\documentclass[11pt]{article}
\usepackage[utf8]{inputenc} % allow utf-8 input
\usepackage[T1]{fontenc}    % use 8-bit T1 fonts
\usepackage{hyperref}       % hyperlinks
\usepackage{url}            % simple URL typesetting
\usepackage{booktabs}       % professional-quality tables
\usepackage{amsfonts}       % blackboard math symbols
\usepackage{nicefrac}       % compact symbols for 1/2, etc.
\usepackage{microtype}      % microtypography

\usepackage{theodor}
\usepackage{hyperref}
\usepackage[toc,page]{appendix}
\usepackage{mathtools}
\usepackage{todonotes}
\usepackage{multirow}
\usepackage[mathscr]{euscript}
\usepackage{caption} 
\DeclareSymbolFont{rsfs}{U}{rsfs}{m}{n}
\DeclareSymbolFontAlphabet{\mathscrsfs}{rsfs}

\def\cuP{\mathscrsfs{P}}

\def\cuP{\mathscrsfs{P}}

\def\boldf{{\boldsymbol f}}
\def\beps{{\boldsymbol \eps}}
\def\bs{{\boldsymbol s}}
\def\bl{{\boldsymbol l}}
\def\bm{{\boldsymbol m}}
\def\bC{{\boldsymbol C}}
\def\bell{{\boldsymbol \ell}}
\def\bsh{{\boldsymbol h}}
\def\bPsi{{\boldsymbol \Psi}}

\def\bS{{\boldsymbol S}}

\def\bX{{\boldsymbol X}}
\def\bE{{\boldsymbol E}}
\def\bY{{\boldsymbol Y}}
\def\bF{{\boldsymbol F}}

\def\bkappa{{\boldsymbol \kappa}}
\def\blambda{{\boldsymbol \lambda}}
\def\bG{{\boldsymbol G}}
\def\bu{{\boldsymbol u}}
\def\bw{{\boldsymbol w}}
\def\bd{{\boldsymbol d}}
\def\bx{{\boldsymbol x}}
\def\btau{{\boldsymbol \tau}}
\def\by{{\boldsymbol y}}
\def\bW{{\boldsymbol W}}
\def\bR{{\boldsymbol R}}
\def\ba{{\boldsymbol a}}
\def\bU{{\boldsymbol U}}
\def\bT{{\boldsymbol T}}
\def\bV{{\boldsymbol V}}
\def\bM{{\boldsymbol M}}
\def\bZ{{\boldsymbol Z}}
\def\ba{{\boldsymbol a}}
\def\bi{{\boldsymbol i}}
\def\bj{{\boldsymbol j}}
\def\bk{{\boldsymbol k}}
\def\bbeta{{\boldsymbol \beta}}
\def\bDelta{{\boldsymbol \Delta}}
\def\bB{{\boldsymbol B}}
\def\bA{{\boldsymbol A}}
\def\be{{\boldsymbol e}}
\def\bu{{\boldsymbol u}}

\def\ba{{\boldsymbol a}}
\def\bA{{\boldsymbol A}}
\def\bD{{\boldsymbol D}}
\def\bb{{\boldsymbol b}}
\def\bv{{\boldsymbol v}}
\def\bxi{{\boldsymbol \xi}}
\def\btheta{{\boldsymbol \theta}}
\def\bTheta{{\boldsymbol \Theta}}
\def\balpha{{\boldsymbol \alpha}}
\def\bG{{\boldsymbol G}}
\def\bb{{\boldsymbol b}}
\def\boldf{{\boldsymbol f}}

\def\bh{{\boldsymbol h}}
\def\by{{\boldsymbol y}}
\def\bH{{\boldsymbol H}}
\def\bdelta{{\boldsymbol\delta}}
\def\bz{{\boldsymbol z}}
\def\bh{{\boldsymbol h}}
\def\bDelta{{\boldsymbol \Delta}}
\def\bt{{\boldsymbol t}}
\def\bm{{\boldsymbol m}}
\def\bl{{\boldsymbol l}}
\def\bone{{\boldsymbol 1}}
\def\btwo{{\boldsymbol 2}}
\def\bp{{\boldsymbol p}}
\def\bq{{\boldsymbol q}}
\def\bone{{\boldsymbol 1}}
\def\obtheta{\overline {\boldsymbol \theta}}
\def\obTheta{\overline {\boldsymbol \Theta}}

\def\obx{\overline {\boldsymbol x}}
\def\oby{\overline {\boldsymbol y}}
\def\of{\overline f}

\def\ox{\overline x}
\def\ha{\hat{a}}
\def\hba{\hat{\boldsymbol a}}
\def\K{\mathbb{K}}
\def\T{\mathbb{T}}
\def\cL{\mathcal{L}}
\def\H{\mathbb{H}}
\def\cF{{\mathcal F}}
\def\cH{{\mathcal H}}
\def\cC{{\mathcal C}}
\def\cQ{{\mathcal Q}}
\def\cP{\mathcal{P}}
\def\cE{{\mathcal E}}
\def\cT{{\mathcal T}}
\def\cE{{\mathcal E}}
\def\cV{{\mathcal V}}

\def\cI{{\mathcal I}}
\def\cR{{\mathcal R}}
\def\cS{{\mathcal S}}
\def\posint{\mathbb{Z}_{\geq 0}}

\def\eff{{\rm eff}}
\def\de{{\rm d}}
\def\Unif{{\sf Unif}}
\def\normal{{\sf N}}
\def\PS{{\rm PS}}
\def\ps{{\rm ps}}
\def\hf{\hat{f}}
\def\tcT{\widetilde{\mathcal T}}

\def\sk{{\rm sk}}

\def\endd{{\rm end}}
\def\proj{{\mathsf P}}
\def\He{{\rm He}}

\def\seff{\mbox{\tiny\rm eff}}
\def\Trace{{\rm Tr}}

\def\diag{{\rm diag}}

\def\Coeff{{\rm Coeff}}
\def\bfone{{\boldsymbol 1}}

\def\M{{\sf M}}
\def\RF{{\sf RF}}
\def\NT{{\sf NT}}
\def\NN{{\sf NN}}
\def\reals{{\mathbb R}}
\def\integers{{\mathbb Z}}
\def\naturals{{\mathbb N}}
\def\de{{\rm d}}
\def\Unif{{\rm Unif}}
\def\proj{{\mathsf P}}
\def\He{{\rm He}}
\def\normal{{\sf N}}

\def\Coeff{{\rm Coeff}}
\def\diag{\text{{\rm diag}}}
\def\rp{\right)}
\def\lp{\left(}
\def\KR{{\rm KRR}}
\def\KRR{{\rm KRR}}
\def\RF{{\rm RF}}
\def\NT{{\rm NT}}
\def\NN{{\rm NN}}
\def\pq{{(q)}}
\def\pqp{{(q')}}
\def\pqq{{(qq)}}
\def\pqqp{{(qq')}}
\def\cl{{\rm cl}}

\def\cG{{\mathcal G}}
\def\cC{{\mathcal C}}
\def\cQ{{\mathcal Q}}

\usepackage{xr}
\usepackage{xr-hyper}

\usepackage{hyperref}
\hypersetup{
    colorlinks,
    linkcolor={blue!80!black},
    citecolor={green!50!black},
}
\colorlet{linkequation}{blue}

\newcommand\blfootnote[1]{%
  \begingroup
  \renewcommand\thefootnote{}\footnote{#1}%
  \addtocounter{footnote}{-1}%
  \endgroup
}

\begin{document}

\title{When Do Neural Networks Outperform Kernel Methods?}

\author{Behrooz Ghorbani\thanks{Department of Electrical Engineering,
    Stanford University}, \;\;Song Mei\thanks{Department of Statistics, University of California, Berkeley},\;\; Theodor Misiakiewicz\thanks{Department of
    Statistics, Stanford University}, \;\;
  Andrea Montanari\footnotemark[1] \footnotemark[3] \thanks{Google Research, Brain Team}}

\maketitle

\begin{abstract}
  For a certain scaling of the initialization of stochastic gradient
  descent (SGD), wide neural networks (NN) have been shown to be
  well approximated by reproducing kernel Hilbert space (RKHS)  methods.
  Recent empirical work showed that, for some classification tasks, RKHS methods
  can replace NNs without a large loss in performance.
  On the other hand, two-layers NNs are known to encode richer smoothness classes than RKHS
  and  we know of special examples for which SGD-trained NN provably outperform RKHS. This is true even in the
  wide network limit, for a different scaling of the  initialization.

  How can we reconcile the above claims? For which tasks do NNs outperform RKHS?
  If covariates are nearly isotropic, RKHS methods
  suffer from the curse of dimensionality, while NNs can
  overcome it by learning the best low-dimensional representation.
  Here we show that this curse of dimensionality becomes milder if
  the covariates display the same low-dimensional structure as the target function,
  and we precisely characterize this tradeoff. Building on these results, we present the spiked covariates model
   that can capture in a unified framework both behaviors observed in earlier work.

  We hypothesize that such a latent low-dimensional structure is present in
  image classification. We test numerically this hypothesis by showing that
  specific perturbations of the training distribution degrade the performances of RKHS methods
  much more significantly than NNs.
\end{abstract}

\tableofcontents

\section{Introduction}
In supervised learning  we are given data $\{(y_i,\bx_i)\}_{i\le n}\sim_{iid}\P\in \cuP(\reals\times \reals^d)$,
with $\bx_i\in\reals^d$ a covariate vector and $y_i\in\reals$ the corresponding label, and would like to learn a function
$f:\reals^d\to\reals$ to predict future labels. In many applications, state-of-the-art systems use multi-layer neural networks
(NN).
The simplest such model is provided by two-layers fully-connected networks:\blfootnote{The code used to produce our results can be accessed at \url{https://github.com/bGhorbani/linearized_neural_networks}.}
\begin{equation}\label{eq:NN}
  \cF_{\NN}^N := \Big\{ \hat f_{\NN}(\bx; \bb,\bW) = \sum_{i=1}^N b_i\sigma(\< \bw_i, \bx\> ): ~~
                b_i \in \R,\, \bw_i\in \R^d,\; \forall i \in [N] \Big\}\, .
  \end{equation}
 $\cF_{\NN}^N$ is a non-linearly parametrized class of functions: while nonlinearity poses a
  challenge to theoreticians, it is often claimed to be crucial in order to learn rich representation of the data.
  Recent efforts to understand NN  have put the spotlight on two linearizations of $\cF_{\NN}^N$,
  the random features \cite{rahimi2008random} and the neural tangent \cite{jacot2018neural}  classes 
\begin{equation}
  \cF_{\RF}^N(\bW) := \Big\{\, \hat f_{\RF}(\bx; \ba; \bW) = \sum_{i=1}^N a_i \sigma(\< \bw_i, \bx\> ): ~~ a_i \in \R, \forall i \in [N] \,\Big\}, \label{eqn:RF_model}
\end{equation}
\begin{equation}
\cF_{\NT}^N(\bW) := \Big\{ \, \hat f_{\NT}(\bx;  \bS, \bW) = \sum_{i=1}^N \<\bs_i , \bx\> \sigma' (\< \bw_i, \bx\>): ~~ \bs_i \in \R^d, \forall i \in [N] \,\Big\}\, .\label{eqn:NT_model}
\end{equation}
  $\cF_{\RF}^N(\bW)$ and   $\cF_{\NT}^N(\bW)$  are linear classes of functions,
  depending on the realization of the input-layer weights $\bW=(\bw_i)_{i\le N}$
  (which are chosen randomly).
  The relation between   NN and these two linear  classes is given by the first-order Taylor expansion: 
  $\hat f_{\NN}(\bx; \bb+\eps\ba,\bW+\eps\bS) -\hat f_{\NN}(\bx; \bb,\bW)  =
  \eps \hat f_{\RF}(\bx; \ba;  \bW) +\eps \hat f_{\NT}(\bx; \bS(\bb); \bW) +O(\eps^2)$, where $ \bS(\bb) = (b_i\bs_i)_{i\le N}$.
  A number of recent papers show that, if weights and SGD updates are suitably scaled, and the network is sufficiently
  wide ($N$ sufficiently large), then SGD converges to a function $\hat f_{\NN}$ that is approximately in
  $\cF_{\RF}^{N}(\bW)+\cF_{\NT}^{N}(\bW)$, with $\bW$ determined by the SGD initialization \cite{jacot2018neural,du2018gradient,du2019gradient,allen2018convergence,zou2018stochastic,oymak2020towards}. This was termed the `lazy regime' in \cite{chizat2019lazy}.

  Does this linear theory convincingly explain the successes of neural networks? Can the performances of NN
  be achieved by the simpler NT or RF models? Is there any fundamental difference between the two classes RF and NT?
  If the weights $(\bw_i)_{i\le N}$ are i.i.d. draws from a  distribution $\nu$ on $\reals^d$,
  the spaces  $\cF_{\RF}^N(\bW)$,  $\cF_{\NT}^N(\bW)$ can be thought as  finite-dimensional approximations
  of a certain RKHS:
  \begin{equation}
    \cH(h) := \cl\Big(\Big\{\,f(\bx) = \sum_{i=1}^Nc_i\, h(\bx,\bx_i):\;\; c_i\in\reals,\, \bx_i\in\reals^d,\,  N\in\naturals\,\Big\}\Big)\, ,
  \end{equation}
  where $\cl(\,\cdot\,)$ denotes closure. From this
  point of view, $\RF$ and $\NT$ differ in that they correspond to slightly different choices of the kernel:
  $h_{\RF}(\bx_1,\bx_2) := \int \sigma(\<\bw,\bx_1\>)
  \sigma(\<\bw,\bx_2\>)\nu(\de\bw)$ versus $h_{\NT}(\bx_1,\bx_2) := \<\bx_1,\bx_2\>
  \int \sigma'(\bw^{\sT}\bx_1) \sigma'(\bw^{\sT}\bx_2)\nu(\de\bw)$. 
  Multi-layer fully-connected $\NN$s in the lazy regime can be viewed as randomized approximations to RKHS as well,
  with some changes in the kernel $h$. 
This motivates analogous questions for $\cH(h)$: can the performances of $\NN$
  be achieved by RKHS methods? 

Recent work addressed the separation between NN and RKHS from several points of view, without providing a unified answer. Some empirical studies on various datasets showed that networks can be replaced by suitable kernels with limited drop in performances \cite{Arora2020Harnessing, li2019enhanced, lee2019wide, novak2019bayesian, lee2018deep, de2018gaussian, garriga2018deep, shankar2020neural}. At least two studies reported a larger gap for convolutional networks and the corresponding kernels \cite{arora2019exact, geiger2019disentangling}.
  On the other hand, theoretical analysis provided a number of separation examples, i.e. target functions $f_*$
  that can be  represented and possibly efficiently learnt using neural networks, but not in the corresponding RKHS
  \cite{yehudai2019power,bach2017breaking,ghorbani2019linearized,ghorbani2019limitations,allen2019can,allen2020backward}.
  For instance, if the target is a single neuron $f_*(\bx) = \sigma(\<\bw_*,\bx\>)$, then training a neural network with one
  hidden neuron learns the target efficiently from approximately $d\log d$ samples \cite{mei2018landscape},
  while the corresponding RKHS has test  error bounded
  away from zero for every sample size polynomial in $d$ \cite{yehudai2019power,ghorbani2019linearized}.
  Further even in the infinite width limit, it is known that two-layers neural networks can actually capture a richer class of functions than the
  associated RKHS, provided SGD training is scaled differently from the lazy regime \cite{mei2018mean,chizat2018global,rotskoff2018neural,sirignano2018mean,chizat2020implicit}.

  Can we reconcile empirical and theoretical results?

  \subsection{Overview}

  In this paper we introduce a stylized scenario -- which we will refer to as the spiked covariates model -- that can explain the above
  seemingly divergent observations in a unified framework. The spiked covariates model is based on two building blocks:
  $(1)$~Target functions depending on low-dimensional projections; $(2)$~Approximately low-dimensional covariates.

  \emph{$(1)$~Target functions depending on low-dimensional projections.}
  We investigate the hypothesis that NNs
  are more efficient at learning target functions that depend on low-dimensional projections of the data (the signal covariates).
  Formally, we consider target functions $f_*:\reals^d\to\reals$ of the form  $f_*(\bx) = \varphi(\bU^{\sT}\bx)$, where $\bU\in\reals^{d\times d_0}$
  is a semi-orthogonal matrix, $d_0\ll d$, and $\varphi:\reals^{d_0}\to\reals$ is a suitably smooth function. 
  This model captures an important property of certain applications. For instance, the labels in an image classification
  problem do not depend equally on the whole Fourier spectrum of the image, but predominantly on the low-frequency components.
  % This is easy to understand if we consider images at very high resolution: as the resolution increases,
  %the  target function changes minimally,
 % while the input dimension $d$ can become arbitrarily large.
  
  As for the example of a single neuron $f_*(\bx) = \sigma(\<\bw_*,\bx\>)$, we expect RKHS to suffer from a curse of dimensionality
  in learning functions of low-dimensional projections. Indeed, this is well understood in low dimension or
  for isotropic covariates \cite{bach2017breaking,ghorbani2019linearized}.

  \begin{figure}
    \centering
\includegraphics[width=\linewidth]{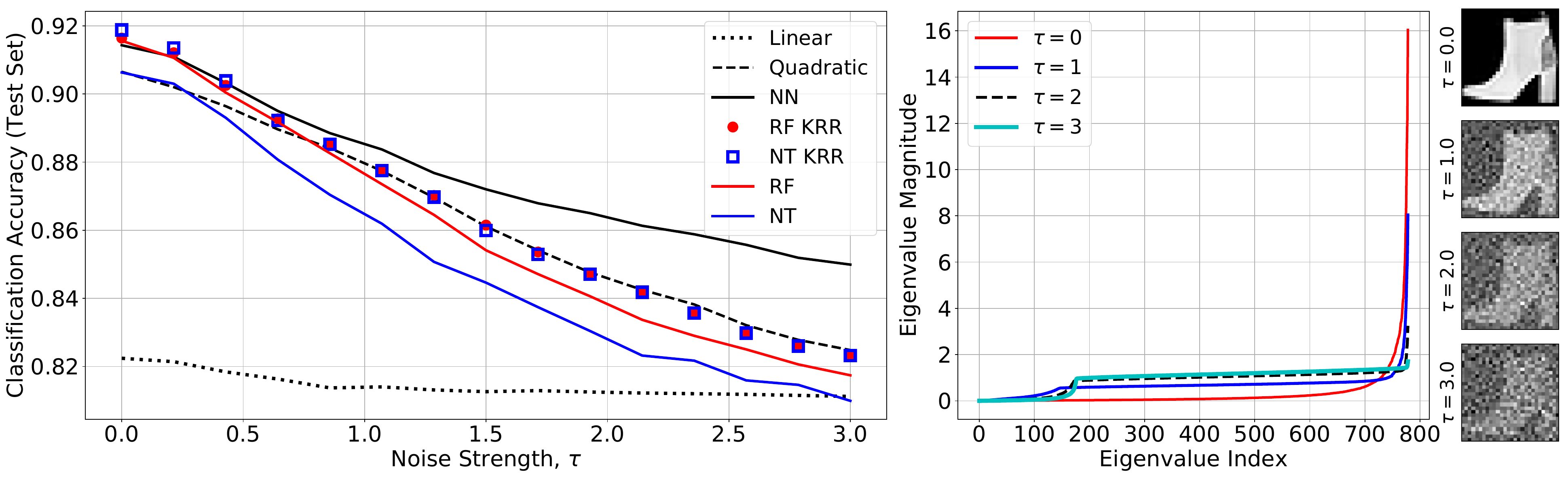}
    \caption{Test accuracy on FMNIST images perturbed by adding noise to the high-frequency Fourier components of the images (see examples on the right). Left: comparison of the accuracy of various methods
      as a function of the added noise. Center: eigenvalues of the empirical covariance of the images. As the noise increases, the images distribution becomes more isotropic.}\label{fig:FMNIST}
\end{figure}
  \emph{$(2)$~Approximately low-dimensional covariates}.  RKHS behave well on certain image classification tasks
  \cite{arora2019exact, li2019enhanced, novak2019bayesian}, and this seems to contradict the previous point. 
  However, the example of image classification naturally brings up another important property of real data that
  helps to clarify this puzzle.  Not only we expect the target function $f_*(\bx)$ to
  depend predominantly on the low-frequency components of image $\bx$, but the image $\bx$ itself to have most of its spectrum concentrated on
  low-frequency components (linear denoising algorithms exploit this very observation).
  
  More specifically, we consider the case in which $\bx=\bU\bz_1+\bU^{\perp}\bz_2$, where $\bU \in \reals^{d \times d_0}$, $\bU^\perp \in \reals^{d \times (d - d_0)}$, and $[\bU|\bU^{\perp}]\in\reals^{d\times d}$ is an orthogonal matrix. Moreover, we assume \linebreak[4] $\bz_1 \sim\Unif(\S^{d_0 - 1}(r_1\sqrt{d_0}))$, $\bz_2 \sim \Unif( \S^{d - d_0 - 1}(r_2\sqrt{d-d_0}))$, and $r_1^2\ge r_2^2$. We find that, 
  if $r_1/ r_2$ (which we will denote later as the covariates signal-to-noise ratio) is sufficiently large, then the curse of dimensionality becomes milder for RKHS methods. We characterize precisely how the performance of these methods depend on the covariate signal-to-noise ratio $r_1/r_2$, the signal dimension $d_0$, and the ambient dimension $d$. 

Notice that the spiked covariate model is highly stylized. For instance, while we expect real images to have a latent low-dimensional structure,
this is best modeled in a nonlinear fashion (e.g. sparsity in wavelet domain \cite{donoho1995adapting}).
Nevertheless the spiked covariate model captures the two basic mechanisms, and provides useful qualitative predictions. As an illustration, consider adding
noise to the high-frequency components of images in a classification task. This will make the distribution of $\bx$ more isotropic,
and --according to our theory-- deteriorate the performances of RKHS methods. On the other hand, NN should be less sensitive to this perturbation.
(Notice that noise is added both to train and test samples.)
In Figure \ref{fig:FMNIST} we carry out such an experiment using Fashion MNIST (FMNIST) data ($d=784$, $n=60000$, 10 classes).
We compare two-layers NN with the RF and NT models. We choose the architectures of NN, NT, RF
as to match the number of parameters: namely we used $N=4096$ for NN
and NT and $N=321126$ for RF. We also fit the corresponding RKHS models (corresponding to $N=\infty$)
using kernel ridge regression (KRR), and two simple polynomial models: $f_\ell(\bx) =\sum_{k=0}^\ell \<\bB_k,\bx^{\otimes k}\>$,
for $\ell \in \{1,2\}$. In the unperturbed dataset, all of these approaches have comparable accuracies (except the linear fit). As noise is added, RF, NT, and RKHS methods deteriorate rapidly. While the accuracy of NN decreases as well, it significantly outperforms other methods.

\subsection{Notations and outline}

Throughout the paper, we use bold lowercase letters $\{\bx, \by, \bz, \ldots\}$ to denote vectors and bold uppercase letters $\{\bA, \bB, \bC, \ldots\}$ to denote matrices. We denote by $\S^{d-1}(r) = \{ \bx \in \R^d: \| \bx \|_2 = r \}$ the set of $d$-dimensional vectors with radius $r$ and $\Unif(\S^{d-1}(r))$ be the uniform probability distribution on $\S^{d-1}(r)$. Further, we let $\normal(\mu, \tau^2)$ be the Gaussian distribution with mean $\mu$ and variance $\tau^2$. 

Let $O_d(\, \cdot \, )$  (respectively $o_d (\, \cdot \,)$, $\Omega_d(\, \cdot \,)$, $\omega_d(\, \cdot \,)$) denote the standard big-O (respectively little-o, big-Omega, little-omega) notation, where the subscript $d$ emphasizes the asymptotic variable. 
%We denote by $O_{d,\P} (\, \cdot \,)$ the big-O in probability notation: $h_1 (d) = O_{d,\P} ( h_2(d) )$ if for any $\eps > 0$, there exists $C_\eps > 0 $ and $d_\eps \in \Z_{>0}$, such that
%\[
%\begin{aligned}
%\P ( |h_1 (d) / h_2 (d) | > C_{\eps}  ) \le \eps, \qquad \forall d \ge d_{\eps}.
%\end{aligned}
%\]
We denote by $o_{d,\P} (\, \cdot \,)$ the little-o in probability notation: $h_1 (d) = o_{d,\P} ( h_2(d) )$, if $h_1 (d) / h_2 (d)$ converges to $0$ in probability.

In section \ref{sec:Rigorous}, we introduce the spiked covariates model and characterize the performance of KRR, RF, NT, and NN models.
Section \ref{sec:Experiments} presents numerical experiments with real and synthetic data.  Section
\ref{sec:Discussion} discusses our results in the context of earlier work.

  \section{Rigorous results for kernel methods and NT, RF NN expansions}
  \label{sec:Rigorous}
  
  \subsection{The spiked covariates model}
  \label{sec:Model}
 %(which we assume to be integral to avoid writing $\lfloor d^{\eta}\rfloor$)
 
Let $d_0= \lfloor d^\eta \rfloor$ for some $\eta \in (0, 1)$. Let $\bU\in\reals^{d\times d_0}$ and $\bU^{\perp}\in\reals^{d\times (d-d_0)}$ be such that $[\bU|\bU^{\perp}]$ is an orthogonal matrix. We denote the subspace spanned by the columns of $\bU$ by $\cV \subseteq \R^d$ which we will refer to as the signal subspace, and the subspace spanned by the columns of $\bU^\perp$ by $\cV^\perp \subseteq \R^d$ which we will refer to as the noise subspace. In the case $\eta\in (0,1)$, the signal dimension $d_0 = \dim(\cV)$ is much smaller than the ambient dimension $d$. Our model for the covariate vector $\bx_i$ is
\[
\bx_{i} = \bU\bz_{0,i}+\bU^{\perp}\bz_{1,i}, \;\;\;\;\; ( \bz_{0,i} , \bz_{1,i}) \sim \Unif(\S^{d_0-1}(r\sqrt{d_0})) \otimes \Unif(\S^{d-d_0-1}(\sqrt{d-d_0})).
\]

We call $\bz_{0,i}$ the signal covariates, $\bz_{1,i}$ the noise covariates, and $r$ the covariates signal-to-noise ratio (or covariates SNR). We will take $r >1$, so that the variance of the signal covariates $\bz_{0,i}$ is larger than that of the noise covariates $\bz_{1,i}$. In high dimension, this model is --for many purposes-- similar to an anisotropic Gaussian model $\bx_i \sim \normal(0,(r^2-1)\bU\bU^{\sT}+\id)$. As shown below, the effect of anisotropy on RKHS methods is significant only if the covariate SNR $r$ is polynomially large in $d$. We shall therefore set $r=d^{\kappa/2}$ for a constant $\kappa>0$.

We are given i.i.d. pairs $(y_i,\bx_i)_{1 \le i \le n}$, where $y_i=f_*(\bx_i)+\eps_i$, and $\eps_i\sim\normal(0,\tau^2)$ is independent of
$\bx_i$. The function $f_*$ only depends on the projection of $\bx_i$ onto the signal subspace $\cV$ (i.e. on the signal covariates $\bz_{0,i}$): $f_*(\bx_i) = \varphi(\bU^{\sT}\bx_i)$, with $\varphi\in L^2(\S^{d_0-1}(r\sqrt{d_0}))$.

For the RF and NT models, we will assume that input layer weights to be
i.i.d. $\bw_i\sim\Unif(\S^{d-1}(1))$. For our purposes, this is essentially the same as 
$w_{ij} \sim \normal(0,1/d)$ independently, but slightly more convenient technically.

We will consider a more general model in  Appendix \ref{sec:GeneralApp}, in which the distribution of $\bx_i$
takes a more general product-of-uniforms form, and we assume a general $f_* \in L^2$.

\subsection{A sharp characterization of RKHS methods}

Given   $h: [-1, 1] \to  \reals$,  consider the rotationally invariant kernel $K_d(\bx_1,\bx_2) = h(\<\bx_1, \bx_2\>/d)$. This class includes the kernels that are obtained by taking the wide limit of the RF and NT models (here expectation is with respect to
$(G_1,G_2)\sim\normal(0,\id_2)$)
\begin{align*}
  h_{\RF}(t) := \E\{\sigma(G_1)\sigma(tG_1+\sqrt{1-t^2} G_2)\}\, ,\;\;\;
  h_{\NT}(t) := t\E\{\sigma'(G_1)\sigma'(tG_1+\sqrt{1-t^2} G_2)\}. 
\end{align*}
(These formulae correspond to $\bw_i\sim\normal(0,\id_d)$, but similar formulae hold for $\bw_i\sim\Unif(\S^{d-1}(\sqrt{d}))$.)
 This correspondence holds beyond two-layers networks: under i.i.d. Gaussian initialization, the NT kernel for an arbitrary number of fully-connected layers is rotationally invariant (see the proof of Proposition 2 of \cite{jacot2018neural}), and hence is covered by the present analysis.

Any RKHS method with kernel $h$ outputs a model of the form $\hf(\bx;\ba) = \sum_{i \le n} a_i h(\< \bx, \bx_i\>/d)$,
with RKHS norm given by $\|\hf(\,\cdot\,;\ba)\|_h^2 =\sum_{i,j\le n} h(\< \bx_i, \bx_j\>/d) a_ia_j$.
We consider kernel ridge regression (KRR) on the dataset $\{(y_i, \bx_i)\}_{i \le n}$ with regularization parameter $\lambda$, namely:
\begin{align*}
\hba(\lambda) := \arg\min_{\ba\in\reals^N} \Big\{\sum_{i=1}^n\big(y_i-\hf(\bx_i;\ba)\big)^2+ \lambda \|\hf(\,\cdot\,;\ba)\|_h^2\Big\} =(\bH + \lambda \id_n)^{-1} \by,
\end{align*}
where $\bH = (H_{ij})_{ij \in [n]}$, with $H_{ij} = h(\<\bx_i, \bx_j\>/d)$.
We denote the prediction error of KRR by 
\[
R_\KR(f_*, \lambda) = \E_\bx\Big[ \Big(f_*(\bx) - \by^\sT (\bH + \lambda \id_n)^{-1} \bh(\bx) \Big)^2 \Big],
\]
where $\bh(\bx) = (h(\< \bx, \bx_1\> / d), \ldots, h(\<\bx, \bx_n\> / d))^\sT$. 

Recall that we assume the target function $f_*(\bx_i) = \varphi(\bU^{\sT}\bx_i)$. We denote $\proj_{\le k}: L^2 \to L^2$ to be the projection operator onto the space of degree $k$ orthogonal polynomials, and $\proj_{> k} = \id - \proj_{\le k}$.
Our next theorem shows that the impact of the low-dimensional latent structure on the generalization error
of KRR is characterized by a certain \emph{`effective dimension'}, $d_{\seff}$.
\begin{theorem}\label{thm:bound_KRR}
Let $h \in C^{\infty}([-1, 1])$. Let $\ell \in \integers_{\ge 0}$ be a fixed integer. We assume that $h^{(k)} (0) > 0$ for all $k \leq \ell$, and assume that there exists a $k > \ell$ such that $h^{(k)} (0)>0$.  
(Recall that $h$ is positive semidefinite whence $h^{(k)}(0) \ge 0$ for all $k$.)

Define the effective dimension $d_{\seff} = \max\{ d_0, d / r^2 \} =
d^{\max(1-\kappa,\eta)}$. If $\omega_d(d_{\seff}^{\ell} \log(
d_{\seff} ) )\le n \le d_{\seff}^{\ell+1 - \delta}$ for some $\delta
>0$, then for any regularization parameter $\lambda = O_d(1)$, the
prediction error of KRR with kernel $h$ is 
\begin{align}
\Big|  R_{\KR }(f_*; \lambda) -\| \proj_{> \ell} f_* \|_{L^2}^2 \Big| \leq o_{d, \P}(1) \cdot ( \|  f_* \|_{L^2}^2 + \tau^2 )\, .
\end{align}
\end{theorem}

Remarkably, the effective dimension $d_{\seff}=d^{\max(1-\kappa,\eta)}$ depends both on the signal dimension $\dim(\cV) = d^{\eta}$ and on the covariate SNR $r = d^{\kappa/2}$. Sample size $n = d_{\seff}^{\ell}$ is necessary to learn a degree $\ell$ polynomial. If we fix $\eta \in (0, 1)$ and take $\kappa = 0+$, we get $d_{\seff}\approx d$: this corresponds to almost isotropic $\bx_i$.
We thus recover  \cite[Theorem 4]{ghorbani2019linearized}.  If instead $\kappa > 1 - \eta$,
then most variance of $\bx_i$  falls in the signal subspace $\cV$, and we get $d_{\seff} = d^\eta = \dim(\cV)$: the test error
is effectively the same as if we had oracle knowledge of the signal subspace $\cV$
and performed KRR on signal covariates $\bz_{0, i} = \bU^\sT \bx_i$. Theorem \ref{thm:bound_KRR} describes
the transition between these two regimes.

\subsection{RF and NT models}

How do the results of the previous section generalize to finite-width approximations of the RKHS? In particular, how do the RF and NT models behave at finite $N$? In order to simplify the picture, we focus here on the approximation error. Equivalently, we assume the sample size to be $n=\infty$ and consider the minimum population risk for $\M \in \{ \RF, \NT\}$
\begin{align}
R_{\M, N}(f_*; \bW) := \inf_{\hat{f}\in\cF^N_{\M}(\bW)}\E\big\{\big[f_*(\bx)-\hat f(\bx)\big]^2\big\}\, .
\end{align}
The next two theorems characterize the asymptotics of the approximation error for RF and NT models. We give generalizations of these statements to other settings and under weaker assumptions in Appendix \ref{sec:GeneralApp}. 
\begin{theorem}[Approximation error for RF]\label{thm:RF}
  Assume $\sigma\in C^{\infty}(\reals)$, with $k$-th derivative  $\sigma^{(k)}(x)^2\le c_{0, k} e^{c_{1, k}x^2/2}$
  for some $c_{0, k} > 0$, $c_{1, k} < 1$, and  all $x\in\reals$ and all $k$. Define its $k$-th Hermite coefficient $\mu_k(\sigma):=\E_{G \sim \normal(0, 1)}[\sigma(G)\He_k(G)]$. Let $\ell\in\integers_{\ge 0}$ be a fixed integer, and assume $\mu_k(\sigma)\neq 0$ for all $k\le \ell$. 
Define $d_{\seff}= d^{\max(1-\kappa,\eta)}$. If $d_{\seff}^{\ell + \delta}\le N \le d_{\seff}^{\ell+1 - \delta}$ for some $\delta > 0$ independent of $N,d$, then
\begin{align}
\big| R_{\RF, N}(f_*;\bW) - \| \proj_{> \ell} f_* \|_{L^2}^2 \big| \leq o_{d, \P}(1) \cdot  \| \proj_{> \ell} f_* \|_{L^2}\|  f_* \|_{L^2}\, .
\end{align}
\end{theorem}

\begin{theorem}[Approximation error for NT]\label{thm:NT}
Assume $\sigma\in C^{\infty}(\reals)$, with $k$-th derivative $\sigma^{(k)}(x)^2\le c_{0, k} e^{c_{1, k} x^2/2}$, for   some $c_{0, k} > 0$, $c_{1, k} < 1$, and all $x\in\reals$ and all $k$. Let $\ell \in \integers_{\ge 0}$, and assume $\mu_k(\sigma)\neq 0$ for all $k\le \ell+1$. Further assume that, for all $L\in\integers_{\ge 0}$, there exist $k_1,k_2$ with $L<k_1< k_2$, such that $\mu_{k_1}(\sigma') \neq 0$, $\mu_{k_2}(\sigma') \neq 0$, and $\mu_{k_1}(x^2\sigma')/\mu_{k_1}(\sigma')\neq \mu_{k_2}(x^2\sigma')/\mu_{k_2}(\sigma')$. 
Define $d_{\seff}= d^{\max(1-\kappa,\eta)}$. If $d_{\seff}^{\ell + \delta}\le N\le d_{\seff}^{\ell+1 - \delta}$ for some $\delta > 0$ independent of $N,d$, then 
\begin{align}
\big| R_{\NT, N}(f_*;\bW) - \| \proj_{> \ell + 1} f_* \|_{L^2}^2 \big| \leq o_{d, \P}(1) \cdot \| \proj_{> \ell + 1} f_* \|_{L^2}\|  f_* \|_{L^2}\, .
\end{align}
\end{theorem}

Here, the definitions of effective dimension $d_{\seff}$ is the same as in Theorem \ref{thm:bound_KRR}. While for the test error of KRR as in Theorem \ref{thm:bound_KRR}, the effective dimension controls the sample complexity $n$ in learning a degree $\ell$ polynomial, in the present case it controls the number of neurons $N$ that is necessary to approximate a degree $\ell$ polynomial. In the case of RF, the latter happens as soon as $N \gg d^{\ell}_{\seff}$, while for NT it happens as soon as $N \gg d^{\ell-1}_{\seff}$. If we take $\eta \in (0, 1)$ and $\kappa = 0+$, the above theorems, again, recover Theorem 1 and 2 of \cite{ghorbani2019linearized}. 

Notice that NT has higher approximation power than RF \emph{in terms of the number of neurons}. This is expected, since NT models contain $Nd$ instead of $N$ parameters. On the other hand, NT has less power \emph{in terms of number of parameters}: to fit a degree $\ell + 1$ polynomial, the parameter complexity for NT is $N d = d_{\seff}^\ell d$ while the parameter complexity for RF is $N = d_{\seff}^{\ell+1} \ll d_{\seff}^\ell d$. While the NT model has $p = Nd$ parameters, only $p_{\seff}^\NT = Nd_{\seff}$ of them appear to matter. We will refer to $p_{\seff}^\NT \equiv Nd_{\seff}$ as the \emph{effective number of parameters} of NT models.

Finally, it is natural to ask what are the behaviors of RF and NT models at finite sample size. Denote by $R_{\M,N,n}(f_*;\bW)$ the corresponding test error (assuming for instance ridge regression, with the optimal regularization $\lambda$). Of course the minimum population risk provides a lower bound: $R_{\M,N,n}(f_*;\bW)\ge R_{\M,N}(f_*;\bW)$. Moreover, we conjecture that the risk is minimized at infinite $N$, $R_{\M,N,n}(f_*;\bW)\gtrsim R_{n}(f_*;h_{\M})$. Altogether this implies the lower bound $R_{\M,N,n}(f_*;\bW)\gtrsim \max(R_{\M,N}(f_*;\bW), R_{n}(f_*;h_{\M}))$. We also conjecture that this lower bound is tight, up to terms vanishing as $N,n,d\to\infty$.

Namely (focusing on NT models), if $Nd_{\seff} \lesssim n$, and $d_{\seff}^{\ell_1} \lesssim N d_{\seff} \lesssim d_{\seff}^{\ell_1+1}$ then the approximation error dominates and $R_{\M,N,n}(f_*;\bW) = \|\proj_{>\ell_1}f_*\|_{L^2}^2+o_{d, \P}(1)\|f_*\|_{L^2}^2$. If on the other hand $Nd_{\seff} \gtrsim n$, and $d_{\seff}^{\ell_2} \lesssim n \lesssim d_{\seff}^{\ell_2+1}$ then the generalization error dominates and $R_{\M,N,n}(f_*;\bW) = \|\proj_{>\ell_2}f_*\|_{L^2}^2+o_{d, \P}(1)\|f_*\|_{L^2}^2$.

\begin{figure}
\includegraphics[width=0.95\linewidth]{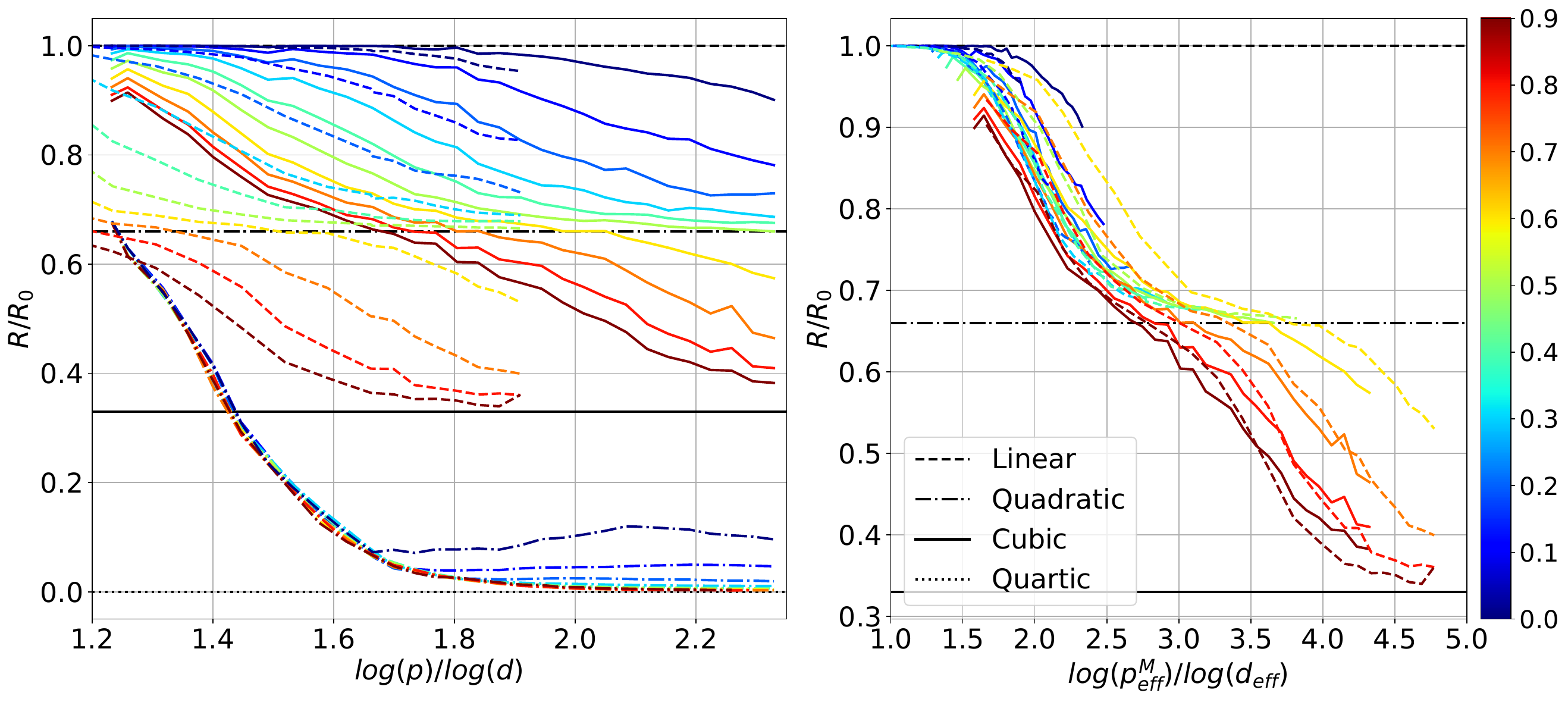}
\caption{Finite-width two-layers NN and their linearizations RF and NT. Models are trained on $2^{20}$ training observations drawn i.i.d from the distribution of Section \ref{sec:Model}. Continuous lines: NT; dashed lines: RF; dot-dashed: NN. Various curves (colors) refer to values of the exponent $\kappa$ (larger $\kappa$ corresponds to stronger low-dimensional component). Right frame: curves for RF and NT as a function of the rescaled quantity $\log(p_{\seff}^\M)/\log(d_{\seff})$.} \label{fig:FiniteWidth}
\end{figure}

\subsection{Neural network models}

Consider the approximation error for NNs
\begin{align}
R_{\NN, N}(f_*) := \inf_{\hat{f}\in \cF^N_{\NN}}\E\big\{\big[f_*(\bx)-\hat f(\bx)\big]^2\big\}. 
\end{align}
Since $\eps^{-1} [ \sigma(\< \bw_i + \eps \ba_i, \bx\>) - \sigma(\< \bw_i, \bx\>)]
\stackrel{\eps\to 0}{\longrightarrow}\< \ba_i, \bx\> \sigma'(\< \bw_i, \bx\>)$,  we have $\cup_{\bW} \cF_{\NT}^{N/2}(\bW) \subseteq {\rm cl}(\cF^N_{\NN})$, and $R_{\NN, N}(f_*) \le \inf_{\bW} R_{\NT, N/2}(f_*, \bW)$. By choosing $\overline \bW = (\bar \bw_i)_{i \le N}$, with $\bar \bw_i = \bU \bar \bv_{i}$ (see Section \ref{sec:Model} for definition of $\bU$), we obtain that $\cF_{\NT}^N(\overline \bW)$  contains all functions of
the form $\bar f(\bU^\sT \bx)$, where $\bar f$ is in the class of functions $\cF_{\NT}^N (\overline \bV)$ on $\reals^{d_0}$.
Hence if $f_*(\bx) = \vphi(\bU^{\sT}\bx)$, $R_{\NN, N}(f_*) $  is at most the error of approximating $\vphi(\bz)$ on the small sphere $\bz \sim \Unif(\S^{d_0 - 1})$
within the class $\cF_{\NT}^N (\overline \bV)$. As a consequence, by Theorem \ref{thm:NT}, if $d_0^{\ell + \delta}\le N\le d_0^{\ell+1 - \delta}$ for some $\delta > 0$, then $R_{\NN, N}(f_*) \le R_{\NT, N / 2}(f_*, \overline \bW) \le (1 + o_{d, \P}(1)) \cdot \| \proj_{> \ell + 1} f_* \|_{L^2}^2$. 

\begin{theorem}[Approximation error for NN]\label{thm:NN}
Assume that $\sigma\in C^{\infty}(\reals)$ satisfies the same assumptions as in Theorem \ref{thm:NT}.  Further assume that $\sup_{x \in \R} \vert \sigma''(x)\vert < \infty$. If $d_0^{\ell + \delta}\le N\le d_0^{\ell+1 - \delta}$ for some $\delta > 0$ independent of $N,d$, then the approximation error of NN models (\ref{eqn:NT_model}) is
\begin{align}
 R_{\NN, N}(f_*) \le (1 + o_d(1)) \cdot \| \proj_{> \ell + 1} f_* \|_{L^2}^2. 
\end{align}
Moreover, the quantity $R_{\NN, N}(f_*)$ is independent of $\kappa \ge 0$. 
\end{theorem}
As a consequence of Theorem \ref{thm:NT} and \ref{thm:NN}, there is a separation between NN and (uniformly sampled) NT models when $d_{\seff} \neq d_0$, i.e., $\kappa < 1 - \eta$. As $\kappa$ increases, the gap between NN and NT becomes smaller and smaller until $\kappa = 1 - \eta$.

\section{Further numerical experiments}
\label{sec:Experiments}

We carried out extensive numerical experiments on synthetic data to check our predictions for RF, NT, RKHS methods at finite sample size $n$, dimension $d$, and width $N$. We simulated two-layers fully-connected NN in the same context in order to compare their behavior to the behavior of the previous models. Finally, we carried out numerical experiments on FMNIST and CIFAR-10 data to test whether our qualitative predictions apply to image  datasets. Throughout we use ReLU activations.

In Figure \ref{fig:FiniteWidth} we investigate the approximation error of RF, NT, and NN models. We generate data $(y_i, \bx_i)_{i \ge 1}$ according to the model of Section \ref{sec:Model}, in $d=1024$ dimensions, with a latent space dimension $d_0 = 16$, hence $\eta=2/5$. The per-coordinate variance in the latent space is $r^2 = d^\kappa$, with $\kappa \in \{0.0, \dots,0.9\}$. Labels are obtained by $y_i=f_*(\bx_i)=\varphi(\bU^{\sT}\bx_i)$ where $\varphi:\reals^{d_0}\to\reals$ is a degree-4 polynomial, without a linear component.
Since we are interested in the minimum population risk, we use a large sample size $n=2^{20}$: we expect the approximation
error to dominate in this regime. (See Appendix \ref{app:exp_details}  for further details.)

We plot the normalized risk  $R_{\RF,N}(f_*, \bW)/R_0$, $R_{\NT,N}(f_*, \bW)/R_0$, $R_{\NN,N}(f_*)/R_0$, $R_0:=\|f_*\|_{L^2}^2$,
for various widths $N$. These are compared with the error of the best polynomial approximation of degrees $\ell=1$ to $3$ (which correspond to $\|\proj_{>\ell}f_*\|^2_{L^2}/\|f_*\|_{L^2}^2$).  As expected, as the number of parameters increases,
the approximation error of each function class decreases. NN provides much better approximations than any of the linear classes, and RF is superior to NT \emph{given the same number of parameters}. This is captured by Theorems \ref{thm:RF} and \ref{thm:NT}: to fit a degree $\ell + 1$ polynomial, the parameter complexity for NT is $N d = d_{\seff}^\ell d$ while  for RF it is $N = d_{\seff}^{\ell+1} \ll d_{\seff}^\ell d$. We denote the effective number of parameters for NT by $p_{\seff}^{\NT} = Nd_{\seff}$ and the effective number of parameter for RF by $p_{\seff}^{\RF} = N$. The right plot reports the same data, but we rescale the x-axis to be $\log(p_{\seff}^\M)/\log(d_{\seff})$. As predicted by the asymptotic theory of Theorems \ref{thm:RF} and \ref{thm:NT}, various curves for NT and RF tend to collapse on this scale.
Finally, the approximation error of RF and NT depends strongly on $\kappa$: larger $\kappa$ leads to smaller effective dimension and hence smaller approximation error. In contrast, the error of NN, besides being smaller in absolute terms, is
much less sensitive to $\kappa$.

\begin{figure}
    \includegraphics[width=0.95\linewidth]{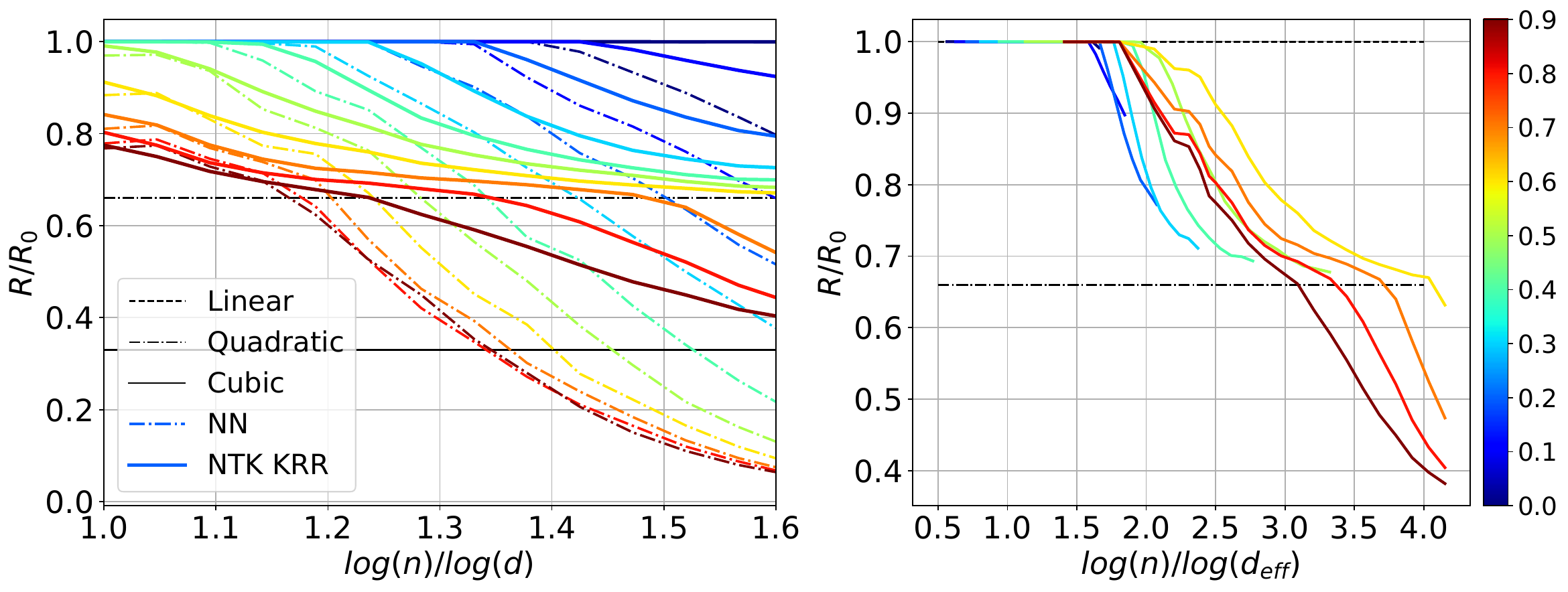}
    \caption{Left: Comparison of the test error of NN (dot-dashed) and NTK KRR (solid) on the distribution of the Section \ref{sec:Model}.
      Various curves (colors) refer to values of the exponent $\kappa$. Right: KRR test error as a function of the number of observations adjusted by the effective dimension.
      Horizontal lines correspond to the best polynomial approximation.}\label{fig:KRR}
    \vspace{-0.25cm}
  \end{figure}

  In Fig.~\ref{fig:KRR} we compare the test error of NN (with $N=4096$) and KRR for the NT kernel (corresponding to
  the $N\to\infty$ limit in the lazy regime), for the same data distribution as in the previous figure.
  We observe that the test error of KRR is substantially larger than the one of NN, and deteriorates rapidly as $\kappa$
  gets smaller (the effective dimension gets larger). In the right frame we plot the test error as a function of
  $\log(n)/\log(d_{\seff})$: we observe that the curves obtained for different $\kappa$ approximately collapse,
  confirming that $d_{\seff}$ is indeed the right dimension parameter controlling the sample complexity.
  Notice that also the error of NN deteriorates as $\kappa$ gets
  smaller, although not so rapidly: this behavior deserves further investigation.
  Notice also
  that the KRR error crosses the level of best degree-$\ell$ polynomial approximation roughly at $\log(n)/\log(d_{\seff})\approx \ell$.
  
\begin{figure}
\centering
\includegraphics[width=\linewidth]{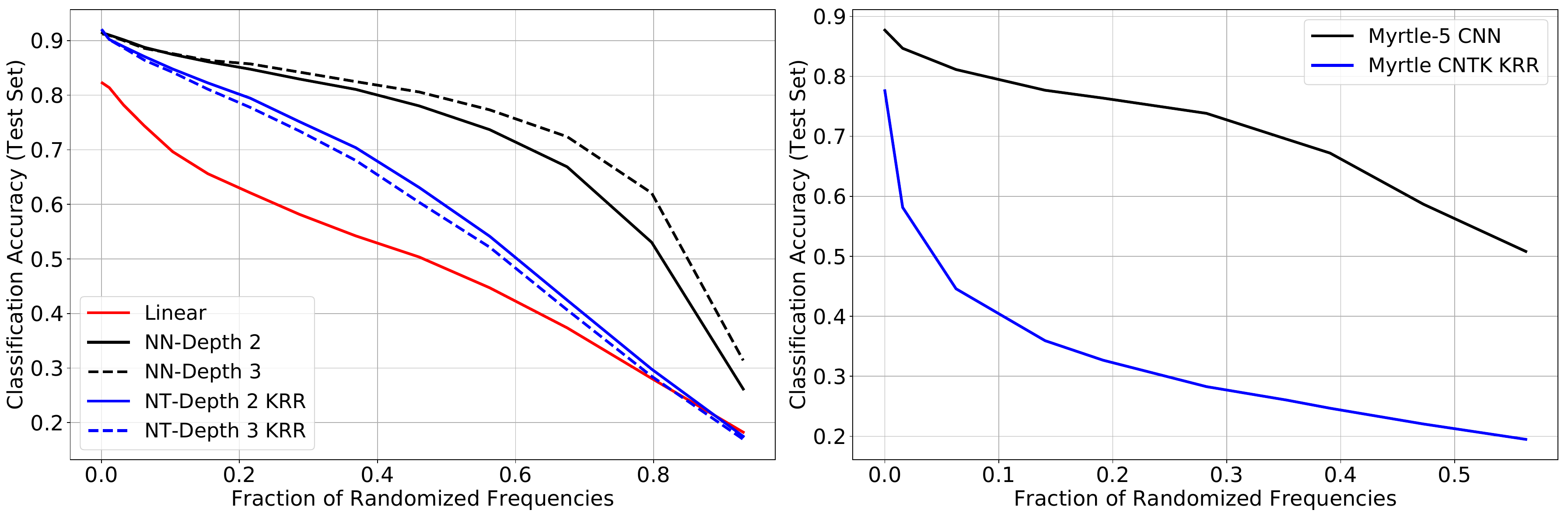}
\caption{Compartison between multilayer NNs and the corresponding NT models under perturtbations in frequency domain.
  Left: Fully connected networks on FMNIST data. Right: Comparison of CNN and CNTK KRR classification accuracy on CIFAR-10. We progressively replace the lowest frequencies of each image with Gaussian noise with matching covariance structure. Right: Accuracy for FMNIST. }\label{fig:FMNIST-Shuffled}
\end{figure}

The basic qualitative insight of our work can be summarized as follows. Kernel methods are effective when a
low-dimensional structure in the target function is aligned with a low-dimensional structure in the covariates.
In image data, both the target function and the covariates are dominated by the low-frequency subspace.
In Figure \ref{fig:FMNIST} we tested this hypothesis by removing the low-dimensional structure of the covariate vectors:
we simply added noise to the high-frequency part of the image. In  Figure \ref{fig:FMNIST-Shuffled} we try the opposite, by
removing the component of the target function that is localized on low-frequency modes.
We decompose each images into a low-frequency and a high-frequency part. We leave the high-frequency part unchanged, and
replace the low-frequency part by Gaussian noise with the first two moments matching the empirical moments
of the data.

In the left frame, we consider FMNIST data and compare
fully-connected NNs with $2$ or $3$ layers  (and $N=4096$ nodes at each hidden layer) with the corresponding
NT KRR model (infinite width). In the right frame, we use  CIFAR-10 data and compare a Myrtle-5 network
(a lightweight convolutional architecture \cite{DavidPage2018, shankar2020neural})
with the corresponding NT KRR. We observe the same behavior as in Figure \ref{fig:FMNIST}. While for the original data
NT is comparable to  NN, as the proportion of perturbed Fourier modes increases, the performance of NT deteriorates
much more rapidly than the one of  NN.

\section{Discussion}
\label{sec:Discussion}
  
The limitations of linear methods ---such as KRR--- in high dimension are well
understood in the context of nonparametric function estimation.
For instance, a basic result in this area establishes that
estimating a Sobolev function $f_*$ in $d$ dimensions  with mean square error $\eps$
requires roughly $\eps^{-2-d/\alpha}$ samples, with $\alpha$ the smoothness parameter \cite{tsybakov2008introduction}.
This behavior is achieved by kernel smoothing and by KRR: however these methods are not expected 
to be adaptive when $f_*(\bx)$ only depends on a low-dimensional projection of $\bx$,
i.e. $f_*(\bx) = \varphi(\bU^{\sT}\bx)$ for an unknown $\bU\in\reals^{d_0\times d}$,
$d_0\ll d$. On the contrary, fully-trained NN can overcome this problem \cite{bach2017breaking}.

However, these classical statistical results have  some limitations.
First, they focus on the low-dimensional regime: $d$ is fixed,
while the sample size $n$ diverges. This is probably unrealistic for many machine learning applications,
in which $d$ is at least of the order of a few hundreds.
Second, classical lower bounds are typically established for the minimax risk, and hence
they do not necessarily apply to specific functions.

To bridge these gaps, we developed a sharp characterization of the test error in
the high-dimensional regime in which both $d$ and $n$ diverge,
while being polynomially related. This characterization holds for any target function $f_*$, and expresses the limiting test error
in terms of the polynomial decomposition. We also present analogous results for finite-width RF and NT models.

Our analysis is analogous and generalizes the recent results of \cite{ghorbani2019linearized}.
However, while \cite{ghorbani2019linearized} assumed the covariates $\bx_i$
to be uniformly distributed over the sphere $\S^{d-1}(\sqrt{d})$, we introduced and analyzed a more general model in which the covariates mostly
lie in the signal subspace with dimension $d_0\ll d$, and the target function is also dependent on that subspace.
In fact our results follow as special cases of a more general model discussed in Appendix \ref{sec:GeneralApp}.

Depending on the relation between signal dimension $d_0$, ambient dimension $d$, and the covariate signal-to-noise ratio $r$,
the model presents a continuum of different behaviors. At one extreme, the covariates are fully $d$-dimensional,
and RKHS methods are highly suboptimal compared to NN. At the other, covariates are close to $d_0$-dimensional
and RKHS methods are instead  more competitive with NN.

Finally, the Fourier decomposition of images is a simple proxy for the decomposition of the covariate vector $\bx$ into
its low-dimensional dominant component (low frequency) and high-dimensional component (high frequency) \cite{yin2019fourier}.

\section*{Acknowledgements}
This work was partially supported by the NSF grants CCF-1714305, IIS-1741162, DMS-1418362, DMS-1407813 and by the ONR grant
N00014-18-1-2729.

\bibliographystyle{amsalpha}
\bibliography{rfntk_anisotropic.bbl}

\newpage

\appendix
\counterwithin{figure}{section}
\counterwithin{table}{section}

\section{Details of numerical experiments} \label{app:exp_details}
\subsection{General training details}

All models studied in the paper are trained with squared loss and $\ell_2$ regularization. For multi-class datasets such as FMNIST, one-hot encoded labels are used for training. All models discussed in the paper use ReLU non-linearity. Fully-connected models are initialized according to mean-field parameterization \cite{nguyen2019mean, nguyen2020rigorous, mei2018mean}. All neural networks are optimized with SGD with $0.9$ momentum. The learning-rate evolves according to the cosine rule 
\begin{align} \label{eqn:cosine_rule}
lr_t = lr_0 \max((1 + \cos(\frac{t \pi}{T})), \frac{1}{15})
\end{align}
where $lr_0 = 10^{-3}$ and $T = 750$ is the total number of training epochs. To ensure the stability of the optimization for wide models, we use $15$ linear warm-up epochs in the beginning. 

When $N \gg 1$, training $\RF$ and $\NT$ with SGD is unstable (unless extremely small learning-rates are used). This makes the optimization prohibitively slow for large datasets. To avoid this issue, instead of SGD, we use conjugate gradient method (CG) for optimizing $\RF$ and $\NT$. Since these two models are strongly convex \footnote{Note that all models are trained with $\ell_2$ regularization.}, the optimizer is unique. Hence, using CG will not introduce any artifacts in the results. 

In order to use CG, we first implement a function to perform Hessian-vector products in TensorFlow \cite{abadi2016tensorflow}. The function handle is then passed to scipy.sparse.cg  for CG. Our Hessian-vector product code uses tensor manipulation utilities implemented by \cite{ghorbani2019investigation}. 

Unfortunately, scipy.sparse.cg does not support one-hot encoded labels. To avoid running CG for each class separately, when the labels are one-hot encoded, we use Adam optimizer \cite{kingma2014adam} instead. When using Adam, the learning-rate still evolves as \eqref{eqn:cosine_rule} with $lr_0 = 10^{-5}$. The batch-size is fixed at $10^4$ to encourage fast convergence to the minimum. 

For $\NN$, $\RF$ and $\NT$, the training is primary done in TensorFlow (v1.12) \cite{abadi2016tensorflow}. For KRR, we generate the kernel matrix first and directly fit the model in regular python. The kernels associated with two-layer models are calculated analytically. For deeper models, the kernels are computed using neural-tangents library in JAX \cite{jax2018github, neuraltangents2020}.

\subsection{Synthetic data experiments} \label{subsec:synthetic_data_exp}

The synthetic data follows the distribution outlined in the main text. In particular,
\begin{align} \label{eqn:data_dist_1}
\bx_i=(\bu_i, \bz_i), \qquad y_i=\vphi(\bu_i), \qquad \bu_i \in \R^{d_0}, \bz_i \in \R^{d - d_0},
\end{align}
where $\bu_i$ and $\bz_i$ are drawn i.i.d from the hyper-spheres with radii $r\sqrt{d_0}$ and $\sqrt{d}$ respectively. We choose
\begin{align} \label{eqn:data_dist_2}
r=d^{\kappa/2},\qquad d_0 = d^{\eta},
\end{align}
where $d$ is fixed to be $1024$ and $\eta=\frac{2}{5}$. We change $\kappa$ in the interval $\{0,…,0.9\}$. For each value of $\kappa$ we generate $2^{20}$ training and $10^4$ test observations.\footnote{Strictly speaking, the model outlined in the main text requires $\bz_i$ to be generated from the hyper-sphere of radius $\sqrt{d - d_0}$.  In order to work with round numbers, in our experiments we use $\sqrt{d}$ instead of $\sqrt{d - d_0}$. The numerical difference between these two choices is negligible.}

The function $\vphi$ is the sum of three orthogonal components $\{\vphi_i\}_{i=1}^3$ with $\Vert \vphi_i \Vert_2 = 1$. To be more specific, 
\begin{align} \label{eqn:data_dist_3}
\vphi_i(\bx) \propto \sum_{j=1}^{d_0 - i} \alpha^{(i)}_j \prod_{k=j}^{j + i} \bx_k, \qquad \alpha_j^{(i)} \overset{i.i.d}{\sim} \exp(1).
\end{align}
This choice of $\vphi_i$ guarantees that each $\vphi_i$ is in the span of degree $i + 1$ spherical harmonics. 

In the experiments presented in Figure \ref{fig:FiniteWidth}, for $\NN$ and $\NT$, the number of hidden units $N$ takes $30$ geometrically spaced values in the interval $[5, 10^4]$. $\NN$ models are trained using SGD with momentum $0.9$ (the learning-rate evolution is described above). We use batch-size of $512$ for the warm-up epochs and batch-size of $1024$ for the rest of the training. For $\RF$, $N$ takes $24$ geometrically spaced values in the interval $[100, 711680]$. The limit $N = 711680$ corresponds to the largest model size we are computationally able to train at this scale. All models are trained with $\ell_2$ regularization. The $\ell_2$ regularization grids used for these experiments are presented in Table \ref{table:synth_hyper_params}. In all our experiments, we choose the $\ell_2$ regularization parameter that yields the best test performance.\footnote{Due to the large size of the test set, choosing these hyper-parameter based on the test set performance has a negligible over-fitting effect. In addition, in studying the approximation error overfitting is not relevant.} In total, we train approximately $10000$ different models just for this subset of experiments.
 
In Figure \ref{fig:KRR} of the main text, we compared the generalization performance of NTK KRR with $\NN$. We use the same training and test data as above to perform this analysis. The number of training data points, $n$, takes $24$ different values ranging from $50$ to $10^5$. The number of test data points is always fixed at $10^4$. 

\begin{table}[h]
\begin{center}
\caption{Hyper-parameter details for synthetic data experiments. \label{table:synth_hyper_params}}

\begin{tabular}{ lll } 
\toprule
Experiment & Model & $\ell_2$ Regularization grid \\
\hline
\hline
\multirow{3}{*}{Approximation error (Fig \ref{fig:FiniteWidth})} & $\NN$ & $\{10^{\alpha_i}\}_{i=1}^{20}$, $\alpha_i$ uniformly spaced in $[-8,-4]$ \\ \cline{2-3}
& $\NT$ & $\{10^{\alpha_i}\}_{i=1}^{10}$, $\alpha_i$ uniformly spaced in $[-4, 2]$ \\ \cline{2-3}
& $\RF$ & $\{10^{\alpha_i}\}_{i=1}^{10}$, $\alpha_i$ uniformly spaced in $[-5, 2]$  \\  \midrule

\hline
\multirow{2}{*}{Generalization error (Fig \ref{fig:KRR})} & $\NN$ & $\{10^{\alpha_i}\}_{i=1}^{25}$, $\alpha_i$ uniformly spaced in $[-8, -2]$ \\ \cline{2-3}
& $\NT$ KRR & $\{10^{\alpha_i}\}_{i=1}^{10}$, $\alpha_i$ uniformly spaced in $[0, 6]$ \\ \cline{2-3}
\bottomrule
\end{tabular}
\end{center}
\end{table}

\subsection{High-frequency noise experiment on FMNIST}
In effort to make the distribution of the covariates more isotropic, in this experiment, we add high-frequency noise to both the training and test data.

Let $\bx \in \R^{k \times k}$ be an image. We first remove the global average of the image and then add high-frequency Gaussian noise to $\bx$ in the following manner:
\begin{enumerate}
\item We convert $\bx$ to frequency domain via Discrete Cosine Transform (DCT II-orthogonal to be precise). We denote the representation of the image in the frequency domain $\tilde{\bx} \in \R^{k \times k}$. 

\item We choose a filter $\bF\in \{0, 1\}^{k \times k}$. $\bF$ determines on which frequencies the noise should be added. The noise matrix $\tilde{\bZ}$ is defined as $\bZ \bigodot \bF$ where $\bZ \in R^{k\times k}$ has i.i.d $\normal(0, 1)$ entries. 
\item We define $\tilde{\bx}_{noisy} = \tilde{\bx} + \tau (\Vert \tilde{\bx} \Vert / \Vert \tilde{\bZ} \Vert) \tilde{\bZ}$. The constant $\tau$ controls the noise magnitude. 
\item We perform Inverse Discrete Cosine Transform (DCT III-orthogonal) on $\tilde{\bx}_{noisy}$ to convert the image to pixel domain. We denote the noisy image in the pixel domain as $\bx_{noisy}$.
\item Finally, we normalize the $\bx_{noisy}$ so that it has norm $\sqrt{d}$.
\end{enumerate}

In the frequency domain, a grayscale image is represented by a matrix $\tilde{\bx} \in \R^{k \times k}$. Qualitatively speaking, elements $(\tilde{\bx})_{i, j}$ with small values of $i$ and $j$ correspond to the low-frequency component of the image and elements with large indices correspond to high-frequency components. The matrix $\bF$ is chosen such that  no noise is added to low frequencies. Specifically, we choose 
\begin{align} \label{eqn:label_formula}
\bF_{i, j} =
 \left\{
\begin{array}{cc}
1 & \mbox{if } (k - i) ^2 + (k - j)^2 \leq (k - 1)^2 \\
0 & \mbox{otherwise}
\end{array}
\right.
\end{align}

This choice of $\bF$ mirrors the average frequency domain representation of FMNIST images (see  Figure \ref{fig:filter} for a comparison). Figure \ref{fig:FMNIST_eigs} shows the eigenvalues of the empirical covariance of the dataset for various noise levels. As discussed in the main text, the distribution of the covariates becomes more isotropic as more and more high-frequency noise is added to the images.

Figure \ref{fig:FMNIST_HF_Perf} shows the normalized squared loss and the classification accuracy of the models as more and more high-frequency noise is added to the data. The normalization factor $R_0 = 0.9$ corresponds to the risk achievable by the (trivial) predictor $\bigg[\hat{y}_{j} (\bx)\bigg]_{1 \leq j\leq 10} = 0.1$.

\begin{figure}[h]
\centering
	\includegraphics[width=\linewidth]{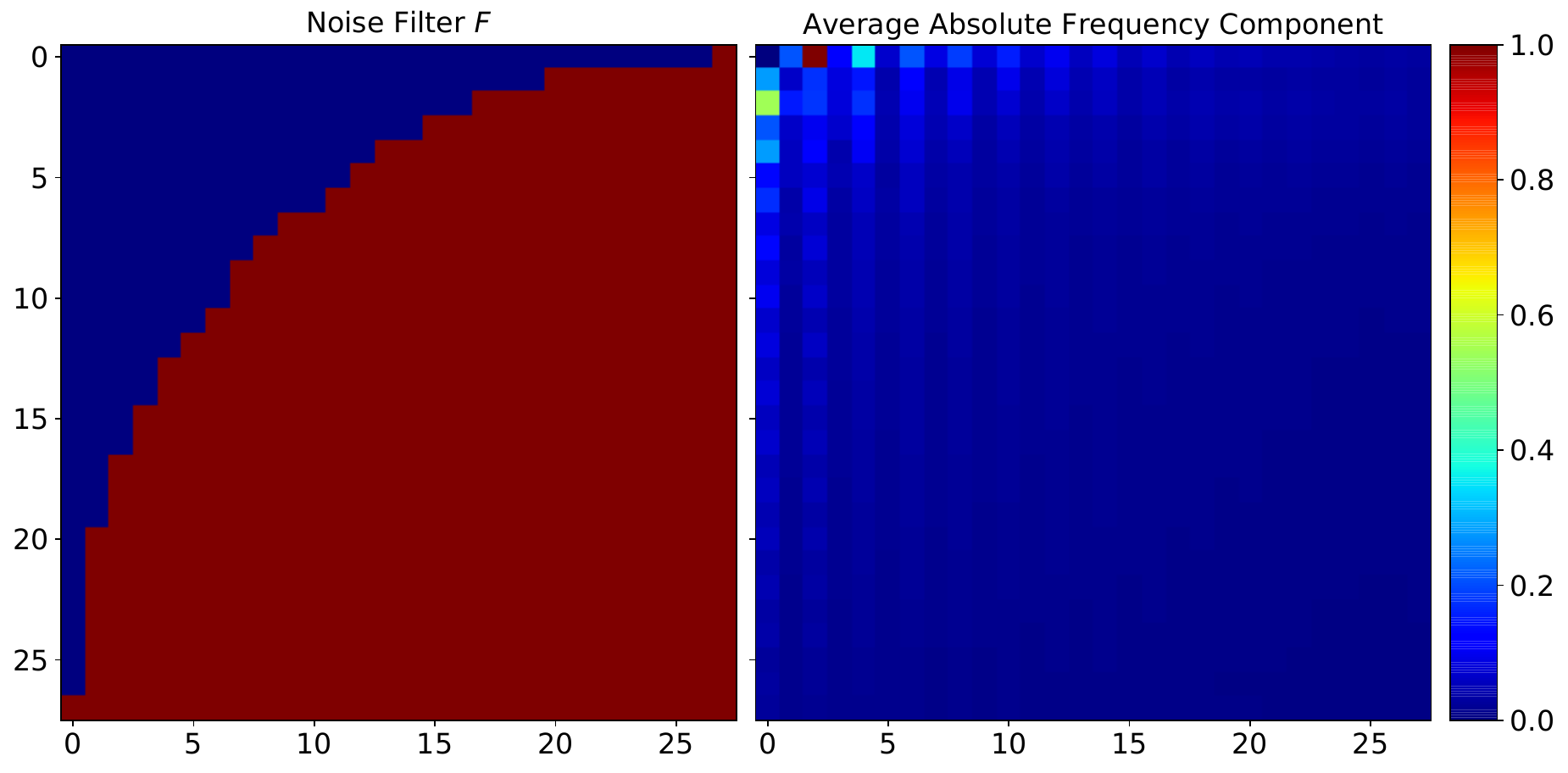}
\caption{Left frame: the pictorial representation of the filter matrix $\bF$ used for the FMNIST experiments. The matrix entries with value zero are represented by color blue while the entries with value one are represented by red. Coordinates on top left-hand side correspond to lower frequency components while coordinates closer to bottom right-hand side represent the high-frequency directions. Right frame: the absolute value of the frequency components of FMNIST images averaged over the training data. The projection of the dataset into the low-frequency region chosen by the filter retains over $95\%$ of the variation in the data. }\label{fig:filter}
\end{figure}

\begin{figure}[h]
\centering
	\includegraphics[width=.8\linewidth]{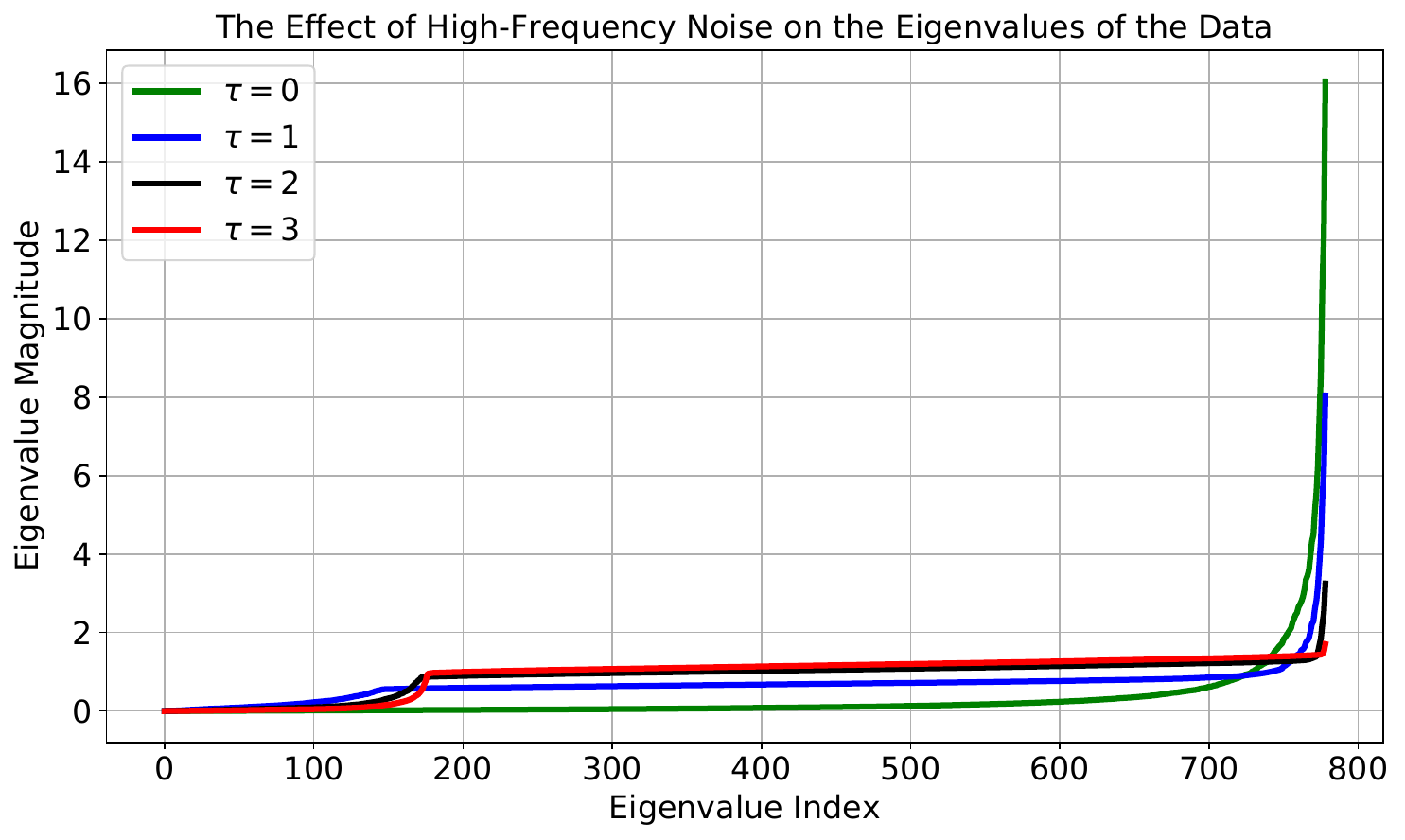}
\caption{The eigenvalues of the empirical covariance matrix of the FMNIST training data. As the noise intensity increases, the distribution of the eigenvalues becomes more isotropic. Note that due to the conservative choice of the filter $\bF$, noise is not added to all of the low-variance directions. These left-out directions corresponds to the small eigenvalues appearing in the left-hand side of the plot.}\label{fig:FMNIST_eigs}
\end{figure}

\begin{figure}[h]
\centering
	\includegraphics[width=\linewidth]{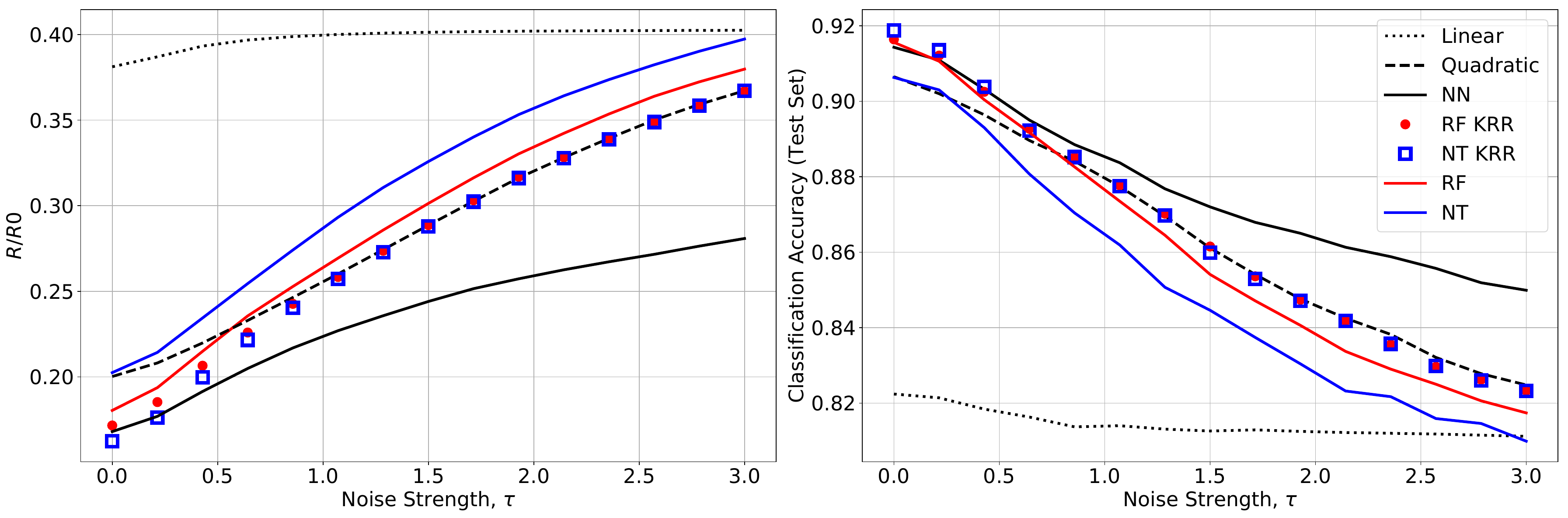}
\caption{The normalized test squared error (left) and the test accuracy (right) of the models trained and evaluated on FMNIST data with high-frequency noise.  }\label{fig:FMNIST_HF_Perf}
\end{figure}

\begin{figure}[h]
\centering
	\includegraphics[width=0.49\linewidth]{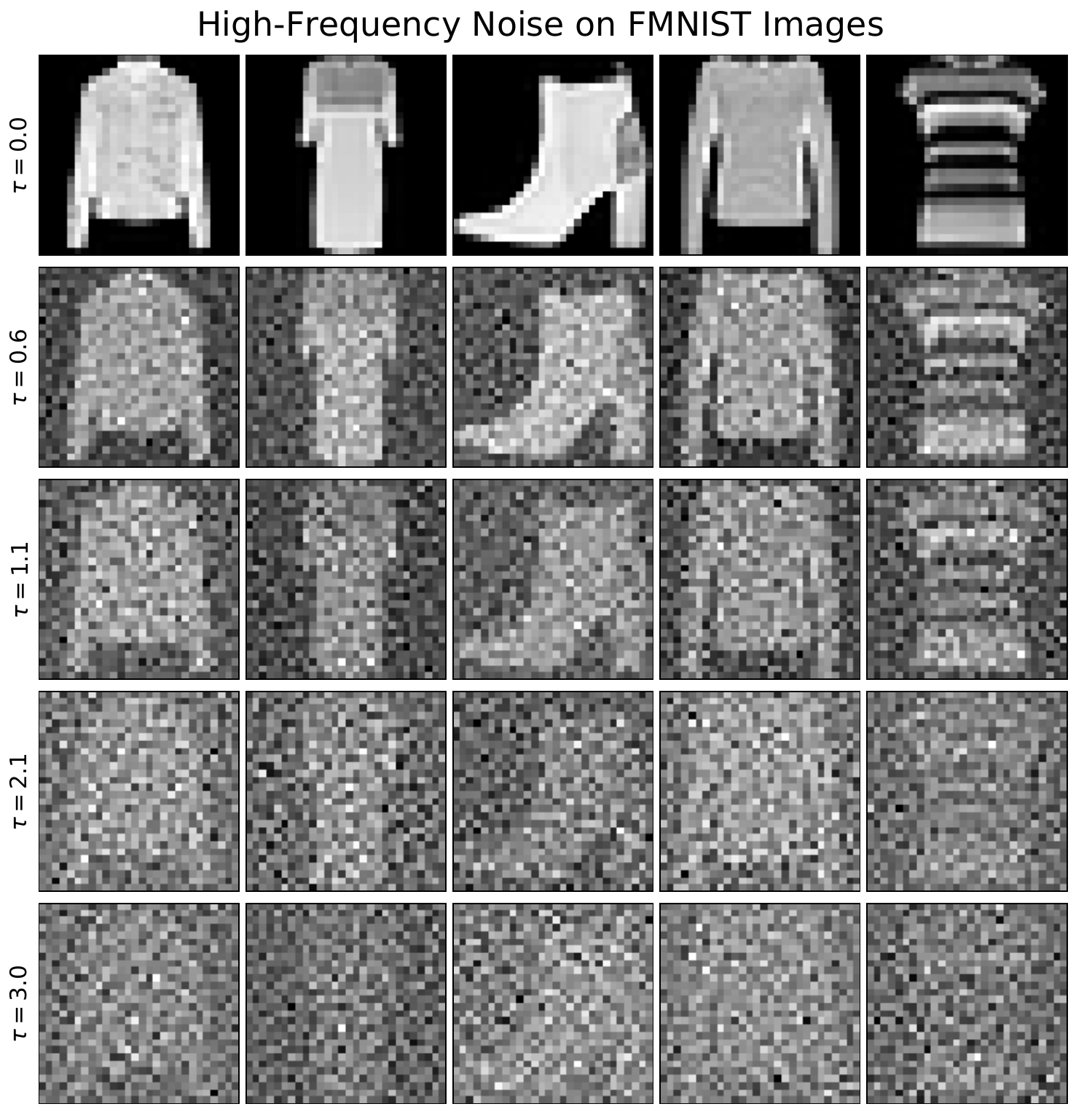}
	\includegraphics[width=0.49\linewidth]{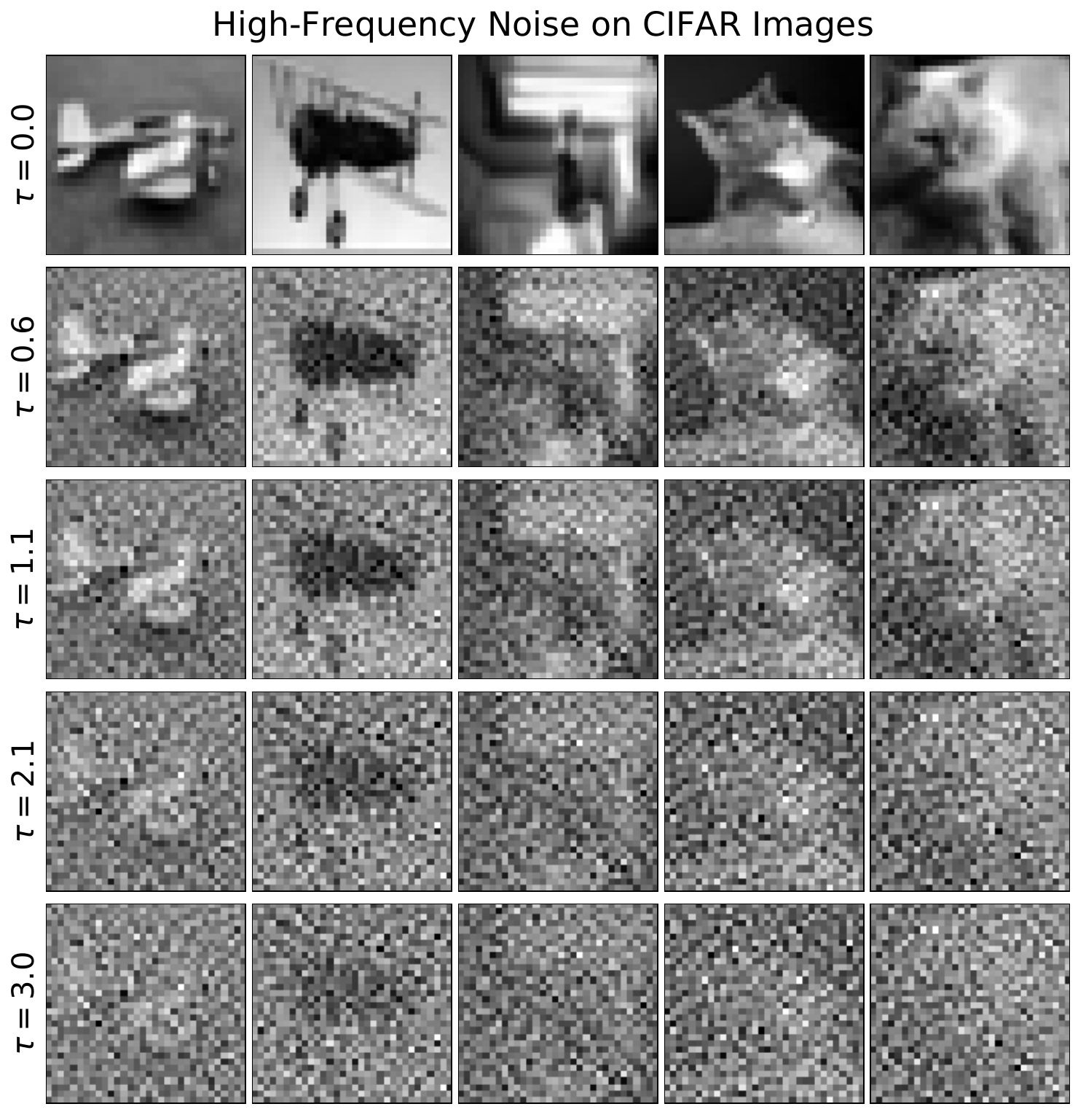}
\caption{Left: FMNIST images with various high-frequency noise levels. Right: CIFAR-2 images with various levels of high-frequency Gaussian noise. The images are converted to grayscale to make the covariate dimension manageable.}\label{fig:HF_images}
\end{figure}

\subsubsection{Experiment hyper-parameters}
For $\NT$ and $\NN$, the number of hidden units $N = 4096$. For $\RF$, we fix $N = 321126$. These hyper-parameter choices ensure that the models have approximately the same number of trainable parameters. $\NN$ is trained with SGD with $0.9$ momentum and learning-rate described by \eqref{eqn:cosine_rule}. The batch-size for the warm-up epochs is $500$. After the warm-up stage is over, we use batch-size of $1000$ to train the network. Since CG is not available in this setting, $\NT$ and $\RF$ are optimized using Adam for $T = 750$ epochs with batch-size of $10^4$. The $\ell_2$ regularization grids used for training these models are listed in Table \ref{table:RidgeRegs}.

\subsection{High-frequency noise experiment on CIFAR-2}
We perform a similar experiment on a subset of CIFAR-10. We choose two classes (airplane and cat) from the ten classes of CIFAR-10. This choice provides us with $10^4$ training and $2000$ test data points. Given that the number of training observations is not very large, we reduce the covariate dimension by converting the images to grayscale. This transformation reduces the covariate dimension to $d = 1024$. 

Figure \ref{fig:CIFAR2_noisy_perf} demonstrates the evolution of the model performances as the noise intensity increases. In the noiseless regime ($\tau = 0$), all models have comparable performances. However, as the noise level increases, the performance gap between $\NN$ and RKHS methods widens. For reference, the accuracy gap between $\NN$ and $\NT$ KRR is only $0.6 \%$ at $\tau = 0$. However, at $\tau = 3$, this gap increases to $4.5 \%$. The normalization factor $R_0 = 0.25$ corresponds to the risk achievable by the trivial estimator $\hat{y}(\bx) = 0.5$.

 \begin{figure}[h]
\centering
	\includegraphics[width=\linewidth]{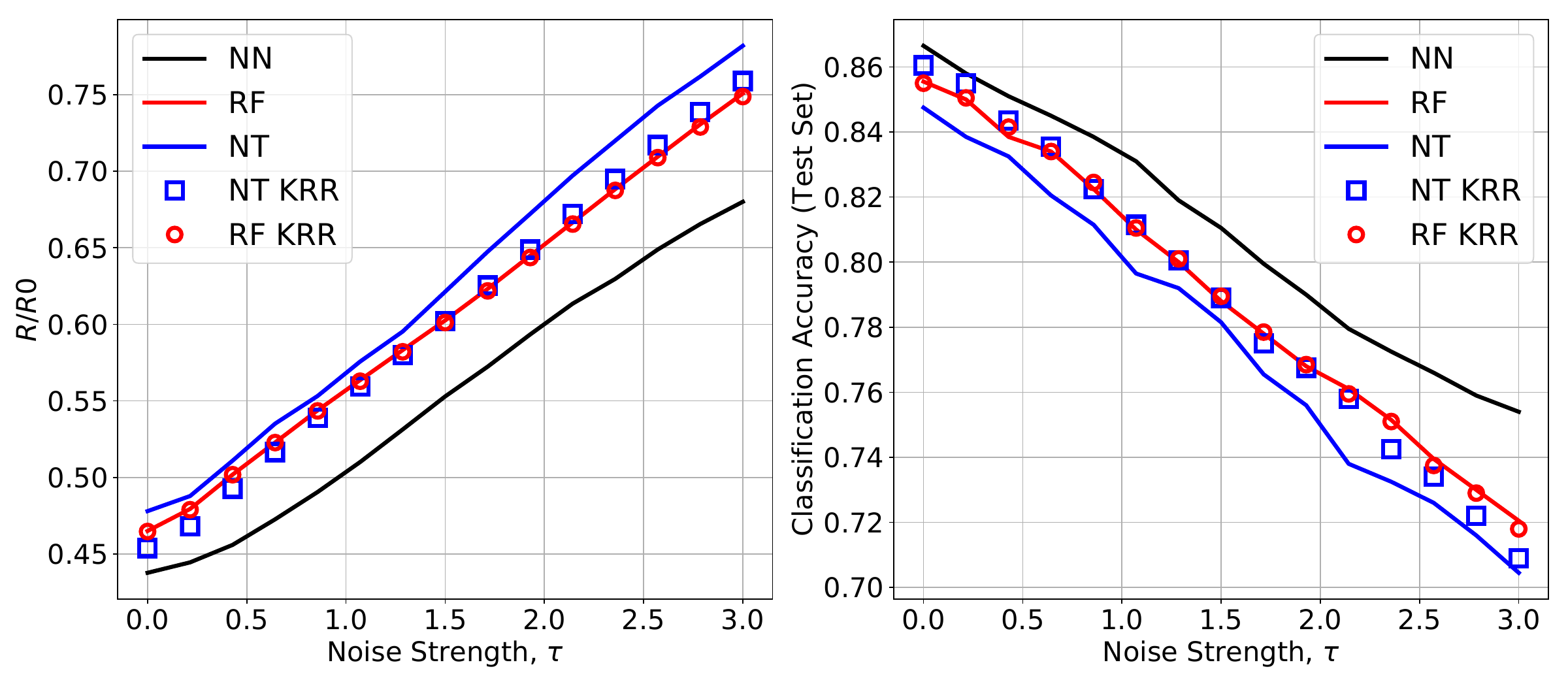}
\caption{Normalized test squared error (left) and test classification accuracy (right) of the models on noisy CIFAR-2. As the noisy intensity increases, the performance gap between $\NN$ and RKHS methods widens. For reference, the accuracy gap between $\NN$ and $\NT$ KRR is only $0.6 \%$ at $\tau = 0$. However, at $\tau = 3$, this gap increases to $4.5 \%$. For finite-width models, $N$ is chosen such that the number of trainable parameters is approximately equal across the models.  For $\NN$ and $\NT$, $N = 4096$ and for $\RF$, $N = 4.2 \times 10^6$. We use the noise filter described in \eqref{eqn:label_formula}.}\label{fig:CIFAR2_noisy_perf}
\end{figure}

\subsubsection{Experiment hyper-parameters}
For $\NT$ and $\NN$, the number of hidden units $N = 4096$. For $\RF$, we fix $N = 4.2 \times 10^6$. These hyper-parameter choices ensure that the models have approximately the same number of trainable parameters. $\NN$ is trained with SGD with $0.9$ momentum and learning-rate described by \eqref{eqn:cosine_rule}. The batch-size is fixed at $250$. $\NT$ is optimized via CG with $750$ maximum iterations. The $\ell_2$ regularization grids used for training these models are listed in Table \ref{table:RidgeRegs}.

\begin{table}[h]
\begin{center}
\caption{Details of regularization parameters used for high-frequency noise experiments. \label{table:RidgeRegs}}

\begin{tabular}{ lll } 
\toprule
Dataset & Model & $\ell_2$ Regularization grid \\
\hline
\hline
\multirow{5}{*}{FMNIST} & $\NN$ & $\{10^{\alpha_i}\}_{i=1}^{20}$, $\alpha_i$ uniformly spaced in $[-6, -2]$ \\ \cline{2-3}
& $\NT$ & $\{10^{\alpha_i}\}_{i=1}^{20}$, $\alpha_i$ uniformly spaced in $[-5, 3]$ \\ \cline{2-3}
& $\RF$ & $\{10^{\alpha_i}\}_{i=1}^{20}$, $\alpha_i$ uniformly spaced in $[-5, 3]$  \\ \cline{2-3}
& $\NT$ KRR & $\{10^{\alpha_i}\}_{i=1}^{20}$, $\alpha_i$ uniformly spaced in $[-1, 5]$ \\ \cline{2-3}
& $\RF$ KRR & $\{10^{\alpha_i}\}_{i=1}^{20}$, $\alpha_i$ uniformly spaced in $[-1, 5]$ \\ \hline
\hline
\multirow{5}{*}{CIFAR-2} & $\NN$ & $\{10^{\alpha_i}\}_{i=1}^{20}$, $\alpha_i$ uniformly spaced in $[-6, -2]$ \\ \cline{2-3}
& $\NT$ & $\{10^{\alpha_i}\}_{i=1}^{20}$, $\alpha_i$ uniformly spaced in $[-4, 4]$ \\ \cline{2-3}
& $\RF$ & $\{10^{\alpha_i}\}_{i=1}^{40}$, $\alpha_i$ uniformly spaced in $[-2, 10]$ \\ \cline{2-3}
& $\NT$ KRR & $\{10^{\alpha_i}\}_{i=1}^{20}$, $\alpha_i$ uniformly spaced in $[-2, 4]$ \\ \cline{2-3}
& $\RF$ KRR & $\{10^{\alpha_i}\}_{i=1}^{20}$, $\alpha_i$ uniformly spaced in $[-2, 4]$ \\
\bottomrule
\end{tabular}
\end{center}
\end{table}

\newpage
\subsection{Low-frequency noise experiments on FMNIST}
To examine the ability of NN and RKHS methods in learning the information in low-variance components of the covariates, we replace the low-frequency components of the image with Gaussian noise. To be specific, we follow the following steps to generate the noisy datasets:

\begin{enumerate}
\item We normalize all images to have mean zero and norm $\sqrt{d}$. 
\item Let $\mathcal{D}_{train}$ denote the set of training images in the DCT-frequency domain. We compute the mean $\mu$ and the covariance $\Sigma$ of the elements of $\mathcal{D}_{train}$.
\item We fix a threshold $\alpha \in \N$ where $1 \leq \alpha \leq k$.
\item Let $\bx$ be an image in the dataset (test or train). We denote the representation of $\bx$ in the frequency domain with $\tilde{\bx}$. For each image, we draw a noise matrix $\bz \sim \cN(\mu, \Sigma)$. We have
\begin{align*}
\big[\tilde{\bx}_{noisy}\big]_{i, j} = 
\left\{
\begin{array}{cc}
(\bz)_{i, j} & \mbox{if } i,j \leq \alpha \\
\tilde{\bx}_{i, j} & \mbox{otherwise}
\end{array}
\right.
\end{align*}
\item We perform IDCT on $\tilde{\bx}_{noisy}$ to get the noisy image $\bx_{noisy}$.
\end{enumerate}
The fraction of the frequencies replaced by noise is $\alpha^2 / k^2$. 
%\todo[inline]{Examples to be added}

 \begin{figure}[h]
\centering
	\includegraphics[width=\linewidth]{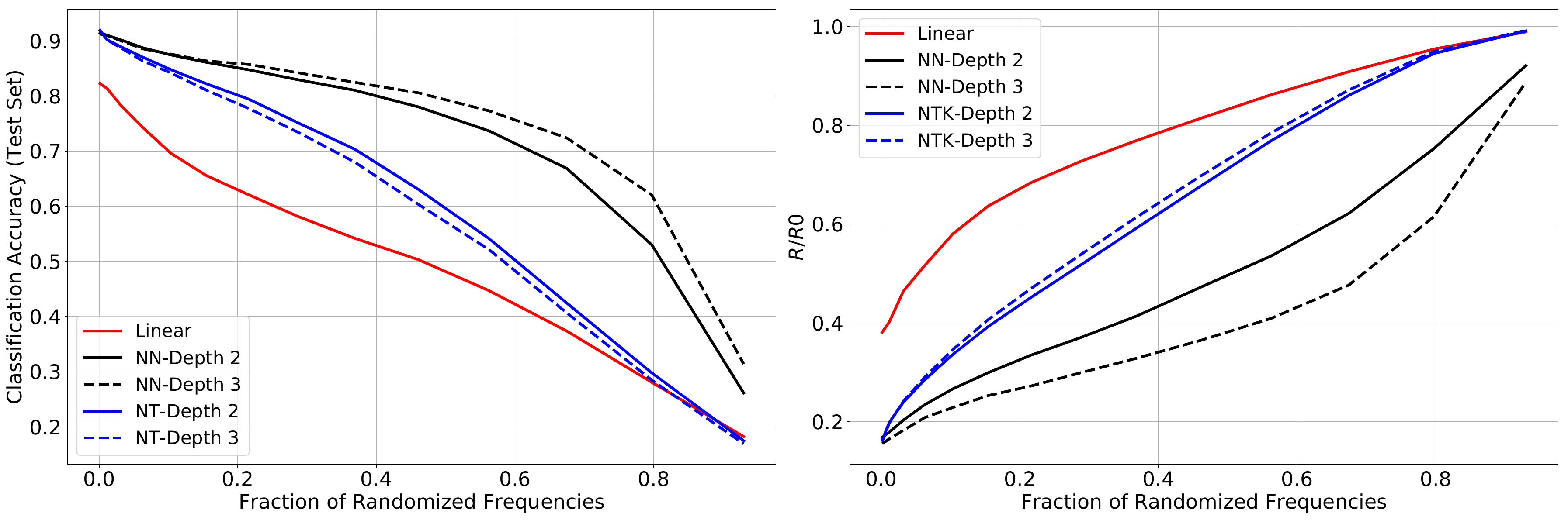}
\caption{Normalized test squared error (left) and test classification accuracy (right) of the models on FMNIST with low-frequency Gaussian noise.}\label{fig:FMNIST_LFG}
\end{figure}

\subsubsection{Experiment hyper-parameters}
For neural networks trained for this experiment, we fix the number of hidden units per-layer to $N = 4096$. This corresponds to approximately $3.2 \times 10^6$ trainable parameters for two-layer networks and $2 \times 10^7$ trainable parameters for three-layer networks. Both models are trained using SGD with momentum with learning rate described by $\eqref{eqn:cosine_rule}$ (with $lr_0 = 10^{-3}$). For the warm-up epochs, we use batch-size of $500$. We increase the batch-size to $1000$ after the warm-up stage. The regularization grids used for training our models are presented in Table \ref{table:LFGaussian}.

\subsection{Low-frequency noise experiments on CIFAR-10}

To test whether our insights are valid for convolutional models, we repeat the same experiment for CNNs trained on CIFAR-10. The noisy data is generated as follows:
\begin{enumerate}
\item Let $\mathcal{D}_{train}$ denote the set of training images in the DCT-frequency domain. Note that CIFAR-10 images have $3$ channels. To convert the images to frequency domain, we apply two-dimensional Discrete Cosine Transform (DCT-II orthogonal) to each channel separately. We compute the mean $\mu$ and the covariance $\Sigma$ of the elements of $\mathcal{D}_{train}$.
\item We fix a threshold $\alpha \in \N$ where $1 \leq \alpha \leq 32$.
\item Let $\bx \in \R^{32 \times 32 \times 3}$ be an image in the dataset (test or train). We denote the representation of $\bx$ in the DCT-frequency domain with $\tilde{\bx}  \in \R^{32 \times 32 \times 3}$. For each image, we draw a noise matrix $\bz \sim \cN(\mu, \Sigma)$. We have
\begin{align*}
\big[\tilde{\bx}_{noisy}\big]_{i, j, k} = 
\left\{
\begin{array}{cc}
(\bz)_{i, j, k} & \mbox{if } i,j \leq \alpha \\
\tilde{\bx}_{i, j, k} & \mbox{otherwise}
\end{array}
\right.
\end{align*}
\item We perform IDCT on $\tilde{\bx}_{noisy}$ to get the noisy image $\bx_{noisy}$.
\item We normalize the noisy data to have zero per-channel mean and unit per-channel standard deviation. The normalization statistics are computed using only the training data.
\end{enumerate}

We use Myrtle-5 architecture for our analysis. The Myrtle family is a collection of simple light-weight high-performance purely convolutional models. The simplicity of these models coupled with their good performance makes them a natural candidate for our analysis. Figure \ref{fig:myrtle} describes the details of this architecture. We fix the number of channels in all convolutional layer to be $N = 512$. This corresponds to approximately $7 \times 10^6$ parameters. Similar to the fully-connected networks, our convolutional models are also optimized via SGD with $0.9$ momentum (learning rate evolves as \eqref{eqn:cosine_rule} with $lr_0 = 0.1$ and $T=70$). We fix the batch-size to $128$. To keep the experimental setting as simple as possible, we do not use any data augmentation for training the network. 

 \begin{figure}[h]
\centering
	\includegraphics[width=\linewidth]{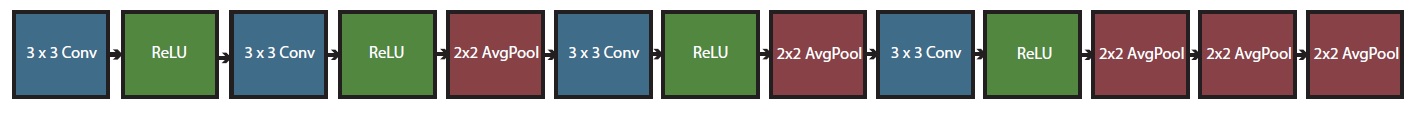}
\caption{Details of Myrtle-5 architecture. The network only uses convolutions and average pooling. In particular, we do not use any batch-normalization \cite{ioffe2015batch} layers in this network. The figure is borrowed from \cite{shankar2020neural}.}\label{fig:myrtle}
\end{figure}

 \begin{figure}[h]
\centering
	\includegraphics[width=\linewidth]{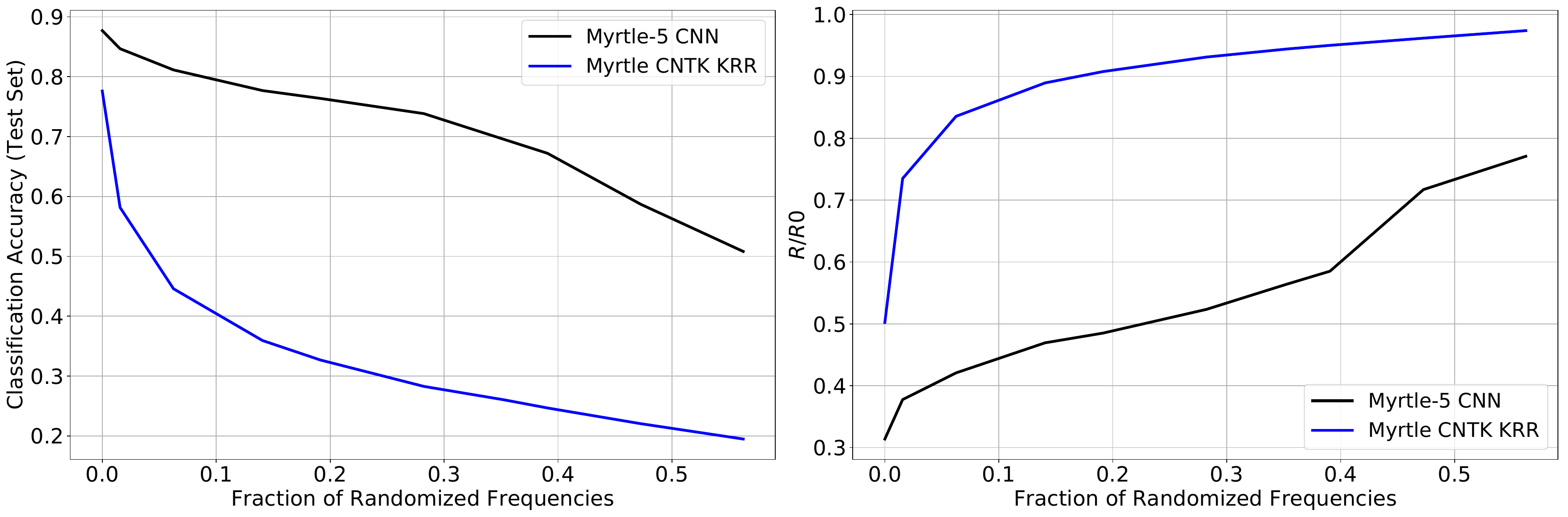}
\caption{Performance of Myrtle-5 and KRR with convolutional neural tangent kernel (CNTK) on noisy CIFAR-10. CNTK is generated from the Myrtle-5 architecture using neural-tangents JAX library. When no noise is present in the data, the CNN achieves $87.7\%$ and the CNTK achieves $77.6\%$ classification accuracy. After randomizing only $1.5\%$ of the frequencies (corresponding to $\alpha = 4$) CNTK classification performance falls to $58.2\%$ while the CNN retains $84.7\%$ accuracy.}\label{fig:CNN_perf}
\end{figure}

\begin{figure}[h]
\centering
	\includegraphics[width=0.49\linewidth]{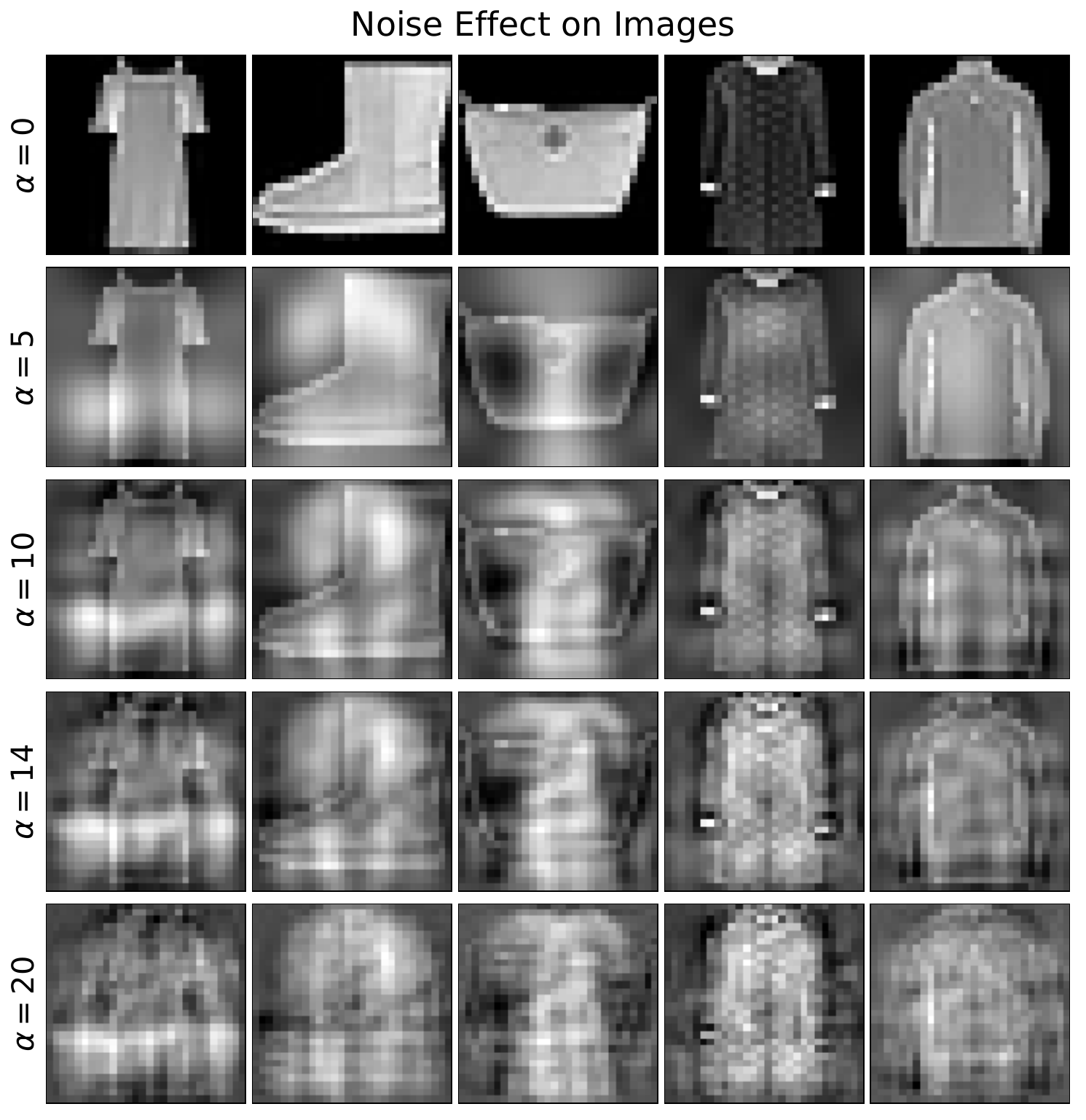}
	\includegraphics[width=0.49\linewidth]{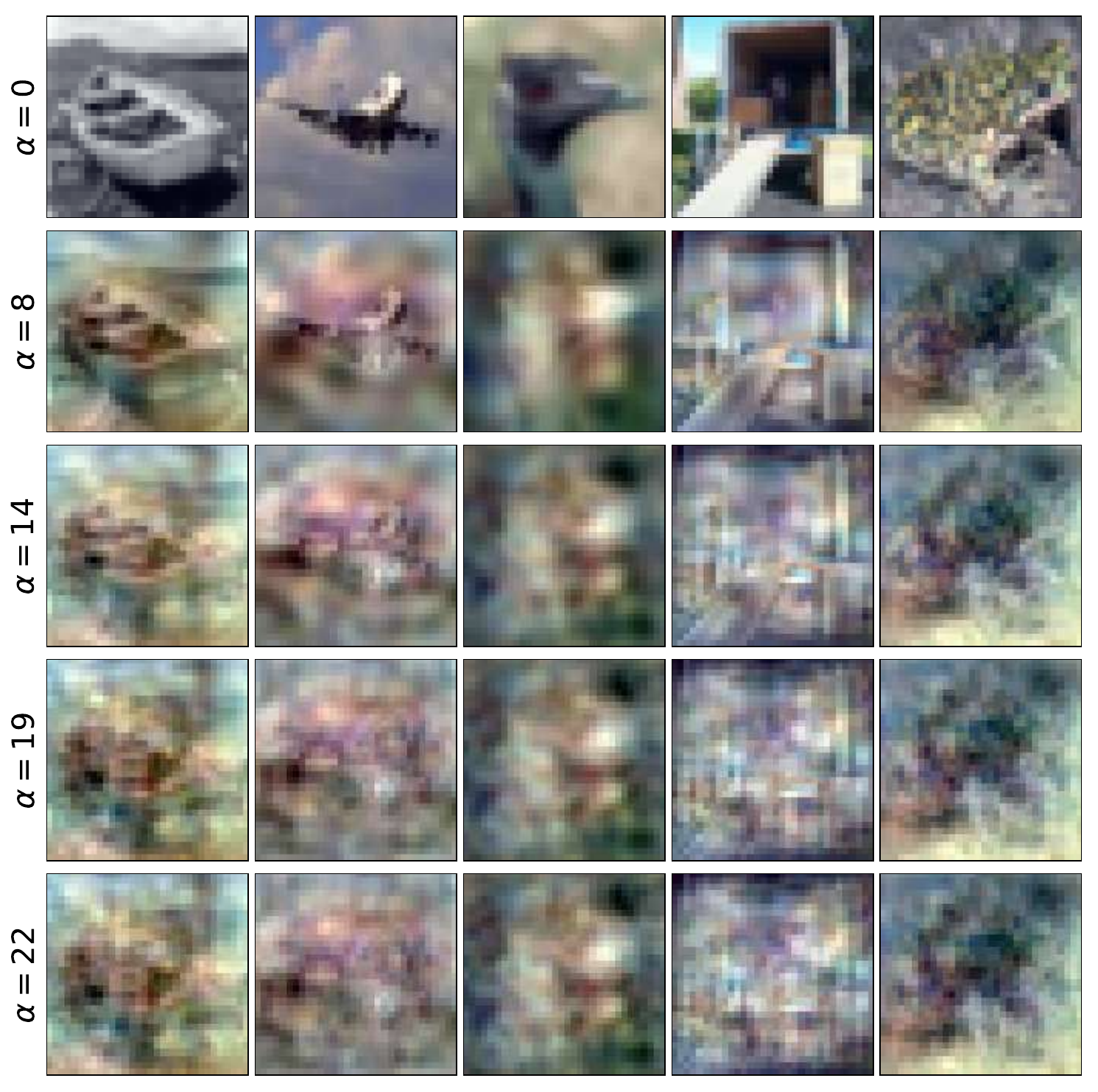}
\caption{The effect of low-frequency noise for various cut-off thresholds, $\alpha$. The left panel corresponds to the noisy FMNIST images and the right panel corresponds to CIFAR-10 images. In order to plot CIFAR-10 images, we rescale them to the interval $[0, 1]$.}\label{fig:LF_images}
\end{figure}

\begin{table}[h]
\begin{center}
\caption{Details of regularization parameters used for low-frequency noise experiments. \label{table:LFGaussian}}
\begin{tabular}{ lll } 
\toprule
Dataset & Model & $\ell_2$ Regularization grid \\
\hline
\hline
\multirow{5}{*}{FMNIST} & $\NN$ depth 2 & $\{10^{\alpha_i}\}_{i=1}^{20}$, $\alpha_i$ uniformly spaced in $[-6, -2]$ \\ \cline{2-3}
& $\NN$ depth 3 & $\{10^{\alpha_i}\}_{i=1}^{10}$, $\alpha_i$ uniformly spaced in $[-7, -5]$ \\ \cline{2-3}
& NTK KRR depth 2 & $\{10^{\alpha_i}\}_{i=1}^{20}$, $\alpha_i$ uniformly spaced in $[-1, 5]$  \\ \cline{2-3}
& NTK KRR depth 3 & $\{10^{\alpha_i}\}_{i=1}^{20}$, $\alpha_i$ uniformly spaced in $[-4, 3]$ \\ \cline{2-3}
& Linear model & $\{10^{\alpha_i}\}_{i=1}^{30}$, $\alpha_i$ uniformly spaced in $[-1, 5]$ \\ \hline
\hline
\multirow{2}{*}{CIFAR-10} & Myrtle-5 & $\{10^{\alpha_i}\}_{i=1}^{10}$, $\alpha_i$ uniformly spaced in $[-5, -2]$ \\ \cline{2-3}
& KRR (Myrtle-5 NTK) & $\{10^{\alpha_i}\}_{i=1}^{20}$, $\alpha_i$ uniformly spaced in $[-6, 1]$ \\ \hline
\bottomrule
\end{tabular}
\end{center}
\end{table}

\section{Technical background on function spaces on the sphere}
\label{sec:Background}

\subsection{Functional spaces over the sphere}

For $d \ge 1$, we let $\S^{d-1}(r) = \{\bx \in \R^{d}: \| \bx \|_2 = r\}$ denote the sphere with radius $r$ in $\reals^d$.
We will mostly work with the sphere of radius $\sqrt d$, $\S^{d-1}(\sqrt{d})$ and will denote by $\mu_{d-1}$ the uniform probability measure on $\S^{d-1}(\sqrt d)$. 
All functions in the following are assumed to be elements of $ L^2(\S^{d-1}(\sqrt d) ,\mu_{d-1})$, with scalar product and norm denoted as $\<\,\cdot\,,\,\cdot\,\>_{L^2}$
and $\|\,\cdot\,\|_{L^2}$:
\begin{align}
\<f,g\>_{L^2} \equiv \int_{\S^{d-1}(\sqrt d)} f(\bx) \, g(\bx)\, \mu_{d-1}(\de \bx)\,.
\end{align}

For $\ell\in\integers_{\ge 0}$, let $\tilde{V}_{d,\ell}$ be the space of homogeneous harmonic polynomials of degree $\ell$ on $\reals^d$ (i.e. homogeneous
polynomials $q(\bx)$ satisfying $\Delta q(\bx) = 0$), and denote by $V_{d,\ell}$ the linear space of functions obtained by restricting the polynomials in $\tilde{V}_{d,\ell}$
to $\S^{d-1}(\sqrt d)$. With these definitions, we have the following orthogonal decomposition
\begin{align}
L^2(\S^{d-1}(\sqrt d) ,\mu_{d-1}) = \bigoplus_{\ell=0}^{\infty} V_{d,\ell}\, . \label{eq:SpinDecomposition}
\end{align}
The dimension of each subspace is given by
\begin{align}
\dim(V_{d,\ell}) = B(d, \ell) = \frac{2 \ell + d - 2}{\ell} { \ell + d - 3 \choose \ell - 1} \, .
\end{align}
For each $\ell\in \integers_{\ge 0}$, the spherical harmonics $\{ Y_{\ell, j}^{(d)}\}_{1\le j \in \le B(d, \ell)}$ form an orthonormal basis of $V_{d,\ell}$:
\[
\<Y^{(d)}_{ki}, Y^{(d)}_{sj}\>_{L^2} = \delta_{ij} \delta_{ks}.
\]
Note that our convention is different from the more standard one, that defines the spherical harmonics as functions on $\S^{d-1}(1)$.
It is immediate to pass from one convention to the other by a simple scaling. We will drop the superscript $d$ and write $Y_{\ell, j} = Y_{\ell, j}^{(d)}$ whenever clear from the context.

We denote by $\proj_k$  the orthogonal projections to $V_{d,k}$ in $L^2(\S^{d-1}(\sqrt d),\mu_{d-1})$. This can be written in terms of spherical harmonics as
\begin{align}
\proj_k f(\bx) \equiv& \sum_{l=1}^{B(d, k)} \< f, Y_{kl}\>_{L^2} Y_{kl}(\bx). 
\end{align}
We also define
$\proj_{\le \ell}\equiv \sum_{k =0}^\ell \proj_k$, $\proj_{>\ell} \equiv \id -\proj_{\le \ell} = \sum_{k =\ell+1}^\infty \proj_k$,
and $\proj_{<\ell}\equiv \proj_{\le \ell-1}$, $\proj_{\ge \ell}\equiv \proj_{>\ell-1}$.

\subsection{Gegenbauer polynomials}
\label{sec:Gegenbauer}

The $\ell$-th Gegenbauer polynomial $Q_\ell^{(d)}$ is a polynomial of degree $\ell$. Consistently
with our convention for spherical harmonics, we view $Q_\ell^{(d)}$ as a function $Q_{\ell}^{(d)}: [-d,d]\to \reals$. The set $\{ Q_\ell^{(d)}\}_{\ell\ge 0}$
forms an orthogonal basis on $L^2([-d,d],\tilde\mu^1_{d-1})$, where $\tilde\mu^1_{d-1}$ is the distribution of $\sqrt{d}\<\bx,\be_1\>$ when $\bx\sim \mu_{d-1}$,
satisfying the normalization condition:
\begin{equation}
\begin{aligned}
\int_{-d}^d Q^{(d)}_k(t) \, Q^{(d)}_j(t ) \, \de \tilde\mu^1_{d-1}  = &  \frac{w_{d-2}}{d w_{d-1}  }  \int_{-d}^d  Q^{(d)}_k(t) \, Q^{(d)}_j(t ) \, \Big( 1 - \frac{t^2}{d^2} \Big)^{(d-3)/2} \, dt  \\  
\label{eq:GegenbauerNormalization}
= & \frac{1}{B(d,k)}\, \delta_{jk}   \, ,
\end{aligned}
\end{equation}
where we denoted $w_{d-1} = \frac{2 \pi^{d/2} }{\Gamma (d/2)}$ the surface area of the sphere $\S^{d-1} (1)$. In particular, these polynomials are normalized so that  $Q_\ell^{(d)}(d) = 1$. 

Gegenbauer polynomials are directly related to spherical harmonics as follows. Fix $\bv\in\S^{d-1}(\sqrt{d})$ and 
consider the subspace of  $V_{\ell}$ formed by all functions that are invariant under rotations in $\reals^d$ that keep $\bv$ unchanged.
It is not hard to see that this subspace has dimension one, and coincides with the span of the function $Q_{\ell}^{(d)}(\<\bv,\,\cdot\,\>)$.

We will use the following properties of Gegenbauer polynomials
\begin{enumerate}
\item For $\bx, \by \in \S^{d-1}(\sqrt d)$
\begin{align}
\< Q_j^{(d)}(\< \bx, \cdot\>), Q_k^{(d)}(\< \by, \cdot\>) \>_{L^2} = \frac{1}{B(d,k)}\delta_{jk}  Q_k^{(d)}(\< \bx, \by\>).  \label{eq:ProductGegenbauer}
\end{align}
\item For $\bx, \by \in \S^{d-1}(\sqrt d)$
\begin{align}
Q_k^{(d)}(\< \bx, \by\> ) = \frac{1}{B(d, k)} \sum_{i =1}^{ B(d, k)} Y_{ki}^{(d)}(\bx) Y_{ki}^{(d)}(\by). \label{eq:GegenbauerHarmonics}
\end{align}
\item Recurrence formula 
\begin{align}
\frac{t}{d}\,  Q_k^{(d)}(t) = \frac{k}{2k + d - 2} Q_{k-1}^{(d)}(t) + \frac{k + d - 2}{2k + d - 2} Q_{k+1}^{(d)}(t). \label{eq:RecursionG}
\end{align}
\item Rodrigues formula
\begin{align}
Q_k^{(d)}(t) = (-1/2)^k d^k \frac{\Gamma((d - 1)/2)}{\Gamma(k + (d - 1)/2)} \Big( 1 -  \frac{t^2}{d^2} \Big)^{(3-d)/2} \Big( \frac{\de }{\de t}\Big)^k \Big(1 - \frac{t^2}{d^2} \Big)^{k + (d-3)/2}. 
\label{eq:Rogrigues_formula}
\end{align}
\end{enumerate}
Note in particular that property 2 implies that --up to a constant-- $Q_k^{(d)}(\< \bx, \by\> )$ is a representation of the projector onto 
the subspace of degree -$k$ spherical harmonics
\begin{align}
(\proj_k f)(\bx) = B(d,k) \int_{\S^{d-1}(\sqrt{d})} \, Q_k^{(d)}(\< \bx, \by\> )\,  f(\by)\, \mu_{d-1}(\de\by)\, .\label{eq:ProjectorGegenbauer}
\end{align}

\subsection{Hermite polynomials}

The Hermite polynomials $\{\He_k\}_{k\ge 0}$ form an orthogonal basis of $L^2(\reals,\gamma)$, where $\gamma(\de x) = e^{-x^2/2}\de x/\sqrt{2\pi}$ 
is the standard Gaussian measure, and $\He_k$ has degree $k$. We will follow the classical normalization (here and below, expectation is with respect to
$G\sim\normal(0,1)$):
\begin{align}
\E\big\{\He_j(G) \,\He_k(G)\big\} = k!\, \delta_{jk}\, .
\end{align}
As a consequence, for any function $g\in L^2(\reals,\gamma)$, we have the decomposition
\begin{align}
g(x) = \sum_{k=0}^{\infty}\frac{\mu_k(g)}{k!}\, \He_k(x)\, ,\;\;\;\;\;\; \mu_k(g) \equiv \E\big\{g(G)\, \He_k(G)\}\, .
\end{align}

Notice that for functions $g$ that are $k$-weakly differentiable with $g^{(k)}$ the $k$-th weak derivative, we have
\begin{equation} \label{eq:weak_derivative_hermite_coefficient}
\mu_k (g) = \E_{G} [ g^{(k)} (G) ].
\end{equation}

The Hermite polynomials can be obtained as high-dimensional limits of the Gegenbauer polynomials introduced in the previous section. Indeed, 
the Gegenbauer polynomials are constructed by Gram-Schmidt orthogonalization of the monomials $\{x^k\}_{k\ge 0}$ with respect to the measure 
$\tilde\mu^1_{d-1}$, while Hermite polynomial are obtained by Gram-Schmidt orthogonalization with respect to $\gamma$. Since $\tilde\mu^1_{d-1}\Rightarrow \gamma$
(here $\Rightarrow$ denotes weak convergence),
it is immediate to show that, for any fixed integer $k$, 
\begin{align}
\lim_{d \to \infty} \Coeff\{ Q_k^{(d)}( \sqrt d x) \, B(d, k)^{1/2} \} = \Coeff\left\{ \frac{1}{(k!)^{1/2}}\,\He_k(x) \right\}\, .\label{eq:Gegen-to-Hermite}
\end{align}
Here and below, for $P$ a polynomial, $\Coeff\{ P(x) \}$ is  the vector of the coefficients of $P$.

\subsection{Tensor product of spherical harmonics}

We will consider in this paper the product space 
\begin{equation}\label{eq:def_PS}
\PS^{\bd} \equiv \prod_{q = 1}^Q \S^{d_q - 1} \lp \sqrt{d_q} \rp,
\end{equation}
and the uniform measure on $\PS^\bd$, denoted $\mu_{\bd} \equiv \mu_{d_1 - 1} \otimes \ldots \otimes \mu_{d_Q -1} = \bigotimes_{q \in [Q]} \mu_{d_q - 1}$, where we recall $\mu_{d_q - 1} \equiv \Unif ( \S^{d_q-1} (\sqrt{d_q}))$. We consider the functional space of $L^2 (\PS^{\bd}, \mu_{\bd})$ with scalar product and norm denoted as $\< \cdot , \cdot \>_{L^2}$ and $\| \cdot \|_{L^2}$:
\[
\< f , g \>_{L^2} \equiv  \int_{\PS^{\bd}} f( \obx ) g ( \obx ) \, \mu_{\bd } ( \de \obx ).
\] 
For $\bell = ( \ell_1 , \ldots , \ell_Q ) \in \integers_{\ge 0}^Q$, let $\tilde{V}^{\bd}_{\bell} \equiv \tilde{V}_{d_1,\ell_1 } \otimes \ldots \otimes \tilde{V}_{d_Q,\ell_Q} $ be the span of tensor products of $Q$ homogeneous harmonic polynomials, respectively of degree $\ell_q$ on $\reals^{d_q}$ in variable $\obx_q$. Denote by $V^{\bd}_{\bell}$ the linear space of functions obtained by restricting the polynomials in $\tilde{V}^{\bd}_{\bell}$ to $\PS^{\bd}$. With these definitions, we have the following orthogonal decomposition
\begin{align}
L^2(\PS^{\bd}  ,\mu_{\bd} ) = \bigoplus_{\bell \in \Z_{\geq 0 }^Q } V^{\bd}_{\bell}\, . \label{eq:ProductSpinDecomposition}
\end{align}
The dimension of each subspace is given by
\[
B(\bd , \bell ) \equiv {\rm dim} ( V^{\bd}_{\bell} ) = \prod_{q = 1}^Q B( d_q , \ell_q) ,
\]
where we recall
\[
B(d, \ell) = \frac{2 \ell + d - 2}{\ell} { \ell + d - 3 \choose \ell - 1} \, .
\]

We recall that for each $\ell \in \Z_{\geq 0}$, the spherical harmonics $\{ Y_{\ell j}^{(d)}   \}_{ j \in [ B(d, \ell)]}$ form an orthonormal basis of $V^{(d)}_{\ell}$ on $\S^{d-1} (\sqrt{d})$. Similarly, for each $\bell \in  \Z_{\geq 0 }^Q$, the tensor product of spherical harmonics $\{ Y_{\bell,\bs}^{\bd} \}_{\bs \in [B(\bd , \bell)]}$ form an orthonormal basis of $V^{\bd}_{\bell}$, where $\bs = (s_1 , \ldots , s_Q ) \in [B(\bd , \bell)]$ signify $s_q \in [B(d_q , \ell_q)]$ for $q = 1 , \ldots , Q$ and
\[
 Y_{\bell,\bs}^{\bd} \equiv Y^{(d_1)}_{\ell_1, s_1} \otimes Y^{(d_2)}_{\ell_2, s_2} \otimes \ldots \otimes Y^{(d_Q)}_{\ell_Q, s_Q} = \bigotimes_{q=1}^Q Y^{(d_q)}_{\ell_q, s_q}.
\]
We have the following orthonormalization property
\[
\< Y^{\bd}_{\bell , \bs} , Y^{\bd}_{\bell'  , \bs ' } \>_{L^2} = \prod_{q=1}^Q \Big\< Y_{\ell_q s_q}^{(d_q)} , Y_{\ell_q ' s_q ' }^{(d_q)} \Big\>_{L^2 \lp \S^{d_q - 1} ( \sqrt{d_q} ) \rp} =\prod_{q=1}^Q \delta_{\ell_q ,\ell_q'} \delta_{s_q , s_q '}  = \delta_{\bell , \bell'} \delta_{\bs , \bs '}.
\]
We denote by $\proj_{\bk}$  the orthogonal projections on $V^{\bd}_{\bk}$ in $L^2(\PS^{\bd}  ,\mu_{\bd})$. This can be written in terms of spherical harmonics as
\begin{align}
\proj_{\bk} f(\obx) \equiv  \sum_{\bs \in [B(\bd , \bk)]}  \< f, Y^{\bd}_{\bk,\bs} \>_{L^2}  Y^{\bd}_{\bk,\bs}  ( \obx). 
\end{align}
We will denote for any $\cQ \subset \Z_{\geq 0 }^Q$, $\proj_{\cQ}$ the orthogonal projection on $\bigoplus_{\bk \in \cQ} V^{\bd}_{\bk}$, given by
\[
\proj_{\cQ} = \sum_{\bk \in \cQ} \proj_{\bk}.
\]
Similarly, the projection on $\cQ^c$, the complementary of the set $\cQ$ in $\Z_{\geq 0 }^Q$, is given by
\[
\proj_{\cQ^c} = \sum_{\bk \not\in \cQ} \proj_{\bk}.
\]

\subsection{Tensor product of Gegenbauer polynomials}

We recall that $\Tilde \mu^1_{d-1}$ denotes the distribution of $\sqrt{d}\<\bx,\be_d \>$ when $\bx \sim \Unif (\S^{d-1} (\sqrt{d}))$. We consider similarly the projection of $\PS^\bd$ on one coordinate per sphere. We define
\begin{equation}\label{eq:def_psd_mu1_d}
\ps^\bd \equiv \prod_{q = 1}^Q  [ -d_q , d_q ], \qquad \Tilde \mu^1_{\bd} \equiv \Tilde \mu^1_{d_1 -1 } \otimes \ldots \otimes \Tilde \mu^1_{d_Q -1 }  = \bigotimes_{q = 1}^Q  \Tilde \mu^1_{d_q -1 },
\end{equation}
and consider $L^2(\ps^\bd , \Tilde \mu^1_{\bd} )$. 

Recall that the Gegenbauer polynomials $\lb Q^{(d)}_k \rb_{k \geq 0}$ form an orthogonal basis of $L^2 ( [-d , d] , \Tilde \mu^1_{d-1})$. 

Define for each $\bk \in \Z_{\geq 0 }^Q$, the tensor product of Gegenbauer polynomials
\begin{equation}\label{def:prod_gegenbauer}
Q^{\bd}_{\bk} \equiv Q^{(d_1)}_{k_1} \otimes \ldots \otimes Q^{(d_q)}_{k_Q} = \bigotimes_{q = 1}^Q  Q^{(d_q)}_{k_q}.
\end{equation} 
We will use the following properties of the tensor product of Gegenbauer polynomials:

\begin{lemma}[Properties of products of Gegenbauer]\label{lem:prod_gegenbauer_prop}
Consider the tensor product of Gegenbauer polynomials $\lb Q^{\bd}_{\bk} \rb_{\bk \in \posint^Q}$ defined in Eq.~\eqref{def:prod_gegenbauer}. Then

\begin{itemize}
\item[$(a)$] The set $\lb Q^{\bd}_{\bk} \rb_{\bk \in \posint^Q}$ forms an orthogonal basis on $L^2(\ps^\bd , \Tilde \mu^1_{\bd})$, satisfying the normalization condition: for any $\bk , \bk' \in \posint^Q$,
\begin{equation}\label{eq:normalization_prod_gegenbauer}
\bla Q^{\bd}_{\bk } , Q^{\bd}_{\bk'} \bra_{L^2 (\ps^\bd)} =  \frac{1}{B(\bd,\bk)}\, \delta_{\bk , \bk'} 
\end{equation}

\item[$(b)$]  For $\obx = ( \obx^{(1)} , \ldots , \obx^{(Q)})$ and $\oby  = ( \oby^{(1)}  , \ldots , \oby^{(Q)}  ) \in \PS^{\bd}$, and $\bk , \bk' \in \posint^Q$,
\begin{equation}
\begin{aligned}
& \Big\< Q^{\bd}_{\bk }  \lp \lb \<\obx^{(q)} ,  \cdot \> \rb_{q\in[Q]} \rp, Q^{\bd}_{\bk'} \lp \lb  \<\oby^{(q)} ,  \cdot \> \rb_{q\in[Q]} \rp \Big\>_{L^2 \lp \PS^\bd \rp} \\
= &   \frac{1}{B(\bd,\bk)}\, \delta_{\bk , \bk'}  \, Q^{\bd}_{\bk }  \lp \lb \<\obx^{(q)} ,  \oby^{(q)} \> \rb_{q\in[Q]} \rp \,.
\end{aligned}
\label{eq:ProductProdGegenbauer}
\end{equation}

\item[$(c)$] For $\obx = ( \obx^{(1)} , \ldots , \obx^{(Q)})$ and $\oby  = ( \oby^{(1)}  , \ldots , \oby^{(Q)}  ) \in \PS^{\bd}$, and $\bk  \in \posint^Q$,
\begin{align}
Q^{\bd}_{\bk }  \lp \lb \<\obx^{(q)} ,  \oby^{(q)} \> \rb_{q\in[Q]} \rp  =  \frac{1}{B(\bd ,\bk )} \sum_{\bs \in B(\bd , \bk) }  Y^{\bd}_{\bk,\bs} ( \obx ) Y^{\bd}_{\bk,\bs} ( \oby ) . \label{eq:ProdGegenbauerHarmonics}
\end{align}

\end{itemize}

\end{lemma}

Notice that Lemma \ref{lem:prod_gegenbauer_prop}.(c) implies that $Q_{\bk }^{\bd}$ is (up to a constant) a representation of the projector onto the subspace $V^{\bd}_{\bk}$
\[
[ \proj_{\bk} f  ] ( \obx ) = B( \bd , \bk ) \int_{\PS^\bd} Q^{\bd}_{\bk }  \lp \lb \<\obx^{(q)} ,  \oby^{(q)} \> \rb_{q\in[Q]} \rp f( \oby) \mu_{\bd} ( \de \oby ) \, .
\]

\begin{proof}[Proof of Lemma \ref{lem:prod_gegenbauer_prop}]

Part $(a)$ comes from the normalization property \eqref{eq:GegenbauerNormalization} of Gegenbauer polynomials,
\[
\begin{aligned}
\bla Q^{\bd}_{\bk } , Q^{\bd}_{\bk'} \bra_{L^2 (\ps^\bd)} = & \Big\< Q^{\bd}_{\bk }  \lp \lb \sqrt{d_q} \<\be_q ,  \cdot \> \rb_{q\in[Q]} \rp, Q^{\bd}_{\bk'} \lp \lb \sqrt{d_q} \<\be_q ,  \cdot \> \rb_{q\in[Q]} \rp \Big\>_{L^2 \lp \PS^\bd \rp} \\
= & \prod_{q = 1}^Q  \Big\< Q_{k_q}^{(d_q)} \lp \sqrt{d_q}\< \be_q, \cdot \> \rp, Q_{k_q'}^{(d_q)} \lp \sqrt{d_q}\< \be_q, \cdot \> \rp \Big\>_{L^2\lp \S^{ d_q - 1} \lp\sqrt{d_q}\rp \rp }\\
 =& \prod_{q = 1}^Q  \frac{1}{B(d_q , k_q)} \delta_{k_q , k_q'} \\
 = &  \frac{1}{B(\bd,\bk)}\, \delta_{\bk , \bk'}  \, ,
\end{aligned}
\]
where the $\lb \be_q \rb_{q \in [Q]}$ are unit vectors in $\R^{d_q}$ respectively. 

Part $(b)$ comes from Eq.~\eqref{eq:ProductGegenbauer},
\[
\begin{aligned}
& \Big\< Q^{\bd}_{\bk }  \lp \lb \<\obx^{(q)} ,  \cdot \> \rb_{q\in[Q]} \rp, Q^{\bd}_{\bk'} \lp \lb  \<\oby^{(q)} ,  \cdot \> \rb_{q\in[Q]} \rp \Big\>_{L^2 \lp \PS^\bd \rp} \\
= & \prod_{q = 1}^Q  \Big\< Q_{k_q}^{(d_q)} \lp \< \obx^{(q)}, \cdot \> \rp, Q_{k_q'}^{(d_q)} \lp \< \oby^{(q)} , \cdot \> \rp \Big\>_{L^2\lp \S^{ d_q - 1} \lp\sqrt{d_q}\rp \rp } \\
=&  \prod_{q = 1}^Q  \frac{1}{B(d_q , k_q)} \delta_{k_q , k_q'}  Q^{(d_q)}_{k_q} \lp \< \obx^{(q)} , \oby^{(q)} \> \rp \\
=&   \frac{1}{B(\bd,\bk)}\, \delta_{\bk , \bk'}  \, Q^{\bd}_{\bk }  \lp \lb \<\obx^{(q)} ,  \oby^{(q)} \> \rb_{q\in[Q]} \rp \,,
\end{aligned}
\]
while part $(c)$ is a direct consequence of Eq.~\eqref{eq:GegenbauerHarmonics}.

\end{proof}

\subsection{Notations}
Throughout the proofs, $O_d(\, \cdot \, )$  (resp. $o_d (\, \cdot \,)$) denotes the standard big-O (resp. little-o) notation, where the subscript $d$ emphasizes the asymptotic variable. We denote $O_{d,\P} (\, \cdot \,)$ (resp. $o_{d,\P} (\, \cdot \,)$) the big-O (resp. little-o) in probability notation: $h_1 (d) = O_{d,\P} ( h_2(d) )$ if for any $\eps > 0$, there exists $C_\eps > 0 $ and $d_\eps \in \Z_{>0}$, such that
\[
\begin{aligned}
\P ( |h_1 (d) / h_2 (d) | > C_{\eps}  ) \le \eps, \qquad \forall d \ge d_{\eps},
\end{aligned}
\]
and respectively: $h_1 (d) = o_{d,\P} ( h_2(d) )$, if $h_1 (d) / h_2 (d)$ converges to $0$ in probability.

We will occasionally hide logarithmic factors  using the  $\Tilde O_d (\, \cdot\, )$ notation (resp. $\Tilde o_d (\, \cdot \, )$): $h_1(d) = \tilde O_d(h_2(d))$ if there exists a constant $C$ 
such that $h_1(d) \le C(\log d)^C h_2(d)$. Similarly, we will denote $\Tilde O_{d,\P} (\, \cdot\, )$ (resp. $\Tilde o_{d,\P} (\, \cdot \, )$) when considering the big-O in probability notation up to a logarithmic factor.

Furthermore, $f = \omega_d (g)$ will denote $f(d)/g(d) \to \infty$.

\section{General framework and main theorems}
\label{sec:GeneralApp}

In this section, define a more general model than the model considered in the main text. In the general model, we will assume the covariate vectors will follow a product of uniform distributions on the sphere, and assume a target function in $L^2$ space. We establish more general versions of Theorems \ref{thm:bound_KRR}, \ref{thm:RF}, \ref{thm:NT} on the two-spheres cases in the main text as Theorems \ref{thm:KRR_lower_upper_bound_aniso}, \ref{thm:RF_lower_upper_bound_aniso}, \ref{thm:NT_lower_upper_bound_aniso}. We will prove Theorems \ref{thm:KRR_lower_upper_bound_aniso}, \ref{thm:RF_lower_upper_bound_aniso}, \ref{thm:NT_lower_upper_bound_aniso} in the following sections. At the end of this section, we will show that Theorems \ref{thm:KRR_lower_upper_bound_aniso}, \ref{thm:RF_lower_upper_bound_aniso}, \ref{thm:NT_lower_upper_bound_aniso} will imply Theorems \ref{thm:bound_KRR}, \ref{thm:RF}, \ref{thm:NT} in the main text.

\subsection{Setup on the product of spheres}

Assume that the data $\bx$ lies on the product of $Q$ spheres,
\[
\bx = \lp \bx^{(1)} , \ldots , \bx^{(Q)} \rp \in \prod_{q \in [Q]} \S^{d_q - 1} ( r_q ),
\] 
where $d_q = d^{\eta_q}$ and $r_q = d^{(\eta_q + \kappa_q)/2}$. Let $\bd = (d_1 , \ldots , d_q) = (d^{\eta_1} , \ldots, d^{\eta_q})$ and $\bkappa =( \kappa_1 , \ldots , \kappa_Q)$, where $\eta_q > 0 $ and $\kappa_q \geq 0$ for $q = 1 , \ldots , Q$. We will denote this space
\begin{equation}\label{def:PS_data}
\PS^\bd_\bkappa = \prod_{q \in [Q]} \S^{d_q - 1} ( r_q ).
\end{equation}
Furthermore, assume that the data is generated following the uniform distribution on $\PS^\bd_{\bkappa}$, i.e.
\begin{equation}\label{def:mu_dist_data}
\bx \, \, \simiid \, \, \Unif( \PS^\bd_{\bkappa} ) =\bigotimes_{q \in [Q]} \Unif \lp \S^{d_q-1} ( r_q) \rp \equiv \mu_{\bd}^\bkappa .
\end{equation}
We have $\bx \in \R^D$ and $\| \bx \|_2 = R$ where $D = d^{\eta_1} + \ldots +  d^{\eta_Q}$ and $R = (d^{\eta_1 + \kappa_1} + \ldots + d^{\eta_Q + \kappa_Q})^{1/2}$. 

We will make the following assumption that will simplify the proofs. Denote 
\begin{equation}\label{def:xi}
\xi \equiv \max_{q \in [Q]} \lb \eta_q + \kappa_q \rb,
\end{equation}
then $\xi$ is attained on only one of the sphere, whose coordinate will be denoted $q_{\xi}$, i.e. $\xi = \eta_{q_\xi} + \kappa_{q_\xi}$ and $\eta_q + \kappa_q < \xi$ for $q \neq q_\xi$.

Let $\sigma : \R \to \R$ be an activation function and $(\bw_i)_{i \in [N]} \sim_{iid} \Unif(\S^{D-1})$ the weights. We introduce the random feature function class
\[
\cF_{\RF}(\bW) = \Big\{ \hat f_{\RF}(\bx; \ba) = \sum_{i=1}^N a_i \sigma(\< \bw_i, \bx\> \sqrt{D} / R): ~~ a_i \in \R, \forall i \in [N] \Big\},
\]
and the neural tangent function class
\[
\cF_{\NT}(\bW) = \Big\{ \hat f_{\RF}(\bx; \ba) = \sum_{i=1}^N \<\ba_i , \bx\> \sigma' (\< \bw_i, \bx\> \sqrt{D} / R): ~~ \ba_i \in \R^D, \forall i \in [N] \Big\}. 
\]
We will denote $\btheta_i = \sqrt{D} \bw_i$. Notice that the normalization in the definition of the function class insures that the scalar product $\< \bx , \btheta_i \> / R$ is of order $1$. This corresponds to normalizing the data.

We consider the approximation of $f$ by functions in function classes $\cF_{\RF}(\bTheta)$ and $\cF_{\NT}(\bTheta)$.

\subsection{Reparametrization}

Recall $(\btheta_i)_{i \in [N]} \sim \Unif(\S^{D-1} (\sqrt{D}))$ independently. We decompose $\btheta_i = ( \btheta_i^{(1)}, \ldots, \btheta_i^{(Q)} )$ into $Q$ sections corresponding to the $d_q$ coordinates associated to the $q$-th sphere. Let us consider the following reparametrization of $(\btheta_i)_{i \in [N]} \simiid \Unif(\S^{D-1}(\sqrt D))$:
\[
\lbp \obtheta^{(1)}_{i}, \ldots , \obtheta^{(Q)}_{i}, \tau_{i}^{(1)} , \ldots , \tau_i^{(Q)} \rbp ,
\]
 where
\[
\obtheta_{i}^{(q)} \equiv \sqrt{d_q} \btheta_{i}^{(q)}/\| \btheta_{i}^{(q)} \|_2 , \qquad \tau_i^{(q)} \equiv \| \btheta_{i}^{(q)} \|_2/\sqrt{d_q}, \qquad \text{ for } q = 1 , \ldots , Q .
\]
Hence 
\[
\btheta_i = \lbp \tau_i^{(1)} \cdot \obtheta_{i}^{(1)}, \ldots, \tau_i^{(Q)} \cdot  \overline \btheta_{i}^{(Q)} \rbp.
\]
It is easy to check that the variables $(\obtheta^{(1)} , \ldots , \obtheta^{(Q)})$ are independent and independent of \linebreak[4] $(\tau_{i}^{(1)} , \ldots , \tau_i^{(Q)})$, and verify
\[
\obtheta^{(q)}_i \sim \Unif ( \S^{d_q - 1} ( \sqrt{d_q} )), \qquad \tau_i^{(q)} \sim d_q^{-1/2} \sqrt{{\rm Beta} \lp \frac{d_q}{2}, \frac{D - d_q}{2} \rp }, \qquad \text{ for } q = 1 , \ldots , Q .
\]
We will denote $\obtheta_i \equiv ( \obtheta^{(1)}_i , \ldots ,\obtheta^{(Q)}_i )$ and $\btau_i \equiv ( \tau^{(1)}_i , \ldots , \tau^{(Q)}_i)$. With these notations, we have
\[
\obtheta_i \in \prod_{q \in [Q]} \S^{d_q - 1} ( \sqrt{d_q} ) \equiv \PS^\bd,
\]
 where $\PS^\bd$ is the `normalized space of product of spheres', and
 \[
 \lp \obtheta_i \rp_{i \in [N]} \,\, \simiid \,\, \bigotimes_{q \in [Q]} \Unif(\S^{d_q-1}(\sqrt{d_q}))  \equiv \mu_{\bd}.
 \]

Similarly, we will denote the rescaled data $\obx \in \PS^\bd$,
 \[
\obx = \lp \obx^{(1)} , \ldots , \obx^{(Q)} \rp \sim \bigotimes_{q \in [Q] } \Unif(\S^{d_q-1}(\sqrt{d_q})) ,
\] 
obtained by taking $\obx^{(q)} = \sqrt{d_q} \bx^{(q)} / r_q = d^{- \kappa_q / 2} \bx^\pq$ for each $q \in [Q]$.

The proof will proceed as follows: first, noticing that $\tau^{(q)}$ concentrates around $1$ for every $q = 1 , \ldots , Q$, we will restrict ourselves without loss of generality to the following high probability event
\[
 \cP_{d,N,\eps} \equiv \Big\lbrace \bTheta \Big\vert \tau_i^{(q)} \in [1-\eps , 1+ \eps], \forall i \in [N]  , \forall q \in [Q] \Big\rbrace \subset  \S^{D-1} (\sqrt{D})^{N},
\]
where $\eps >0$ will be chosen sufficiently small. Then, we rewrite the activation function
\[
 \sigma ( \< \cdot , \cdot \> /R ) : \S^{D-1} ( \sqrt{D}) \times \PS^\bd_\bkappa \to \R,
 \]
 as a function, for a random $\btau$ (but close to $(1, \ldots ,1)$)
\[
\sigma_{\bd,\btau} : \PS^{\bd} \times \PS^\bd \to \R,
\]
given for $\btheta = (\obtheta, \btau)$ by
\[
  \sigma_{\bd, \btau} \lp \lb \< \obtheta^\pq , \obx^\pq \> / \sqrt{d_q} \rb_{q\in[Q]} \rp =  \sigma \lp \sum_{q\in[Q]}
  \frac{\tau^\pq r_q}{R} \cdot \frac{\< \obtheta^\pq , \obx^\pq \> }{ \sqrt{d_q}} \rp.
\]
We can therefore apply the algebra of tensor product of spherical harmonics and use the machinery developed in \cite{ghorbani2019linearized}.

\subsection{Notations}

Recall the definitions $\bd = ( d_1 , \ldots , d_q )$, $\bkappa =( \kappa_1 , \ldots , \kappa_Q)$, $d_q = d^{\eta_q}$, $r_q = d^{(\eta_q + \kappa_q)/2}$, $D = d^{\eta_1} + \ldots +  d^{\eta_Q}$ and $R = (d^{\eta_1 + \kappa_1} + \ldots + d^{\eta_Q + \kappa_Q})^{1/2}$. Let us denote $\xi = \max_{q \in [Q]} \lb \eta_q + \kappa_q \rb$ and $q_{\xi} =\arg\min_{q \in [Q]} \lb \eta_q + \kappa_q \rb$.

Recall that $(\btheta_i)_{i \in [N]} \sim \Unif(\S^{D-1} (\sqrt{D}))$ independently. Let $\bTheta = (\btheta_1, \ldots, \btheta_N)$. We denote $\E_\btheta$ to be the expectation operator with respect to $\btheta \sim \Unif(\S^{D-1}(\sqrt D))$ and $\E_{\bTheta}$ the expectation operator with respect to $\bTheta = (\btheta_1 , \ldots , \btheta_N ) \sim \Unif ( \S^{D-1} ( \sqrt{D} ) )^{\otimes N}$. 

We will denote $\E_{\obtheta}$ the expectation operator with respect to $\obtheta \equiv (\obtheta^{(1)} , \ldots, \obtheta^{(Q)}) \sim \mu_{\bd}$, $\E_{\obTheta}$ the expectation operator with respect to $\obTheta = ( \obtheta_1 , \ldots , \obtheta_N)$, and $\E_{\btau}$ the expectation operator with respect to $\btau$ (we recall $\btau \equiv ( \tau^{(1)} , \ldots , \tau^{(Q)})$) or $(\btau_1 , \ldots , \btau_N)$ (where the $\btau_i $ are independent) depending on the context. In particular, notice that $\E_{\btheta} = \E_{\btau} \E_{\obtheta}$ and $\E_{\bTheta} = \E_{\btau} \E_{\obTheta}$.

We will denote $\E_{\bTheta_{\eps}}$ the expectation operator with respect to $\bTheta = ( \btheta_1 , \ldots , \btheta_N)$ restricted to $ \cP_{d,N,\eps} $ and $\E_{\btau_{\eps}}$ the expectation operator with respect to $\btau$ restricted to $[1-\eps , 1+\eps]^Q$. Notice that $\E_{\bTheta_{\eps}} = \E_{\btau_\eps} \E_{\obTheta}$.

Let $\E_\bx$ to be the expectation operator with respect to $\bx \sim \mu^\bd_\bkappa$, and $\E_{\obx}$ the expectation operator with respect to $\obx \sim \mu_{\bd}$.

\subsection{Generalization error of kernel ridge regression}

We consider the Kernel Ridge Regression solution $\ha_i$, namely
\[
\hba = (\bH + \lambda \id_n)^{-1} \by,
\]
where the kernel matrix $\bH = (H_{ij})_{ij \in [n]}$ is assumed to be given by
\[
H_{ij} = \bar h_d(\<\bx_i, \bx_j\>/R^2) = \E_{\obtheta \sim \Unif ( \PS^\bd ) } [ \sigma ( \< \obtheta , \bx \> / R ) \sigma ( \< \obtheta , \by  \> /R ) ] ,
\]
and $\by = (y_1, \ldots, y_n)^\sT = \boldf + \beps$, with 
\[
\begin{aligned}
\boldf & = (f_d(\bx_1), \ldots, f_d(\bx_n))^\sT, \\
\beps & = (\eps_1, \ldots, \eps_n)^\sT. 
\end{aligned}
\]
The prediction function at location $\bx$ gives
\[
\hat f_\lambda(\bx) = \by^\sT (\bH + \lambda \id_n)^{-1} \bh(\bx),
\]
with 
\[
\bh(\bx) = [ \bar h_d(\< \bx, \bx_1\>/R^2), \ldots,\bar h_d(\< \bx, \bx_n\>/R^2)]^\sT. 
\]
The test error of empirical kernel ridge regression is defined as
\[
\begin{aligned}
R_\KR(f_d, \bX, \lambda) \equiv& \E_\bx\Big[ \Big(f_d(\bx) - \by^\sT (\bH + \lambda \id_n)^{-1} \bh(\bx) \Big)^2 \Big]. 
\end{aligned}
\]

We define the set $\overline{\cQ}_\KRR ( \gamma )  \subseteq \posint^Q$ as follows (recall that $\xi\equiv\max_{q\in[Q]}(\eta_q+\kappa_q)$):
\begin{align}
\overline{\cQ}_\KRR ( \gamma ) =& \blb \bk \in \posint^Q \Big\vert \sum_{q = 1}^Q (\xi - \kappa_q) k_q \leq \gamma \brb , \label{eq:def_overcQ_KRR}
\end{align}
and the function $m: \R_{\geq 0} \to \R_{\geq 0}$ which at $\gamma$ associates
\[
m ( \gamma) = \min_{\bk \not\in \overline{\cQ}_\KRR ( \gamma ) }  \sum_{q \in [Q]} ( \xi - \kappa_q ) k_q .
\]
Notice that by definition $m(\gamma) >\gamma$.

We consider sequences of problems indexed by the integer $d$, and we view the problem parameters
(in particular, the dimensions $d_q$, the radii $r_q$, the kernel $h_d$, and so on) as functions of $d$.
\begin{assumption}\label{ass:activation_lower_upper_KRR_aniso}
Let $\{ h_d \}_{d \geq 1}$ be a sequence of functions $h_d : [-1,1] \to \R$ such that $H_d ( \bx_1 , \bx_2 ) = h_d ( \< \bx_1 , \bx_2 \> / d )$ is a positive semidefinite kernel.
\begin{itemize}

\item[(a)] For $\gamma > 0$ (which is specified in the theorem), we denote $L = \max_{q \in [Q]} \lceil  \gamma / \eta_q \rceil$. We assume that $h_d$ is $L$-weakly differentiable. We assume that for $0 \leq k \leq  L$, the $k$-th weak derivative verifies almost surely $h_d^{(k)} (u) \leq C$ for some constants $C>0$ independent of $d$. Furthermore, we assume there exists $k > L $ such that $h_d^{(k)} (0) \geq c >0$ with $c$ independent of $d$.

\item[(b)] For $\gamma > 0$ (which is specified in the theorem), we define 
\[
\overline K = \max_{\bk \in \overline{\cQ}_\KRR ( \gamma )} | \bk |.
\]
We assume that $\sigma$ verifies for $k \leq \overline{K}$, $h_d^{(k)} (0) \geq c$, with $c>0$ independent of $d$. 

\end{itemize}
\end{assumption}

\begin{theorem}[Risk of the \KRR\, model]\label{thm:KRR_lower_upper_bound_aniso}
  Let $\{ f_d \in L^2(\PS^\bd_\bkappa , \mu_{\bd}^\bkappa )\}_{d \ge 1}$ be a sequence of functions. Assume $w_d ( d^\gamma \log d ) \leq n \leq O_d( d^{m(\gamma) - \delta})$ for some $\gamma >0$ and $\delta > 0$. Let $\{ h_d \}_{d \geq 1}$ be a sequence of functions that satisfies Assumption \ref{ass:activation_lower_upper_KRR_aniso} at level $\gamma$. Let $\bX = (\bx_i)_{i \in [n]}$ with $(\bx_i )_{ i \in [n]} \sim \Unif (\PS^\bd_{\bkappa})$ independently, and $y_i = f_d (\bx_i) + \eps_i$ and $\eps_i \sim_{iid} \normal ( 0 ,\tau^2 )$ for some $\tau^2 \geq 0$. Then for any $\eps>0$,
  and for any $\lambda = O_d(1)$, with high probability we have
\begin{align}
\Big \vert R_{\KR}(f_{d}, \bX , \lambda )  - \| \proj_{\cQ^c} f_d \|_{L^2}^2 \Big \vert &\le \eps ( \| f_d \|_{L^2}^2 + \tau^2 )\, .
\label{eq:KRR_bound}
\end{align}
\end{theorem}

See Section \ref{sec:proof_RFK_bound} for the proof of this Theorem.

\subsection{Approximation error of the random features model}

We consider the minimum population error for the random features model
\[
R_{\RF}(f_{d}, \bW)  = \inf_{ f \in \cF_{\RF} ( \bW )} \E \big[ ( f_* ( \bx ) - f ( \bx ) )^2 \big].
\]

Let us define the sets:
\begin{align}
\cQ_\RF ( \gamma ) =& \blb \bk \in \posint^Q \Big\vert \sum_{q = 1}^Q (\xi - \kappa_q) k_q < \gamma \brb , \label{eq:def_cQ_RF} \\
\overline{\cQ}_\RF ( \gamma ) =& \blb \bk \in \posint^Q \Big\vert \sum_{q = 1}^Q (\xi - \kappa_q) k_q \leq \gamma \brb . \label{eq:def_overcQ_RF}
\end{align}

\begin{assumption}\label{ass:activation_lower_upper_RF_aniso}
Let $\sigma$ be an activation function. 
\begin{itemize}
\item[(a)] There exists constants $c_0,c_1$, with $c_0 > 0$ and $c_1 < 1$ such that the activation function $\sigma$ verifies $\sigma (u)^2 \leq c_0 \exp ( c_1 u^2 / 2)$ almost surely for $u \in \R$.

\item[(b)] For $\gamma > 0$ (which is specified in the theorem), we denote $L = \max_{q \in [Q]} \lceil  \gamma / \eta_q \rceil$. We assume that $\sigma$ is $L$-weakly differentiable. Define 
\[
K = \min_{\bk \in \cQ_\RF ( \gamma )^c } | \bk |.
\]
We assume that for $K  \leq k \leq  L$, the $k$-th weak derivative verifies almost surely $\sigma^{(k)} (u)^2 \leq c_0 \exp ( c_1 u^2 / 2)$ for some constants $c_0>0$ and $c_1 < 1$.

Furthermore we will assume that $\sigma$ is not a degree-$\lfloor \gamma / \eta_{q_\xi} \rfloor$ polynomial where we recall that $q_\xi$ corresponds to the unique $\arg \min_{q \in [Q]} \lb \eta_q + \kappa_q \rb$.

\item[(c)] For $\gamma > 0$ (which is specified in the theorem), we define 
\[
\overline K = \max_{\bk \in \overline{\cQ}_\RF ( \gamma )} | \bk |.
\]
We assume that $\sigma$ verifies for $k \leq \overline{K}$, $\mu_k (\sigma) \neq 0$. Furthermore we assume that for $k \leq \overline{K}$, the $k$-th weak derivative verifies almost surely $\sigma^{(k)} (u)^2 \leq c_0 \exp ( c_1 u^2 / 2)$ for some constants $c_0>0$ and $c_1 < 1$.

\end{itemize}
\end{assumption}

Assumption \ref{ass:activation_lower_upper_RF_aniso}.(a) implies that $\sigma \in  L^2(\R , \gamma)$ where $\gamma (\de x) = e^{-x^2 / 2} \de x /\sqrt{2\pi} $ is the standard Gaussian measure. We recall the Hermite decomposition of $\sigma$, 
\begin{equation}
\sigma (x) = \sum_{k =0}^\infty \frac{\mu_k (\sigma) }{k!} \He_k (x) , \qquad \mu_k ( \sigma) \equiv \E_{G \sim \normal (0,1)} [ \sigma(G) \He_k (G)].
\end{equation}

\begin{theorem}[Risk of the \RF\, model]\label{thm:RF_lower_upper_bound_aniso}
Let $\{ f_d \in L^2(\PS^\bd_\bkappa , \mu_{\bd}^\bkappa )\}_{d \ge 1}$ be a sequence of functions. Let $\bW = (\bw_i)_{i \in [N]}$ with $(\bw_i)_{i \in [N]} \sim \Unif(\S^{D-1})$ independently. We have the following results. 
\begin{itemize}
\item[(a)] Assume $N \le o_d(d^{\gamma })$ for a fixed $\gamma > 0$. Let $\sigma$ satisfy Assumptions \ref{ass:activation_lower_upper_RF_aniso}.(a) and \ref{ass:activation_lower_upper_RF_aniso}.(b) at level $\gamma$. Then, for any $\eps > 0$, the following holds with high probability: 
\begin{align}
\Big \vert R_{\RF}(f_{d}, \bW) - R_{\RF}(\proj_{\cQ} f_d, \bW) - \| \proj_{\cQ^c} f_d \|_{L^2}^2 \Big \vert &\le \eps \| f_d \|_{L^2} \| \proj_{\cQ^c} f_d \|_{L^2}\, ,
\label{eq:RF_lower_bound}
\end{align}
where $\cQ \equiv \cQ_\RF (\gamma)$ is defined in Equation~\eqref{eq:def_cQ_RF}.
\item[(b)] Assume $N \ge w_d(d^{\gamma})$ for some positive constant $\gamma> 0$, and $ \sigma$ satisfy Assumptions \ref{ass:activation_lower_upper_RF_aniso}.(a) and \ref{ass:activation_lower_upper_RF_aniso}.(c) at level $\gamma$. Then for any $\eps > 0$, the following holds with high probability: 
\begin{align}\label{eqn:RF_upper_bound}
0 \le R_{\RF}(\proj_{\cQ} f_d, \bW) \le \eps \| \proj_{\cQ} f_d \|_{L^2}^2 \, ,
\end{align}
where $\cQ \equiv \overline{\cQ}_\RF (\gamma)$ is defined in Equation~\eqref{eq:def_overcQ_RF}.
\end{itemize}
\end{theorem}

See Section \ref{sec:proof_RFK_lower_aniso} for the proof of the lower bound \eqref{eq:RF_lower_bound}, and Section \ref{sec:proof_RFK_upper_aniso} for the proof of the upper bound \eqref{eqn:RF_upper_bound}. 

\begin{remark}
This theorems shows that for each $\gamma \not\in  (\xi - \kappa_1) \posint+ \ldots + (\xi - \kappa_Q) \posint$, we can decompose our functional space as
\[
L^2 ( \PS^\bd_{\bkappa} , \mu_\bd^\bkappa ) = \cF (\bbeta, \bkappa , \gamma) \oplus \cF^c (\bbeta , \bkappa , \gamma),
\] 
where
\[
\begin{aligned}
\cF (\bbeta, \bkappa , \gamma) = &\bigoplus_{\bk\in \cQ_{\RF} (\gamma) } \bV^\bd_\bk, \\
 \cF^c  (\bbeta, \bkappa , \gamma) =& \bigoplus_{\bk \not\in \cQ_{\RF} (\gamma) } \bV^\bd_\bk,
\end{aligned}
\]
such that for $N = d^{\gamma}$, $\RF$ model fits the subspace of low degree polynomials $\cF (\bbeta, \bkappa , \gamma)$ and cannot fit $ \cF^c  (\bbeta, \bkappa , \gamma)$, i.e.
\[
R_{\RF} ( f_d , \bW)\approx \| \proj_{Q_{\RF} (\gamma)^c} f_d \|_{L^2}^2.
\]
\end{remark}

\begin{remark}
In other words, we can fit a polynomial of degree $\bk \in \posint^Q$, if and only if
\[
d^{(\xi - \kappa_1) k_1 } \cdot \ldots \cdot d^{(\xi - \kappa_Q) k_Q} = d_{1,\eff}^{k_1} \ldots d_{Q,\eff}^{k_Q} = o_d(N).
\]
Each subspace has therefore an effective dimension $d_{q,\eff} \equiv d^{\xi - \kappa_q} = d_q^{(\xi - \kappa_q) / \eta_q} \asymp D^{(\xi - \kappa_q) / \max_{q \in [Q]} \eta_q}$. This can be understood intuitively as follows,
\[
\sigma \lp \< \btheta , \bx \> /R \rp = \sigma \lp \sum_{q \in [Q]} \< \btheta^\pq , \bx^\pq \> / R \rp.
\]
The term $q_\xi$ (recall that $q_{\xi}=\arg\max_q(\eta_q+\kappa_q)$ and $\xi = \eta_{q_\xi} + \kappa_{q_\xi}$) verifies $\< \btheta^{(q_\xi)} , \bx^{(q_\xi)} \> / R  = \Theta_d (1)$ and has the same effective dimension $d_{q_\xi,\eff} = d^{\eta_{q_\xi}}$ has in the uniform case restricted to the sphere $\S^{d^{\eta_q}-1} ( \sqrt{d^{\eta_q}})$ (the scaling of the sphere do not matter because of the global normalization factor $R^{-1}$). However, for $ \eta_q + \kappa_q < \xi$, we have $\< \btheta^\pq , \bx^\pq \> / R  = \Theta_d (d^{(\eta_q + \kappa_q - \xi)/2})$ and we will need $d^{\xi - \kappa_q - \eta_q}$ more neurons to capture the dependency on the $q$-th sphere coordinates. The effective dimension is therefore given by $d_{q,\eff} = d_q \cdot d^{\xi - \kappa_q - \eta_q} = d^{\xi - \kappa_q}$.
\end{remark}

\subsection{Approximation error of the neural tangent model}

We consider the minimum population error for the random features model
\[
R_{\NT}(f_{d}, \bW)  = \inf_{ f \in \cF_{\NT} ( \bW )} \E \big[ ( f_* ( \bx ) - f ( \bx ) )^2 \big].
\]
For $\bk \in \posint^Q$,  we denote by $S(\bk)\subseteq [Q]$ the subset of indices $q\in [Q]$ such that $k_q > 0$.

 We define the sets
\begin{align}
\cQ_\NT ( \gamma ) =& \blb \bk \in \posint^Q \Big\vert \sum_{q = 1}^Q (\xi - \kappa_q) k_q < \gamma + \lbp \xi - \min_{q \in S(\bk)} \kappa_q \rbp \brb , \label{eq:def_cQ_NT} \\
\overline{\cQ}_\NT ( \gamma ) =& \blb \bk \in \posint^Q \Big\vert \sum_{q = 1}^Q (\xi - \kappa_q) k_q \leq \gamma + \lbp \xi - \min_{q \in S(\bk)} \kappa_q \rbp \brb . \label{eq:def_overcQ_NT}
\end{align}

\begin{assumption}\label{ass:activation_lower_upper_NT_aniso}
Let $\sigma:\R \to \R$ be an activation function. 
\begin{itemize}
\item[(a)] The activation function $\sigma$ is weakly differentiable with weak derivative $\sigma'$. There exists constants $c_0,c_1$, with $c_0 > 0$ and $c_1 < 1$ such that the activation function $\sigma$ verifies $\sigma' (u)^2 \leq c_0 \exp ( c_1 u^2 / 2)$  almost surely for $u \in \R$.

\item[(b)] For $\gamma > 0$ (which is specified in the theorem), we denote $L = \max_{q \in [Q]} \lceil  \gamma / \eta_q \rceil$. We assume that $\sigma'$ is $L$-weakly differentiable. Define 
\[
K = \min_{\bk \in \cQ_\NT ( \gamma )^c} | \bk |.
\]
We assume that for $K - 1 \leq k \leq  L$, the $k$-th weak derivative verifies almost surely $\sigma^{(k+1)} (u)^2 \leq c_0 \exp ( c_1 u^2 / 2)$ for some constants $c_0>0$ and $c_1 < 1$.

Furthermore, we assume that $\sigma'$ verifies a non-degeneracy condition. Recall that $\mu_k(h) \equiv \E_{G\sim\normal(0,1)}[h(G)\He_k(G)]$ denote the $k$-th coefficient of the Hermite expansion of 
$h\in L_2(\reals,\gamma)$ (with $\gamma$ the standard Gaussian measure). Then there exists $k_1,k_2\ge 2L+7 [ \max_{q \in [Q]} \xi / \eta_q ]$ such that $\mu_{k_1}(\sigma') ,\mu_{k_2}(\sigma')\neq 0$ and
\begin{align}
\frac{\mu_{k_1}(x^2\sigma')}{\mu_{k_1}(\sigma')}\neq \frac{\mu_{k_2}(x^2\sigma')}{\mu_{k_2}(\sigma')} \, .
\end{align}

\item[(c)] For $\gamma > 0$ (which is specified in the theorem), we define
\[
\overline K = \max_{\bk \in \overline{\cQ}_\NT ( \gamma )} | \bk |.
\]
We assume that $\sigma$ verifies for $k \leq \overline{K} +1$, $\mu_k (\sigma ' ) = \mu_{k+1} (\sigma) \neq 0$. Furthermore we assume that for $k \leq \overline{K}+1$ , the $k$-th weak derivative verifies almost surely $\sigma^{(k+1)} (u)^2 \leq c_0 \exp ( c_1 u^2 / 2)$ for some constants $c_0>0$ and $c_1 < 1$.

\end{itemize}
\end{assumption}

Assumption \ref{ass:activation_lower_upper_NT_aniso}.(a) implies that $\sigma' \in  L^2(\R , \gamma)$ where $\gamma (\de x) = e^{-x^2 / 2} \de x /\sqrt{2\pi} $ is the standard Gaussian measure. We recall the Hermite decomposition of $\sigma'$:
\begin{equation}
\sigma' (x) = \sum_{k =0}^\infty \frac{\mu_k (\sigma') }{k!} \He_k (x) , \qquad \mu_k ( \sigma') \equiv \E_{G \sim \normal (0,1)} [ \sigma'(G) \He_k (G)].
\end{equation}

In the Assumption \ref{ass:activation_lower_upper_NT_aniso}.(b), it is useful to notice that the Hermite coefficients of $x^2\sigma'(x)$
can be computed from the ones of $\sigma'(x)$ using the relation $\mu_k (x^2 \sigma ' )  = \mu_{k+2} (\sigma') + [1+2k] \mu_k(\sigma') + k(k-1) \mu_{k-2} (\sigma') $.

\begin{theorem}[Risk of the \NT\, model]\label{thm:NT_lower_upper_bound_aniso}
Let $\{ f_d \in L^2(\PS^\bd_\bkappa , \mu_{\bd}^\bkappa )\}_{d \ge 1}$ be a sequence of functions. Let $\bW = (\bw_i)_{i \in [N]}$ with $(\bw_i)_{i \in [N]} \sim \Unif(\S^{D-1})$ independently. We have the following results. 
\begin{itemize}
\item[(a)] Assume $N \le o_d(d^{\gamma })$ for a fixed $\gamma > 0$. Let $\sigma$ satisfy Assumptions \ref{ass:activation_lower_upper_NT_aniso}.(a) and \ref{ass:activation_lower_upper_NT_aniso}.(b) at level $\gamma$. Then, for any $\eps > 0$, the following holds with high probability: 
\begin{align}
\Big \vert R_{\NT}(f_{d}, \bW) - R_{\NT}(\proj_{\cQ} f_d, \bW) - \| \proj_{\cQ^c} f_d \|_{L^2}^2 \Big \vert &\le \eps \| f_d \|_{L^2} \| \proj_{\cQ^c} f_d \|_{L^2}\, ,
\label{eq:NT_lower_bound}
\end{align}
where $\cQ \equiv \cQ_\NT (\gamma)$ is defined in Equation~\eqref{eq:def_cQ_NT}.
\item[(b)] Assume $N \ge w_d(d^{\gamma})$ for some positive constant $\gamma> 0$, and $ \sigma$ satisfy Assumptions \ref{ass:activation_lower_upper_NT_aniso}.(a) and \ref{ass:activation_lower_upper_NT_aniso}.(c) at level $\gamma$. Then for any $\eps > 0$, the following holds with high probability: 
\begin{align}\label{eqn:NT_upper_bound}
0 \le R_{\NT}(\proj_{\cQ} f_d, \bW) \le \eps \| \proj_{\cQ} f_d \|_{L^2}^2 \, ,
\end{align}
where $\cQ \equiv \overline{\cQ}_\NT (\gamma)$ is defined in Equation~\eqref{eq:def_overcQ_NT}.
\end{itemize}
\end{theorem}

See Section \ref{sec:proof_NTK_lower_aniso} for the proof of lower bound, and Section \ref{sec:proof_NTK_upper_aniso} for the proof of upper bound. 

\begin{remark}
This theorems shows that each for each $\gamma>0$ such that $\cQ_{\NT} (\gamma)^c \cap \overline{\cQ}_{\NT} (\gamma) = \emptyset$, we can decompose our functional space as
\[
L^2 ( \PS^\bd_{\bkappa} , \mu_\bd^\bkappa ) = \cF (\bbeta, \bkappa , \gamma) \oplus \cF^c (\bbeta , \bkappa , \gamma),
\] 
where
\[
\begin{aligned}
\cF (\bbeta, \bkappa , \gamma) = &\bigoplus_{\bk\in \cQ_{\NT} (\gamma) } \bV^\bd_\bk, \\
 \cF^c  (\bbeta, \bkappa , \gamma) =& \bigoplus_{\bk \not\in \cQ_{\NT} (\gamma) } \bV^\bd_\bk,
\end{aligned}
\]
such that for $N = d^{\gamma}$, $\NT$ model fits the subspace of low degree polynomials $\cF (\bbeta, \bkappa , \gamma)$ and cannot fit $ \cF^c  (\bbeta, \bkappa , \gamma)$ at all, i.e.
\[
R_{\NT} ( f_d , \bW)\approx \| \proj_{\cQ_{\NT} (\gamma)^c} f_d \|_{L^2}^2.
\]
\end{remark}

\begin{remark}
In other words, we can fit a polynomial of degree $\bk \in \posint^Q$, if and only if
\[
d^{(\xi - \kappa_1) k_1 } \cdot \ldots \cdot d^{(\xi - \kappa_Q) k_Q} = d_{1,\eff}^{k_1} \ldots d_{Q,\eff}^{k_Q} = o_d(d^\beta N),
\]
where $\beta = \xi - \min_{q \in S(\bk)} \kappa_q $.
\end{remark}

\subsection{Connecting to the theorems in the main text}

 Let us connect the above general results to the two-spheres setting described in the main text. We consider two spheres with $\eta_1 = \eta$, $\kappa_1 = \kappa$ for the first sphere, and $\eta_2 = 1$, $\kappa_2 = 0$ for the second sphere. We have $\xi = \max ( \eta + \kappa , 1 )$.
 
 Let $w_d ( d^\gamma \log d ) \leq n \leq O_d ( d^{\gamma + \delta})$ with $\delta>0$ constant sufficiently  small, then by Theorem \ref{thm:KRR_lower_upper_bound_aniso} the function subspace learned by KRR is given by the polynomials of degree $k_1$ in the first sphere coordinates and $k_2$ in the second sphere with 
 \[
 \max ( \eta , 1 - \kappa) k_1 + \max ( \eta + \kappa , 1 ) k_2  < \gamma.
 \]
 We consider functions that only depend on the first sphere, i.e., $k_2 = 0$ and denote $d_\eff = d^{ \max ( \eta , 1 - \kappa)}$. Then the subspace of approximation is given by the $k$ polynomials in the first sphere such that $d_\eff^k \leq d^\gamma$. Furthermore, one can check that the Assumptions listed in Theorem \ref{thm:bound_KRR} in the main text verifies Assumption \ref{ass:activation_lower_upper_KRR_aniso}.
 
 Similarly, for $w_d ( d^\gamma ) \leq N \leq O_d ( d^{\gamma + \delta})$ with $\delta>0$ constant sufficiently small, Theorem \ref{thm:RF_lower_upper_bound_aniso} implies that the $\RF$ models can only approximate $k$ polynomials in the first sphere such that $d_\eff^k \leq d^\gamma$. Furthermore, Assumptions listed in Theorem \ref{thm:RF} in the main text verifies Assumption \ref{ass:activation_lower_upper_RF_aniso}.
 
 In the case of $\NT$, we only consider $\bk = (k_1 , 0)$ and $S(\bk) = \{1 \}$. We get $\min_{q \in S(\bk)} \kappa_q = \kappa$. The subspace approximated is given by the $k$ polynomials in the first sphere such that $d_\eff^k \leq d^\gamma d_\eff$. Furthermore, Assumptions listed in Theorem \ref{thm:NT} in the main text verifies Assumption \ref{ass:activation_lower_upper_NT_aniso}.

\section{Proof of Theorem \ref{thm:KRR_lower_upper_bound_aniso}}
\label{sec:proof_RFK_bound}

The proof follows closely the proof of \cite[Theorem 4]{ghorbani2019linearized}.

\subsection{Preliminaries}

Let us rewrite the kernel functions $ \{ h_d \}_{d \geq 1}$ as functions on the product of normalized spheres: for $\bx = \lb \bx^\pq \rb_{q\in[Q]}$ and $\by  =  \lb \by^\pq \rb_{q\in[Q]} \in \PS^\bd_\bkappa$:
\begin{equation}\label{eq:def_sigma_d_tau}
\begin{aligned}
h_d( \< \by , \bx \> / R^2 ) = & h_d \lp \sum_{q \in [Q]}   (r_q^2/R^2 \sqrt{d_q} ) \cdot  \< \oby^\pq , \obx^\pq \> / \sqrt{d_q} \rp \\
\equiv & h_{\bd} \lp \lb \< \oby^\pq , \obx^\pq \> / \sqrt{d_q} \rb_{q \in [Q]} \rp.
\end{aligned}
\end{equation}
Consider the expansion of $h_{\bd}$ in terms of tensor product of Gegenbauer polynomials. We have
\[
h_d (\<\by , \bx \>/ R^2 ) =  \sum_{\bk \in \posint^Q} \lambda^{\bd}_{\bk} (h_{\bd}  ) B(\bd,\bk) Q^{\bd}_{\bk} \lp \lb \< \oby^\pq , \obx^\pq \> \rb_{q \in [Q]} \rp, 
\]
where
\[
\lambda^{\bd}_{\bk} (h_{\bd} ) =  \E_{\obx} \Big[ h_{\bd} \lp \ox_1^{(1)} , \ldots , \ox^{(Q)}_1 \rp  Q^{\bd}_{\bk} \Big( \sqrt{d_1}\ox_1^{(1)} , \ldots , \sqrt{d_Q} \ox_1^{(Q)} \Big)  \Big],
\]
where the expectation is taken over $\obx = (\obx^{(1)} , \ldots , \obx^{(Q)} ) \sim \mu_{\bd}$.

\begin{lemma} \label{lem:bound_gegenbauer_KRR}
Let $\{ h_d \}_{d\geq 1}$ be a sequence of kernel functions that satisfies Assumption \ref{ass:activation_lower_upper_KRR_aniso}. Assume $w_d ( d^\gamma ) \leq n \leq o_d (d^{m(\gamma)} )$ for some $\gamma >0$. Consider $\cQ = \overline{\cQ}_{\KR} (\gamma)$ as defined in Eq.~\eqref{eq:def_overcQ_KRR}. Then there exists constants $c,C > 0$ such that for $d$ large enough,
\[
\begin{aligned}
 \max_{\bk \not\in \cQ}  \lambda^{\bd}_{\bk} ( h_{\bd}  )  \leq &  C d^{-m ( \gamma) }, \\
 \min_{\bk \in \cQ} \lambda^{\bd}_{\bk} ( h_{\bd }  ) \geq & c d^{- \gamma}.
 \end{aligned}
\]
\end{lemma}

\begin{proof}[Proof of Lemma \ref{lem:bound_gegenbauer_KRR}]
Notice that by Lemma \ref{lem:formula_gegenbauer_prod_gegenbauer},
\[
\lambda_{\bk}^\bd ( h_{\bd} ) = \lp \prod_{q \in [Q]} \alpha^{k_q}_q \rp \cdot   R(\bd,\bk) \cdot \E_{\obx } \left[ \lp \prod_{q \in [Q]} \lp 1 - \frac{(\ox_1^\pq)^2}{d_q} \rp^{k_q} \rp \cdot h_d^{(|\bk|)} \lp \sum_{q \in [Q]}  \alpha_q \ox_1^\pq \rp  \right],
\]
where $\alpha_q = d_q^{-1/2} r_q^2 / R^2 = (1 + o_d(1)) d^{\eta_q/2 + \kappa_q - \xi}$. By Assumption \ref{ass:activation_lower_upper_KRR_aniso}.$(a)$, we have
\[
\lambda_{\bk}^\bd ( h_{\bd} )  B ( \bd , \bk) \leq C  \prod_{q \in [Q]}d^{( \kappa_q - \xi) k_q}.
\]
Furthermore, by Assumption \ref{ass:activation_lower_upper_KRR_aniso}.$(b)$ and dominated convergence,
\[
\E_{\obx } \left[ \lp \prod_{q \in [Q]} \lp 1 - \frac{(\ox_1^\pq)^2}{d_q} \rp^{k_q} \rp \cdot h_d^{(|\bk|)} \lp \sum_{q \in [Q]}  \alpha_q \ox_1^\pq \rp  \right] \to h_d^{(|\bk|)} ( 0 ) \geq  c > 0,
\]
for $k \geq \overline{K}$. The lemma then follows from the same proof as in Lemma \ref{lem:bound_gegenbauer_RF_LB} and Lemma \ref{lem:bound_gegenbauer_RF_UB}, where we adapt the proofs of Lemma \ref{lem:convergence_proba_Gegenbauer_coeff} and \ref{lem:convergence_Gegenbauer_coeff_0_l} to $h_\bd$.
\end{proof}

\subsection{Proof of Theorem \ref{thm:KRR_lower_upper_bound_aniso}}

\noindent
\textbf{Step 1. Rewrite the $\by$, $\bE$, $\bH$, $\bM$ matrices.} 

The test error of empirical kernel ridge regression gives
\[
\begin{aligned}
R_{\KR}(f_d, \bX, \lambda) \equiv& \E_\bx\Big[ \Big( f_d(\bx) - \by^\sT (\bH + \lambda \id_n)^{-1} \bh(\bx)) \Big)^2 \Big]\\
=& \E_\bx [ f_d(\bx)^2] - 2 \by^\sT (\bH + \lambda \id_n)^{-1} \bE + \by^\sT (\bH + \lambda \id_n)^{-1} \bM (\bH + \lambda \id_n)^{-1} \by,
\end{aligned}
\]
where $\bE = (E_1, \ldots, E_n)^\sT$ and $\bM = (M_{ij})_{ij \in [n]}$, with
\[
\begin{aligned}
E_i =& \E_\bx[f_d(\bx) h_d(\< \bx, \bx_i\>/d)], \\
M_{ij} =& \E_{\bx}[h_d(\< \bx_i, \bx\>/d) h_d(\< \bx_j, \bx\>/d)]. 
\end{aligned}
\]

Let $B = \sum_{\bk \in \cQ} B(\bd, \bk)$. Define for any $\bk \in \posint^Q$,
\[
\begin{aligned}
\bD_{\bk} =&  \lambda^\bd_{\bk} (h_{\bd}) \id_{B(\bd,\bk)}, \\
\bY_{\bk} =& (Y^{\bd}_{\bk,\bs}(\obx_i))_{i \in [n], \bs \in [B(\bd, \bk)] } \in \R^{n \times B(\bd, \bk)},\\
\blambda_{\bk} =& ( \lambda^\bd_{\bk,\bs}(f_d))_{\bs \in [B(\bd,\bk)]}^\sT \in \R^{B(\bd, \bk)}, \\
\bD_{\cQ} =& \diag \lp \lp \lambda^{\bd}_{\bk} (h_{\bd} ) \id_{B(\bd, \bk)} \rp_{\bk \in \cQ} \rp  \in \R^{B \times B}\\
\bY_{\cQ} =& (\bY_{\bk} )_{\bk \in \cQ} \in \R^{n \times B}, \\
\blambda_{\cQ} =& \lp \lp \blambda_{\bk}^\sT\rp_{\bk \in \cQ} \rp^\sT \in \R^{B}. 
\end{aligned}
\]
Let the spherical harmonics decomposition of $f_{d}$ be 
\[
f_{d}(\bx) = \sum_{\bk \in \posint^Q} \sum_{\bs \in [B(\bd , \bk)] } \lambda^{\bd}_{\bk,\bs}(f_d) Y^{\bd}_{\bk,\bs} (\obx) , 
\]
and the Gegenbauer decomposition of $h_{\bd}$ be 
\[
h_{\bd} \lp \ox_1^{(1)} , \ldots , \ox_1^{(Q)} \rp = \sum_{ \bk \in \posint^Q } \lambda^{\bd}_{\bk} (h_{\bd} ) B(\bd, \bk) Q_{\bk}^{\bd} \lp \sqrt{d_1} \ox_1^{(1)} , \ldots , \sqrt{d_Q} \ox_1^{(Q)} \rp . 
\]

We write the decompositions of vectors $\boldf$, $\bE$, $\bH$, and $\bM$. We have 
\[
\begin{aligned}
\boldf =& \bY_{\cQ} \blambda_{\cQ} + \sum_{\bk \in \cQ^c} \bY_{\bk} \blambda_{\bk}, \\
\bE =&\bY_{\cQ} \bD_{\cQ} \blambda_{\cQ} + \sum_{\bk \in \cQ^c} \bY_{\bk} \bD_{\bk} \blambda_{\bk} ,\\
\bH =& \bY_{\cQ} \bD_{\cQ} \bY_{\cQ}^\sT + \sum_{\bk\in \cQ^c} \bY_{\bk} \bD_{\bk} \bY_{\bk}^\sT , \\
\bM =& \bY_{\cQ} \bD_{\cQ}^2 \bY_{\cQ}^\sT + \sum_{\bk  \in \cQ^c} \bY_{\bk} \bD_{\bk}^2 \bY_{\bk}^\sT . 
\end{aligned}
\]
From Lemma \ref{lem:control_cQc_terms}, we can rewrite
\[
\begin{aligned}
\bH =& \bY_{\cQ} \bD_{\cQ} \bY_{\cQ}^\sT + \kappa_h (\id_n + \bDelta_h), \\
\bM =& \bY_{\cQ} \bD_{\cQ}^2 \bY_{\cQ}^\sT + \kappa_u  \bDelta_u, 
\end{aligned}
\]
where $\kappa_h = \Theta_d (1)$, $\kappa_u = O_d ( d^{-m ( \gamma )})$, $\| \bDelta_h \|_{\op} = o_{d,\P} (1)$ and $\| \bDelta_u \|_{\op} = O_{d,\P} (1)$.

\noindent
\textbf{Step 2. Decompose the risk}

The rest of the proof follows closely from \cite[Theorem 4]{ghorbani2019linearized}. We decompose the risk as follows
\[
\begin{aligned}
R_{\KR}(f_d, \bX, \lambda) =& \| f_d \|_{L^2}^2  - 2 T_1 + T_2 +  T_3 - 2 T_4 + 2 T_5. 
\end{aligned}
\]
where
\[
\begin{aligned}
T_1 =& \boldf^\sT (\bH + \lambda \id_n)^{-1} \bE, \\
T_2 =&  \boldf^\sT (\bH + \lambda \id_n)^{-1} \bM (\bH + \lambda \id_n)^{-1} \boldf, \\
T_3 =& \beps^\sT (\bH + \lambda \id_n)^{-1} \bM (\bH + \lambda \id_n)^{-1} \beps,\\ 
T_4 =& \beps^\sT (\bH + \lambda \id_n)^{-1} \bE,\\
T_5 =& \beps^\sT (\bH + \lambda \id_n)^{-1} \bM (\bH + \lambda \id_n)^{-1} \boldf. 
\end{aligned}
\]
Further, we denote $\boldf_{\cQ}$, $\boldf_{\cQ^c}$, $\bE_{\cQ}$, and $\bE_{\cQ^c}$, 
\[
\begin{aligned}
\boldf_{\cQ} =&  \bY_{\cQ} \blambda_{\cQ},& \bE_{\cQ} =&  \bY_{\cQ} \bD_{\cQ} \blambda_{\cQ}, \\
\boldf_{\cQ^c} =&  \sum_{\bk \in \cQ^c} \bY_\bk \blambda_\bk,& \bE_{\cQ^c} =&  \sum_{\bk \in \cQ^c} \bY_\bk \bD_\bk \blambda_\bk. 
\end{aligned}
\]

\noindent
\textbf{Step 3. Term $T_2$}

Note we have 
\[
T_2 = T_{21}  + T_{22} + T_{23}, 
\]
where
\[
\begin{aligned}
T_{21} =&\boldf_{\cQ}^\sT (\bH + \lambda \id_n)^{-1} \bM (\bH + \lambda \id_n)^{-1} \boldf_{\cQ},\\
T_{22} =& 2 \boldf_{\cQ}^\sT (\bH + \lambda \id_n)^{-1} \bM (\bH + \lambda \id_n)^{-1} \boldf_{\cQ^c},\\
T_{23} =&\boldf_{\cQ^c}^\sT (\bH + \lambda \id_n)^{-1} \bM (\bH + \lambda \id_n)^{-1} \boldf_{\cQ^c}.\\
\end{aligned}
\]
By Lemma \ref{lem:key_H_U_H_bound}, we have 
\begin{equation}\label{eqn:H_U_H_bound}
\| n (\bH + \lambda \id_n)^{-1} \bM (\bH + \lambda \id_n)^{-1} - \bY_{\cQ} \bY_{\cQ}^\sT / n\|_{\op} = o_{d, \P}(1),
\end{equation}
hence
\[
\begin{aligned}
T_{21} = & \blambda_{\cQ} \bY_{\cQ}^\sT (\bH + \lambda \id_n)^{-1} \bM (\bH + \lambda \id_n)^{-1} \bY_{\cQ} \blambda_{\cQ} \\
= & \blambda_{\cQ}^\sT \bY_{\cQ}^\sT \bY_{\cQ} \bY_{\cQ}^\sT \bY_{\cQ} \blambda_{\cQ} / n^2 + [\| \bY_{\cQ} \blambda_{\cQ} \|_2^2  / n] \cdot o_{d, \P}(1). 
\end{aligned}
\]
By Lemma \ref{lem:concentration_YY}, we have (with $\| \bDelta \|_2 = o_{d, \P}(1)$)
\[
\begin{aligned}
\blambda_{\cQ}^\sT \bY_{\cQ}^\sT \bY_{\cQ} \bY_{\cQ}^\sT\bY_{\cQ} \blambda_{\cQ} / n^2 =& \blambda_{\cQ}^\sT(\id_B + \bDelta)^2 \blambda_{\cQ}  = \| \blambda_{\cQ} \|_2^2 (1 + o_{d, \P}(1)). 
\end{aligned}
\]
Moreover, we have 
\[
\| \bY_{\cQ} \blambda_{\cQ} \|_2^2 / n = \blambda_{\cQ}^\sT(\id_B + \bDelta) \blambda_{\cQ}  = \| \blambda_{\cQ} \|_2^2 (1 + o_{d, \P}(1)). 
\]
As a result, we have
\begin{align}\label{eqn:term_R21}
T_{21} =& \| \blambda_{\cQ} \|_{2}^2 (1 + o_{d, \P}(1)) = \| \proj_{\cQ} f_d \|_{L^2}^2 (1 + o_{d, \P}(1)). 
\end{align}

By Eq. (\ref{eqn:H_U_H_bound}) again, we have 
\[
\begin{aligned}
T_{23} =& \Big(\sum_{\bk \in \cQ^c} \blambda_{\bk}^\sT \bY_\bk^\sT \Big) (\bH + \lambda \id_n)^{-1} \bM (\bH + \lambda \id_n)^{-1}\Big(\sum_{\bk \in \cQ^c} \bY_\bk \blambda_\bk \Big)\\
=&  \Big(\sum_{\bk \in \cQ^c} \blambda_\bk^\sT \bY_\bk^\sT \Big) \bY_\cQ \bY_\cQ^\sT \Big(\sum_{\bk \in \cQ^c} \bY_\bk \blambda_\bk \Big) / n^2 +  \Big[\Big\|\sum_{\bk \in \cQ^c} \bY_\bk \blambda_\bk \Big\|_2^2 / n\Big] \cdot o_{d, \P}(1).
\end{aligned}
\]
By Lemma \ref{lem:YYYY_expectation}, we have 
\[
\begin{aligned}
\E\Big[\Big(\sum_{\bk \in \cQ^c} \blambda_\bk^\sT \bY_\bk^\sT \Big) \bY_{\cQ} \bY_{\cQ}^\sT \Big(\sum_{\bk \in \cQ^c} \bY_\bk \blambda_\bk \Big) \Big] /n^2 = &\sum_{\bu,\bv \in \cQ^c} \blambda_\bu^\sT\{ \E [(  \bY_\bu^\sT  \bY_{\cQ} \bY_{\cQ}^\sT  \bY_\bv ) ] /n^2 \} \blambda_\bv 
\\
=& \frac{B}{n} \sum_{\bk \in \cQ^c} \| \blambda_\bk \|_2^2. \\
\end{aligned}
\]
Moreover
\[
\E \Big[\Big\|\sum_{\bk \in \cQ^c} \bY_\bk \blambda_\bk \Big\|_2^2 / n\Big] = \sum_{\bk \in \cQ^c} \| \blambda_\bk \|_2^2 = \| \proj_{\cQ^c} f_d \|_{L^2}^2. 
\]
This gives
\begin{align}\label{eqn:term_R23}
T_{23} = o_{d, \P}(1) \cdot \| \proj_{\cQ^c} f_d \|_{L^2}^2. 
\end{align}
Using Cauchy Schwarz inequality for $T_{22}$, we get
\begin{align}\label{eqn:term_R22}
T_{22} \le 2 (T_{21} T_{23})^{1/2} = o_{d, \P}( 1) \cdot \| \proj_{\cQ} f_d \|_{L^2}  \| \proj_{\cQ^c} f_d \|_{L^2} .  
\end{align}
As a result, combining Eqs. (\ref{eqn:term_R21}), (\ref{eqn:term_R22}) and (\ref{eqn:term_R23}), we have 
\begin{equation}\label{eqn:KRR_term_T2}
T_2 =  \| \proj_{\cQ} f_d \|_{L^2}^2 + o_{d, \P}(1) \cdot \| f_d \|_{L^2}^2. 
\end{equation}

\noindent
\textbf{Step 4. Term $T_{1}$. }
Note we have 
\[
T_1 = T_{11} + T_{12} + T_{13}, 
\]
where
\[
\begin{aligned}
T_{11} =& \boldf_{\cQ}^\sT (\bH + \lambda \id_n)^{-1} \bE_{\cQ}, \\
T_{12} =& \boldf_{\cQ^c}^\sT (\bH + \lambda \id_n)^{-1} \bE_{\cQ}, \\
T_{13} =& \boldf^\sT (\bH + \lambda \id_n)^{-1} \bE_{\cQ^c}. \\
\end{aligned}
\]
By Lemma \ref{lem:lem_for_error_bound_R11}, we have 
\[
\| \bY_{\cQ}^\sT( \bH + \lambda \id_n)^{-1}\bY_{\cQ} \bD_{\cQ} - \id_B \|_{\op} = o_{d, \P}(1). 
\]
so that 
\begin{equation}\label{eqn:term_R11}
T_{11} = \blambda_{\cQ}^\sT  \bY_{\cQ}^\sT( \bH + \lambda \id_n)^{-1}\bY_{\cQ} \bD_{\cQ} \blambda_{\cQ} = \| \blambda_{\cQ} \|_2^2 (1 + o_{d, \P}(1)) = \| \proj_{\cQ} f_d \|_2^2 (1 + o_{d, \P}(1)). 
\end{equation}

Using Cauchy Schwarz inequality for $T_{12}$, and by the expression of $\bM = \bY_{\cQ} \bD_{\cQ}^2 \bY_{\cQ}^\sT + \kappa_u \bDelta_u $ with $\| \bDelta_u \|_{\op} = O_{d, \P}(1)$ and $\kappa_u = O_{d} ( d^{- m(\lambda)} )$, we get with high probability
\begin{equation}\label{eqn:term_R12}
\begin{aligned}
\vert T_{12} \vert =& \Big\vert \sum_{\bk \in \cQ^c} \blambda_\bk^\sT \bY_\bk^\sT (\bH + \lambda \id_n)^{-1} \bY_{\cQ} \bD_{\cQ}  \blambda_{\cQ} \Big\vert \\
\le& \Big\| \sum_{\bk \in \cQ^c} \blambda_\bk^\sT \bY_\bk^\sT (\bH + \lambda \id_n)^{-1} \bY_{\cQ} \bD_{\cQ} \Big\|_2 \| \blambda_{\cQ} \|_2 \\
=& \Big[ \Big( \sum_{\bk \in \cQ^c} \blambda_\bk^\sT \bY_\bk^\sT \Big) (\bH + \lambda \id_n)^{-1} \bY_{\cQ} \bD_{\cQ}^2 \bY_{\cQ}^\sT (\bH + \lambda \id_n)^{-1} \Big( \sum_{\bk \in \cQ^c} \blambda_\bk^\sT \bY_\bk^\sT \Big)\Big]^{1/2} \| \blambda_{\cQ} \|_2\\
\le & \Big[ \Big( \sum_{\bk \in \cQ^c } \blambda_\bk^\sT \bY_\bk^\sT \Big) (\bH + \lambda \id_n)^{-1} \bM (\bH + \lambda \id_n)^{-1} \Big( \sum_{\bk \in \cQ^c} \blambda_\bk^\sT \bY_\bk^\sT \Big) \Big]^{1/2} \| \blambda_{\cQ} \|_2\\
=&  T_{23}^{1/2} \| \blambda_{\cQ} \|_2 = o_{d, \P}(1) \cdot \| \proj_{\cQ} f_d \|_{L^2} \| \proj_{\cQ^c} f_d \|_{L^2}.
\end{aligned}
\end{equation}

For term $T_{13}$, we have 
\[
\begin{aligned}
\vert T_{13} \vert =&  \vert \boldf^\sT (\bH + \lambda \id_n)^{-1} \bE_{\cQ^c}\vert \le  \| \boldf \|_2 \| (\bH + \lambda \id_n)^{-1} \|_{\op} \| \bE_{\cQ^c} \|_2. 
\end{aligned}
\]
Note we have $\E[\| \boldf \|_2^2] = n \| f_d \|_{L^2}^2$, and $\| (\bH + \lambda \id_n)^{-1} \|_{\op} \le 2/(\kappa_h + \lambda)$ with high probability, and 
\[
\E [ \| \bE_{\cQ^c} \|_2^2 ] = n \sum_{\bk \in \cQ^c} \lambda^\bd_\bk(h_\bd)^2 \| \proj_{\bk} f_d \|_{L^2}^2 \le n \Big[ \max_{\bk \in \cQ^c} \lambda^\bd_\bk(h_\bd)^2\Big] \| \proj_{\cQ^c} f_d \|_{L^2}^2. 
\]
As a result, we have
\begin{equation}\label{eqn:term_R13}
\begin{aligned}
\vert T_{13} \vert \le& O_d(1) \cdot  \| \proj_{\cQ^c} f_d \|_{L^2} \| f_d \|_{L^2} \Big[n^2 \max_{\bk \in \cQ^c} \lambda^\bd_\bk(h_\bd)^2 \Big]^{1/2} / (\kappa_h + \lambda) \\
=& o_{d, \P}(1) \cdot \| \proj_{\cQ^c} f_d \|_{L^2} \| f_d \|_{L^2},
\end{aligned}
\end{equation}
where the last equality used the fact that $ n \le O_d(d^{m(\gamma) - \delta})$ and Lemma \ref{lem:bound_gegenbauer_KRR}. Combining Eqs. (\ref{eqn:term_R11}), (\ref{eqn:term_R12}) and (\ref{eqn:term_R13}), we get 
\begin{equation}\label{eqn:KRR_term_T1}
T_1 =  \| \proj_{\cQ} f_d \|_{L^2}^2 + o_{d, \P}(1) \cdot \| f_d \|_{L^2}^2. 
\end{equation}

\noindent
\textbf{Step 5. Terms $T_3, T_4$ and $T_5$. } By Lemma \ref{lem:key_H_U_H_bound} again, we have 
\[
\begin{aligned}
\E_\beps[T_3] / \tau^2 =& \tr((\bH + \lambda \id_n)^{-1} \bM (\bH + \lambda \id_n)^{-1}) = \tr(\bY_{\cQ} \bY_{\cQ}^\sT / n^2) + o_{d, \P}(1), 
\end{aligned}
\]
By Lemma \ref{lem:concentration_YY}, we have 
\[
\tr(\bY_{\cQ} \bY_{\cQ}^\sT / n^2) = \tr(\bY_{\cQ}^\sT \bY_{\cQ}) / n^2 = n B / n^2 + o_{d, \P}(1) = o_{d, \P}(1). 
\]
This gives
\begin{align}\label{eqn:term_varT3}
T_3 = o_{d, \P}(1) \cdot \tau^2. 
\end{align}

Let us consider $T_4$ term:
\[
\begin{aligned}
\E_{\beps} [T_4^2 ]/\tau^2 = & \E_{\beps} [ \beps^\sT (\bH + \lambda \id_n)^{-1} \bE \bE^\sT  (\bH + \lambda \id_n)^{-1} \beps ]/\tau^2 \\
= & \bE^\sT  (\bH + \lambda \id_n)^{-2 }  \bE.
\end{aligned}
\]
For any integer $L$, denote $\cL \equiv [0,L]^Q \cap \posint^Q$, and $\bY_{\cL} = ( \bY_{\bk} )_{\bk \in \cL}$ and $\bD_{\cL} = ( \bD_{\bk})_{\bk \in \cL}$. Then notice that by Lemma \ref{lem:concentration_YY}, Lemma \ref{lem:key_H_U_H_bound} and the definition of $\bM$, we get  
\[
\begin{aligned}
\| \bD_{\cL} \bY_{\cL }^{\sT} (\bH + \lambda \id_n)^{-2 } \bY_{\cL } \bD_{\cL} \|_{\op} = & \| (\bH + \lambda \id_n)^{-1  }\bY_{\cL } \bD_{\cL}^2  \bY_{\cL }^{\sT} (\bH + \lambda \id_n)^{-1  } \|_{\op} \\
\leq& \| (\bH + \lambda \id_n )^{-1} \bM ( \bH + \lambda \id_n )^{-1}  \|_{\op}.\\
\leq &  \| \bY_{\cQ} \bY_{\cQ}^{\sT} / n \|_{\op}/n + o_{\P,d} (1) \cdot / n \\
 = & o_{d,\P} (1)
\end{aligned}
\]
Therefore, 
\[ 
\begin{aligned}
\bE^\sT  (\bH + \lambda \id_n)^{-2 }  \bE = & \lim_{L \to \infty} \bE^\sT_{\cL}  (\bH + \lambda \id_n)^{-2 }  \bE_{\cL} \\
= & \lim_{L \to \infty} \blambda^\sT_{\cL}  [\bD_{\cL} \bY_{\cL }^{\sT} (\bH + \lambda \id_n)^{-2 } \bY_{\cL} \bD_{\cL}] \blambda_{\cL} \\
\leq & \| (\bH + \lambda \id_n )^{-1} \bM ( \bH + \lambda \id_n )^{-1}  \|_{\op} \cdot \lim_{L \to \infty} \| \blambda_{\cL} \|_2^2 \\
\leq & o_{d,\P}(1) \cdot \| f_d \|_{L^2}^2,
\end{aligned}
\]
which gives
\begin{align}\label{eqn:term_varT4}
T_4 = o_{d, \P}(1) \cdot \tau \| f_d \|_{L^2} = o_{d,\P} (1) \cdot ( \tau^2 +  \| f_d \|_{L^2}^2 ). 
\end{align}
We decompose $T_5$ using $\boldf = \boldf_{\cQ} + \boldf_{\cQ^c}$,
\[
T_5 = T_{51} + T_{52},
\]
where
\[
\begin{aligned}
T_{51} = & \beps^\sT (\bH + \lambda \id_n )^{-1} \bM ( \bH + \lambda \id_n )^{-1}  \boldf_{\cQ}, \\
T_{52} = & \beps^\sT (\bH + \lambda \id_n )^{-1} \bM ( \bH + \lambda \id_n )^{-1}  \boldf_{\cQ^c} .
\end{aligned}
\]
First notice that
\[
\begin{aligned}
\| \bM^{1/2} ( \bH + \lambda \id_n)^{-2} \bM^{1/2} \|_{\op} =& \| (\bH + \lambda \id_n)^{-1} \bM (\bH + \lambda \id_n)^{-1} \|_{\op} =  o_{d,\P} (1) .
\end{aligned}
\]
Then by Lemma \ref{lem:key_H_U_H_bound}, we get
\[
\begin{aligned}
\E_{\beps} [T_{51}^2 ]/\tau^2 = & \E_{\beps} [ \beps^\sT (\bH + \lambda \id_n)^{-1} \bM (\bH + \lambda \id_n)^{-1} \boldf_{\cQ} \boldf_{\cQ}^\sT (\bH + \lambda \id_n)^{-1} \bM (\bH + \lambda \id_n)^{-1} \beps ]/\tau^2 \\
= & \boldf^\sT_{\cQ} [ (\bH + \lambda \id_n)^{-1} \bM (\bH + \lambda \id_n)^{-1} ]^2 \boldf_{\cQ} \\
\leq &  \| \bM^{1/2} ( \bH + \lambda \id_n)^{-2} \bM^{1/2} \|_{\op} \| \bM^{1/2} ( \bH + \lambda \id_n)^{-1 } \boldf_{\cQ} \|_{2}^2 \\
= &  o_{d,\P}(1)  \cdot T_{21} \\
= &  o_{d,\P} (1)  \cdot \| \proj_{\cQ } f_d \|_{L^2}^2.
\end{aligned}
\]
Similarly, we get 
\[
\begin{aligned}
\E_{\beps} [T_{52}^2 ]/\tau^2   = &   o_{d,\P} (1) \cdot T_{23} 
= &  o_{d,\P}  (1) \cdot \| \proj_{\cQ^c } f_d \|_{L^2}^2.
\end{aligned}
\]
By Markov's inequality, we deduce that
\begin{align}\label{eqn:term_varT5}
T_5 = o_{d, \P}(1) \cdot \tau ( \| \proj_{\cQ} f_d \|_{L^2} + \| \proj_{\cQ^c} f_d \|_{L^2} )= o_{d,\P} (1) \cdot ( \tau^2 +  \| f_d \|_{L^2}^2 ). 
\end{align}

\noindent
\textbf{Step 6. Finish the proof. }

Combining Eqs. (\ref{eqn:KRR_term_T1}), (\ref{eqn:KRR_term_T2}), (\ref{eqn:term_varT3}), (\ref{eqn:term_varT4}) and (\ref{eqn:term_varT5}), we have
\[
\begin{aligned}
R_{\KR}( f_d, \bX, \lambda) = & \| f_d \|_{L^2}^2 - 2 T_{1} + T_{2} + T_3 - 2 T_4 + 2T_5 \\
=& \| \proj_{\cQ^c} f_d \|_{L^2}^2 + o_{d, \P}(1) \cdot (\| f_d \|_{L^2}^2 + \tau^2),
\end{aligned}
\]
which concludes the proof.

\subsection{Auxiliary results}

\begin{lemma}\label{lem:concentration_YY}
Let $\{ Y_{\bk,\bs}^\bd \}_{\bk \in \posint^Q, \bs \in [B(\bd, \bk)]}$ be the collection of tensor product of spherical harmonics on $\PS^\bd$. Let $(\obx_i)_{i \in [n]} \sim_{iid} \Unif(\PS^\bd)$. Denote
\[
\begin{aligned}
\bY_{\bk} = (Y_{\bk,\bs}^\bd(\obx_i))_{i \in [n], \bs \in [B(\bd, \bk)]} \in \R^{n \times B(\bd, \bk)}. 
\end{aligned}
\]
Assume that $n\geq w_d (d^\gamma \log d)$ and consider
\[
\cR =  \Big\{ \bk \in \posint^Q \Big\vert \sum_{q \in [Q]} \eta_q k_q < m( \gamma ) \Big\}.
\]
Denote $A = \sum_{\bk \in \cR} B(\bd, \bk)$ and
\[
\bY_{\cR} = ( \bY_\bk )_{\bk \in \cR}  \in \R^{n \times A}. 
\]
Then we have 
\[
\bY_{\cR}^\sT \bY_{\cR} / n  =  \id_A + \bDelta, 
\]
with $\bDelta \in \R^{A \times A}$ and $\E[\| \bDelta \|_{\op}] = o_{d}(1)$. 
\end{lemma}
\begin{proof}[Proof of Lemma \ref{lem:concentration_YY}. ] 

Let $\bPsi = \bY_{\cR}^\sT \bY_{\cR} /n \in \R^{A \times A}$. We can rewrite $\bPsi$ as
\[
\bPsi = \frac{1}{n} \sum_{i=1}^n \bsh_i \bsh_i^\sT, 
\]
where $\bsh_i = (Y_{\bk,\bs}^\bd ( \obx_i) )_{\bk \in \cR, \bs \in [B(\bd,\bk)]} \in \R^A$. We use matrix Bernstein inequality. Denote $\bX_i = \bsh_i \bsh_i - \id_A \in \R^{A \times A}$. Then we have $\E[\bX_i] = \bzero$, and 
\[
\begin{aligned}
\| \bX_i \|_{\op} \le \| \bsh_i \|_2^2 + 1 = & \sum_{\bk \in \cR} \sum_{\bs \in [B(\bd, \bk)]}Y_{\bk, \bs}^\bd (\obx_i)^2 + 1 \\
=& \sum_{\bk \in \cR} B(\bd, \bk) Q_{\bk}^\bd  \lp \lb \< \obx^\pq_i , \obx^\pq_i \> \rb_{q \in [Q]} \rp + 1 = A + 1,
\end{aligned}
\]
where we use formula \eqref{eq:GegenbauerHarmonics} and the normalization $Q_{\bk}^\bd (d_1 , \ldots , d_Q ) = 1$. Denote $V = \| \sum_{i = 1}^n \E[\bX_i^2] \|_{\op}$. Then we have 
\[
V  = n \| \E[(\bsh_i \bsh_i^\sT - \id_A)^2] \|_{\op} = n \| \E[ \bsh_i \bsh_i^\sT \bsh_i \bsh_i^\sT - 2 \bsh_i \bsh_i^\sT + \id_A ] \|_{\op} = n \| (A-1) \id_A\|_{\op} = n (A-1), 
\]
where we used $\bsh_i^\sT \bsh_i  = \| \bsh_i \|_2^2 = A$ and $\E [  \bsh_i (\obx_i)  \bsh_i^\sT ( \obx_i) ] = ( \E [ Y^\bd_{\bk,\bs} (\obx_i) Y^\bd_{\bk ' , \bs '} (\obx_i) ] )_{\bk \bs , \bk ' \bs '} = \id_A$. 
As a result, we have for any $t > 0$, 
\begin{equation}\label{eq:tail_bound}
\begin{aligned}
\P( \| \bPsi - \id_A \|_{\op} \ge t ) \le & A \exp\{ - n^2 t^2 / [2 n (A - 1)  + 2 (A + 1)n t/ 3] \} \\
 \le & \exp\{ - (n/A) t^2 / [10  (1 + t)] + \log A \}. 
\end{aligned}
\end{equation}
Notice that there exists $C>0$ such that $A \leq C\max_{\bk \in \cR} \prod_{q \in[Q]} d^{\eta_q k_q} \leq C d^{\gamma}$ (by definition of $m(\gamma)$ and $\cR$) and therefore $n \geq w_d(A \log A)$. Integrating the tail bound \eqref{eq:tail_bound} proves the lemma.
\end{proof}

\begin{lemma}\label{lem:control_cQc_terms}
Let $\sigma$ be an activation function satisfying Assumption \ref{ass:activation_lower_upper_KRR_aniso}. Let $w_d(d^\gamma \log d ) \leq n \leq O_d ( d^{m(\gamma) - \delta} )$ for some $\gamma > 0$ and $\delta> 0$. Then there exists sequences $\kappa_h$ and $\kappa_u$ such that
\begin{align}
\bH = & \sum_{\bk \in \posint^Q} \bY_\bk \bD_\bk \bY_{\bk}^\sT = \bY_{\cQ} \bD_{\cQ} \bY_{\cQ}^\sT + \kappa_h ( \id_n + \bDelta_h ), \label{eq:decomposition_H} \\
\bM = & \sum_{\bk \in \posint^Q} \bY_\bk \bD_\bk^2 \bY_{\bk}^\sT = \bY_{\cQ} \bD_{\cQ}^2 \bY_{\cQ}^\sT + \kappa_m  \bDelta_m , \label{eq:decomposition_M}
\end{align}
where $\kappa_h = \Theta_d (1)$, $\kappa_m = O_d ( d^{-m ( \gamma )})$, $\| \bDelta_h \|_{\op} = o_{d,\P} (1)$ and $\| \bDelta_m \|_{\op} = O_{d,\P} (1)$.
\end{lemma}

\begin{proof}[Proof of Lemma \ref{lem:control_cQc_terms}]
Define
\[
\begin{aligned}
\cR = & \Big\{ \bk \in \posint^Q \Big\vert \sum_{q \in [Q]} \eta_q k_q < m( \gamma ) \Big\} ,\\
\cS = & \Big\{ \bk \in \posint^Q \Big\vert \sum_{q \in [Q]} \eta_q k_q \geq m( \gamma ) \Big\} ,
\end{aligned}
\]
such that $\cR \cup \cS = \posint^Q$. The proof comes from bounding the eigenvalues of the matrix $\bY_{\bk} \bY_{\bk}^\sT$ for $\bk \in \cR$ and $\bk \in \cS$ separately.  From Corollary \ref{coro:Delta_bound_aniso_unif}, we have
\[
\sup_{\bk \in \cS} \| \bY_{\bk} \bY_{\bk}^\sT / B(\bd , \bk) - \id_n \|_{\op}  = o_{d,\P} (1).
\]
Hence, we can write
\begin{equation} \label{eq:first_bound_decompo_H}
\sum_{\bk \in \cS}  \bY_\bk \bD_\bk \bY_{\bk}^\sT = \kappa_h (\id_n + \bDelta_{h,1}),
\end{equation}
with $\kappa_h = \sum_{\bk \in \cS} \lambda^\bd_{\bk} (h_{\bd} ) B(\bd,\bk)  = O_d (1)$. From Assumption \ref{ass:activation_lower_upper_KRR_aniso}.$(b)$ and a proof similar to Lemma \ref{lem:convergence_Gegenbauer_coeff_0_l}, there exists $\bk = ( 0, \ldots , k , \ldots , 0)$ (for $k >L$ at position $q_\xi$) such that $\lim\inf_{d \to \infty}  \lambda^\bd_{\bk} (h_{\bd} ) B(\bd,\bk) >0$. Hence, $\kappa_h = \Theta_{d} (1)$.

From Lemma \ref{lem:concentration_YY} we have for $\bk \in \cR \cap \cQ^c$,
\[
\bY^\sT_{\bk} \bY_{\bk}  / n = \id_{B(\bd , \bk)} + \bDelta,
\]
with $\| \bDelta \|_{\op} = o_{d,\P}(1)$. We deduce that $\| \bY_{\bk} \bY^\sT_{\bk} \|_{\op} = O_{d,\P} (n)$. Hence, 
\[
\| \bY_{\bk} \bD_{\bk} \bY^\sT_{\bk} \|_{\op} = O_{d,\P} ( n \lambda_{\bk}^\bd (h_{\bd} ) ) = o_{d,\P} (1), \] where we used Lemma \ref{lem:bound_gegenbauer_KRR}. We deduce that
\begin{equation} \label{eq:second_bound_decompo_H}
\sum_{\bk \in \cR \cap \cQ^c}  \bY_\bk \bD_\bk \bY_{\bk}^\sT = \kappa_h  \bDelta_{h,2},
\end{equation}
with $\| \bDelta_{h,2} \|_{\op} = o_{d,\P}(1)$ where we used $\kappa_h^{-1} = O_{d} (1)$. Combining Eqs.~\eqref{eq:first_bound_decompo_H} and \eqref{eq:second_bound_decompo_H} yields Eq.~\eqref{eq:decomposition_H}.

Similarly, we get
\[
\sum_{\bk \in \cQ^c} \bY_{\bk} \bD_{\bk}^2 \bY_{\bk}^\sT = \sum_{\bk \in \cR \cap \cQ^c} [\lambda_{\bk}^\bd (h_{\bd} )^2 n  ] \bY_{\bk} \bY_{\bk}^\sT/  n +  \sum_{\bk \in \cS} [\lambda_{\bk}^\bd (h_{\bd} )^2 B (\bd,\bk) ]  \bY_{\bk} \bY_{\bk}^\sT/ B(\bd , \bk) .
\]
Using Lemma \ref{lem:bound_gegenbauer_KRR}, we have $\lambda_{\bk}^\bd (h_{\bd} )^2 n \leq C d^{-2m(\gamma)} n = O_{d,\P} (d^{-m(\gamma)} )$ and $\lambda_{\bk}^\bd (h_{\bd} )^2 B (\bd,\bk) \leq C \lambda_{\bk}^\bd (h_{\bd} ) \leq C' d^{-m(\gamma)} $. Hence Eq.~\eqref{eq:decomposition_M} is verified with
\[
\kappa_m = \sum_{\bk \in \cR \cap \cQ^c} \lambda_{\bk}^\bd (h_{\bd} )^2 n    +  \sum_{\bk \in \cS} \lambda_{\bk}^\bd (h_{\bd} )^2 B (\bd,\bk) .
\]
\end{proof}

\begin{lemma}\label{lem:YYYY_expectation}
Let $\{ Y^\bd_{\bk,\bs} \}_{\bk \in \posint^Q, \bs \in [B(\bd, \bk)]}$ be the collection of product of spherical harmonics on \linebreak[4] $L^2(\PS^\bd , \mu_{\bd} )$. Let $(\obx_i)_{i \in [n]} \sim_{iid} \Unif(\PS^\bd)$. Denote 
\[
\begin{aligned}
\bY_\bk = (Y^\bd_{\bk ,\bs}(\obx_i))_{i \in [n], \bs \in [B(\bd, \bk)]} \in \R^{n \times B(\bd, \bk)}. 
\end{aligned}
\]
Then for $\bu , \bv , \bt \in \posint^Q$ and $\bu \neq \bv$, we have 
\[
\E[ \bY_{\bu}^\sT  \bY_{\bt} \bY_{\bt}^\sT  \bY_{\bv} ] = \bzero. 
\]
For $\bu , \bt \in \posint^Q$, we have 
\[
\E[ \bY_{\bu}^\sT  \bY_{\bt} \bY_{\bt}^\sT  \bY_{\bu} ] = [ B(\bd, \bt) n + n(n-1) \delta_{\bu , \bt} ] \id_{B(\bd, \bu)}. 
\]
\end{lemma}

\begin{proof}
We have
\begin{equation}\label{eqn:YYYY_expectation}
\begin{aligned}
& \E[ \bY_{\bu}^\sT  \bY_{\bt} \bY_{\bt}^\sT  \bY_{\bv}] \\
=& \sum_{i,j\in[n]} \sum_{\bm \in [B(\bd, \bt )]} (\E[ Y^\bd_{\bu , \bp}(\obx_i) \Big(Y^\bd_{\bt ,  \bm}(\obx_i) Y^\bd_{\bt , \bm}(\obx_j)\Big) Y^\bd_{\bv , \bq}(\obx_j) ] )_{\bp \in [B(\bd, \bu)], \bq \in [B(\bd, \bv)]}\\
=& \sum_{i\in[n]} \Big(\E \Big[ Y^\bd_{\bu , \bp}(\obx_i)  \Big(\sum_{\bm \in [B(\bd, \bt)]} Y^\bd_{\bt , \bm}(\obx_i) Y^\bd_{\bt , \bm}(\obx_i)\Big) Y^\bd_{\bv , \bq}(\obx_i) \Big] \Big)_{\bp \in [B(\bd, \bu)], \bq \in [B(\bd, \bv)]}\\
& + \sum_{i \neq j\in[n]} \sum_{\bm \in [B(\bd, \bt)]} (\E[ Y^\bd_{\bu , \bp}(\obx_i) Y^\bd_{\bt , \bm}(\obx_i) Y^\bd_{\bt , \bm}(\obx_j) Y^\bd_{\bv , \bq}(\obx_j) ] )_{\bp \in [B(\bd, \bu)], \bq \in [B(\bd, \bv)]}\\
=& B(\bd, \bt)  \sum_{i\in[n]} (\E[ Y^\bd_{\bu , \bp}(\obx_i) Y^\bd_{\bv , \bq}(\obx_i)] )_{\bp \in [B(\bd, \bu)], \bq \in [B(\bd, \bv)]} \\
& + \sum_{i \neq j\in[n]}  \sum_{\bm \in [B(\bd, \bt)]} ( \delta_{\bu , \bt} \delta_{\bp , \bm} \delta_{\bt , \bv} \delta_{\bq , \bm})_{\bp \in [B(\bd, \bu)], \bq \in [B(\bd, \bv)]}\\
=& ( B(\bd, \bt) n \delta_{\bu , \bv} \delta_{\bp , \bq} + n(n-1) \delta_{\bu , \bt} \delta_{\bt , \bv} \delta_{\bp , \bq} )_{\bp \in [B(\bd, \bu)], \bq \in [B(\bd, \bv)]}. 
\end{aligned}
\end{equation}
This proves the lemma. 
\end{proof}

\begin{lemma}\label{lem:key_H_U_H_bound}
Let $\sigma$ be an activation function satisfying Assumption \ref{ass:activation_lower_upper_KRR_aniso}. Assume $\omega_d(d^{\gamma} \log d) \le n \le O_d(d^{m(\gamma) - \delta })$ for some $\gamma > 0$ and $\delta >0$. We have 
\[
\| n (\bH + \lambda \id_n)^{-1} \bM (\bH + \lambda \id_n)^{-1} - \bY_{\cQ} \bY_{\cQ}^\sT / n \|_{\op} = o_{d, \P}(1). 
\]
\end{lemma}

\begin{proof}[Proof of Lemma \ref{lem:key_H_U_H_bound}]
Denote 
\begin{align}
\bY_\bk = (Y^\bd_{\bk , \bs}(\obx_i))_{i \in [n], \bs \in [B(\bd, \bk)]} \in \R^{n \times B(\bd, \bk)}.
\end{align}
Denote $B = \sum_{\bk \in \cQ} B(\bd, \bk)$, and
\[
\bY_{\cQ} =( \bY_{\bk} )_{\bk \in \cQ} \in \R^{n \times B},
\]
and 
\[
\bD_{\cQ} = \diag( ( \lambda^\bd_{\bk} (h_{\bd} ) \id_{B(\bd, \bk)} )_{\bk \in \cQ} )  \in \R^{B \times B}. 
\]
From Lemma \ref{lem:control_cQc_terms}, we have 
\[
\begin{aligned}
& n (\bH + \lambda \id_n)^{-1} \bM (\bH + \lambda \id_n)^{-1} \\
=& n (\bY_{\cQ} \bD_{\cQ} \bY_{\cQ}^\sT + (\kappa_h + \lambda) \id_n + \kappa_h \bDelta_h)^{-1} (\bY_{\cQ} \bD_{\cQ}^2 \bY_{\cQ}^\sT + \kappa_m  \bDelta_m ) (\bY_{\cQ} \bD_{\cQ} \bY_{\cQ}^\sT + (\kappa_h + \lambda) \id_n + \kappa_h \bDelta_h)^{-1} \\
=& T_1 + T_2, 
\end{aligned}
\]
where $\| \bDelta_h \|_{\op} = o_{d, \P}(1)$, $\| \bDelta_u \|_{\op} = O_{d,\P} (1)$ and $\kappa_m = O_d (d^{- m(\gamma)} )$, and
\[
\begin{aligned}
T_1 =&  n   \kappa_m (\bY_{\cQ} \bD_{\cQ} \bY_{\cQ}^\sT + (\kappa_h + \lambda) \id_n + \kappa_h \bDelta_h)^{-1}   \bDelta_m (\bY_{\cQ} \bD_{\cQ} \bY_{\cQ}^\sT + (\kappa_h + \lambda) \id_n + \kappa_h \bDelta_h)^{-1}, \\
T_2 =& n (\bY_{\cQ} \bD_{\cQ} \bY_{\cQ}^\sT + (\kappa_h + \lambda) \id_n + \kappa_h \bDelta_h)^{-1} \bY_{\cQ} \bD_{\cQ}^2 \bY_{\cQ}^\sT  (\bY_{\cQ} \bD_{\cQ} \bY_{\cQ}^\sT + (\kappa_h + \lambda) \id_n + \kappa_h \bDelta_h)^{-1}.
\end{aligned}
\] 
Then, we can use the same proof as in \cite[Lemma 13]{ghorbani2019linearized} to bound $\| T_1 \|_{\op}$ (recall $n = O_d(d^{m(\gamma) - \delta})$)
\[
\begin{aligned}
\| T_1 \|_{\op}  \le 2 n \kappa_m / (\kappa_h + \lambda)^2 \| \bDelta_m \|_{\op} = o_{d,\P} (1),
\end{aligned}
\]
and $\| T_2 - \bY_{\cQ} \bY_{\cQ}^\sT / n \|_{\op} = o_{d,\P} (1)$, where we only need to check that
\[
\lambda_{\min} ( \bD_{\cQ} / [ (\kappa_h + \lambda) /n ] ) = \min_{\bk \in \cQ} [ n \lambda^\bd_{\bk} (h_{\bd} )] / (\kappa_h + \lambda) = w_d (1),
\]
which directly follows from Lemma \ref{lem:bound_gegenbauer_KRR}.
\end{proof}

\begin{lemma}\label{lem:lem_for_error_bound_R11}
Let $\sigma$ be an activation function satisfying Assumption \ref{ass:activation_lower_upper_KRR_aniso}. Assume $\omega_d(d^{\gamma} \log d) \le n \le O_d(d^{m(\gamma) - \delta })$ for some $\gamma > 0$ and $\delta>0$. We have 
\[
\| \bY_{\cQ}^\sT( \bH + \lambda \id_n)^{-1}\bY_{\cQ} \bD_{\cQ} - \id_B \|_{\op} = o_{d, \P}(1). 
\]
\end{lemma}

\begin{proof}[Proof of Lemma \ref{lem:lem_for_error_bound_R11}]~

This lemma can be deduced directly from \cite[Lemma 14]{ghorbani2019linearized}, by noticing that 
\[
\lambda_{\min} ( \bD_{\cQ} / [ (\kappa_h + \lambda) /n ] = \min_{\bk \in \cQ} [ n \lambda^\bd_{\bk} (h_{\bd} ) ] / (\kappa_h + \lambda) = \omega_d (1),
\]
from Lemma \ref{lem:bound_gegenbauer_KRR}.
\end{proof}

\section{Proof of Theorem \ref{thm:RF_lower_upper_bound_aniso}.(a): lower bound for the \RF\,model}
\label{sec:proof_RFK_lower_aniso}

\subsection{Preliminaries}

In the theorems, we show our results in high probability with respect to $\bTheta$. Hence, in the proof we will restrict the sample space to the high probability event $\cP_{\eps} \equiv \cP_{d,N,\eps}$ for $\eps > 0$ small enough, where
\begin{equation}\label{eq:def_P_eps}
 \cP_{d,N,\eps} \equiv \Big\lbrace \bTheta \Big\vert \tau_i^\pq \in [1-\eps , 1+ \eps], \forall i \in [N],\forall q \in [Q] \Big\rbrace \subset \Big( \S^{D-1} (\sqrt{D}) \Big)^{\otimes N}.
\end{equation}
We will denote $\E_{\btau_{\eps}}$ the expectation over $\btau$ restricted to $\tau^\pq \in [1-\eps, 1+ \eps]$ for all $q\in [Q]$, and $\E_{\bTheta_{\eps}}$ the expectation over $\bTheta$ restricted to the event $\cP_{\eps}$.

\begin{lemma}\label{lem:bound_proba_cPeps}
Assume $N = o( d^{\gamma})$ for some $\gamma > 0$. We have for any fixed $\eps > 0$, 
\[
\P ( \cP_{\eps}^c ) = o_d (1).
\]
\end{lemma}

\begin{proof}[Proof of Lemma \ref{lem:bound_proba_cPeps}]
The tail inequality in Lemma \ref{lem:tail_tau} and the assumption $N = o( d^{\gamma})$  imply that there exists some constants $C,c >0$ such that
\[
\P ( \cP_{d,N,\eps}^c ) \leq  \sum_{q \in [Q]} N \P ( | \tau^\pq -1| > \eps ) \leq  \sum_{q \in [Q]} C \exp (\gamma \log (d) - c d^{\eta_q} \eps ) = o_d (1).
\]
\end{proof}

We consider the activation function $\sigma : \R \to \R$. Let $\btheta \sim \S^{D-1} (\sqrt{D})$ and $\bx = \lb \bx^\pq \rb_{q\in[Q]} \in \PS^\bd_\bkappa$. We introduce the function $\sigma_{\bd,\btau} : \ps^\bd \to \R$ such that 
\begin{equation}\label{eq:def_sigma_d_tau}
\begin{aligned}
\sigma ( \< \btheta , \bx \> / R ) = & \sigma \lp \sum_{q \in [Q]}  \tau^\pq \cdot (r_q/R) \cdot  \< \obtheta^\pq , \obx^\pq \> / \sqrt{d_q} \rp \\
\equiv & \sigma_{\bd,\btau} \lp \lb \< \obtheta^\pq , \obx^\pq \> / \sqrt{d_q} \rb_{q \in [Q]} \rp.
\end{aligned}
\end{equation}
Consider the expansion of $\sigma_{\bd,\btau}$ in terms of tensor product of Gegenbauer polynomials. We have
\begin{equation}\label{eq:gegenbauer_decomposition_sigma_d_tau}
\sigma (\<\btheta , \bx \>/ R ) =  \sum_{\bk \in \posint^Q} \lambda^{\bd}_{\bk} (\sigma_{\bd,\btau}  ) B(\bd,\bk) Q^{\bd}_{\bk} \lp \lb \< \obtheta^\pq , \obx^\pq \> \rb_{q \in [Q]} \rp, 
\end{equation}
where
\[
\lambda^{\bd}_{\bk} (\sigma_{\bd,\btau} ) =  \E_{\obx} \Big[ \sigma_{\bd,\btau} \lp \ox_1^{(1)} , \ldots , \ox^{(Q)}_1 \rp  Q^{\bd}_{\bk} \Big( \sqrt{d_1}\ox_1^{(1)} , \ldots , \sqrt{d_Q} \ox_1^{(Q)} \Big)  \Big],
\]
where the expectation is taken over $\obx = (\obx^{(1)} , \ldots , \obx^{(Q)} ) \sim \mu_{\bd}$.

\begin{lemma} \label{lem:bound_gegenbauer_RF_LB}
Let $\sigma$ be an activation function that satisfies Assumptions \ref{ass:activation_lower_upper_RF_aniso}.(a) and \ref{ass:activation_lower_upper_RF_aniso}.(b). Consider $N \leq o_d( d^{\gamma })$ and $\cQ = \cQ_{\RF} (\gamma)$ as defined in Theorem \ref{thm:RF_lower_upper_bound_aniso}.(a). Then there exists $\eps_0 > 0 $ and $d_0$ and a constant $C > 0$ such that for $d \ge d_0$ and $\btau \in [1- \eps_0 , 1+ \eps_0]^Q$,
\[
 \max_{\bk \not\in \cQ}  \lambda^{\bd}_{\bk} ( \sigma_{\bd,\btau}  )^2  \leq  C d^{-\gamma}.
\]
\end{lemma}

\begin{proof}[Proof of Lemma \ref{lem:bound_gegenbauer_RF_LB}]
Notice that by Assumption \ref{ass:activation_lower_upper_RF_aniso}.$(b)$ we can apply Lemma \ref{lem:convergence_proba_Gegenbauer_coeff} to any $\bk \in \cQ^c$ such that $| \bk | = k_1 + \ldots + k_Q \leq L$. In particular, there exists $C >0$, $\eps_0' > 0$ and $d_0'$ such that for any $\bk \in \cQ^c$ with $| \bk| \leq L$, $d \ge d_0'$ and $ \btau \in [ 1 - \eps_0 ' , 1+ \eps_0' ]^Q$,
\[
\lp \prod_{q \in [Q] } d^{(\xi - \eta_q - \kappa_q)k_q} \rp B(\bd,\bk) \lambda^{\bd}_{\bk} ( \sigma_{\bd,\btau}  )^2 \leq C < \infty,
\]
Furthermore, using that $B(\bd,\bk) = \Theta ( d_1^{k_1} d_2^{k_2} \ldots d_Q^{k_Q})$, there exists $C ' >0$ such that for $\bk \in \cQ^c$ with $| \bk| \leq L$,
\begin{equation}\label{eq:prelem_bound_1_RF}
\lambda^{\bd}_{\bk} ( \sigma_{\bd,\btau}  )^2 \leq C' \prod_{q \in [Q] } d^{(\eta_q + \kappa_q - \xi)k_q} d_q^{- k_q} = C' \prod_{q \in [Q] } d^{(\kappa_q - \xi)k_q}  \le C' d^{-\gamma},
\end{equation}
where we used in the last inequality $\bk \not\in \cQ_{\RF} (\gamma)$ implies $ (\xi - \kappa_1)k_1 + \ldots + (\xi - \kappa_Q) k_Q \geq \gamma$ by definition.

Furthermore, from Assumption \ref{ass:activation_lower_upper_RF_aniso} and Lemma \ref{lem:coupling_convergence_Gaussian}.$(b)$, there exists $\eps_0'' > 0$, $d_0''$ and $C  < \infty$, such that 
\[
\sup_{d \ge d_0''} \, \sup_{\btau \in [1 - \eps_0'', 1+ \eps_0'']^Q } \E_{\overline \bx} \lsb \sigma_{\bd, \btau} \lp \lb \<\bw^\pq,\obx^\pq\> \rb_{q \in [Q]} \rp^2  \rsb < C.
\]
From the Gegenbauer decomposition \eqref{eq:gegenbauer_decomposition_sigma_d_tau}, this implies that for any $\bk \in \posint^Q$, $d \leq d_0''$ and $\btau \in [1- \eps_0'' , 1+ \eps_0'']^Q$,
\[
B( \bd , \bk ) \lambda^\bd_{\bk} (\sigma_{\bd,\btau})^2 \leq C.
\]
In particular, for $| \bk | = k_1 + \ldots + k_Q > L =  \max_{q \in [Q]} \lceil  \gamma / \eta_q \rceil$, we have
\begin{equation}\label{eq:prelem_bound_2_RF}
\lambda^\bd_{\bk} (\sigma_{\bd,\btau})^2 \leq \frac{C}{B(\bd , \bk)} \leq C' \prod_{q \in [Q]} d^{-\eta_q k_q} \leq C' \prod_{q \in [Q]} d^{-\gamma k_q / L} \leq C' d^{-\gamma}.
\end{equation}
Combining Eqs~\eqref{eq:prelem_bound_1_RF} and \eqref{eq:prelem_bound_2_RF} yields the result.
\end{proof}

\subsection{Proof of Theorem \ref{thm:RF_lower_upper_bound_aniso}.(a): Outline}

Let $\cQ \equiv \cQ_{\RF} (\gamma)$ as defined in Theorem \ref{thm:RF_lower_upper_bound_aniso}.(a) and $\bTheta = \sqrt{D} \bW$ such that $\btheta_i = \sqrt{D} \bw_i \sim_{iid} \Unif( \S^{D-1} (\sqrt{D}))$.

Define the random vectors $\bV = (V_1, \ldots, V_N)^\sT$, $\bV_{\cQ} = (V_{1, \cQ}, \ldots, V_{N, \cQ})^\sT$, $\bV_{\cQ^c} = (V_{1, \cQ^c}, \ldots, V_{N, \cQ^c})^\sT$, with
\begin{align}
V_{i, \cQ} \equiv& \E_{\bx}[[\proj_{\cQ} f_d](\bx) \sigma (\< \btheta_i, \bx\>/R )],\\
V_{i, \cQ^c} \equiv& \E_{\bx}[[\proj_{\cQ^c} f_d](\bx) \sigma (\< \btheta_i, \bx\> /R)],\\
V_i \equiv& \E_{\bx}[f_d(\bx) \sigma (\< \btheta_i, \bx\>/R)] = V_{i, \cQ} + V_{i, \cQ^c}. 
\end{align}
Define the random matrix $\bU = (U_{ij})_{i, j \in [N]}$, with 
\begin{align}
U_{ij} = \E_{\bx}[\sigma (\< \bx, \btheta_i\>/R) \sigma (\< \bx, \btheta_j\>/ R)]. \label{eq:KernelMatrix}
\end{align}
In what follows, we write $R_{\RF}(f_d) = R_{\RF}(f_d,\bW) =  R_{\RF}(f_d,\bTheta/\sqrt{D})$ for the random features risk, omitting the dependence on the weights $\bW = \bTheta/\sqrt{D}$.
By the definition and a simple calculation, we have 
\[
\begin{aligned}
R_{\RF}(f_d) =& \min_{\ba \in \R^N} \Big\{ \E_{\bx}[f_d(\bx)^2] - 2 \< \ba, \bV \> + \< \ba, \bU \ba\> \Big\} = \E_{\bx}[f_d(\bx)^2] - \bV^\sT \bU^{-1} \bV,\\
R_{\RF}(\proj_{\cQ} f_d) =& \min_{\ba \in \R^N} \Big\{ \E_{\bx}[\proj_{\cQ} f_d(\bx)^2] - 2 \< \ba, \bV_{\le \ell} \> + \< \ba, \bU \ba\> \Big\} = \E_{\bx}[\proj_{\cQ} f_d(\bx)^2] - \bV_{\cQ}^\sT \bU^{-1} \bV_{\cQ}. 
\end{aligned}
\]

By orthogonality, we have
\[
\E_{\bx}[f_d(\bx)^2] = \E_{\bx}[[\proj_{\cQ}f_d](\bx)^2] + \E_{\bx}[[\proj_{\cQ^c} f_d](\bx)^2], 
\]
which gives
\begin{equation}\label{eqn:decomposition_risk}
\begin{aligned}
& \Big\vert R_{\RF}(f_d) - R_{\RF}(\proj_{\cQ} f_d) - \E_{\bx}[[\proj_{\cQ^c} f_d](\bx)^2] \Big\vert \\
=& \Big\vert \bV_{\cQ}^\sT \bU^{-1} \bV_{\cQ} - \bV^\sT \bU^{-1} \bV \Big\vert = \Big\vert \bV_{\cQ}^\sT \bU^{-1} \bV_{\cQ} - (\bV_{\cQ} + \bV_{\cQ^c})^\sT \bU^{-1} (\bV_{\cQ} + \bV_{\cQ^c}) \Big\vert\\
=& \Big\vert 2 \bV^\sT \bU^{-1} \bV_{\cQ^c} - \bV_{\cQ^c}^\sT \bU^{-1} \bV_{\cQ^c} \Big\vert \le 2 \| \bU^{-1/2} \bV_{\cQ^c} \|_2 \| \bU^{-1/2} \bV \|_{2} +  \| \bU^{-1} \|_{\op} \| \bV_{\cQ^c}\|_2^2\\
\le&  2 \| \bU^{-1/2} \|_{\op} \| \bV_{\cQ^c} \|_2 \| f_d \|_{L^2}+  \| \bU^{-1} \|_{\op} \| \bV_{\cQ^c}\|_2^2,
\end{aligned}
\end{equation}
where the last inequality used the fact that
\[
0 \le R_{\RF}(f_d) = \| f_d \|_{L^2}^2 - \bV^\sT \bU^{-1} \bV, 
\]
so that
\[
\| \bU^{-1/2} \bV \|_2^2 = \bV^\sT \bU^{-1} \bV \le \| f_d \|_{L^2}^2. 
\]

The Theorem follows from the following two claims 
\begin{align}
\| \bV_{\cQ^c } \|_2 / \| \proj_{\cQ^c } f_d \|_{L^2} =& o_{d,\P}(1), \label{eqn:bound_V_RFK_NU}\\
\| \bU^{-1} \|_{\op} =& O_{d,\P} (1),\label{eqn:bound_inverse_U_RFK_NU}
\end{align}
This is achieved by the Proposition \ref{prop:bound_V_RFK_NU} and \ref{prop:kernel_lower_bound_RFK_NU} stated below.

\begin{proposition}[Expected norm of $\bV$]\label{prop:bound_V_RFK_NU}
Let $\sigma$ be an activation function satisfying Assumptions \ref{ass:activation_lower_upper_RF_aniso}.(a) and \ref{ass:activation_lower_upper_RF_aniso}.(b) for a fixed $\gamma>0$. Denote $\cQ = \cQ_{\RF} (\gamma)$. Let $\eps>0$ and define $\cE_{\cQ^c,\eps}$ by
\[
\cE_{\cQ^c,\eps} \equiv  \E_{\btheta_\eps}[\< \proj_{\cQ^c,0} f_d, \sigma (\<\btheta, \cdot\>/R) \>_{L^2}^2],
\]
where we recall that $\E_{\btheta_\eps} = \E_{\btau_{\eps}} \E_{\obtheta}$ the expectation with respect to $\btau$ restricted to $[1 -\eps , 1 + \eps ]^Q$ and $\obtheta \sim \Unif ( \PS^\bd)$.

Then there exists a constant $C >0$ and $\eps_0 >0$ (depending only on the constants of Assumptions \ref{ass:activation_lower_upper_RF_aniso}.(a) and \ref{ass:activation_lower_upper_RF_aniso}.(b)) such that for $d$ sufficiently large,
\[
\cE_{\cQ^c,\eps_0} \le C  d^{-\gamma  } \cdot \| \proj_{\cQ^c} f_d \|_{L^2}^2 \, .
\]
\end{proposition}

\begin{proposition}[Lower bound on the kernel matrix]\label{prop:kernel_lower_bound_RFK_NU}
Assume $N = o_d (d^\gamma)$ for a fixed integer $\gamma > 0$. Let  $(\btheta_i)_{i \in [N]} \sim \Unif(\S^{D-1}(\sqrt D))$ independently, and $\sigma$ be an activation function satisfying Assumption \ref{ass:activation_lower_upper_RF_aniso}.(a). Let $\bU \in \R^{N \times N}$ be the  kernel matrix defined by Eq.~\eqref{eq:KernelMatrix}.
Then there exists a constant $\eps > 0 $ that depends on the activation function $\sigma$, such that
\[
\lambda_{\min}(\bU) \ge \eps,
\]
with high probability as $d \to \infty$. 
\end{proposition}

The proofs of these two propositions are provided in the next sections. 

Proposition \ref{prop:bound_V_RFK_NU} shows that there exists $\eps_0 > 0$ such that
\[
\E_{\bTheta_{\eps_0}} [ \| \bV_{\cQ^c} \|_2^2 ] = N \cE_{\cQ^c, \eps_0}  \leq C N d^{-\gamma } \| \proj_{\cQ^c} f_d \|_{L^2}^2.
\]
Hence, by Markov's inequality, we get for any $\eps > 0$,
\[
\begin{aligned}
\P (  \| \bV_{\cQ^c} \|_2 \geq \eps \cdot \| \proj_{\cQ^c} f_d \|_{L^2} ) \leq & \P (  \lbrace \| \bV_{\cQ^c} \|_2 \geq \eps \cdot \| \proj_{\cQ^c} f_d \|_{L^2} \rbrace \cap \cP_{\eps_0} ) + \P ( \cP_{\eps_0}^c ) \\
\leq & \frac{N \cE_{\cQ^c , \eps_0 }}{\eps^2 \| \proj_{\cQ^c} f_d \|_{L^2}^2 }  + o_d (1)\\
\leq & C ' N d^{-\gamma} + o_d(1),
\end{aligned}
\]
where we used Lemma \ref{lem:bound_proba_cPeps}. By assumption, we have $N = o_d ( d^{\gamma})$, hence Eq.~\eqref{eqn:bound_V_RFK_NU} is verified. Furthermore Eq.~\eqref{eqn:bound_inverse_U_RFK_NU} follows simply from Proposition \ref{prop:kernel_lower_bound_RFK_NU}. This proves the theorem.

\subsection{Proof of Proposition \ref{prop:bound_V_RFK_NU}}

We will denote:
\[
\overline f_{d}(\obx) = f_{d}(\bx), 
\]
such that $\of$ is a function on the normalized product of spheres $\PS^\bd$ (Note that we defined $\proj_\bk f_d (\bx) \equiv \proj_{\bk} \of_d (\obx )$ the unambiguous polynomial approximation of $f_d$ with polynomial of degree $\bk$).
We have
\[
\begin{aligned}
V_{i, \cQ^c } =& \E_{\bx} \left[ [\proj_{\cQ^c} f_{d}](\bx) \sigma\lp \sum_{q \in [Q]} \< \bx^\pq, \btheta_{i}^\pq \> / R \rp \right] \\
=& \E_{\obx} \left[ [\proj_{\cQ^c} \of_{d}](\obx) \sigma_{\bd,\btau_i} \lp \lb \< \obx^\pq, \obtheta_{i}^\pq\> / \sqrt{d_q} \rb_{q \in [Q]} \rp \right].
\end{aligned}
\]
We recall the expansion of $\sigma_{\bd,\btau}$ in terms of tensor product of Gegenbauer polynomials
\[
\begin{aligned}
\sigma (\<\btheta , \bx \>/ R ) = & \sum_{\bk \in \posint^Q} \lambda^{\bd}_{\bk} (\sigma_{\bd,\btau}  ) B(\bd,\bk) Q^{\bd}_{\bk} \lp \lb \< \obtheta^\pq , \obx^\pq \> \rb_{q \in [Q]} \rp, \\
\lambda^{\bd}_{\bk} (\sigma_{\bd,\btau} ) = & \E_{\obx} \Big[ \sigma_{\bd,\btau} \lp \ox_1^{(1)} , \ldots , \ox^{(Q)}_1 \rp  Q^{\bd}_{\bk} \Big( \sqrt{d_1}\ox_1^{(1)} , \ldots , \sqrt{d_Q} \ox_1^{(Q)} \Big)  \Big].
\end{aligned}
\]
For any $\bk \in \posint^Q$, the spherical harmonics expansion of $P_{\bk} \overline f_{d}$ gives
\[
P_{\bk} \overline f_{d} ( \obx) =  \sum_{\bs \in [B(\bd , \bk)] }  \lambda^{\bd}_{\bk,\bs} ( \overline f_{d} ) Y^{\bd}_{\bk,\bs} ( \obx ) .
\]  
Using Eq.~\eqref{eq:ProdGegenbauerHarmonics} to get the following property
\begin{equation}
\begin{aligned}
 \E_{\obx} \lbb Q^{\bd}_{\bk '} \lp \lb \< \overline \btheta^\pq , \obx^\pq \> \rb_{q \in [Q]} \rp  Y^{\bd}_{\bk,\bs} ( \obx) \rbb 
=  &\frac{1}{B(\bd , \bk ' )} \sum_{\bs ' \in [B(\bd , \bk)]} Y^{\bd}_{\bk',\bs'} ( \obtheta ) \E_{\obx} \lbb Y^{\bd}_{\bk ' ,\bs' } ( \obx ) Y^{\bd}_{\bk,\bs} ( \obx) \rbb \\
=  & \frac{1}{B(\bd , \bk)}  Y^{\bd}_{\bk,\bs} ( \obtheta) \delta_{\bk,\bk'} \, ,
\end{aligned}\label{eq:prod_gegenbauer_harmonics}
\end{equation}
we get
\[
\begin{aligned}
& \E_{\obx}\lbb [\proj_{\bk} \overline f_{d}] (\obx ) \sigma_{\bd,\btau} \lp \lb \< \overline \btheta^\pq , \obx^\pq \> / \sqrt{d_q} \rb_{q \in [Q]} \rp  \rbb \\
=  & \sum_{\bk' \geq \bzero } \lambda^{\bd}_{\bk'} (\sigma_{\bd,\btau}  ) B(\bd,\bk') \sum_{\bs \in [B(\bd , \bk)] }  \lambda^{\bd}_{\bk,\bs} ( \overline f_{d} ) \E_{\obx} \lbb Y^{\bd}_{\bk,\bs} ( \obx ) Q^{\bd}_{\bk} \lp \lb \< \overline \btheta^\pq , \obx^\pq \> \rb_{q \in [Q]} \rp \rbb \\
=&
 \sum_{\bs \in [B(\bd,\bk)] }  \lambda^{\bd}_{\bk,\bs} ( \overline f_{d} )  \lambda^{\bd}_{\bk} ( \sigma_{\bd,\btau} )  Y^{\bd}_{\bk,\bs} ( \obtheta ) . 
\end{aligned}
\]
Let $\eps_0>0$ be a constant as specified in Lemma \ref{lem:bound_gegenbauer_RF_LB}. We consider 
\begin{equation}\label{eq:bound_cE_dec}
\begin{aligned}
\cE_{\cQ^c, \eps_0 } = &\E_{\obtheta,\btau_{\eps_0}} \lbb \E_{\obx} \lbb [ \proj_{\cQ^c} \of_{d} ] (\obx) \sigma_{\bd,\btau} \lp \lb \< \obtheta^\pq , \obx^\pq \> /\sqrt{d_q} \rb_{q \in [Q]} \rp   \rbb^2 \rbb \\
= & \sum_{\bk,\bk' \in \cQ^c } \E_{\obtheta,\btau_{\eps_0}} \Big[ \E_{\obx} \lbb [ \proj_{\bk} \of_{d}  ] (\obx) \sigma_{\bd,\btau}  \lp \lb \< \obtheta^\pq , \obx^\pq \> /\sqrt{d_q} \rb_{q \in [Q]} \rp \rbb \\
& \phantom{AAAAAA} \times \E_{\oby} \lbb [ \proj_{\bk'} \of_{d} ] (\oby) \sigma_{\bd,\btau}  \lp \lb \< \obtheta^\pq , \oby^\pq \> /\sqrt{d_q} \rb_{q \in [Q]} \rp  \rbb \Big]  \\ 
= &  \sum_{\bk,\bk' \in \cQ^c }\E_{\btau_{\eps_0} } \lbb  \lambda^{\bd}_{\bk} ( \sigma_{\bd,\btau} )  \lambda^{\bd}_{\bk'} ( \sigma_{\bd,\btau} ) \rbb \\
& \phantom{AAAAAA} \times \sum_{\bs \in [B(\bd , \bk)] }  \sum_{\bs' \in [B(\bd , \bk')] } \lambda^{\bd}_{\bk,\bs} ( \overline f_{d} ) \lambda^{\bd}_{\bk',\bs'} ( \overline f_{d} ) \E_{\obtheta} [ Y^{\bd}_{\bk,\bs} ( \obtheta) Y^{\bd}_{\bk',\bs'} ( \obtheta) ] \\
= &  \sum_{\bk \in \cQ^c } \E_{\btau_{\eps_0} } [  \lambda^{\bd}_{\bk} ( \sigma_{\bd,\btau} )^2 ] \sum_{\bs \in [B(\bd , \bk)] }  \lambda^{\bd}_{\bk,\bs} ( \overline f_{d} )^2 \\
\le & \lbb \max_{\bk \in \cQ^c}  \E_{\btau_{\eps_0} } [  \lambda^{\bd}_{\bk} ( \sigma_{\bd,\btau} )^2 ] \rbb \cdot \sum_{\bk \in \cQ^c} \sum_{\bs \in [B(\bd , \bk ) ]  }   \lambda^{\bd}_{\bk,\bs} ( \overline f_{d} )^2 \\
= & \lbb \max_{\bk \in \cQ^c}  \E_{\btau_{\eps_0} } [  \lambda^{\bd}_{\bk} ( \sigma_{\bd,\btau} )^2 ] \rbb \cdot \| \proj_{\cQ^c} \of_{d} \|_{L^2}.
\end{aligned}
\end{equation}

From Lemma \ref{lem:bound_gegenbauer_RF_LB}, there exists a constant $C>0$ such that for $d$ sufficiently large, we have for any $\bk \in \cQ^c$,
\begin{equation}\label{eq:bound_coeff_cE}
\E_{\btau_{\eps_0} } [  \lambda^{\bd}_{\bk} ( \sigma_{\bd,\btau} )^2 ] \leq \sup_{\btau \in [1 - \eps_0 , 1 + \eps_0]^Q } \lambda^{\bd}_{\bk} ( \sigma_{\bd,\btau} )^2 \leq C d^{-\gamma}.
\end{equation}
Combining Eq.~\eqref{eq:bound_cE_dec} and Eq.~\eqref{eq:bound_coeff_cE} yields
\[
\cE_{\cQ^c, \eps_0 } \leq C d^{- \gamma } \cdot \| \proj_{\cQ^c} \of_{d} \|_{L^2}.
\]

\subsection{Proof of Proposition \ref{prop:kernel_lower_bound_RFK_NU}}

\noindent
\textbf{Step 1. Construction of the activation functions $\hat \sigma$, $\bar \sigma$. }

Without loss of generality, we will assume that $q_\xi =1$. From Assumption \ref{ass:activation_lower_upper_RF_aniso}.$(b)$, $\sigma$ is not a degree $\lfloor \gamma / \eta_1 \rfloor$-polynomial. This is equivalent to having $m \geq \lfloor \gamma / \eta_1 \rfloor +1$ such that $\mu_m ( \sigma) \neq 0$. Let us denote
\[
m = \inf \lbrace k \ge \lfloor \gamma / \eta_1 \rfloor + 1| \mu_m ( \sigma ) \neq 0 \rbrace.
\]

Recall the expansion of $\sigma_{\bd,\btau}$  in terms of product of Gegenbauer polynomials
\[
\sigma_{\bd , \btau} \lp \lb \<\obtheta^\pq , \obx^\pq \>/ \sqrt{d_q} \rb_{q \in [Q]} \rp =  \sum_{\bk \in \posint^Q} \lambda^{\bd}_{\bk} (\sigma_{\bd,\btau}  ) B(\bd,\bk) Q^{\bd}_{\bk} \lp \lb \< \overline \btheta^\pq , \obx^\pq \> \rb_{q \in [Q]} \rp, 
\]
where
\[
\lambda^{\bd}_{\bk} (\sigma_{\bd,\btau} ) =  \E_{\obx} \lbb \sigma_{\bd,\btau} \lp \ox_1^{(1)} , \ldots , \ox^{(Q)}_1 \rp Q^{\bd}_{\bk} \lp  \sqrt{d_1} \ox^{(1)}_1 , \ldots , \sqrt{d_Q} \ox^{(Q)}_1 \rp \rbb.
\]
Denoting $\bm = ( m , 0 , \ldots , 0) \in \posint^Q$ and using the Gegenbauer coefficients of $\sigma_{\bd , \btau}$, we define an activation function $\bar \sigma _{\bd,\btau}$ which is a degree $m$ polynomial in $\obx^{(1)}$  and do not depend on $\obx^\pq$ for $q \ge 2$. 
\[
\begin{aligned}
\bar \sigma_{\bd,\btau} \lp \lb \obtheta^\pq , \obx^\pq \> / \sqrt{d_q} \rb_{q \in [Q]} \rp & = \lambda^{\bd}_{\bm} (\sigma_{\bd,\btau}  ) B(\bd,\bm) Q^{\bd}_{\bm} \lp \lb \< \obtheta^\pq , \obx^\pq \> / \sqrt{d_q} \rb_{q \in [Q]} \rp \\
& = \lambda^{\bd}_{\bm} (\sigma_{\bd,\btau}  ) B(d_1,m)  Q^{(d_1)}_{m} ( \sqrt{d_1} \ox^{(1)}_1 ),
\end{aligned}
\]
and an activation function 
\[
\hat \sigma_{\bd,\btau} \lp \lb \obtheta^\pq , \obx^\pq \> / \sqrt{d_q} \rb_{q \in [Q]} \rp =  \sum_{\bk \neq \bm \in \posint^Q } \lambda^{\bd}_{\bk} (\sigma_{\bd,\btau}  ) B(\bd,\bk) Q^{\bd}_{\bk} \lp \lb \< \obtheta^\pq , \obx^\pq \> / \sqrt{d_q} \rb_{q \in [Q]} \rp.
\]

\noindent
\textbf{Step 2. The kernel functions $u_d$, $\hat u_d$ and $\bar u_d$.} 

Let $u_d$, $\hat u_d$ and $\bar u_d$  be defined by
\begin{equation}
\begin{aligned}
&u_{\bd}^{\btau_1 , \btau_2} \lp \lb \<\obtheta_1^\pq , \obtheta_2^\pq \> / \sqrt{d_q} \rb_{q \in [Q]} \rp \\
=& \E_{\bx}[\sigma(\<\btheta_1, \bx\>/R) \sigma(\<\btheta_2, \bx\>/R)]\\
 = & \sum_{\bk \in \posint^Q} \lambda^{\bd}_{\bk} (\sigma_{\bd,\btau_1}  ) \lambda^{\bd}_{\bk} (\sigma_{\bd,\btau_2}  ) B(\bd,\bk) Q^{\bd}_{\bk} \lp \lb\< \obtheta_1^\pq , \obtheta_2^\pq \> / \sqrt{d_q} \rb_{q \in [Q]} \rp \, \\
 \end{aligned}
 \end{equation}
 and
 \begin{equation}
 \begin{aligned}
& \hat u_{\bd}^{\btau_1 , \btau_2} \lp \lb \<\obtheta_1^\pq , \obtheta_2^\pq \> / \sqrt{d_q} \rb_{q \in [Q]} \rp \\
=& \E_{\bx}[\hat \sigma(\<\btheta_1, \bx\>/R) \hat \sigma(\<\btheta_2, \bx\>/R)]\\
 = & \sum_{\bk \neq \bm \in \posint^Q} \lambda^{\bd}_{\bk} (\sigma_{\bd,\btau_1}  ) \lambda^{\bd}_{\bk} (\sigma_{\bd,\btau_2}  ) B(\bd,\bk) Q^{\bd}_{\bk} \lp \lb \<\obtheta_1^\pq , \obtheta_2^\pq \> / \sqrt{d_q} \rb_{q \in [Q]} \rp \, \\
  \end{aligned}
 \end{equation}
and
\begin{equation}
\begin{aligned}
 \bar u_{\bd}^{\btau_1 , \btau_2} \lp \lb \<\obtheta_1^\pq , \obtheta_2^\pq \> / \sqrt{d_q} \rb_{q \in [Q]} \rp =& \E_{\bx}[\bar \sigma(\<\btheta_1, \bx\>/R) \bar \sigma(\<\btheta_2, \bx\>/R)]\\
 = & \lambda^{\bd}_{\bm} (\sigma_{\bd,\btau_1}  ) \lambda^{\bd}_{\bm} (\sigma_{\bd,\btau_2}  ) B(d_1,m) Q^{(d_1)}_{m} ( \< \obtheta^{(1)}_1 , \obtheta^{(1)}_2 \>).
\end{aligned}\label{eq:decomposition_U_bar_gegenbauer}
\end{equation}
We immediately have $u_{\bd}^{\btau_1 , \btau_2} = \hat u_{\bd}^{\btau_1 , \btau_2} + \bar u_{\bd}^{\btau_1 , \btau_2}$. Note that all three correspond to positive semi-definite kernels. 

\noindent
\textbf{Step 3. Analyzing the kernel matrix. }

Let $\bU, \hat \bU, \bar \bU \in \R^{N \times N}$ with 
\[
\begin{aligned}
\bU_{ij} =& u_{\bd}^{\btau_i , \btau_j} \lp \lb\< \obtheta_i^\pq , \obtheta_j^\pq \> / \sqrt{d_q} \rb_{q \in [Q]} \rp, \\
\hat \bU_{ij} =& \hat u_{\bd}^{\btau_i , \btau_j} \lp \lb \<\obtheta_i^\pq , \obtheta_j^\pq \> / \sqrt{d_q} \rb_{q \in [Q]} \rp, \\
\bar \bU_{ij} =& \bar u_{\bd}^{\btau_i , \btau_j} \lp \lb \<\obtheta_i^\pq , \obtheta_j^\pq \> / \sqrt{d_q} \rb_{q \in [Q]} \rp.
\end{aligned}
\]
Since $\hat \bU = \bU - \bar \bU \succeq 0$, we immediately have $\bU \succeq \bar \bU$. In the following, we will lower bound $\bar \bU$. 

By the decomposition of $\bar \bU$ in terms of Gegenbauer polynomials \eqref{eq:decomposition_U_bar_gegenbauer}, we have
\[
\bar \bU =  B(d_1,m)\, \diag \Big( \lambda^{\bd}_{\bm} ( \sigma_{\bd,\btau_i}) \Big) \cdot  \bW_m \cdot  \diag \Big(\lambda^{\bd}_{\bm} ( \sigma_{\bd,\btau_i}) \Big),
\]
where $\bW_m \in \R^{N \times N}$ with $W_{m,ij} =  Q_m^{(d_1)} ( \<\obtheta^{(1)}_i , \obtheta^{(1)}_j \>)$.
From Proposition \ref{prop:Delta_bound} (recalling that by definition of $m > \gamma / \eta_1$, i.e.\,$\gamma < m \eta_1$, we have $N < d^{\eta_1 m - \delta} = d_1^{m - \delta ' }$ for some $\delta >0$), we have 
\[
\| \bW_{m} -  \id_N \|_{\op} = o_{d,\P } (1).
\]
Hence we get
\begin{equation}\label{eq:RF_prop_2_bound_I}
\Big\| \overline \bU   - B(d_1,m) \diag \Big(\lambda^{\bd}_{\bm} ( \sigma_{\bd,\btau_i})^2 \Big) \Big\|_{\op} =  \max_{i \in [N]} \Big\lbrace B(d_1,m)  \lambda^{\bd}_{\bm} ( \sigma_{\bd,\btau_i})^2 \Big\rbrace \cdot o_{d,\P} (1).
\end{equation}
From Assumption \ref{ass:activation_lower_upper_RF_aniso}.$(a)$ and Lemma \ref{lem:convergence_Gegenbauer_coeff_0_l} applied to coefficient $\bm$, as well as the assumption that $\mu_m (\sigma) \neq 0 $, there exists $\eps_0>0$ and $ C,c>0$ such that for $d$ large enough,
\begin{equation}\label{eq:RF_prop_2_bound_II}
\begin{aligned}
&\sup_{\btau \in [1 - \eps_0 , 1+\eps_0]^Q} B(d_1,m)  \lambda^{\bd}_{\bm} ( \sigma_{\bd,\btau})^2 \le C < \infty, \\
&\inf_{\btau \in [1 - \eps_0 , 1+\eps_0]^Q} B(d_1,m)  \lambda^{\bd}_{\bm} ( \sigma_{\bd,\btau})^2 \ge c > 0.
\end{aligned}
\end{equation}
We restrict ourselves to the event $\cP_{\eps_0}$ defined in Eq.~\eqref{eq:def_P_eps}, which happens with high probability (Lemma \ref{lem:bound_proba_cPeps}). Hence from Eqs.~\eqref{eq:RF_prop_2_bound_I} and \eqref{eq:RF_prop_2_bound_II}, we deduce that with high probability
\[
\overline \bU  = B(d_1,m) \diag \Big(\lambda^{\bd}_{\bm} ( \sigma_{\bd,\btau_i})^2 \Big) + o_{d,\P} (1)  \succeq \frac{c}{2} \id_N.
\]
We conclude that with high probability
\[
\bU = \bar \bU + \hat \bU  \succeq \overline \bU \succeq \frac{c}{2} \id_N.
\]

\clearpage

\section{Proof of Theorem \ref{thm:RF_lower_upper_bound_aniso}.(b): upper bound for \RF\,model}
\label{sec:proof_RFK_upper_aniso}

\subsection{Preliminaries}

\begin{lemma} \label{lem:bound_gegenbauer_RF_UB}
Let $\sigma$ be an activation function that satisfies Assumptions \ref{ass:activation_lower_upper_RF_aniso}.(a) and \ref{ass:activation_lower_upper_RF_aniso}.(b). Let $\|\bw^\pq \|_2=1$ be unit vectors of $\R^{d_q}$, for $q = 1 , \ldots , Q$. Fix $\gamma > 0$ and denote $\cQ = \overline \cQ_{\RF} (\gamma)$. Then there exists $\eps_0 > 0 $ and $d_0$ and constants $C, c > 0$ such that for $d \ge d_0$ and $\btau \in [1- \eps_0 , 1+ \eps_0]^Q$,
\begin{align}
\E_{\overline \bx} \lbb \sigma_{\bd, \btau} \lp \lb \< \bw^\pq , \obx^\pq \> \rb_{q \in [Q]} \rp^2 \rbb \le  C < \infty, \label{eq:bound_gegen_RF_UB_1}\\
 \min_{\bk \in \cQ  } \lambda^{\bd}_{k,0} ( \sigma_{d,\tau}  )^2   \ge  c d^{-\gamma} > 0. \label{eq:bound_gegen_RF_UB_2}
\end{align}
\end{lemma}

\begin{proof}[Proof of Lemma \ref{lem:bound_gegenbauer_RF_UB}]
The first inequality comes simply from Assumption \ref{ass:activation_lower_upper_RF_aniso}.$(a)$ and Lemma \ref{lem:coupling_convergence_Gaussian}.$(b)$. For the second inequality, notice that by Assumption \ref{ass:activation_lower_upper_RF_aniso}.$(c)$ we can apply Lemma \ref{lem:convergence_proba_Gegenbauer_coeff} to any $\bk \in \cQ$. Hence (using that $\mu_k (\sigma)^2 >0$ and we can choose $\delta$ sufficiently small), we deduce that there exists $c >0$, $\eps_0 > 0$ and $d_0$ such that for any $d \geq d_0$, $\btau \in [1- \eps_0 , 1+ \eps_0]^Q$ and $\bk \in \cQ$, 
\[
\lp \prod_{q \in [Q] } d^{(\xi - \eta_q - \kappa_q)k_q} \rp B(\bd,\bk) \lambda^{\bd}_{\bk} ( \sigma_{\bd,\btau}  )^2 \geq c > 0.
\]
Furthermore, using that $B(\bd,\bk) = \Theta ( d_1^{k_1} d_2^{k_2} \ldots d_Q^{k_Q})$, there exists $c ' >0$ such that for any $\bk \in \cQ$,
\[
\lambda^{\bd}_{\bk} ( \sigma_{\bd,\btau}  )^2 \geq c '  \prod_{q \in [Q] } d^{(\eta_q + \kappa_q - \xi)k_q} d_q^{k_q} = c '  \prod_{q \in [Q] } d^{(\kappa_q - \xi)k_q}  \ge c d^{-\gamma},
\]
where we used in the last inequality $\bk \in \overline{\cQ}_{\RF} (\gamma)$ implies $ (\xi - \kappa_1)k_1 + \ldots + (\xi - \kappa_Q) k_Q \leq \gamma$ by definition.
\end{proof}

\subsection{Properties of the limiting kernel}

Similarly to the proof of \cite[Theorem 1.$(b)$]{ghorbani2019linearized}, we construct a limiting kernel which is used as a proxy to upper bound the \RF\,risk.

We recall the definition of $\PS^{\bd} = \prod_{q \in [Q]} \S^{d_q - 1} ( \sqrt{d_q})$ and $\mu_\bd = \Unif ( \PS^\bd )$. Let us denote $\mathcal{L} = L^2 ( \PS^{\bd}, \mu_\bd )$. Fix $\btau \in \R_{>0}^Q$ and recall the definition for a given $\btheta = (\obtheta , \btau)$ of $\sigma_{\bd,\btau} ( \lb \< \obtheta^\pq , \cdot \> /\sqrt{d_q} \rb ) \in \mathcal{L} $,
\[
\sigma_{\bd,\btau}   \lp \lb \< \obtheta_q , \overline \bx^\pq \> \rb_{q \in [Q]} \rp = \sigma \lp \sum_{q \in [Q]}  \tau^\pq (r_q / R) \< \obtheta_q , \overline \bx^\pq \> \rp.
\]

Define the operator $\T_{\btau} : \cL \to \cL$, such that for any $g \in \cL$,
\[
\T_{\btau} g ( \obtheta ) =  \E_{\overline \bx} \lbb  \sigma_{\bd,\btau}  \lp \lb \< \obtheta^\pq , \obx^\pq \> /\sqrt{d_q} \rb_{q \in [Q]} \rp g (\obx) \rbb.
\]
It is easy to check that the adjoint operator $\T^*_{\btau} : \cL \to \cL$ verifies $\T^* = \T$ with variables $\obx$ and $\obtheta$ exchanged.

We define the operator $\K_{\btau,\btau'} : \cL \to \cL $ as $\K_{\btau,\btau'} \equiv \T_{\btau} \T_{\btau'}^*$. For $g \in \cL$, we can write
\[
\K_{\btau_1, \btau_2} g ( \obtheta_1 ) = \E_{\obtheta_2 } [ K_{\btau_1, \btau_2} ( \obtheta_1 , \obtheta_2 ) g ( \obtheta_2 ) ],
\]
where 
\[
K_{\btau_1,\btau_2} ( \obtheta_1 , \obtheta_2 ) = \E_{\obx} \lbb \sigma_{\bd,\btau_1}   \lp \lb \< \obtheta^\pq_1 , \obx^\pq \> /\sqrt{d_q} \rb_{q \in [Q]} \rp \sigma_{\bd,\btau_2}   \lp \lb \< \obtheta^\pq_2 , \obx^\pq \> /\sqrt{d_q} \rb_{q \in [Q]} \rp  \rbb .
\]

We recall the decomposition of $\sigma_{\bd,\btau} $ in terms of tensor product of Gegenbauer polynomials
\[
\begin{aligned}
\sigma_{\bd , \btau} \lp \lb \ox^\pq_1  \rb_{q \in [Q]} \rp = & \sum_{\bk \in \posint^Q} \lambda^{\bd}_{\bk} (\sigma_{\bd,\btau}  ) B(\bd,\bk) Q^{\bd}_{\bk} \lp \lb \ox^\pq_1  \rb_{q \in [Q]} \rp, \\
\lambda^{\bd}_{\bk} (\sigma_{\bd,\btau} ) = & \E_{\obx} \lbb \sigma_{\bd,\btau} \lp \lb \ox^\pq_1  \rb_{q \in [Q]} \rp  Q^{\bd}_{\bk} \lp \lb \sqrt{d_q} \ox^\pq_1  \rb_{q \in [Q]} \rp  \rbb.
\end{aligned}
\]
Recall that $\lbrace Y^\bd_{\bk,\bs}  \rbrace_{\bk \in \posint^Q, \bs \in [B(\bd , \bs)]}$ forms an orthonormal basis of $\cL$.
From Eq.~\eqref{eq:prod_gegenbauer_harmonics}, we have for any $\bk \geq 0$ and $\bs \in [B(\bd,\bk)]$,
\[
\begin{aligned}
\T_{\btau} Y^\bd_{\bk,\bs} ( \obtheta) =& \sum_{\bk' \in \posint^Q} \lambda_{\bk'}^{\bd} ( \sigma_{\bd,\btau}) B(\bd , \bk') \E_{\obx} \lbb Q^{\bd}_{\bk'} \lp \lb \< \obtheta^\pq, \obx^\pq \>  \rb_{q \in [Q]} \rp  Y^\bd_{\bk,\bs} (\obx ) \rbb \\
 =  & \lambda_{\bk}^{\bd} ( \sigma_{\bd,\btau})Y^\bd_{\bk,\bs} ( \obtheta),
\end{aligned}
\]
where we used 
\[
\E_{\obx} \lbb Q^{\bd}_{\bk'} \lp \lb \< \obtheta^\pq, \obx^\pq \>  \rb_{q \in [Q]} \rp  Y^\bd_{\bk,\bs} (\obx ) \rbb =\frac{\delta_{\bk , \bk'}}{B (\bd , \bk)} \bY_{\bk, \bs} (\obtheta).
\]
The same equation holds for $\T^*_{\btau}$. Therefore, we directly deduce that
\[
\K_{\btau,\btau'} Y^\bd_{\bk,\bs} ( \obtheta) = ( \T_{\btau} \T^*_{\btau'} ) Y^\bd_{\bk,\bs} ( \obtheta) = \lambda_{\bk}^{\bd} ( \sigma_{\bd,\btau})  \lambda_{\bk}^{\bd} ( \sigma_{\bd,\btau ' }) Y^\bd_{\bk,\bs} ( \obtheta).
\]
We deduce that $\lbrace Y^\bd_{\bk,\bs}  \rbrace_{\bk \in \posint^Q, \bs \in [B(\bd , \bs)]}$ is an orthonormal basis that diagonalizes the operator $\K_{\btau,\btau'} $. 

Let $\eps_0 > 0 $ be defined as in Lemma \ref{lem:bound_gegenbauer_RF_UB}. We will consider $\btau,\btau' \in [1- \eps_0 , 1+ \eps_0]^Q$ and restrict ourselves to the subspace $V^\bd_{\cQ}$. From the choice of $\eps_0$ and for $d$ large enough, the eigenvalues $\lambda_{\bk}^{\bd} ( \sigma_{\bd,\btau})  \lambda_{\bk}^{\bd} ( \sigma_{\bd,\btau ' }) \neq 0 $ for any $\bk \in \cQ$. Hence, the operator $\K_{\btau,\btau'}\vert_{V^\bd_{\cQ}} $ is invertible.

 \subsection{Proof of Theorem \ref{thm:RF_lower_upper_bound_aniso}.(b)}

Without loss of generality, let us assume that $\lbrace f_d \rbrace$ are polynomials contained in $V^\bd_{\cQ}$, i.e.\,$\of_{d} = \proj_{\cQ } \of_{d}$. 

Consider 
\[
\hat{f} (\bx ; \bTheta, \ba ) = \sum_{i = 1}^N a_i \sigma   (\<\btheta_i , \bx \> / R).
\]
Define $\alpha_{\btau} ( \obtheta) \equiv \K_{\btau,\btau}^{-1} \T_{\btau} \of_{d} (\obtheta)$ and choose $a_i^* = N^{-1} \alpha_{\btau_i} ( \obtheta_i)$,  where we denoted $\obtheta_i = ( \obtheta^\pq_i )_{q \in [Q]} $ with $\obtheta^\pq_i = \btheta^\pq_i /\tau^\pq_i \in \S^{d_q-1} (\sqrt{d_q})$ and  $\tau^\pq_i = \| \btheta^\pq_i \|_2 / \sqrt{d_q}$ independent of $\obtheta^\pq_i$.

Let $\eps_0 > 0 $ be defined as in Lemma \ref{lem:bound_gegenbauer_RF_UB} and consider the expectation over $\cP_{\eps_0}$ of the \RF\,risk (in particular, $\ba^* = (a_1^* , \ldots , a_N^*)$ are well defined):
\[
\begin{aligned}
\E_{\bTheta_{\eps_0}} [ R_{\RF}(f_d , \bTheta) ] =& \E_{\bTheta_{\eps_0} } \Big[ \inf_{\ba \in \R^{N}} \E_{\bx} [(f_d (\bx) - \hat f (\bx; \bTheta, \ba))^2] \Big]\\
\leq & \E_{\bTheta_{\eps_0}} \lbb \E_{\bx} \lbb (f_d (\bx) - \hat f (\bx; \bTheta, \ba^* (\bTheta) ))^2 \rbb \rbb.
\end{aligned}
\]
We can expand the squared loss at $\ba^*$ as
\begin{equation}
\begin{aligned}
\E_\bx [ (f_d (\bx) - \hat{f} (\bx; \bTheta, \ba^*) )^2 ] = & \| f_d \|^2_{L^2} -2 \sum_{i=1}^N \E_{\bx} [a_i^* \sigma  ( \< \btheta_i , \bx \> / R  ) f_d (\bx) ] \\
\label{eq:expansion_squared_loss_RF}
& \phantom{AAAAA}+  \sum_{i,j=1}^N \E_{\bx} [ a_i^* a_j^* \sigma  ( \< \btheta_i , \bx \> / R ) \sigma  ( \< \btheta_j , \bx \> / R  ) ].
\end{aligned}
\end{equation}
The second term of the expansion \eqref{eq:expansion_squared_loss_RF} around $\ba^*$ verifies
\begin{equation}\label{eq:upper_bound_RF_NU_term1}
\begin{aligned}
& \E_{\bTheta_{\eps_0}} \lbb \sum_{i=1}^N \E_{\bx} \lbb a_i^* \sigma  ( \< \btheta_i , \bx \> / R  ) f_d (\bx) \rbb \rbb \\
= & \E_{\btau_{\eps_0}} \lbb \E_{\obtheta} \lbb  \alpha_{\btau} ( \obtheta)  \E_{\obx} \lbb  \sigma_{\bd,\btau}  \lp \lb \< \obtheta^\pq , \obx^\pq \> /\sqrt{d_q} \rb_{q \in [Q]} \rp  \of_{d} (\obx) \rbb \rbb\rbb \\
= &  \E_{\btau_{\eps_0} } \lbb \< \K_{\btau,\btau}^{-1} \T_{\btau} \of_{d}  , \T_{\btau} \of_{d} \>_{L^2} \rbb\\
= & \| f_d \|^2_{L^2},
\end{aligned}
\end{equation}
where we used that for each $\btau \in [1-\eps_0,1+\eps_0]^Q$, we have $\T_{\btau}^* \K_{\btau,\btau}^{-1} \T_{\btau} \vert_{V^\bd_{\cQ}} = \id\vert_{V^\bd_{\cQ}}$.

Let us consider the third term in the expansion \eqref{eq:expansion_squared_loss_RF} around $\ba^*$: the non diagonal term verifies
\[
\begin{aligned}
& \E_{\bTheta_{\eps_0}} \lbb \sum_{i\neq j} \E_{\bx} \lbb a_i^*a_j^* \sigma  ( \< \btheta_i , \bx \> / R ) \sigma  ( \< \btheta_j , \bx \> / R ) \rbb \rbb \\
= & \lp 1 - N^{-1} \rp \E_{\btau_{\eps_0}^1 , \btau_{\eps_0}^2,\obtheta_1 , \obtheta_2} \Big[  \alpha_{\btau^1} ( \obtheta_1)   \alpha_{\btau^2} ( \obtheta_2)  \\
& \phantom{AAAAAA} \times E_{\obx} \Big[ \sigma_{\bd,\btau^1}   \lp \lb \< \obtheta^\pq_1 , \obx^\pq \> / \sqrt{d_q} \rb_{q \in [Q]} \rp  \sigma_{\bd,\btau^2}   \lp \lb \< \obtheta^\pq_2 , \obx^\pq \> / \sqrt{d_q} \rb_{q \in [Q]} \rp \Big] \Big]  \\
= & \lp 1 - N^{-1} \rp \E_{\btau_{\eps_0}^1 , \btau_{\eps_0}^2,\obtheta_1 , \obtheta_2} \Big[  \K^{-1}_{\btau^1,\btau^1} \T_{\btau^1} \of_{d} ( \obtheta_1 )   \K_{\btau^1 , \btau^2 } ( \obtheta_1 , \obtheta_2 ) \K^{-1}_{\btau^2,\btau^2} T_{\btau^2} \of_{d} ( \obtheta_2 ) \Big] \\
= &  \lp 1 - N^{-1} \rp \E_{\btau_{\eps_0}^1, \btau_{\eps_0}^2 } \Big[ \< \K_{\btau^1,\btau^1}^{-1} \T_{\btau^1} \of_{d}  , \K_{\btau^1,\btau^2} \K^{-1}_{\btau^2,\btau^2} \T_{\btau^2} \of_{d} \>_{L^2} \Big].
\end{aligned}
\]
For $\bk \in \cQ$ and $\bs \in [B(\bd , \bk)]$ and $\btau^1,\btau^2 \in [1-\eps_0,1+\eps_0]^Q$, we have (for $d$ large enough)
\[
\begin{aligned}
\T_{\btau^1}^* \K^{-1}_{\btau^1,\btau^1} \K_{\btau^1,\btau^2} \K^{-1}_{\btau^2,\btau^2} \T_{\btau^2}  Y^\bd_{\bk,\bs} =&  \Big(\T_{\btau^1}^* \K^{-1}_{\btau^1,\btau^1} \T_{\btau^1} \Big) \cdot \Big( \T_{\btau^2}^* \K^{-1}_{\btau^2,\btau^2} \T_{\btau^2} \Big) \cdot Y^\bd_{\bk,\bs}   =  Y^\bd_{\bk,\bs}.
\end{aligned}
\]
Hence for any $\btau^1,\btau^2 \in [1-\eps_0,1+\eps_0]^Q$, $\T_{\btau^1}^* \K^{-1}_{\btau^1,\btau^1} \K_{\btau^1,\btau^2} \K^{-1}_{\btau^2,\btau^2} \T_{\btau^2}\vert_{V^\bd_{\cQ}}   = \id\vert_{V^\bd_{\cQ}}$. Hence
\begin{equation}\label{eq:upper_bound_RF_NU_term2}
\E_{\bTheta_{\eps_0}} \lbb \sum_{i\neq j} \E_{\bx} \lbb a_i^*a_j^* \sigma  ( \< \btheta_i , \bx \> / R  ) \sigma  ( \< \btheta_j , \bx \> / R ) \rbb \rbb  = \lp 1 - N^{-1} \rp \| f_d \|^2_{L^2}.
\end{equation}
The diagonal term verifies
\[
\begin{aligned}
 & \E_{\bTheta_{\eps_0}} \lbb \sum_{i \in [N]} \E_{\bx} \lbb (a_i^*)^2  \sigma  ( \< \btheta_i , \bx \> / R  )^2 \rbb \rbb \\ = & N^{-1} \E_{\btau_{\eps_0} , \obtheta} \lbb  \alpha_{\btau} ( \obtheta)^2 K_{\btau,\btau} ( \obtheta , \obtheta ) \rbb  \\
\le & N^{-1} \lbb \max_{\obtheta, \btau \in [1- \eps_0, 1+\eps_0]^Q}  K_{\btau,\btau} ( \obtheta , \obtheta ) \rbb \cdot \E_{\btau_{\eps_0}} [ \| \K^{-1}_{\btau,\btau} \T_{\btau} \of_{d} \|^2_{L^2} ].
\end{aligned}
\]
We have by definition of $\K_{\btau,\btau}$
\[
 \sup_{\btau \in [1- \eps_0,1+\eps_0]^Q} K_{\btau,\btau} ( \obtheta , \obtheta ) = \sup_{\btau \in [1- \eps_0,1+\eps_0]^Q}  \| \sigma_{\bd,\btau} \|^2_{L^2} \leq C,
\]
 for $d$ large enough (using Lemma \ref{lem:bound_gegenbauer_RF_UB}). Furthermore
\[
\begin{aligned}
\| \K^{-1}_{\btau,\btau} \T_{\btau} \of_{d} \|^2_{L^2} = & \sum_{\bk \in \cQ} \frac{1}{\lambda^{\bd}_{\bk} ( \sigma_{\bd,\btau} )^2 } \sum_{\bs \in [B(\bd,\bk)]} \lambda^\bd_{\bk,\bs} ( \of_{d})^2 \\
\le & \lbb \max_{\bk \in \cQ}  \frac{1}{\lambda^{\bd}_{\bk} ( \sigma_{\bd,\btau} )^2 }  \rbb \cdot \| \proj_{\cQ} f_d \|^2_{L^2}.
\end{aligned}
\]
From Lemma \ref{lem:bound_gegenbauer_RF_UB}, we get
\[
\E_{\btau_{\eps_0}} [ \| \K^{-1}_{\btau,\btau} \T_{\btau} \of_{d} \|^2_{L^2} ] \leq C d^{\gamma} \cdot \| \proj_{\cQ}  f_d \|^2_{L^2}.
\]
Hence,
\begin{equation}\label{eq:upper_bound_RF_NU_term3}
 \E_{\bTheta_{\eps_0}} \lbb \sum_{i \in [N]} \E_{\bx} \lbb (a_i^*)^2  \sigma  ( \< \btheta_i , \bx \> / R  )^2 \rbb \rbb \leq C \frac{d^{\gamma}}{N} \| \proj_{\cQ}  f_d \|^2_{L^2}.
\end{equation}

Combining Eq.~\eqref{eq:upper_bound_RF_NU_term1}, Eq.~\eqref{eq:upper_bound_RF_NU_term2} and Eq.~\eqref{eq:upper_bound_RF_NU_term3}, we get
\[
\begin{aligned}
& \E_{\bTheta_{\eps_0}} [ R_{\RF}(f_d , \bTheta) ]\\
 \leq & \E_{\bTheta_{\eps_0}} \lbb \E_{\bx} \lbb \lp f_d (\bx) - \hat f (\bx; \bTheta, \ba^* (\bTheta) )\rp^2 \rbb \rbb \\ 
=& \| f_d \|_{L^2}^2 - 2 \| f_d \|_{L^2}^2 + (1 - N^{-1})  \| f_d \|_{L^2}^2 +N^{-1} \E_{\btau_{\eps_0} , \obtheta} \Big[  (\alpha_{\btau} ( \obtheta) )^2 K_{\btau,\btau} ( \obtheta , \obtheta ) \Big]  \\
 \le &C \frac{d^{\gamma}}{N} \| \proj_{\cQ}  f_d \|^2_{L^2}.
\end{aligned}
\]
By Markov's inequality, we get for any $\eps >0$ and $d$ large enough,
\[
\P ( R_{\RF}(f_d , \bTheta) > \eps \cdot \| f_d \|_{L^2}^2 ) \leq \P ( \lbrace R_{\RF}(f_d , \bTheta) > \eps \cdot  \| f_d \|_{L^2}^2 \rbrace \cap \cP_{\eps_0} )  + \P (\cP_{\eps_0}^c ) \leq C '  \frac{d^{\gamma}}{N } + \P (\cP_{\eps_0}^c ).
\]
The assumption $N = \omega_d ( d^{\gamma} )$ and Lemma \ref{lem:bound_proba_cPeps} conclude the proof.

\clearpage

\section{Proof of Theorem \ref{thm:NT_lower_upper_bound_aniso}.(a): lower bound for \NT\,model}
\label{sec:proof_NTK_lower_aniso}

\subsection{Preliminaries}

We consider the activation function $\sigma : \R \to \R$ with weak derivative $\sigma'$. Consider $\sigma_{\bd,\btau}' : \ps^\bd \to \R$ defined as follows
\begin{equation}\label{eq:def_sigma_d_tau_prime}
\begin{aligned}
\sigma '( \< \btheta , \bx \> / R ) = & \sigma ' \lp \sum_{q \in [Q]}  \tau^\pq \cdot (r_q/R) \cdot  \< \obtheta^\pq , \obx^\pq \> / \sqrt{d_q} \rp \\
\equiv & \sigma_{\bd,\btau} ' \lp \lb \< \obtheta^\pq , \obx^\pq \> / \sqrt{d_q} \rb_{q \in [Q]} \rp.
\end{aligned}
\end{equation}
Consider the expansion of $\sigma_{\bd,\btau} '$ in terms of product of Gegenbauer polynomials. We have
\begin{equation}\label{eq:sigma_prime_gegenbauer_expansion}
\sigma ' (\<\btheta , \bx \>/ R ) =  \sum_{\bk \in \posint^Q} \lambda^{\bd}_{\bk} (\sigma_{\bd,\btau} ' ) B(\bd,\bk) Q^{\bd}_{\bk} \lp \lb \< \obtheta^\pq , \obx^\pq \> \rb_{q \in [Q]} \rp, 
\end{equation}
where
\[
\lambda^{\bd}_{\bk} (\sigma_{\bd,\btau} ' ) =  \E_{\obx} \Big[ \sigma_{\bd,\btau} ' \lp \ox_1^{(1)} , \ldots , \ox^{(Q)}_1 \rp  Q^{\bd}_{\bk} \Big( \sqrt{d_1}\ox_1^{(1)} , \ldots , \sqrt{d_Q} \ox_1^{(Q)} \Big)  \Big],
\]
where the expectation is taken over $\obx = (\obx^{(1)} , \ldots , \obx^{(Q)} ) \sim \mu_{\bd}$.

\begin{lemma}\label{lem:upper_bound_A}
Let $\sigma$ be an activation function that satisfies Assumptions \ref{ass:activation_lower_upper_NT_aniso}.$(a)$ and \ref{ass:activation_lower_upper_NT_aniso}.$(b)$. Define for $\bk \in \posint^Q$ and $\btau\in \R_{\ge 0}^Q$, 
\begin{equation}\label{eq:def_A_lemma}
A^{\pq}_{\btau,\bk} = r_q^2 \cdot [ t_{d_q, k_q - 1} \lambda^{\bd}_{\bk_{q-}}(\sigma'_{\bd,\btau})^2 B(\bd, \bk_{q-} ) + s_{d_q, k_q+1} \lambda^{\bd}_{\bk_{q+}}(\sigma'_{\bd,\btau})^2 B(\bd, \bk_{q+} )],
\end{equation}
with $\bk_{q+} = (k_1 , \ldots , k_q+1 , \ldots , k_Q)$ and $\bk_{q-} = (k_1 , \ldots , k_q-1 , \ldots , k_Q)$, and 
\[
s_{d, k} = \frac{k}{2k + d - 2}, \qquad t_{d, k} = \frac{k + d - 2}{2k + d - 2},
\]
with the convention $t_{d,-1} = 0$.
Then there exists constants $\eps_0 >0$ and $C>0$ such that for $d$ large enough, we have for any $\btau \in [1 - \eps_0 , 1+ \eps_0]^Q$ and $\bk \in \cQ_\NT (\gamma)^c$,
\[
\frac{A^{\pq}_{\btau,\bk}}{B(\bd , \bk)}  \leq \begin{dcases}
 C d^\xi d^{- \gamma - (\xi - \min_{q \in S(\bk)} \kappa_q) } & \mbox{ if $k_q > 0$,}\\
 C d^{\eta_q + 2 \kappa_q - \xi} d^{- \gamma - (\xi - \min_{q \in S(\bk)} \kappa_q)} & \mbox{ if $k_q = 0$,}
\end{dcases}
\]
where we recall $S(\bk) \subset [Q]$ is the subset of indices corresponding to the non zero integers $k_q > 0$.
\end{lemma}

\begin{proof}[Proof of Lemma \ref{lem:upper_bound_A}]
Let us fix an integer $M$ such that $\cQ \subset [M]^Q$. We will denote $\cQ \equiv \cQ_{\NT} (\gamma)$ for simplicity. Following the same proof as in Lemma \ref{lem:bound_gegenbauer_RF_LB}, there exists $\eps_0>0$, $d_0$ and $C >0$ such that for any $d\geq d_0$ and $\btau \in [1 - \eps_0 , 1+ \eps_0]^Q$, we have for any $\bk \in \cQ^c \cap [M]^Q$,
\[
\begin{aligned}
 \lambda^\bd_{\bk} (\sigma_{\bd,\btau} ')^2 \leq & C d^{-\gamma - (\xi - \min_{q \in S(\bk)}  \kappa_q )}, \\
  \lambda^\bd_{\bk_{q-}} (\sigma_{\bd,\btau}')^2 \leq & C d^{\xi - \kappa_q -\gamma - (\xi - \min_{q \in S(\bk)}  \kappa_q ) }, \\
   \lambda^\bd_{\bk_{q+}} (\sigma_{\bd,\btau} ')^2 \leq & C d^{\kappa_q - \xi -\gamma - (\xi - \min_{q \in S(\bk)}  \kappa_q )},
 \end{aligned}
\]
while for $\bk \not\in [M]^Q$, we get
\[
\max \lb  \lambda^\bd_{\bk} (\sigma_{\bd,\btau} ')^2  ,  \lambda^\bd_{\bk_{q-}} (\sigma_{\bd,\btau} ')^2 , \lambda^\bd_{\bk_{q+}} (\sigma_{\bd,\btau} ')^2 \rb \leq C d^{- (M -1) \min_{q\in[Q]} \eta_q }.
\]
Injecting this bound in the formula \eqref{eq:def_A_lemma} of $\bA^\pq_{\btau,\bk}$, we get for $d\geq d_0$, $\btau \in [1 - \eps_0 , 1+ \eps_0]^Q$ and any $\bk \in \cQ^c \cap [M]^Q$: if $k_q >0$,
\[
\frac{A^{\pq}_{\btau,\bk}}{B(\bd , \bk)} \leq C' d^{\eta_q + \kappa_q} d^{ \xi - \kappa_q -\gamma - (\xi - \min_{q \in S(\bk)}  \kappa_q )} = C ' d^{\xi} d^{- \gamma - (\xi - \min_{q \in S(\bk)} \kappa_q) },
\]
while for $k_q = 0$,
\[
\frac{A^{\pq}_{\btau,\bk}}{B(\bd , \bk)} \leq C' d^{\eta_q + \kappa_q} d^{- \eta_q} d^{ \kappa_q +\eta_q - \xi - \gamma - (\xi - \min_{q \in S(\bk)}  \kappa_q )} = C ' d^{\eta_q + 2 \kappa_q - \xi} d^{- \gamma - (\xi - \min_{q \in S(\bk)} \kappa_q) },
\]
where we used that for $k_q \in [M]$, there exists a constant $c>0$ such that $s_{d_q , k_q} \leq c d^{-\eta_q}$ and $t_{d_q, k_q} \leq c$.
Similarly, we get for $\bk \not\in [M]^Q$
\[
\frac{A^{\pq}_{\btau,\bk}}{B(\bd , \bk)} \leq C'' d^{\kappa_q + \eta_q - (M-1) \min_{q \in [Q]} \eta_q },
\]
where we used that $s_{d_q , k} , t_{d_q , k} \leq 1$ for any $k \in \posint$. Taking $M$ sufficiently large yields the result.

\end{proof}

\subsection{Proof of Theorem \ref{thm:NT_lower_upper_bound_aniso}.(a): Outline}

The structure of the proof for the \NT\, model is the same as for the \RF\, case, however some parts of the proof requires more work.

We define  the random vector $\bV = (\bV_1, \ldots, \bV_N)^\sT\in\reals^{Nd}$, where, for each $j\le N$, $\bV_j\in\reals^D$, and analogously
$\bV_{\cQ} = (\bV_{1, \cQ}, \ldots, \bV_{N,\cQ})^\sT\in\reals^{ND}$, $\bV_{\cQ^c} = (\bV_{1, \cQ^c}, \ldots, \bV_{N, \cQ^c })^\sT\in\reals^{ND}$, as follows
\[
\begin{aligned}
\bV_{i, \cQ} =& \E_{\bx}[[\proj_{\cQ} f_d](\bx) \sigma'(\< \btheta_i, \bx\>/R) \bx],\\
\bV_{i, \cQ^c } =& \E_{\bx}[[\proj_{\cQ^c} f_d](\bx) \sigma'(\< \btheta_i, \bx\>/R) \bx],\\
\bV_i =& \E_{\bx}[f_d(\bx) \sigma'(\< \btheta_i, \bx\>/R) \bx] = \bV_{i, \cQ } + \bV_{i, \cQ^c}. \\
\end{aligned}
\]
We define the random matrix $\bU = (\bU_{ij})_{i, j \in [N]}\in\reals^{ND\times ND}$, where for each $i,j\le N$, $\bU_{ij}\in\reals^{D\times D}$, 
is given by
\begin{align}
\bU_{ij} = \E_{\bx}[\sigma'(\< \bx, \btheta_i\>/R) \sigma'(\< \bx, \btheta_j\>/R) \bx \bx^\sT]. 
\label{eq:NT-Kernel}
\end{align}
Proceeding as for the \RF\, model, we obtain 
\[
\begin{aligned}
& \Big\vert R_{\NT}(f_d) - R_{\NT}(\proj_{\cQ} f_d) - \| \proj_{\cQ^c} f_d \|^2_{L^2} \Big\vert \\
=& \Big\vert \bV_{\cQ}^\sT \bU^{-1} \bV_{\cQ} - \bV^\sT \bU^{-1} \bV \Big\vert = \Big\vert \bV_{\cQ}^\sT \bU^{-1} \bV_{\cQ} - (\bV_{\cQ} + \bV_{\cQ^c})^\sT \bU^{-1} (\bV_{\cQ} + \bV_{\cQ^c}) \Big\vert\\
=& \Big\vert 2 \bV^\sT \bU^{-1} \bV_{\cQ^c} - \bV_{\cQ^c}^\sT \bU^{-1} \bV_{\cQ^c} \Big\vert \\
\le&  2  \| \bU^{-1/2} \bV_{\cQ^c} \|_2 \| f_d \|_{L^2}+  \bV_{\cQ^c}^\sT \bU^{-1} \bV_{\cQ^c}.
\end{aligned}
\]

We claim that we have
\begin{align}
\| \bU^{-1/2} \bV_{\cQ^c} \|_2^2 = \bV_{\cQ^c}^\sT \bU^{-1} \bV_{\cQ^c} =& o_{d,\P}(\| \proj_{\cQ^c} f_d \|_{L^2}^2 ), \label{eqn:bound_VUV_NTK_NU}.
\end{align}

To show this result, we will need the following two propositions.

\begin{proposition}[Expected norm of $\bV$]\label{prop:expected_V_NTK_NU}
Let $\sigma$ be a weakly differentiable activation function with weak derivative $\sigma'$ and $\cQ \subset \posint^Q$. Let $\eps >0$ and define $\cE^\pq_{\cQ^c,\eps}$ by
\[
\begin{aligned}
\cE^\pq_{\cQ^c, \eps} \equiv & \E_{\btheta_\eps}\Big[ \< \E_\bx[ [ \proj_{\cQ^c} f_d] (\bx) \sigma'(\<\btheta, \bx\>/R) \bx^\pq], \E_\bx[ [\proj_{\cQ^c} f_d] (\bx) \sigma'(\<\btheta, \bx\>/R) \bx^\pq] \>\Big], 
\end{aligned}
\]
where the expectation is taken with respect to $\bx = ( \bx^{(1)} , \ldots , \bx^{(Q)} ) \sim \mu_\bd^\bkappa$.
Then, 
\[
\begin{aligned}
\cE^\pq_{\cQ^c,\eps_0} \le \lbb \max_{\bk \in \cQ^c}  B(\bd,\bk)^{-1} \E_{\btau_{\eps_0}} [A_{\btau,\bk}^{\pq}  ]    \rbb \cdot \| \proj_{\cQ^c}  f_{d} \|_{L^2}^2 \, .
\end{aligned}
\]
\end{proposition}

\begin{proposition}[Lower bound on the kernel matrix]\label{prop:kernel_lower_bound_NTK_NU}
Let $ N= o_d(d^{\gamma})$ for some $\gamma > 0$, and  $(\btheta_i)_{i \in [N]} \sim \Unif(\S^{D-1}(\sqrt D))$ independently. Let $\sigma$ be an activation that satisfies
 Assumptions \ref{ass:activation_lower_upper_NT_aniso}.$(a)$ and  \ref{ass:activation_lower_upper_NT_aniso}.$(b)$. 
Let $\bU \in \R^{ND \times ND}$ be the  kernel matrix with $i,j$ block $\bU_{ij}\in\reals^{D\times D}$ defined by Eq.~\eqref{eq:NT-Kernel}. Then there exists two matrices $\bD$ and $\bDelta$ such that 
\[
\begin{aligned}
\bU \succeq & \bD + \bDelta,
\end{aligned}
\]
with $\bD = \diag ( \bD_{ii} )$ block diagonal. Furthermore, $\bD$ and $\bDelta$ verifies the following properties:
\begin{enumerate}

\item[$(a)$] $\| \bDelta \|_{\op} = o_{d,\P} ( d^{- \max_{q \in [Q]} \kappa_q} ) $

\item[$(b)$] For each $i \in [N]$, we can decompose the matrix $\bD_{ii}$ into block matrix form $( \bD_{ii}^{qq'} )_{q,q' \in [Q]} \in \R^{DN \times DN}$ with $\bD_{ii}^{qq'} \in \R^{d_q N\times d_{q'} N}$ such that
\begin{itemize}
\item For any $q \in [Q]$, there exists constants $c_q,C_q >0$ such that we have with high probability 
\begin{equation}\label{eq:property_D_diagonal}
 0 < c_q \frac{r_q^2}{d_q} = c_q d^{\kappa_q} \leq \min_{i \in [N]} \lambda_{\min} (\bD_{ii}^{qq}) \leq \max_{i \in [N]} \lambda_{\max} (\bD_{ii}^{qq}) \leq C_q  \frac{r_q^2}{d_q} = C_q d^{\kappa_q} < \infty,
\end{equation}
as $d \to \infty$.
\item For any $q \neq q' \in [Q]$, we have 
\begin{equation}\label{eq:property_D_non_diagonal}
\max_{i \in [N]} \sigma_{\max} ( \bD_{ii}^{qq'} ) = o_{d,\P} (r_q r_{q'} / \sqrt{d_q d_{q'}}).
\end{equation}
\end{itemize}
\end{enumerate}
\end{proposition}

The proofs of these two propositions are provided in the next sections. 

From Proposition \ref{prop:kernel_lower_bound_NTK_NU}, we can upper bound Eq.~\eqref{eqn:bound_VUV_NTK_NU} as follows
\begin{equation}\label{eq:decomposition_upper_bound_2_terms}
\begin{aligned}
\bV_{\cQ^c}^\sT \bU^{-1} \bV_{\cQ^c}  \preceq & \bV_{\cQ^c}^\sT (\bD + \bDelta)^{-1} \bV_{\cQ^c} =  \bV_{\cQ^c}^\sT \bD^{-1} \bV_{\cQ^c} - \bV_{\cQ^c}^\sT \bD^{-1} \bDelta (\bD + \bDelta )^{-1} \bV_{\cQ^c}. 
\end{aligned}
\end{equation}

Let us fix $\eps_0>0$ as prescribed in Lemma \ref{lem:upper_bound_A}. We decompose the vector $\bV_{i,\cQ^c} =  (\bV^\pq_{i,\cQ^c})_{q \in [Q]}$ where
\[
\begin{aligned}
\bV^\pq_{i, \cQ^c } =& \E_{\bx}[[\proj_{\cQ^c} f_d](\bx) \sigma'(\< \btheta_i, \bx\>/R) \bx^\pq].
\end{aligned}
\]
We denote $\bV^\pq_{\cQ^c} = (\bV^\pq_{1, \cQ^c } , \ldots , \bV^\pq_{N, \cQ^c } ) \in \R^{d_q N}$. 
From Proposition \ref{prop:expected_V_NTK_NU}, we have 
\[
\frac{d_q}{r_q^2} \E_{\bTheta_{\eps_0}} [ \| \bV^\pq_{\cQ^c} \|_2^2 ] \leq  \lbb \max_{\bk \in \cQ^c}  N d^{-\kappa_q} B(\bd,\bk)^{-1} \E_{\btau_{\eps_0}} [A_{\btau,\bk}^{\pq}  ]    \rbb \cdot \| \proj_{\cQ^c}  f_{d} \|_{L^2}^2 \, .
\]
Hence, using the upper bounds on $A_{\btau,\bk}^{\pq}$ in Lemma \ref{lem:upper_bound_A}, we get for $\bk \in \cQ^c$ with $k_q > 0$:
\[
N d^{-\kappa_q} B(\bd,\bk)^{-1} \E_{\tau_{\eps_0}} [A_{\btau,\bk}^{\pq}  ] \leq C N d^{-\kappa_q}  d^{-\gamma + \min_{q \in S(\bk)} \kappa_q} = o_{d} (1),
\]
where we used that $N = o_d(d^\gamma)$ and $\kappa_q \geq \min_{q \in S(\bk)} \kappa_q$ (we have $k_q > 0$ and therefore $q \in S(\bk)$ by definition). Similarly for $\bk \in \cQ^c$ with $k_q = 0$:
\[
N d^{-\kappa_q} B(\bd,\bk)^{-1} \E_{\btau_{\eps_0}} [A_{\btau,\bk}^{\pq}  ] \leq C N d^{\eta_q+\kappa_q - \xi}  d^{-\gamma -(\xi - \min_{q \in S(\bk)} \kappa_q)} = o_{d} (1),
\]
where we used that by definition of $\xi$ we have $ \eta_q + \kappa_q \leq \xi$ and $\min_{q \in S(\bk)} \kappa_q \leq \xi$. We deduce that 
\[
\frac{d_q}{r_q^2} \E_{\bTheta_{\eps_0}} [ \| \bV^\pq_{\cQ^c} \|_2^2 ] = o_d (1) \cdot \| \proj_{\cQ^c}  f_{d} \|_{L^2}^2,
\]
and therefore by Markov's inequality that 
\begin{equation}\label{eq:bound_E_pq_cQc}
\frac{d_q}{r_q^2 }  \| \bV^\pq_{\cQ^c} \|_2^2  = o_{d,\P}(1) \cdot \| \proj_{\cQ^c}  f_{d} \|_{L^2}^2.
\end{equation}

Notice that the properties \eqref{eq:property_D_diagonal} and \eqref{eq:property_D_non_diagonal} imply that there exists $c >0$ such that with high probability $\lambda_{\min} ( \bD) \geq \min_{i \in [N]} \lambda_{\min} ( \bD_{ii}) \geq c$. In particular, we deduce that $\| (\bD + \bDelta )^{-1} \|_{\op} \leq c^{-1}/2$ with high probability. Combining these bounds and Eq.~\eqref{eq:bound_E_pq_cQc} and recalling that $\| \bDelta \|_{\op} = o_{d,\P}(d^{-\max_{q \in [Q]} \kappa_q})$ show that
\begin{equation}\label{eq:VUV_first_bound}
\begin{aligned}
| \bV_{\cQ^c}^\sT \bD^{-1} \bDelta (\bD + \bDelta )^{-1} \bV_{\cQ^c} | \leq & \| \bD^{-1} \|_{\op}  \| (\bD + \bDelta )^{-1} \|_{\op} \sum_{q \in [Q]} \| \bDelta \|_{\op} \| \bV^\pq_{\cQ^c} \|_2^2  \\
= & o_{d,\P} (1) \cdot \| \proj_{\cQ^c} f_d \|_{L^2}^2.
\end{aligned}
\end{equation}

We are now left to show $\bV_{\cQ^c}^\sT \bD^{-1} \bV_{\cQ^c}  = o_{d,\P} (1)$. For each $i \in [N]$, denote $\bB_{ii} = \bD^{-1}_{ii}$ and notice that we can apply Lemma \ref{lem:bound_inv_block_matrix} to $\bB_{ii}$ and get
\[
 \max_{i \in [N]} \| \bB^{qq}_{ii} \|_{\op} = O_{d,\P} \lp \frac{d_q}{ r_q^2} \rp , \qquad  \max_{i \in [N]} \| \bB^{qq'}_{ii} \|_{\op} = o_{d,\P} \lp  \frac{\sqrt{d_q d_{q'}}}{r_q r_{q'}} \rp.
\]
Therefore,
\begin{equation}\label{eq:decomposition_first_term}
\begin{aligned}
\bV_{\cQ^c}^\sT \bD^{-1} \bV_{\cQ^c} =& \sum_{i \in [N]} \sum_{q,q' \in [Q]}  (\bV^\pq_{i,\cQ^c})^\sT \bB_{ii}^{qq'} \bV^\pqp_{i,\cQ^c} \\
\leq & \sum_{q,q' \in [Q]} O_{d,\P} (1) \cdot \lp \frac{d_q}{r_q^2} \| \bV^\pq_{\cQ^c} \|_2^2 \rp^{1/2} \lp \frac{d_{q'}}{r_{q'}^2} \| \bV^\pqp_{\cQ^c} \|_2^2 \rp^{1/2} .
\end{aligned}
\end{equation}

Using Eq.~\eqref{eq:bound_E_pq_cQc} in Eq.~\eqref{eq:decomposition_first_term}, we get 
\begin{equation}\label{eq:VUV_second_bound}
\bV_{\cQ^c }^\sT \bD^{-1} \bV_{\cQ^c}  = o_{d,\P} (1)\cdot \| \proj_{\cQ^c} f_d \|_{L^2}^2.
\end{equation}
Combining Eq.~\eqref{eq:VUV_first_bound} and Eq.~\eqref{eq:VUV_second_bound} yields Eq.~\eqref{eqn:bound_VUV_NTK_NU}. This proves the theorem.

\subsection{Proof of Proposition \ref{prop:expected_V_NTK_NU}}

\begin{proof}[Proof of Proposition \ref{prop:expected_V_NTK_NU}]

Let us consider $\eps_0 > 0$ as prescribed in Lemma \ref{lem:upper_bound_A}. We have for $q \in [Q]$
\[
\begin{aligned}
\cE^\pq_{\cQ^c, \eps_0} =  & \E_{\btheta_\eps}\Big[ \< \E_\bx[ [ \proj_{\cQ^c} f_d] (\bx) \sigma'(\<\btheta, \bx\>/R) \bx^\pq], \E_\by[ [\proj_{\cQ^c} f_d] (\by) \sigma'(\<\btheta, \by\>/R) \by^\pq] \>\Big] \\
= & \E_{\btau_{\eps_0}}\Big[ \E_{\obx, \oby} \Big[[\proj_{\cQ^c} \of_d] (\obx) [\proj_{\cQ^c} \of](\oby) H^\pq_{\btau} ( \obx, \oby) \Big] \Big],
\end{aligned}
\]
where we denoted $H^\pq_\btau$ the kernel given by 
\[
\begin{aligned}
& H^\pq_{\btau} (\obx , \oby) \\
= &  \E_{\obtheta} \lbb \sigma'_{\bd,\btau} \lp \lb \< \obtheta^\pq , \obx^\pq \> / \sqrt{d_q} \rb_{q \in [Q]} \rp  \sigma'_{\bd,\btau} \lp \lb \< \obtheta^\pq , \oby^\pq \> / \sqrt{d_q} \rb_{q \in [Q]} \rp \rbb \< \bx^\pq,  \by^\pq \>, 
\end{aligned}
\]
Then we have
\begin{equation}\label{eqn:expression_Kernel_NTK}
H^\pq_{\btau} (\obx , \oby) = \sum_{\bk \in \posint^Q}  A_{\btau,\bk}^{\pq} Q^{\bd}_{\bk} \lp \lb \< \obx^\pq , \oby^\pq \> \rb_{q \in [Q]} \rp , 
\end{equation}
where $A_{\btau,\bk}^{\pq}$ is given in Lemma \ref{lem:formula_NT_Kernel}. Hence we get
\[
\begin{aligned}
\cE^\pq_{\cQ^c, \eps_0} = &\E_{\btau_{\eps_0}} [ \E_{\obx , \oby} [ P_{\cQ^c} \of_{d} ( \obx) P_{\cQ^c}  \of_{d} (\oby ) H^\pq_{\btau} (\obx, \oby) ] ] \\
= & \sum_{ \bk \in \posint^Q} \E_{\btau_{\eps_0}} [A_{\btau,\bk}^{\pq} ] \E_{\obx , \oby} \Big[ Q^{\bd}_{\bk} \lp \lb \< \obx^\pq , \oby^\pq \> \rb_{q \in [Q]} \rp  [P_{\cQ^c}  \of_{d} ] ( \obx) [ P_{\cQ^c}  \of_{d} ] (\oby ) \Big] .
\end{aligned}
\]
We have
\[
\begin{aligned}
&\E_{\obx , \oby} \Big[ Q^{\bd}_{\bk} \lp \lb \< \obx^\pq , \oby^\pq \> \rb_{q \in [Q]} \rp  [P_{\cQ^c}  \of_{d} ] ( \obx) [ P_{\cQ^c}  \of_{d} ] (\oby ) \Big] \\
=& \sum_{\bl , \bl' \in \cQ^c} \sum_{\bs \in [B(\bd, \bl)] } \sum_{\bs' \in [B(\bd , \bl')]} \lambda^\bd_{\bl,\bs} (\of_d) \lambda^\bd_{\bl' , \bs'} (\of_d ) \E_{\obx , \oby} \Big[ Q^{\bd}_{\bk} \lp \lb \< \obx^\pq , \oby^\pq \> \rb_{q \in [Q]} \rp  Y^\bd_{\bl,\bs} ( \obx) Y^\bd_{\bl' , \bs'} (\oby ) \Big]  \\
= &  \delta_{\bk \in \cQ^c} \sum_{\bs \in [B(\bd , \bk)]} \frac{\lambda^\bd_{\bk,\bs} (\of_d)^2 }{B(\bd,\bk)},
\end{aligned}
\]
where we used in the third line
\[
\begin{aligned}
 & \E_{\obx , \oby} \Big[ Q^{\bd}_{\bk} \lp \lb \< \obx^\pq , \oby^\pq \> \rb_{q \in [Q]} \rp  Y^\bd_{\bl,\bs} ( \obx) Y^\bd_{\bl' , \bs'} (\oby ) \Big] \\
 = &  \E_{\oby} \Big[ \E_{\obx} \Big[ Q^{\bd}_{\bk} \lp \lb \< \obx^\pq , \oby^\pq \> \rb_{q \in [Q]} \rp  Y^\bd_{\bl,\bs} ( \obx) \Big] Y^\bd_{\bl' , \bs'} (\oby ) \Big] \\
 =& \frac{\delta_{\bk , \bl} }{B ( \bd , \bk) }  \E_{\oby} \Big[Y^\bd_{\bk , \bs} ( \oby ) Y^\bd_{\bl' , \bs'} (\oby ) \Big] \\
 = &  \frac{\delta_{\bk , \bl}  \delta_{\bk , \bl'} \delta_{\bs , \bs'} }{B ( \bd , \bk) }.
\end{aligned}
\]
We conclude that
\[
\begin{aligned}
\cE^\pq_{\cQ^c, \eps_0} = & \sum_{\bk \in \posint^Q} \E_{\btau_{\eps_0}} [A_{\btau,\bk}^{\pq}  ]   \sum_{\bs \in [B(\bd , \bk)]} \frac{\lambda^\bd_{\bk,\bs} (\of_d)^2 }{B(\bd,\bk)} \delta_{\bk \in \cQ^c}   \\
= & \sum_{\bk \in \cQ^c}  \frac{\E_{\btau_{\eps_0}} [A_{\btau,\bk}^{\pq}  ] }{B(\bd,\bk)} \| P_{\bk} f_d \|^2_{L^2} \\
\le & \lbb \max_{\bk \in \cQ^c}  B(\bd,\bk)^{-1} \E_{\btau_{\eps_0}} [A_{\btau,\bk}^{\pq}  ]    \rbb \cdot \| \proj_{\cQ^c}  f_{d} \|_{L^2}^2.
\end{aligned}
\]
\end{proof}

\begin{lemma}\label{lem:formula_NT_Kernel}
Let $\sigma$ be a weakly differentiable activation function with weak derivative $\sigma'$. For a fixed $\btau \in \R_{\ge 0}^Q$, define the kernels for $q \in [Q]$,
\[
\begin{aligned}
& H^\pq_{\btau} ( \obx , \oby )\\ =  &\frac{r_q^2}{d_q}   \E_{\obtheta} \lbb \sigma'_{\bd,\btau} \lp \lb \< \obtheta^\pq , \obx^\pq \> / \sqrt{d_q} \rb_{q \in [Q]} \rp  \sigma'_{\bd,\btau} \lp \lb \< \obtheta^\pq , \oby^\pq \> / \sqrt{d_q} \rb_{q \in [Q] } \rp \rbb \< \obx^\pq,  \oby^\pq \>.
\end{aligned}
\]
Then, we have the following decomposition in terms of product of Gegenbauer polynomials,
\[
H^\pq_{\btau} ( \obx , \oby ) = \sum_{\bk \in \posint^2} A^{\pq}_{\btau,\bk}  Q^{\bd}_{\bk} \lp \lb \< \obx^\pq , \oby^\pq \>  \rb_{q \in [Q] } \rp , 
\]
where 
\[
A^{\pq}_{\btau,\bk} = r_q^2 \cdot [ t_{d_q, k_q - 1} \lambda^{\bd}_{\bk_{q-}}(\sigma'_{\bd,\btau})^2 B(\bd, \bk_{q-} ) + s_{d_q, k_q+1} \lambda^{\bd}_{\bk_{q+}}(\sigma'_{\bd,\btau})^2 B(\bd, \bk_{q+} )],
\]
with $\bk_{q+} = (k_1 , \ldots , k_q+1 , \ldots , k_Q)$ and $\bk_{q-} = (k_1 , \ldots , k_q-1 , \ldots , k_Q)$, and 
\[
s_{d, k} = \frac{k}{2k + d - 2}, \qquad t_{d, k} = \frac{k + d - 2}{2k + d - 2},
\]
with the convention $t_{d,-1} = 0$.
\end{lemma}

\begin{proof}[Proof of Lemma \ref{lem:formula_NT_Kernel}]

Recall the decomposition of $\sigma'$ in terms of tensor product of Gegenbauer polynomials,
\[
\begin{aligned}
\sigma' (\<\btheta , \bx \>/ R ) = & \sum_{\bk \in \posint^Q }  \lambda^{\bd}_{\bk} (\sigma_{\bd,\btau} '  ) B(\bd,\bk) Q^{\bd}_{\bk} \lp \lb \< \obtheta^\pq , \obx^\pq \> \rb_{q \in [Q]} \rp , \\
\lambda^{\bd}_{\bk} (\sigma_{\bd,\btau} ' ) = & \E_{\obx} \lbb \sigma_{\bd,\btau} ' \lp \ox_1^{(1)} , \ldots , \ox_1^{(Q)} \rp Q^{\bd}_{\bk} \lp  \sqrt{d_1} \ox_1^{(1)} , \ldots , \sqrt{d_q} \ox_1^{(Q)} \rp  \rbb ,
\end{aligned}
\] 
Injecting this decomposition into the definition of $H^\pq_{\btau}$ yields
\[
\begin{aligned}
 & H^\pq_{\btau} ( \obx , \oby ) \\
 =& \frac{r_q^2}{d_q}   \E_{\obtheta} \lbb \sigma'_{\bd,\btau} \lp \lb \< \obtheta^\pq , \obx^\pq \> / \sqrt{d_q} \rb_{q \in [Q]} \rp  \sigma'_{\bd,\btau} \lp \lb \< \obtheta^\pq , \oby^\pq \> / \sqrt{d_q} \rb_{q \in [Q] } \rp \rbb \< \obx^\pq,  \oby^\pq \>  \\
 = &  \frac{r^2_q}{d_q}  \sum_{\bk,\bk' \in \posint^Q}  \lambda_{\bk}^{\bd} ( \sigma_{\bd,\btau}')   \lambda_{\bk'}^{\bd} ( \sigma_{\bd,\btau}' ) B(\bd , \bk) B(\bd , \bk' ) \times\\
 & \phantom{AAAAAAAAAA} \E_{\overline \btheta} \lbb  Q^{\bd}_{\bk} \lp \lb \< \obtheta^\pq , \obx^\pq \>  \rb_{q \in [Q]} \rp Q^{\bd}_{\bk'} \lp \lb \< \obtheta^\pq , \oby^\pq \>  \rb_{q \in [Q]} \rp \rbb \< \obx^\pq , \oby^\pq \>.
\end{aligned}
\]
Recalling Eq.~\eqref{eq:ProductProdGegenbauer}, we have
\[
\begin{aligned}
\E_{\overline \btheta} \lbb  Q^{\bd}_{\bk} \lp \lb \< \obtheta^\pq , \obx^\pq \>  \rb_{q \in [Q]} \rp Q^{\bd}_{\bk'} \lp \lb \< \obtheta^\pq , \oby^\pq \>  \rb_{q \in [Q]} \rp \rbb  = & \delta_{\bk , \bk'}  \frac{Q^{\bd}_{\bk} \lp \lb \< \obx^\pq , \oby^\pq \>  \rb_{q \in [Q] } \rp  }{B (\bd , \bk)}.
\end{aligned}
\]
Hence,
\[
\begin{aligned}
&  H^\pq_{\btau} ( \obx , \oby ) \\ =& \frac{r_q^2}{d_q} \sum_{\bk \in \posint^Q}  \lambda_{\bk}^{\bd} ( \sigma_{\bd,\btau}')^2 B(\bd,\bk ) Q^\bd_\bk \lp \lb \< \obx^\pq , \oby^\pq \> \rb_{q \in [Q]} \rp \< \obx^\pq , \oby^\pq \> \\
  = & r_q^2 \sum_{\bk \in \posint^Q}  \lambda_{\bk}^{\bd} ( \sigma_{\bd,\btau}')^2 B(\bd,\bk ) \lbb Q^{(d_q)}_{k_q} ( \< \obx^\pq , \oby^\pq \>  )  \< \obx^\pq , \oby^\pq \>/d_q \rbb  \prod_{q' \neq q } Q^\bd_{k_{q'}} ( \< \obx^{(q')} , \oby^{(q')} \> ) . 
\end{aligned}
\]
By the recurrence relationship for Gegenbauer polynomials \eqref{eq:RecursionG}, we have
\[
\frac{t}{d_q } Q^{(d_q)}_{k_q} (t) = s_{d_q , k_q} Q^{(d_q)}_{k_q-1} (t) + t_{d_q, k_q} Q^{(d_q)}_{k_q+1} (t),
\]
where (we use the convention $t_{d_q,-1} = 0$)
\[
s_{d_q , k_q } = \frac{k_q}{2k_q+d_q-2}, \qquad t_{d_q , k_q } = \frac{k_q + d_q - 2}{2k_q + d_q - 2}.
\]
Hence we get,
\[
\begin{aligned}
 & H^\pq_{\btau} ( \obx , \oby )\\
  = & r_q^2 \sum_{\bk \in \posint^Q}  \lambda_{\bk}^{\bd} ( \sigma_{\bd,\btau}')^2 B(\bd,\bk ) \lbb Q^{(d_q)}_{k_q} ( \< \obx^\pq , \oby^\pq \>  )  \< \obx^\pq , \oby^\pq \>/d_q \rbb  \prod_{q' \neq q } Q^\bd_{k_{q'}} ( \< \obx^{(q')} , \oby^{(q')} \> )  \\
 = & r_q^2 \sum_{\bk \in \posint^Q}  \lambda_{\bk}^{\bd} ( \sigma_{\bd,\btau}')^2 B(\bd,\bk) \Big[ s_{d_q , k_q} Q^{(d_q)}_{k_q-1} (\< \obx^\pq ,  \oby^\pq \>  ) + t_{d_q, k_q} Q^{(d_q)}_{k_q+1} (\< \obx^\pq ,  \oby^\pq \>  )  \Big] \\
 & \phantom{AAAAAAAAAAAAAAAAAAAAAAA} \times \prod_{q' \neq q } Q^\bd_{k_{q'}} ( \< \obx^{(q')} , \oby^{(q')} \> )     \\
 =& \sum_{\bk \in \posint^2} A^{\pq}_{\btau,\bk}  Q^{\bd}_{\bk} \lp \lb \< \obx^\pq , \oby^\pq \>  \rb_{q \in [Q] } \rp,
\end{aligned}
\]
where we get by matching the coefficients, 
\[
A^{\pq}_{\btau,\bk} = r_q^2 \cdot [ t_{d_q, k_q - 1} \lambda^{\bd}_{\bk_{q-}}(\sigma'_{\bd,\btau})^2 B(\bd, \bk_{q-} ) + s_{d_q, k_q+1} \lambda^{\bd}_{\bk_{q+}}(\sigma'_{\bd,\btau})^2 B(\bd, \bk_{q+} )],
\]
with $\bk_{q+} = (k_1 , \ldots , k_q+1 , \ldots , k_Q)$ and $\bk_{q-} = (k_1 , \ldots , k_q-1 , \ldots , k_Q)$.
\end{proof}

\subsection{Proof of Proposition \ref{prop:kernel_lower_bound_NTK_NU}}

\subsubsection{Preliminaries}

\begin{lemma}\label{lem:prod_gegenbauer_decomposition}
Let $\psi: \R^Q \to \R$ be a function such that $\psi ( \lb \< \be_q , \cdot \> \rb_{q \in [Q]} ) \in L^2 (\PS^{\bd} , \mu_{\bd}  )$. We will consider for integers $\bi = (i_1 , \ldots , i_Q) \in \posint^Q$, the associated function $\psi^{(\bi)}$ given by:
\[
\psi^{(\bi)} \lp \ox^{(1)}_1 , \ldots ,   \ox^{(Q)}_1 \rp =\lp\ox^{(1)}_1 \rp^{i_1} \ldots \lp\ox^{(Q)}_1 \rp^{i_Q} \psi \lp \ox^{(1)}_1 , \ldots ,   \ox^{(Q)}_1 \rp .
\]
Assume that $\psi^{(\bi)} ( \lb \< \be_q , \cdot \> \rb_{q \in [Q]} )  \in L^2 (\PS^{\bd} , \mu_{\bd}  )$. Let $\lbrace \lambda_{\bk}^{\bd} (\psi) \rbrace_{\bk \in \posint^Q } $ be the coefficients of the expansion of $\psi$ in terms of the product of Gegenbauer polynomials
\[
\begin{aligned}
\psi \lp \ox^{(1)}_1 , \ldots ,   \ox^{(Q)}_1 \rp = & \sum_{\bk \in \posint^Q} \lambda_{\bk}^{\bd} (\psi) B(\bd,\bk)  Q_{\bk}^{\bd} \lp \sqrt{d_1} \ox_1^{(1)}  , \ldots ,  \sqrt{d_Q} \ox_1^{(Q)} \rp , \\
\lambda^{\bd}_{\bk} (\psi ) = & \E_{\obx} \lbb \psi \lp  \ox^{(1)}_1 , \ldots ,   \ox^{(Q)}_1 \rp Q^{\bd}_{\bk} \lp \sqrt{d_1} \ox_1^{(1)}  , \ldots ,  \sqrt{d_Q} \ox_1^{(Q)}  \rp  \rbb.
\end{aligned}
\]
Then we can write 
\[
 \psi^{(\bi)} \lp \ox^{(1)}_1 , \ldots ,   \ox^{(Q)}_1 \rp  = \sum_{\bk \in \posint^Q} \lambda_{\bk}^{\bd,\bi} (\psi) B(\bd,\bk)  Q_{\bk}^{\bd} \lp \sqrt{d_1} \ox_1^{(1)}  , \ldots ,  \sqrt{d_Q} \ox_1^{(Q)} \rp,
\]
where the coefficients $\lambda_{\bk}^{\bd,\bi} (\psi)$ are given recursively: denoting $\bi_{q+} = ( i_1 , \ldots, i_{q}+1, \ldots , i_{Q})$, if $k_q = 0$,
\[
\begin{aligned}
\lambda_{\bk}^{\bd,\bi_{q+}} ( \psi) = \sqrt{d_q} \lambda_{\bk_{q+}}^{\bd,\bi} (\psi),
\end{aligned}
\]
and for $k_q >0$,
\[
\lambda_{\bk}^{\bd,\bi_{q+}} ( \psi)  = \sqrt{d_q}\frac{k_q+d_q-2}{2k_q+d_q-2} \lambda_{\bk_{q+} }^{\bd,\bi} (\psi)  + \sqrt{d_q} \frac{k_q}{2k_q+d_q-2} \lambda_{\bk_{q-} }^{\bd,\bi} (\psi),
\]
where we recall the notations $\bk_{q+} = (k_1 , \ldots , k_q+1 , \ldots , k_Q)$ and $\bk_{q-} = (k_1 , \ldots , k_q-1 , \ldots , k_Q)$.
\end{lemma}

\begin{proof}[Proof of Lemma \ref{lem:prod_gegenbauer_decomposition}]
We recall the following two formulas for $k\ge 1$ (see Section \ref{sec:Gegenbauer}):
\[
\begin{aligned}
\frac{x}{d} Q^{(d)}_{k} (x) & = \frac{k}{2k +d - 2} Q^{(d)}_{k-1} (x) + \frac{k+d-2}{2k+d-2} Q^{(d)}_{k+1} (x), \\
B(d,k) & = \frac{2k+d-2}{k} \binom{k+d-3}{k-1}. 
\end{aligned}
\]
Furthermore, we have $Q^{(d)}_0 (x) = 1$, $Q^{(d)}_1 (x)  = x/d$ and therefore therefore 
$x Q^{(d)}_0 (x) = d Q^{(d)}_1 (x)$. Similarly to the proof of \cite[Lemma 6]{ghorbani2019linearized}, we insert these expressions in the expansion of the function $\psi$. Matching the coefficients of the expansion yields the result.
\end{proof}

Let $\bu : \S^{D-1} (\sqrt{D} ) \times \S^{D-1} (\sqrt{D}) \to \R^{D \times D}$ be a matrix-valued function defined by
\[
\bu ( \btheta_1 , \btheta_2 ) = \E_{\bx} [ \sigma ' ( \< \btheta_1 , \bx \> /R ) \sigma ' ( \< \btheta_2, \bx \> /R) \bx \bx^\sT ].
\]
We can write this function as a $Q$ by $Q$ block matrix function $\bu = ( \bu^\pqqp )_{q,q' \in [Q]}$, where $\bu^\pqqp : \S^{D-1} (\sqrt{D} ) \times \S^{D-1} (\sqrt{D}) \to \R^{d_q \times d_{q'}}$ are given by
\[
\bu^\pqqp  ( \btheta_1 , \btheta_2 )=  \E_{\bx} [ \sigma ' ( \< \btheta_1 , \bx \> /R ) \sigma ' ( \< \btheta_2, \bx \> /R) \bx^\pq (\bx^\pqp)^\sT ].
\]

We have the following lemma which is a generalization of \cite[Lemma 7]{ghorbani2019linearized}, that shows essentially the same decomposition of the matrix $\bu ( \btheta_1 , \btheta_2)$ as by integration by part if we had $\bx \sim \normal ( 0 , \id)$. 

\begin{lemma}\label{lem:decomposition_U_NT_NU}
For $q \in [Q]$, there exists functions $u_1^\pqq , u_2^\pqq , u_{3,1}^\pqq , u_{3,2}^\pqq :\S^{D-1} (\sqrt{D}) \times \S^{D-1} (\sqrt{D}) \to \R$ such that
\[
\begin{aligned}
\bu^\pqq ( \btheta_1 , \btheta_2 ) = & u^\pqq_1 (\btheta_1 , \btheta_2 ) \id_{d_q} + u^\pqq_2 ( \btheta_1 , \btheta_2 )  [ \btheta^\pq_1 (\btheta^\pq_2 )^\sT +\btheta^\pq_2  (\btheta^\pq_1)^\sT ]\\
& + u^\pqq_{3,1} ( \btheta_1 , \btheta_2 )  \btheta^\pq_1 (\btheta^\pq_1)^\sT +  u^\pqq_{3,2} ( \btheta_1 , \btheta_2 )  \btheta^\pq_2  (\btheta^\pq_2 )^\sT.
\end{aligned}
\]
For $q, q' \in [Q]$, there exists functions $u_{2,1}^\pqqp , u_{2,2}^\pqqp , u_{3,1}^\pqqp , u_{3,2}^\pqqp :\S^{D-1} (\sqrt{D}) \times \S^{D-1} (\sqrt{D}) \to \R$ such that
\[
\begin{aligned}
\bu^\pqqp ( \btheta_1 , \btheta_2 ) = &  u^\pqqp_{2,1} ( \btheta_1 , \btheta_2 )   \btheta^\pq_1 (\btheta^\pqp_2)^\sT + u^\pqqp_{2,2} ( \btheta_1 , \btheta_2 ) \btheta^\pq_2  (\btheta^\pqp_1 )^\sT \\
& + u^\pqqp_{3,1} ( \btheta_1 , \btheta_2 )  \btheta^\pq_1 (\btheta^\pqp_1 )^\sT +  u^\pqqp_{3,2} ( \btheta_1 , \btheta_2 ) \btheta^\pq_2  (\btheta^\pqp_2)^\sT.
\end{aligned}
\]
\end{lemma}

\begin{proof}[Proof of Lemma \ref{lem:decomposition_U_NT_NU}]

Denote $\gamma^\pq = \< \obtheta^\pq_1 , \obtheta^\pq_2 \>/ d_q$. Let us rotate each sphere $q \in [Q]$ such that
\begin{equation}\label{eq:basis_rotation}
\begin{aligned}
\btheta^\pq_1 = & \lp \tau^\pq_1 \sqrt{d_q} , 0 , \ldots , 0 \rp, \\
\btheta^\pq_2 = & \lp \tau^\pq_2 \sqrt{d_q} \gamma^\pq , \tau^\pq_2 \sqrt{d_q} \sqrt{1 - (\gamma^\pq)^2}, 0  , \ldots ,  0\rp.
\end{aligned}
\end{equation}

\noindent \textbf{Step 1:}  $\bu^\pqq$.

Let us start with $\bu^\pqq$. For clarity, we will denote (in the rotated basis~\eqref{eq:basis_rotation})
\[
\begin{aligned}
\alpha_1 =  \< \btheta_1 , \bx \> /R= & \sum_{q \in [Q]} \tau_1^\pq \sqrt{d_q} /R \cdot \ox^\pq_1 , \\
\alpha_2 =  \< \btheta_2 , \bx \> /R = &   \sum_{q \in [Q]}  \lbb \tau_2^\pq \sqrt{d_q} \gamma^\pq /R \cdot \ox^\pq_1   +  \tau_2^\pq \sqrt{d_q} \sqrt{1 - (\gamma^\pq)^2}  / R \cdot \ox^\pq_2   \rbb .
\end{aligned}
\]
Then it is easy to show that we can rewrite
\[
\begin{aligned}
 \bu^\pqq (\btheta_1 , \btheta_2) = \E_{\bx} [ \sigma ' ( \alpha_1 ) \sigma ' ( \alpha_2 ) \bx^\pq (\bx^\pq)^\sT = \begin{bmatrix} 
\bu^\pqq_{1:2,1:2} & \bzero  \\
\bzero & \E_{\bx}[\sigma'(\alpha_1 ) \sigma'(\alpha_2 ) (x^\pq_3)^2]  \id_{d_q-2}
\end{bmatrix},
\end{aligned}
\]
with
\[
\begin{aligned}
\bu^\pqq_{1:2,1:2}  = \begin{bmatrix} 
 \E_{\bx}[\sigma'(\alpha_1 ) \sigma'(\alpha_2 ) (x^\pq_1)^2] & \E_{\bx}[\sigma'(\alpha_1 ) \sigma'(\alpha_2 ) x^\pq_1 x^\pq_2 ]  \\
\E_{\bx}[\sigma'(\alpha_1 ) \sigma'(\alpha_2 )  x^\pq_2  x^\pq_1] & \E_{\bx}[\sigma'(\alpha_1 ) \sigma'(\alpha_2 ) (x_2^\pq)^2]
\end{bmatrix}.
\end{aligned}
\]

\noindent \textbf{Case (a):} $\btheta^\pq_1 \neq \btheta^\pq_2$.

Given any functions $u^\pqq_1 , u^\pqq_2 , u^\pqq_{3,1} , u^\pqq_{3,2} : \S^{D-1} (\sqrt{D}) \times \S^{D-1} (\sqrt{D}) \to \R$, we define
\[
\begin{aligned}
\Tilde  \bu^\pqq ( \btheta_1 , \btheta_2 ) =&  u^\pqq_1 (\btheta_1 , \btheta_2 ) \id_{d_1} + u^\pqq_2 (\btheta_1 , \btheta_2 )  [ \btheta^\pq_1 (\btheta^\pq_2 )^\sT +\btheta^\pq_2  (\btheta^\pq_1)^\sT ] \\
& + u^\pqq_{3,1} (\btheta_1 , \btheta_2 )  \btheta^\pq_1 (\btheta^\pq_1)^\sT +  u^\pqq_{3,2} (\btheta_1 , \btheta_2 )  \btheta^\pq_2  (\btheta^\pq_2 )^\sT.
\end{aligned}
\]
In the rotated basis \eqref{eq:basis_rotation}, we have 
\[
\begin{aligned}
\Tilde  \bu^\pqq (\btheta_1 , \btheta_2) = \begin{bmatrix} 
\Tilde \bu^\pqq_{1:2,1:2} & \bzero  \\
\bzero & u^\pqq_1 (\btheta_1 , \btheta_2 )   \id_{d_q-2}
\end{bmatrix},
\end{aligned}
\]
where (we dropped the dependency on $(\btheta_1 , \btheta_2)$ for clarity)
\[
\begin{aligned}
 \bu^\pqq_{11} =&u^\pqq_1  + 2 \tau^\pq_1 \tau^\pq_2 d_q \gamma^\pq u^\pqq_2  + (\tau^\pq_1)^2 d_q u^\pqq_{3,1} +  (\tau^\pq_2)^2 d_q (\gamma^\pq )^2 u^\pqq_{3,2},  \\
 \bu^\pqq_{12} =&  \tau_1^\pq \tau_2^\pq d_q \sqrt{1 - (\gamma^\pq)^2} u^\pqq_2 + (\tau^\pq_2)^2 d_q \gamma^\pq \sqrt{1 - (\gamma^\pq)^2} \, u^\pqq_{3,2} , \\
\bu^\pqq_{22} =& u^\pqq_1 + (\tau^\pq_2)^2 d_q ( 1 - (\gamma^\pq)^2 ) u^\pqq_{3,2} .
\end{aligned}
\]
We see that $\bu^\pqq$ and $\Tilde \bu^\pqq$ will be equal if and only if we have the following equalities:
\[
\begin{aligned}
\Trace ( \bu^\pqq (\btheta_1 , \btheta_2 )) = & \Trace ( \Tilde \bu^\pqq (\btheta_1 , \btheta_2 )) \\
=&d_q  u^\pqq_1 + 2 \tau^\pq_1 \tau^\pq_2 d_q \gamma^\pq u^\pqq_2 + (\tau^\pq_1)^2 d_q u^\pqq_{3,1} +  (\tau^\pq_2)^2 d_q u^\pqq_{3,2} , \\
\< \btheta^\pq_1 , \bu^\pqq (\btheta_1 , \btheta_2 ) \btheta^\pq_2 \> = & \< \btheta^\pq_1 , \Tilde \bu^\pqq (\btheta_1 , \btheta_2 ) \btheta^\pq_2 \>  \\
=&  \tau_1^\pq \tau_2^\pq d_q \gamma^\pq u^\pqq_1  + (\tau_1^\pq)^2 ( \tau^\pq_2)^2 d_q^2 (1 + (\gamma^\pq)^2) u^\pqq_2 \\
& + (\tau_1^\pq)^3 \tau_2^\pq d_q^2 \gamma^\pq  u^\pqq_{3,1} +  \tau_1^\pq (\tau_2^\pq)^3 d_q^2 \gamma^\pq  u^\pqq_{3,1} ,  \\
\< \btheta^\pq_1 , \bu^\pqq (\btheta_1 , \btheta_2 ) \btheta^\pq_1 \> = & \< \btheta^\pq_1 , \Tilde \bu^\pqq (\btheta_1 , \btheta_2 ) \btheta^\pq_1 \> \\
 =&( \tau_1^\pq)^2 d_q  u^\pqq_1  + 2 (\tau^\pq_1)^3 \tau^\pq_2 d_q^2 \gamma^\pq u^\pqq_2\\
 &  + (\tau^\pq_1)^4 d_q^2 u^\pqq_{3,1} +  ( \tau_1^\pq)^2  (\tau^\pq_2)^2 d_q^2 (\gamma^\pq )^2 u^\pqq_{3,2},  \\
\< \btheta^\pq_2 , \bu^\pqq (\btheta_1 , \btheta_2 ) \btheta^\pq_2 \> = & \< \btheta^\pq_2 , \Tilde \bu^\pqq (\btheta_1 , \btheta_2 ) \btheta^\pq_2 \> \\
 = & (\tau^\pq_2)^2 d_q u^\pqq_1  + 2 \tau_1^\pq (\tau_2^\pq)^3 d^2_q \gamma^\pq u^\pqq_2 \\
 & + (\tau_1^\pq)^2 ( \tau_2^\pq)^2 d_q^2 (\gamma^\pq)^2 u^\pqq_{3,1} + ( \tau^\pq_2)^4 d_q^2 u^\pqq_{3,2}.
\end{aligned}
\]

Hence $\Tilde \bu^\pqq = \bu^\pqq$ if and only if
\begin{equation}\label{eq:forumula_uzz}
\begin{bmatrix} 
u^\pqq_1 \\
u^\pqq_2  \\
u^\pqq_{3,1}   \\
u^\pqq_{3,2}
\end{bmatrix} = d_q^{-1} (\bM^\pqq)^{-1} \times \begin{bmatrix} 
\Trace ( \bu^\pqq (\btheta_1 , \btheta_2 )) \\
\< \btheta^\pq_1 , \bu^\pqq (\btheta_1 , \btheta_2 ) \btheta^\pq_2 \>\\
\< \btheta^\pq_1 , \bu^\pqq (\btheta_1 , \btheta_2 ) \btheta^\pq_1 \> \\
\< \btheta^\pq_2 , \bu^\pqq (\btheta_1 , \btheta_2 ) \btheta^\pq_2 \>
\end{bmatrix} ,
\end{equation}
where
\[
\bM^\pqq = \begin{bmatrix} 
1 & 2 \tau^\pq_1 \tau^\pq_2  \gamma^\pq & (\tau^\pq_1)^2 & (\tau^\pq_2)^2 \\
\tau_1^\pq \tau_2^\pq  \gamma^\pq & (\tau_1^\pq)^2 ( \tau^\pq_2)^2 d_q (1 + (\gamma^\pq)^2) & (\tau_1^\pq)^3 \tau_2^\pq d_q \gamma^\pq  &  \tau_1^\pq (\tau_2^\pq)^3 d_q \gamma^\pq  \\
( \tau_1^\pq)^2   & 2 (\tau^\pq_1)^3 \tau^\pq_2 d_q \gamma^\pq  & (\tau^\pq_1)^4 d_q  & ( \tau_1^\pq)^2  (\tau^\pq_2)^2 d_q (\gamma^\pq )^2  \\
(\tau^\pq_2)^2  &  2 \tau_1^\pq (\tau_2^\pq)^3 d_q \gamma^\pq & (\tau_1^\pq)^2 ( \tau_2^\pq)^2 d_q (\gamma^\pq)^2  & ( \tau^\pq_2)^4 d_q
\end{bmatrix}
\]
is invertible almost surely (for $\tau^\pq_1 , \tau^\pq_2 \neq 0$ and $\gamma^\pq \neq 1$).

\noindent \textbf{Case (b):} $\btheta^\pq_1 = \btheta^\pq_2$.

Similarly, for some fixed $\alpha$ and $\beta$, we define
\[
\Tilde  \bu^\pqq ( \btheta_1 , \btheta_1 ) = \alpha \id_{d_q} + \beta \btheta^\pq_1  (\btheta^\pq_1)^\sT.
\]
Then $\bu^\pqq ( \btheta_1 , \btheta_1 )$ and $\Tilde \bu^\pqq ( \btheta_1 , \btheta_1 )$ are equal if and only if 
\[
\begin{bmatrix} 
\alpha \\
\beta
\end{bmatrix} = d_q^{-1} ( \bM^\pqq_{\|} )^{-1} \times \begin{bmatrix} 
\Trace ( \bu^\pqq ( \btheta_1 , \btheta_1 ) ) \\
\< \btheta_1^\pq ,\bu^\pqq ( \btheta_1 , \btheta_1 ) \btheta_1^\pq \> 
\end{bmatrix},
\]
where 
\[
 \bM^\pqq_{\|}  = \begin{bmatrix} 
1 & (\tau_1^\pq)^2 \\
(\tau_1^\pq)^2 & (\tau_1^\pq)^4 d_q
\end{bmatrix}.
\]

\noindent \textbf{Step 2:}  $\bu^\pqqp$ for $q \neq q'$.

Similarly to the two previous steps, we define for any functions $u^\pqqp_{2,1},u^\pqqp_{2,2},u^\pqqp_{3,1},u^\pqqp_{3,2}:\S^{D-1} (\sqrt{D}) \times \S^{D-1} (\sqrt{D}) \to \R$,
\[
\begin{aligned}
\Tilde \bu^\pqqp ( \btheta_1 , \btheta_2 ) = &  u^\pqqp_{2,1} (\btheta_1, \btheta_2 )   \btheta^\pq_1 (\btheta^\pqp_2)^\sT + u^\pqqp_{2,2} (\btheta_1 , \btheta_2 ) \btheta^\pq_2  (\btheta^\pqp_1 )^\sT\\
& + u^\pqqp_{3,1} (\btheta_1 , \btheta_2 )  \btheta^\pq_1 (\btheta^\pqp_1 )^\sT +  u^\pqqp_{3,2} (\btheta_1 , \btheta_2 )  \btheta^\pq_2  (\btheta^\pqp_2)^\sT.
\end{aligned}
\]
We can rewrite $\Tilde \bu^\pqqp$ as
\[
\begin{aligned}
\Tilde  \bu^\pqqp (\btheta_1 , \btheta_2) = \begin{bmatrix} 
\Tilde \bu^\pqqp_{1:2,1:2} & \bzero  \\
\bzero & \bzero
\end{bmatrix},
\end{aligned}
\]
where
\[
\begin{aligned}
\Tilde \bu^\pqqp_{11} = & u^\pqqp_{2,1} \tau_1^\pq \tau_2^\pqp \gamma^\pqp  + u^\pqqp_{2,2} \tau_2^\pq \tau_1^\pqp \gamma^\pq + u^\pqqp_{3,1} \tau_1^\pq \tau_1^\pqp + u^\pqqp_{3,2} \tau_2^\pq \tau_2^\pqp \gamma^\pq \gamma^\pqp, \\
\Tilde \bu^\pqqp_{12} = &  u^\pqqp_{2,1} \tau_1^\pq \tau_2^\pqp \sqrt{1 - (\gamma^\pqp)^2}  + u^\pqqp_{3,2} \tau_2^\pq \tau_2^\pqp \gamma^\pq \sqrt{1 - (\gamma^\pqp)^2} , \\
\Tilde \bu^\pqqp_{21} = & u^\pqqp_{2,2} \tau_2^\pq \tau_1^\pqp \sqrt{1 - ( \gamma^\pq)^2}  + u^\pqqp_{3,2} \tau_2^\pq \tau_2^\pqp \sqrt{1 - (\gamma^\pq)^2} \gamma^\pqp , \\
\Tilde \bu^\pqqp_{22} = &  u^\pqqp_{3,2} \tau_2^\pq \tau_2^\pqp \sqrt{1 - (\gamma^\pq)^2} \sqrt{1 - (\gamma^\pqp)^2} .
\end{aligned}
\]

\noindent \textbf{Case (a):} $\btheta^\pq_1 \neq \btheta^\pq_2$.

We have equality $\Tilde \bu^\pqqp = \bu^\pqqp$ if and only if 
\[
\begin{bmatrix} 
u^\pqqp_{2,1} \\
u^\pqqp_{2,2}  \\
u^\pqqp_{3,1}   \\
u^\pqqp_{3,2}
\end{bmatrix} = (d_q d_{q'} )^{-1} (\bM^\pqqp)^{-1} \times \begin{bmatrix} 
\< \btheta^\pq_1 , \bu^\pqqp (\btheta_1 , \btheta_2 ) \btheta^\pqp_1 \> \\
\< \btheta^\pq_1 , \bu^\pqqp (\btheta_1 , \btheta_2 ) \btheta^\pqp_2 \>\\
\< \btheta^\pq_2 , \bu^\pqqp (\btheta_1 , \btheta_2 ) \btheta^\pqp_1 \> \\
\< \btheta^\pq_2 , \bu^\pqqp (\btheta_1 , \btheta_2 ) \btheta^\pqp_2 \>
\end{bmatrix} ,
\]
where $\bM^\pqqp$ is given by
\[
 \begin{bmatrix} 
(\tau_1^\pq)^2 \tau_1^\pqp \tau_2^\pqp \gamma^\pqp & \tau^\pq_1 \tau^\pq_2  (\tau^\pqp_1)^2 \gamma^\pq & (\tau^\pq_1)^2 ( \tau^\pqp_1 )^2 & \tau^\pq_1 \tau^\pq_2  \tau^\pqp_1 \tau^\pqp_2 \gamma^\pq \gamma^\pqp \\
(\tau^\pq_1)^2 ( \tau^\pqp_2 )^2 & \tau^\pq_1 \tau^\pq_2  \tau^\pqp_1 \tau^\pqp_2 \gamma^\pq \gamma^\pqp & (\tau_1^\pq)^2 \tau_1^\pqp \tau_2^\pqp \gamma^\pqp & \tau^\pq_1 \tau^\pq_2  (\tau^\pqp_2)^2 \gamma^\pq \\
\tau^\pq_1 \tau^\pq_2  \tau^\pqp_1 \tau^\pqp_2 \gamma^\pq \gamma^\pqp & (\tau^\pq_2)^2 ( \tau^\pqp_1 )^2  & \tau^\pq_1 \tau^\pq_2  (\tau^\pqp_1)^2 \gamma^\pq & (\tau_2^\pq)^2 \tau_1^\pqp \tau_2^\pqp \gamma^\pqp  \\
\tau^\pq_1 \tau^\pq_2  (\tau^\pqp_2)^2 \gamma^\pq & (\tau_2^\pq)^2 \tau_1^\pqp \tau_2^\pqp \gamma^\pqp & \tau^\pq_1 \tau^\pq_2  \tau^\pqp_1 \tau^\pqp_2 \gamma^\pq \gamma^\pqp  & (\tau^\pq_2)^2 ( \tau^\pqp_2 )^2 
\end{bmatrix},
\]
which is invertible almost surely (for $\tau^\pq_1 , \tau^\pq_2 \neq 0$ and $\gamma^\pq \neq 1$).

\noindent \textbf{Case (b):} $\btheta^\pq_1= \btheta^\pq_2$.

It is straightforward to check that
\[
\bu^\pqqp ( \btheta_1 , \btheta_1 ) = \beta \btheta^\pq_1 (\btheta^\pqp_1 )^\sT,
\]
where 
\[
\beta = ( d_q d_{q'} )^{-1} ( \tau^\pq_1 \tau^\pqp_1 )^{-2} \bla \btheta^\pq_1 , \bu^\pqqp (\btheta_1 , \btheta_1) \btheta^\pqp_1 \bra.
\]
\end{proof}

\subsubsection{Proof of Proposition \ref{prop:kernel_lower_bound_NTK_NU}}

\noindent
\textbf{Step 1. Construction of the activation function $\hat \sigma$.}

Recall the definition of $\sigma_{\bd,\btau}$ in Eq.~\eqref{eq:def_sigma_d_tau_prime} and its expansion in terms of tensor product of Gegenbauer polynomials:
\[
\begin{aligned}
\sigma' (\<\btheta , \bx \>/ R ) = & \sum_{\bk \in \posint^Q }  \lambda^{\bd}_{\bk} (\sigma_{\bd,\btau} ' ) B(\bd,\bk) Q^{\bd}_{\bk} \lp \lb \< \obtheta^\pq , \obx^\pq \> \rb_{q \in [Q]} \rp, \\
\lambda^{\bd}_{\bk} (\sigma_{\bd,\btau} ' ) = & \E_{\obx} \lbb \sigma'_{\bd , \btau} \lp \ox^{(1)}_1 , \ldots , \ox^{(Q)}_1 \rp  Q^{\bd}_{\bk} \lp \sqrt{d_1}  \ox^{(1)}_1 , \ldots ,  \sqrt{d_Q} \ox^{(Q)}_1 \rp \rbb.
\end{aligned}
\]
We recall the definition of $q_{\xi} = \arg \max_{q \in[Q]} \lb \eta_q + \kappa_q \rb$. Let  $l_2 > l_1 \ge 2 L + 5$ be two indices that satisfy the conditions of Assumption \ref{ass:activation_lower_upper_NT_aniso}.$(b)$ and we define $\bl_1 = ( 0 ,\ldots , 0, l_1 , 0, \ldots, 0) $ ($l_1$ at position $q_\xi$) and $\bl_2 = ( 0 , \ldots , 0, l_2 , 0, \ldots , 0)$ ($l_2$ at position $q_\xi$). Using the Gegenbauer coefficients of $\sigma'$, we define a new activation function $\hat \sigma'$ by
\begin{align}\label{eqn:hat_sigma_definition}
\hat \sigma' (\<\btheta , \bx \>/ R) =&  \sum_{\bk \in \posint^Q \setminus \lb \bl_1 , \bl_2 \rb } \lambda^{\bd}_{\bk} (\sigma_{\bd,\btau} ' ) B(\bd,\bk) Q^{\bd}_{\bk} \lp \lb \< \obtheta^\pq , \obx^\pq \> \rb_{q \in [Q]} \rp \\
& + \sum_{t = 1,2} (1 -\delta_t)\lambda^{\bd}_{\bl_t} (\sigma_{\bd,\btau} ' )  B(d_{q_{\xi}},l_t) Q^{(d_{q_{\xi}} )}_{l_t} (\< \overline \btheta^{(q_\xi)} , \obx^{(q_{\xi})} \>),
\end{align}
for some $\delta_1, \delta_2$ that we will fix later (with $|\delta_t|\le 1$).

\noindent
\textbf{Step 2. The functions $\bu, \hat \bu$ and $\bar \bu$.}

Let $\bu$ and $\hat \bu$ be the matrix-valued functions associated respectively to $\sigma'$ and $\hat \sigma'$
\begin{align}
\bu ( \btheta_1 , \btheta_2 ) & = \E_{\bx}[\sigma'(\<\btheta_1, \bx\>/ R) \sigma'(\<\btheta_2, \bx\>/R) \bx \bx^\sT]\, , \label{eq:bu} \\
\hat \bu ( \btheta_1 , \btheta_2 ) & = \E_{\bx}[\hat \sigma'(\<\btheta_1, \bx\>/R) \hat \sigma'(\<\btheta_2, \bx\>/R) \bx \bx^\sT]\, .\label{eq:hbu}
\end{align}
From Lemma \ref{lem:decomposition_U_NT_NU}, there exists functions $u^{ab}_1, u^{ab}_{2,1}, u^{ab}_{2,2},u^{ab}_{3,1},u^{ab}_{3,2}$ and $\hat u^{ab}_1, \hat u^{ab}_{2,1}, \hat u^{ab}_{2,2},\hat u^{ab}_{3,1},\hat u^{ab}_{3,2}$ (for $a,b \in [Q]$), which decompose $\bu$ and $\hat \bu$ along $\btheta_1$ and $\btheta_2$ vectors. We define $\bar \bu = \bu - \hat \bu $. Then we have the same decomposition for $\bar u^{ab}_{k,j} = u^{ab}_{k,j}  - \hat u^{ab}_{k,j} $ for $a,b \in [Q], k =1,2,3, j=1,2$.

\noindent
\textbf{Step 3. Construction of the kernel matrices.}

Let $\bU, \hat \bU , \bar \bU \in \R^{ND \times ND}$ with $i,j$-th block (for $i, j \in [N]$) given by
\begin{align}
\bU_{ij} & = \bu ( \btheta_i , \btheta_j )\, , \label{eq:bU}\\
\hat \bU_{ij} & = \hat \bu ( \btheta_i , \btheta_j )\, , \label{eq:hbU}\\
\bar \bU_{ij} & = \bar \bu ( \btheta_i , \btheta_j ) =  \bu ( \btheta_i , \btheta_j ) - \hat \bu ( \btheta_i , \btheta_j )\, . \label{eqn:bar_U_expression}
\end{align}
Note that we have $\bU = \hat \bU + \bar \bU$. By Eq. (\ref{eq:hbU}) and (\ref{eq:hbu}), it is easy to see that $\hat \bU \succeq 0$. Then we have $\bU \succeq \bar \bU$. In the following, we would like to lower bound matrix $\bar \bU$. 

We decompose $\bar \bU  $ as
\[
\bar \bU = \bD + \bDelta,
\]
where $\bD \in \R^{DN \times DN}$ is a block-diagonal matrix, with
\begin{equation}\label{eqn:expression_D_matrix}
\bD = \diag(\bar \bU_{11} , \ldots , \bar \bU_{NN}),
\end{equation}
and $\bDelta \in \R^{DN \times DN}$ is formed by blocks $\bDelta_{ij}\in\reals^{D\times D}$ for $i, j \in [n]$,
defined by
\begin{equation}\label{eqn:expression_Delta_matrix}
\bDelta_{ij} = \begin{cases}
0, & ~~~~ i = j,\\
\bar \bU_{ij}, &~~~~ i \neq j.
\end{cases}
\end{equation}
In the rest of the proof, we will prove that $ \| \bDelta \|_{\op} = o_{d,\P} (d^{- \max_{q \in [Q]} \kappa_q } )$ and the block matrix $\bD$ verifies the properties \eqref{eq:property_D_diagonal} and \eqref{eq:property_D_non_diagonal}.

\noindent
\textbf{Step 4. Prove that $ \| \bDelta \|_{\op} = o_{d,\P} (d^{- \max_{q \in [Q]} \kappa_q } )$.}

We will prove in fact that $\| \bDelta \|^2_{F} = o_{d,\P} (d^{- 2 \max_{q \in [Q]} \kappa_q } )$. For the rest of the proof, we fix $\eps_0 \in (0,1)$ and we restrict ourselves without loss of generality to the set $\cP_{\eps_0}$.

Let us start with $\overline \bu^\pqq$ for $q \in [Q]$. Denoting $\gamma_{ij}^\pq = \< \obtheta^\pq_i , \obtheta^\pq_j \>/d_q < 1$, we get, from Eq.~\eqref{eq:forumula_uzz},
\begin{align}
\begin{bmatrix} 
\bar u^\pqq_1 ( \btheta_i, \btheta_j)  \\
\bar u^\pqq_2 ( \btheta_i, \btheta_j)\\
\bar u^\pqq_{3,1} ( \btheta_i, \btheta_j) \\
\bar u^\pqq_{3,2} ( \btheta_i, \btheta_j)
\end{bmatrix} = \begin{bmatrix} 
u_1 ( \btheta_i, \btheta_j) - \hat u_1 ( \btheta_i, \btheta_j)  \\
u_2 ( \btheta_i, \btheta_j) - \hat u_2 ( \btheta_i, \btheta_j)  \\
u_{3,1} ( \btheta_i, \btheta_j) - \hat u_{3,1} ( \btheta_i, \btheta_j) \\
u_{3,2} ( \btheta_i, \btheta_j) - \hat u_{3,2} ( \btheta_i, \btheta_j) 
\end{bmatrix} 
&= d_q^{-1} (\bM^\pqq_{ij})^{-1} \times \begin{bmatrix} 
\Trace ( \bar \bu^\pqq (\btheta_1 , \btheta_2 )) \\
\< \btheta^\pq_1 , \bar \bu^\pqq (\btheta_1 , \btheta_2 ) \btheta^\pq_2 \>\\
\< \btheta^\pq_1 ,\bar \bu^\pqq (\btheta_1 , \btheta_2 ) \btheta^\pq_1 \> \\
\< \btheta^\pq_2 ,\bar \bu^\pqq (\btheta_1 , \btheta_2 ) \btheta^\pq_2 \>
\end{bmatrix}\, ,
\label{eq:u_bar_coeff}
\end{align}
where $\bM_{ij}^\pqq$ is given by
\begin{align}
 \begin{bmatrix} 
1 & 2 \tau^\pq_1 \tau^\pq_2  \gamma^\pq_{ij} & (\tau^\pq_1)^2 & (\tau^\pq_2)^2 \\
\tau_1^\pq \tau_2^\pq  \gamma^\pq_{ij} & (\tau_1^\pq)^2 ( \tau^\pq_2)^2 d_q (1 + (\gamma^\pq_{ij})^2) & (\tau_1^\pq)^3 \tau_2^\pq d_q \gamma^\pq_{ij}  &  \tau_1^\pq (\tau_2^\pq)^3 d_q \gamma^\pq_{ij}   \\
( \tau_1^\pq)^2   & 2 (\tau^\pq_1)^3 \tau^\pq_2 d_q \gamma^\pq  & (\tau^\pq_1)^4 d_q  & ( \tau_1^\pq)^2  (\tau^\pq_2)^2 d_q (\gamma^\pq_{ij} )^2  \\
(\tau^\pq_2)^2  &  \tau_1^\pq (\tau_2^\pq)^3 d_q \gamma^\pq_{ij} & (\tau_1^\pq)^2 ( \tau_2^\pq)^2 d_q (\gamma^\pq)^2  & ( \tau^\pq_2)^4 d_q
\end{bmatrix}\, . \label{eq:u_bar_coeff-Mmatrix}
\end{align}

Using the notations of Lemma \ref{lem:prod_gegenbauer_decomposition}, we get
\[
\begin{aligned}
\Trace ( \bU^\pqq_{ij} ) = &  \E_{\bx}[\sigma'(\<\btheta_i, \bx\>/R) \sigma'(\<\btheta_j, \bx\>/R) \|\bx^\pq\|_2^2 ]  \\
=& r_q^2  \sum_{\bk \in \posint^Q} \lambda^{\bd}_{\bk} (\sigma_{\bd,\btau_i} ' )  \lambda^{\bd}_{\bk} (\sigma_{\bd,\btau_j} ' )  B(\bd,\bk) Q^{\bd}_{\bk} \lp \lb \< \obtheta^\pq_i , \obtheta^\pq_j \> \rb_{q \in [Q]} \rp , \\
\< \btheta_i^\pq , \bU^\pqq_{ij} \btheta_j^\pq \> = & \E_{\bx}[\sigma'(\<\btheta_i, \bx\> /R) \<\btheta_i^\pq, \bz^\pq\> \sigma'(\<\btheta_j, \bx\> /R ) \<\btheta_j^\pq, \bx^\pq \> ] \\
=& r_q^2 \tau^\pq_i \tau^\pq_j \sum_{\bk \in \posint^Q}\lambda^{\bd,\bone_{q} }_{\bk} (\sigma_{\bd,\btau_i} ' )  \lambda^{\bd,\bone_q }_{\bk} (\sigma_{\bd,\btau_j} ' )  B(\bd,\bk) Q^{\bd}_{\bk} \lp \lb \< \obtheta^\pq_i , \obtheta^\pq_j \> \rb_{q \in [Q]} \rp ,  \\
\< \btheta_i^\pq , \bU^\pqq_{ij} \btheta_i^\pq \> = & \E_{\bx}[\sigma'(\<\btheta_i, \bx\> / R) \<\btheta_i^\pq, \bx^\pq\>^2 \sigma'(\<\btheta_j, \bx\> /R ) ] \\
=& r_q^2 (\tau^\pq_i)^2  \sum_{\bk \in \posint^Q} \lambda^{\bd,\btwo_{q}}_{\bk} (\sigma_{\bd,\btau_i} ' )  \lambda^{\bd}_{\bk} (\sigma_{\bd,\btau_j} ' )  B(\bd,\bk) Q^{\bd}_{\bk} \lp \lb \< \obtheta^\pq_i , \obtheta^\pq_j \> \rb_{q \in [Q]} \rp, \\
\< \btheta_j^\pq , \bU^\pqq_{ij} \btheta_j^\pq \> = & \E_{\bx}[\sigma'(\<\btheta_i, \bx\> / R) \sigma'(\<\btheta_j, \bx\> /R ) \<\btheta_j^\pq, \bx^\pq \>^2 ] \\
=& r_q^2 ( \tau^\pq_j )^2 \sum_{\bk \in \posint^Q} \lambda^{\bd}_{\bk} (\sigma_{\bd,\btau_i} ' )  \lambda^{\bd,\btwo_{q}}_{\bk} (\sigma_{\bd,\btau_j} ' )  B(\bd,\bk) Q^{\bd}_{\bk} \lp \lb \< \obtheta^\pq_i , \obtheta^\pq_j \> \rb_{q \in [Q]} \rp ,
\end{aligned}
\]
where we denoted $\bone_{q} = ( 0 , \ldots , 0, 1 , 0, \ldots , 0)$ (namely the $q$'th coordinate vector in $\R^Q$) and $\btwo_{q} = ( 0 , \ldots , 0, 2 , 0 , \ldots , 0) = 2 \bone_q$.

We get similar expressions for $\hat \bU_{ij} $ with $\lambda^{\bd}_{\bk} (\sigma_{\bd,\btau} ' ) $ replaced by $\lambda^{\bd}_{\bk} (\hat \sigma_{\bd,\btau} ' ) $. Because we defined $\sigma'$ and $\hat \sigma'$ by only modifying the $\bl_1$-th and $\bl_2$-th coefficients, we get 
\begin{equation}
\begin{aligned}
\Trace ( \bar \bU_{ij}^\pqq ) = & \Trace ( \bU_{ij}^\pqq  - \hat \bU_{ij}^\pqq ) \\
= & r_q^2 \sum_{t = 1, 2}  \delta_t (2 - \delta_t)   \lambda^{\bd}_{\bl_t} ( \sigma_{\bd,\btau_i}')    \lambda^{\bd}_{\bl_t} ( \sigma_{\bd,\btau_j}') B(\bd , \bl_t) Q_{\bl_t}^{\bd} \lp \lb d_q \gamma^\pq_{ij} \rb_{q \in [Q]} \rp   .
\end{aligned}
\label{eq:trace_ubar}
\end{equation}
Recalling that $\lambda^{\bd,\bone_q}_{\bk }$ only depend on $\lambda^{\bd}_{\bk - \bone_q}$ and $\lambda^{\bd}_{\bk + \bone_q}$, and $\lambda^{\bd,\btwo_q}_{\bk }$ on $\lambda^{\bd}_{\bk - \btwo_q}$, $\lambda_{\bk}^\bd$ and $\lambda^{\bd}_{\bk + \btwo_q}$, (Lemma \ref{lem:prod_gegenbauer_decomposition}), we get 
\begin{equation}
\begin{aligned}
 & \< \btheta_i^\pq ,  \bar \bU_{ij}^\pqq  \btheta_j^\pq \> \\
= &  r_q^2 \tau^\pq_i \tau^\pq_j  \sum_{t = \lb 1, 2 \rb, \bk \in \lb \bl_t \pm \bone_q\rb }  \delta_t (2 - \delta_t) \lambda^{\bd,\bone_q}_{\bk } (\sigma_{\bd,\btau_i} ' )  \lambda^{\bd,\bone_q}_{\bk } (\sigma_{\bd,\btau_j} ' )  B(\bd, \bk ) Q^{\bd}_{\bk } \lp \lb d_q \gamma^\pq_{ij} \rb_{q \in [Q]} \rp, \\
  & \< \btheta_i^\pq ,  \bar \bU_{ij}^\pqq  \btheta_i^\pq  \>\\
= &  r_q^2  (\tau^\pq_i )^2 \sum_{t \in \lbrace 1,2\rbrace , \bk \in \lbrace \bl_t , \bl_t \pm \btwo_q \rbrace } \delta_t (2 - \delta_t) \lambda^{\bd,\btwo_q }_{\bk} (\sigma_{\bd,\btau_i} ' )  \lambda^{\bd}_{\bk} (\sigma_{\bd,\btau_j} ' )  B(\bd,\bk)  Q^{\bd}_{\bk} \lp \lb d_q \gamma^\pq_{ij} \rb_{q \in [Q]} \rp, \\
 &  \< \btheta_j^\pq ,  \bar \bU_{ij}^\pqq  \btheta_j^\pq  \>\\
 \label{eq:quadratic_ubar}
= &  r_q^2  (\tau^\pq_j )^2 \sum_{t \in \lbrace 1,2\rbrace,  \bk \in \lbrace \bl_t , \bl_t \pm \btwo_q \rbrace } \delta_t (2 - \delta_t) \lambda^{\bd}_{\bk} (\sigma_{\bd,\btau_i} ' )  \lambda^{\bd,\btwo_q}_{\bk} (\sigma_{\bd,\btau_j} ' )  B(\bd,\bk) Q^{\bd}_{\bk}\lp \lb d_q \gamma^\pq_{ij} \rb_{q \in [Q]} \rp ,
\end{aligned}
\end{equation}
where we used the convention $\lambda^\bd_{\bk} ( \sigma_{\bd,\btau}' )  =0 $ if one of the coordinates verifies $k_q < 0$.

From Lemma \ref{lem:prod_gegenbauer_decomposition}, Lemma \ref{lem:convergence_proba_Gegenbauer_coeff} and Lemma \ref{lem:convergence_Gegenbauer_coeff_0_l}, we get for $t = 1,2$ and $q \neq q_\xi$:
\begin{equation}
\begin{aligned}
\lim_{(d, \btau_i , \btau_j ) \rightarrow (+\infty, \bone , \bone )} \lambda^{\bd}_{\bl_t} ( \sigma_{\bd,\btau_i}') \lambda^{\bd}_{\bl_t} ( \sigma_{\bd,\btau_j}')   B(\bd,\bl_t)   & = \frac{\mu_{l_t} (\sigma')^2}{l_t !}, \\
\label{eq:lambdaone_limit}
\lim_{(d, \btau_i , \btau_j ) \rightarrow (+\infty, \bone , \bone )}  \lambda^{\bd,\bone_q}_{\bl_t + \bone_q} (\sigma_{\bd,\btau_i} ' )  \lambda^{\bd,\bone_q}_{\bl_t + \bone_q} (\sigma_{\bd,\btau_j} ' )  B(\bd,\bl_t + \bone_q) & = \frac{\mu_{l_t} (\sigma')^2}{l_t !}, \\
\lim_{(d, \btau_i , \btau_j ) \rightarrow (+\infty, \bone , \bone )} \lambda^{\bd,\btwo_q}_{\bl_t} (\sigma_{\bd,\btau_i} ' )  \lambda^{\bd}_{\bl_t} (\sigma_{\bd,\btau_j} ' )  B(\bd,\bl_t) & = \frac{\mu_{l_t} (\sigma')^2}{l_t!},\\
\lim_{(d, \btau_i , \btau_j ) \rightarrow (+\infty, \bone , \bone )}  \lambda^{\bd,\btwo_q}_{\bl_t + \btwo_q} (\sigma_{\bd,\btau_i} ' )  \lambda^{\bd}_{\bl_t + \btwo_q} (\sigma_{\bd, \btau_j} ' )  B(\bd,\bl_t + \btwo_q) & =0,
\end{aligned}
\end{equation}
while for $q = q_\xi$ and $u \in \lb -1 , 1 \rb$,
\begin{equation}\label{eq:gegenbauer_conv_qxi_1}
\begin{aligned}
& \lim_{(d, \btau_i , \btau_j ) \rightarrow (+\infty, \bone , \bone )}  \lambda^{\bd,\bone_{q_\xi}}_{\bl_t + u \bone_{q_{\xi}} } (\sigma_{\bd,\btau_i} ' ) [ B ( \bd, \bl_t + u\bone_{q_{\xi}} ) (l_t + u ) !]^{1/2} 
\\
= & \mu_{l_t + u + 1} ( \sigma ') + (l_t + u ) \mu_{l_t + u -1} (\sigma ') , 
\end{aligned}
\end{equation}
and for $v \in \lb -2 , 0 ,2 \rb$,
\begin{equation}\label{eq:gegenbauer_conv_qxi_2}
\begin{aligned}
& \lim_{(d, \btau_i , \btau_j ) \rightarrow (+\infty, \bone , \bone )}  \lambda^{\bd,\btwo_{q_\xi}}_{\bl_t + v  \bone_{q_{\xi}} } (\sigma_{\bd,\btau_i} ' ) [ B ( \bd, \bl_t + v \bone_{q_{\xi}} ) (l_t + v) !]^{1/2} \\
= & \mu_{l_t +v + 2} ( \sigma ') + (2 l_t + 2v +1) \mu_{l_t +v} (\sigma ') + (l_t+v) ( l_t + v - 1) \mu_{l_t + v-2} (\sigma ' ).
\end{aligned}
\end{equation}

From Lemma \eqref{lem:gegenbauer_coefficients}, we recall that the coefficients of the $k$-th Gegenbauer polynomial $Q_k^{(d)} ( x ) = \sum_{s=0}^k p^{(d)}_{k,s} x^s$
satisfy
\begin{align}
p^{(d)}_{k, s} = O_d (d^{-k/2 -s/2})\, .
\end{align}
Furthermore, Lemma \ref{lem:random_matrix_bound} shows that $\max_{i \neq j} |\< \overline \btheta_i^\pq , \overline \btheta_j^\pq \> | = O_{d,\P} (\sqrt{d_q \log d_q} )$. We deduce that
\begin{equation}
\begin{aligned}
\label{eq:sup_gegenbauer_polynomials}
\max_{i \neq j} \vert Q_{k_q}^{(d_q)} ( \< \overline \btheta_i^\pq , \overline \btheta_j^\pq \> ) \vert = & \Tilde O_{d,\P} ( d_q^{-k_q/2})
\end{aligned}
\end{equation}
Plugging the estimates \eqref{eq:lambdaone_limit} and \eqref{eq:sup_gegenbauer_polynomials} into
Eqs.~\eqref{eq:trace_ubar} and \eqref{eq:quadratic_ubar}, we obtain that
\begin{equation}
\begin{aligned}
& \max_{i \neq j} \Big\lbrace \big|\Trace (\bar \bU_{ij}^\pqq  )\big|  , \;\big|\< \btheta_i^\pq ,  \bar \bU_{ij}^\pqq  \btheta_j^\pq \>\big|   ,\;
\big|\< \btheta_i^\pq ,  \bar \bU_{ij}^\pqq \btheta_i^\pq \>,\;
\big|\< \btheta_j^\pq ,  \bar \bU_{ij}^\pqq \btheta_j^\pq \> \big| \Big\rbrace \\
= & \Tilde O_{d,\P} ( d^{2 \xi} d^{ - \eta_q   l_1 /2}).
\label{eq:boundUbar}
\end{aligned} 
\end{equation}
From Eq.~\eqref{eq:u_bar_coeff-Mmatrix}, using the fact that $\max_{i\neq j} |\gamma_{ij}^\pq|  = O_{d,\P} (\sqrt{(\log d_q)/d_q})$ and Cramer's rule for matrix inversion, it 
is easy to see that 
\begin{align}\label{eqn:M_matrix_bound}
\max_{i\neq j}\max_{l,k \in [4]}\big|((\bM_{ij}^\pqq)^{-1})_{lk}\big| = O_{d,\P} (1)\, .
\end{align}
We deduce from \eqref{eq:boundUbar}, \eqref{eq:u_bar_coeff} and \eqref{eqn:M_matrix_bound} that for $a \in [3], b \in [2]$,
\begin{equation}\label{eqn:bound_bar_u}
\max_{i \neq j } \lbrace \vert \bar u_{a,b}^\pqq (\btheta^\pq_i, \btheta^\pq_j ) \vert \rbrace = \Tilde O_{d,\P} ( d^{2 \xi } d^{- \eta_q l_1/2}). 
\end{equation}
As a result, combining Eq. (\ref{eqn:bound_bar_u}) with Eq. (\ref{eqn:bar_U_expression}) in the expression of $\overline u^\pqq$ given in Lemma \ref{lem:decomposition_U_NT_NU}, we get
\[
\begin{aligned}
 & \max_{i\neq j}\| \bar \bU_{ij}^\pqq \|_{F}^2 \\
 =& \max_{i\neq j} \| \bar u^\pqq_1  \id_{d_q} + \bar u^\pqq_2  [ \btheta^\pq_i (\btheta^\pq_j )^\sT +\btheta^\pq_j  (\btheta^\pq_i)^\sT ]+ \bar u^\pqq_{3,1} \btheta^\pq_i (\btheta^\pq_i)^\sT +  \bar u^\pqq_{3,2} \btheta^\pq_j  (\btheta^\pq_j )^\sT  \|_F^2 \\
 \le &  \Tilde O_{d,\P} ( d^{6 \xi } d^{- \eta_q l_1} ).
 \end{aligned}
\]
A similar computation shows that
\[
 \max_{i\neq j}\| \bar \bU_{ij}^\pqqp \|_{F}^2 \le   \Tilde O_{d,\P} ( d^{6\xi} d^{ -\eta_q l_1}).
\]
By the expression of $\bDelta$ given by \eqref{eqn:expression_Delta_matrix}, we conclude that
\[
\| \bDelta \|_{\op}^2 \le \| \bDelta \|^2_{F} = \sum_{q,q' \in [Q]}  \sum_{i,j = 1, i \neq j}^N \| \bar \bU_{ij}^\pqqp   \|_{F}^2 = \Tilde O_{d,\P} ( N^2 d^{6 \xi - \eta_q l_1 }) .
\]
By assumption, $ N = o_d ( d^\gamma ) $. Hence, since by assumption $\eta_q l_1 \ge 2\gamma +7\xi $, we deduce that $ \| \bDelta \|_{\op} = o_{d,\P} (d^{-\xi}) = o_{d,\P} (d^{-\max_{q \in [Q]} \kappa_q } )$.

\noindent
\textbf{Step 5. Checking the properties of matrix $\bD$.}

By Lemma \ref{lem:decomposition_U_NT_NU}, we can express $\bar \bU_{ii}$ as a block matrix with
\[
\bar \bU_{ii}^\pqq = \alpha^\pq \id_{d_q} + \beta^\pq \btheta_i^\pq (\btheta_i^\pq)^\sT, \qquad \bar \bU_{ii}^\pqqp = \beta^\pqqp \btheta_i^\pq (\btheta_i^\pqp)^\sT,
\]
with coefficients given by
\begin{equation}\label{eq:u_bar_cst}
\begin{aligned}
\begin{bmatrix} 
\alpha^\pq \\
\beta^\pq
\end{bmatrix} =&  [d_q(d_q-1) (\tau^\pq_i)^4 ]^{-1} \begin{bmatrix} 
d_q (\tau^\pq_i)^4 & - (\tau^\pq_i)^2 \\
- (\tau^\pq_i)^2 & 1  
\end{bmatrix} \times \begin{bmatrix} 
\Trace ( \bar \bU_{ii}^\pqq) \\
\< \btheta_i^\pq , \bar \bU_{ii}^\pqq \btheta_i^\pq \> 
\end{bmatrix} , \\
\beta^\pqqp = &  ( d_q d_{q'} )^{-1} ( \tau^\pq_i \tau^\pqp_i )^{-2} \< \btheta^\pq_i , \bar \bU^\pqqp_{ii} (\btheta_i , \btheta_i) \btheta^\pqp_i \>.
\end{aligned}
\end{equation}

Let us first focus on the $q = q_\xi$ sphere. Using Eqs.~\eqref{eq:trace_ubar} and \eqref{eq:quadratic_ubar} with the expressions \eqref{eq:gegenbauer_conv_qxi_1} and \eqref{eq:gegenbauer_conv_qxi_2}, we get the following convergence in probability (using that $\lbrace \tau_i^{(q)} \rbrace_{i \in [N]}$ concentrates on $1$),
\begin{equation}
\begin{aligned}
\sup_{i \in [N]} \Big\vert r_{q_\xi}^{-2} \Trace ( \bar \bU_{ii}^{(q_\xi q_\xi)}) - F_1 ( \bdelta ) \Big\vert \overset{\P}{\to} & 0 \, , \\
\label{eq:convergence_to_F_1_F_2_qxi}
 \sup_{i \in [N]} \Big\vert r_{q_\xi}^{-2} \< \btheta_i^{(q_\xi)} , \bar \bU_{ii}^{(q_\xi q_\xi)} \btheta_i^{(q_\xi)} \> - F_2 ( \bdelta ) \Big\vert \overset{\P}{\to} & 0 \, ,
\end{aligned}
\end{equation}
where we denoted $\bdelta=(\delta_1,\delta_2)$ (where $\delta_1, \delta_2$ first appears in the definition of $\hat \sigma$ in Eq. (\ref{eqn:hat_sigma_definition}), and till now $\delta_1, \delta_2$ are still not determined) and, similarly to the proof of \cite[Proposition 5]{ghorbani2019linearized} and letting $\mu_k \equiv \mu_k (\sigma')$,
we have
\begin{align}
F_1(\bdelta) & = \sum_{t \in\{ 1,2\}}\delta_t (2 - \delta_t) \frac{\mu_{l_t}^2 }{l_t !}\, ,
\end{align}
while, for $l_2\neq l_1+2$
\begin{align*}
F_2(\bdelta) & =\sum_{t \in\{ 1,2\}}\left\{\frac{1}{(l_t-1)!} \Big[(\mu_{l_t}+(l_t-1)\mu_{l_t-2})^2-((1-\delta_t)\mu_{l_t}+(l_t-1)\mu_{l_t-2})^2\Big]\right.\\
&\phantom{AAAAA}+\left. \frac{1}{(l_t+1)!}\Big[(\mu_{l_t+2}+(l_t+1)\mu_{l_t})^2-(\mu_{l_t+2}+(1-\delta_t)(l_t+1)\mu_{l_t})^2\Big]
\right\}\, ,
\end{align*}
while, for $l_2= l_1+2$
\begin{align*}
F_2(\bdelta)  = &\frac{1}{(l_1-1)!} \Big[(\mu_{l_1}+(l_1-1)\mu_{l_1-2})^2-((1-\delta_1)\mu_{l_1}+(l_1-1)\mu_{l_1-2})^2\Big]\\
&+\frac{1}{(l_1+1)!} \Big[(\mu_{l_1+2}+(l_1+1)\mu_{l_1})^2-((1-\delta_2)\mu_{l_1+2}+(1-\delta_1)(l_1+1)\mu_{l_1})^2\Big]\\
&+\frac{1}{(l_2+1)!}\Big[(\mu_{l_2+2}+(l_2+1)\mu_{l_2})^2-(\mu_{l_2+2}+(1-\delta_2)(l_2+1)\mu_{l_2})^2\Big]\, .
\end{align*}
We have from Eq.~\eqref{eq:u_bar_cst},
\[
\begin{aligned}
& \lambda_{\min} ( \bar U_{ii}^\pqq ) \\
=& \min \Big\lbrace \alpha^\pq , \alpha^\pq + \beta^\pq d_q ( \tau_i^\pq )^2 \Big\rbrace \\
= & \min \left\{ \frac{1}{d_q-1}\Trace ( \bar \bU_{ii}^\pqq )- \frac{1}{d_q(d_q-1) (\tau_i^\pq )^2 }\< \btheta_i^\pq , \bar \bU_{ii}^\pqq \btheta_i^\pq \> ,  \frac{1}{d_q (\tau_i^\pq )^2}\< \btheta_i^\pq , \bar \bU_{ii}^\pqq \btheta_i^\pq \> \right\}.
\end{aligned}
\]
Hence, using Eq.~\eqref{eq:convergence_to_F_1_F_2_qxi}, we get
\begin{align}
 \sup_{i \in [N]} \Big\vert  \frac{d_{q_\xi}}{r_{q_\xi}^2} \lambda_{\min} (  \bar U_{ii}^{(q_\xi q_\xi)} ) - \min \lb F_1(\bdelta),F_2(\bdelta) \rb \Big\vert \overset{\P}{\to} 0.
\label{eq:lim_U_uu}
\end{align}
Following the same reasoning as in \cite[Proposition 5]{ghorbani2019linearized}, we can verify that under Assumption \ref{ass:activation_lower_upper_NT_aniso}.$(b)$, we have  $\nabla F_1(\bzero), \nabla F_2(\bzero)\neq \bzero$ and 
$\det (\nabla F_1(\bzero) ,\nabla F_2(\bzero)) \neq 0$. We can therefore find $\bdelta = (\delta_1 , \delta_2)$ such that $F_1(\bdelta)>0$, $F_2(\bdelta)>0$.  Furthermore,
\begin{align}
 \sup_{i \in [N]} \Big\vert  \frac{d_{q_\xi}}{r_{q_\xi}^2} \lambda_{\max} (  \bar U_{ii}^{(q_\xi q_\xi)} ) - \max \lb F_1(\bdelta),F_2(\bdelta) \rb \Big\vert \overset{\P}{\to} 0.
\label{eq:lim_U_zz}
\end{align}

Similarly, we get for $q \neq q_\xi$ from Eqs.~\eqref{eq:trace_ubar} and \eqref{eq:quadratic_ubar} with the expressions \eqref{eq:lambdaone_limit} (recalling that $\lbrace \tau_i^{(q)} \rbrace_{i \in [N]}$ concentrates on $1$),
\begin{equation}
\begin{aligned}
 \sup_{i \in [N]} \Big\vert  r_q^{-2} \Trace ( \bar \bU_{ii}^\pqq) - F_1 ( \bdelta ) \Big\vert \overset{\P}{\to} & 0\, , \\
 \sup_{i \in [N]} \Big\vert r_q^{-2} \< \btheta_i^\pq , \bar \bU_{ii}^\pqq \btheta_i^\pq \> - F_1 ( \bdelta ) \Big\vert \overset{\P}{\to} & 0 \, ,  \\
\label{eq:convergence_to_F_1_F_2}
 \sup_{i \in [N]} \Big\vert ( r_q r_{q'} )^{-1} \< \btheta_i^\pq , \bar \bU_{ii}^\pqqp \btheta_i^\pqp \> \Big\vert \overset{\P}{\to} &  0 \,     .
\end{aligned}
\end{equation}
We deduce that for $q \neq q_\xi$ and $q \neq q'$,
\[
\begin{aligned}
 \sup_{i \in [N]} \Big\vert   \frac{d_q}{r_q^2} \lambda_{\min} (  \bar U_{ii}^\pqq ) - F_1(\bdelta) \Big\vert & \overset{\P}{\to}   0, \\ 
  \sup_{i \in [N]} \Big\vert   \frac{d_q}{r_q^2} \lambda_{\max} (  \bar U_{ii}^\pqq ) - F_1(\bdelta) \Big\vert &  \overset{\P}{\to}   0, \\
   \sup_{i \in [N]} \Big\vert   \frac{(d_q d_{q'} )^{1/2} }{r_q r_{q'}} \sigma_{\max} (  \bar U_{ii}^\pqqp )\Big\vert  & \overset{\P}{\to}  0,
\end{aligned}
\]
which finishes to prove properties \eqref{eq:property_D_diagonal}	and \eqref{eq:property_D_non_diagonal}.

\clearpage

\section{Proof of Theorem \ref{thm:NT_lower_upper_bound_aniso}.(b): upper bound for \NT\,model}
\label{sec:proof_NTK_upper_aniso}

\subsection{Preliminaries}

\begin{lemma}\label{lem:lower_bound_A}
Let $\sigma$ be an activation function that satisfies Assumptions \ref{ass:activation_lower_upper_NT_aniso}.$(a)$ and \ref{ass:activation_lower_upper_NT_aniso}.$(c)$ for some level $\gamma >0$. Let $\cQ = \overline{\cQ}_{\NT} (\gamma)$ as defined in Eq.~\eqref{eq:def_overcQ_NT}. Define for integer $\bk \in \posint^Q$ and $\btau,\btau' \in \R_{\ge 0}^Q$, 
\begin{equation} \label{eq:def_A_tau_tauprime}
\begin{aligned}
A^{\pq}_{(\btau,\btau'), \bk }=&  r_q^2 \cdot \Big[ t_{d_q , k_q-1}  \lambda_{\bk_{q-} }^{\bd} ( \sigma_{\bd,\btau}') \lambda_{\bk_{q-} }^{\bd} ( \sigma_{\bd,\btau' }') B(\bd, \bk_{q-}  ) \\
& \phantom{AAAAAAA} +  s_{d_q, k_q+1}  \lambda_{\bk_{q+}}^{\bd} ( \sigma_{\bd,\btau}')\lambda_{\bk_{q+}}^{\bd} ( \sigma_{\bd,\btau'} ')  B(\bd, \bk_{q+} ) \Big].
\end{aligned}
\end{equation}
with $\bk_{q+} = (k_1 , \ldots , k_q+1 , \ldots , k_Q)$ and $\bk_{q-} = (k_1 , \ldots , k_q-1 , \ldots , k_Q)$, and 
\[
s_{d, k} = \frac{k}{2k + d - 2}, \qquad t_{d, k} = \frac{k + d - 2}{2k + d - 2},
\]
with the convention $t_{d,-1} = 0$.

Then there exists constants $\eps_0 >0$ and $C>0$ such that for $d$ large enough, we have for any $\btau,\btau' \in [1 - \eps_0 , 1+ \eps_0]^Q$,
\[
\max_{\bk \in \cQ} \frac{B(\bd,\bk)}{A^{\pq}_{(\btau,\btau'), \bk }} \leq C d^{\gamma - \kappa_q} .
\]
\end{lemma}

\begin{proof}[Proof of Lemma \ref{lem:lower_bound_A}]
From Assumptions \ref{ass:activation_lower_upper_NT_aniso}.$(a)$ and \ref{ass:activation_lower_upper_NT_aniso}.$(c)$ and Lemma \ref{lem:convergence_proba_Gegenbauer_coeff}, there exists $c>0$ and $\eps_0 > 0$ such that for any $\btau , \btau' \in [1 - \eps_0 , 1 + \eps_0 ]^Q$ and $\bk \in \cQ$,
\[
\lambda_{\bk}^\bd ( \sigma_{\bd , \btau} ' ) \lambda_{\bk}^\bd ( \sigma_{\bd , \btau'} ' ) \geq c \prod_{q \in [Q]} d^{k_q ( \kappa_q - \xi ) }.
\]
Hence for $k_q > 0$, we get $\lambda_{\bk}^\bd ( \sigma_{\bd , \btau} ' ) \lambda_{\bk}^\bd ( \sigma_{\bd , \btau'} ' ) \geq c d^{-\gamma - \xi + \kappa_q}$, and for $k_q = 0$, we get $\lambda_{\bk}^\bd ( \sigma_{\bd , \btau} ' ) \lambda_{\bk}^\bd ( \sigma_{\bd , \btau'} ' ) \geq c d^{-\gamma + \xi - \eta_q - \kappa_q}$. Carefully injecting these bounds in Eq.~\eqref{eq:def_A_tau_tauprime} yields the lemma.
\end{proof}

\subsection{Proof of Theorem \ref{thm:NT_lower_upper_bound_aniso}.(b): outline}

In this proof, we will consider $Q$ sub-classes of functions corresponding to the \NT\,model restricted to the $q$-th sphere:
\[
\cF_{\NT^\pq} ( \bW) \equiv  \Big\{  f (\bx ) = \sum_{i=1}^N  \< \ba_i , \bx^\pq \> \sigma ' (\< \bw_i, \bx\> / R)\, : \,\,\, \ba_i \in \R^{d_q}, i \in [N] \Big\}.
\]
We define similarly the risk associated to this sub-model
\[
R_{\NT^\pq} (f_d , \bW ) = \inf_{f \in \cF_{\NT^\pq} ( \bW) } \E [ ( f_d (\bx) - f (\bx ) )^2 ].
\]
and approximation subspace
\begin{equation}\label{eq:def_overcQ_NT_q}
\begin{aligned}
\overline{\cQ}_{\NT^\pq} (\gamma ) = & \Big\{ \bk \in \posint^Q \Big\vert k_q > 0 \text{ and } \sum_{q \in [Q]} k_q ( \xi - \kappa_q ) \leq \gamma + (\xi - \kappa_q ) \Big\} \\
& \cup  \Big\{ \bk \in \posint^Q \Big\vert k_q = 0 \text{ and } \sum_{q \in [Q]} k_q ( \xi - \kappa_q ) \leq \gamma - (\xi - \kappa_q -\eta_q ) \Big\}.
\end{aligned}
\end{equation}

\begin{theorem}\label{thm:upper_bound_NT_q}
Let $\{ f_d \in L^2(\PS^\bd_\bkappa , \mu_{\bd}^\bkappa )\}_{d \ge 1}$ be a sequence of functions. Let $\bW = (\bw_i)_{i \in [N]}$ with $(\bw_i)_{i \in [N]} \sim \Unif(\S^{D-1})$ independently. Assume $N \ge \omega_d(d^{\gamma})$ for some positive constant $\gamma> 0$, and $ \sigma$ satisfy Assumptions \ref{ass:activation_lower_upper_NT_aniso}.(a) and \ref{ass:activation_lower_upper_NT_aniso}.(c) at level $\gamma$. Then for any $\eps > 0$, the following holds with high probability: 
\begin{align}\label{eqn:NT_upper_bound_q}
0 \le R_{\NT^\pq}(\proj_{\cQ} f_d, \bW) \le \eps \| \proj_{\cQ} f_d \|_{L^2}^2 \, ,
\end{align}
where $\cQ \equiv \overline{\cQ}_{\NT^\pq} (\gamma)$ is defined in Equation~\eqref{eq:def_overcQ_NT_q}.
\end{theorem}

\begin{remark}
From the proof of Theorem \ref{thm:NT_lower_upper_bound_aniso}.$(a)$, we have a matching lower bound for $\cF_{\NT^\pq} $. 
\end{remark}

We recall
\[
\overline{\cQ}_\NT ( \gamma ) = \blb \bk \in \posint^Q \Big\vert \sum_{q = 1}^Q (\xi - \kappa_q) k_q \leq \gamma + \lbp \xi - \min_{q \in S(\bk)} \kappa_q \rbp \brb .
\]
Notice that 
\[
\overline{\cQ}_\NT ( \gamma ) = \bigcup_{q \in Q} \overline{\cQ}_{\NT^\pq} (\gamma )
\]
Denote $q_{\bk} = \arg \min_{q\in S(\bk)} \kappa_q$, such that $\bk \in\overline{\cQ}_{\NT^{(q_{\bk}) }}$ for any $\bk \in \overline{\cQ}_\NT ( \gamma )$. Furthermore, notice that by definition for any $f \in L^2(\PS^\bd_\bkappa , \mu_{\bd}^\bkappa )$ and $q \in [Q]$,
\[
R_{\NT}(f , \bW) \leq R_{\NT^\pq}(f, \bW).
\]

Let us deduce Theorem \ref{thm:NT_lower_upper_bound_aniso}.$(b)$ from Theorem \ref{thm:upper_bound_NT_q}. Denote $\cQ = \overline{\cQ}_{\NT} (\gamma)$. We divide the $N$ neurons in $|\cQ |$ sections of size $N ' =N / |\cQ |$, i.e.\,$\bW = ( \bW_\bk )_{\bk \in \cQ}$ where $\bW_{\bk} \in \R^{ N' \times d}$. For any $\eps > 0$, we get from Theorem \ref{thm:upper_bound_NT_q} that with high probability
\[
R_{\NT}( \proj_{\cQ} f_d , \bW) \leq \sum_{\bk \in \cQ} R_{\NT^{(q_\bk)}}( \proj_{\bk} f , \bW_{\bk}) \leq \sum_{\bk \in \cQ} \eps \| \proj_{\bk} f_d \|_{L^2}^2 = \eps \| \proj_{\cQ} f_d \|_{L^2}^2.
\]

\subsection{Proof of Theorem \ref{thm:upper_bound_NT_q}}

\subsubsection{Properties of the limiting kernel}

Similarly to the proof of Theorem \ref{thm:RF_lower_upper_bound_aniso}.$(b)$, we construct a limiting kernel which is used as a proxy to upper bound the $\NT^\pq$ risk.

We recall the definition of $\PS^{\bd} = \prod_{q \in [Q]} \S^{d_q - 1} ( \sqrt{d_q} )$.  We introduce $\mathcal{L} = L^2 ( \PS^\bd  \rightarrow \R, \mu_{\bd}) $ and $\mathcal{L}_{d_q} = L^2 ( \PS^\bd  \rightarrow \R^{d_q} , \mu_{\bd} )$. For a given $\btheta \in \S^{D-1} (\sqrt{D})$ and associated vector $\btau \in \R_{\geq 0}^Q$, recall the definition of $\sigma_{\bd,\btau} '  \in \mathcal{L} $:
\[
\sigma_{\bd,\btau} '  \lp \lb \< \obtheta^\pq , \obx^\pq  \> / \sqrt{d_q} \rb_{q \in [Q]} \rp = \sigma ' \lp \sum_{q \in [Q]} \tau^\pq \cdot (r_q / R) \cdot \< \obtheta^\pq , \obx^\pq  \> / \sqrt{d_q} \rp
\]

For any $\btau \in \R_{\geq 0}^Q$, define the operator $\T_{\btau} : \cL \to \cL_{d_q}$, such that for any $g \in \cL$,
\[
\T_{\btau} g ( \obtheta ) =  \frac{r_q}{\sqrt{d_q}} \E_{\obx} \Big[  \obx^\pq \sigma_{\bd,\btau} ' \lp \lb \< \obtheta^\pq , \obx^\pq \> / \sqrt{d_q} \rb_{q \in [Q]} \rp g (\obx) \Big].
\]
The adjoint operator $\T^*_{\btau} : \cL_{d_q} \to \cL$ verifies for any $h \in \cL_{d_q}$,
\[
\T_{\btau}^* h ( \obx ) = \frac{r_q}{\sqrt{d_q}} ( \obx^\pq)^\sT \E_{\obtheta} \Big[ \sigma_{\bd,\btau} '  \lp \lb \< \obtheta^\pq , \obx^\pq \> / \sqrt{d_q} \rb_{q \in [Q]} \rp h ( \obtheta ) \Big].
\] 
We define the operator $\K_{\btau,\btau'} : \cL_{d_q} \to \cL_{d_q} $ as $\K_{\btau,\btau'} \equiv \T_{\btau} \T_{\btau'}^*$. For $h \in \cL_{d_q}$, we can write
\[
\K_{\btau,\btau'} h ( \obtheta_1 ) = \E_{\obtheta_2 } [ \K_{\btau,\btau'} ( \obtheta_1 , \obtheta_2 ) h ( \obtheta_2 ) ],
\]
where 
\[
\begin{aligned}
&\K_{\btau_1,\btau_2} ( \obtheta_1 , \obtheta_2 )\\  =&  \frac{r_q^2}{d_q}\E_{\obx} \Big[ \obx^\pq (\obx^\pq)^\sT \sigma_{\bd,\btau_1} '  \lp \lb \< \obtheta_1^\pq , \obx^\pq \> / \sqrt{d_q} \rb_{q \in [Q]} \rp \sigma_{\bd,\btau_2} '   \lp \lb \< \obtheta_2^\pq , \obx^\pq \> / \sqrt{d_q} \rb_{q \in [Q]} \rp \Big] .
\end{aligned}
\]
Define $\H_{\btau,\btau'} : \cL \to \cL$ as $\H_{\btau,\btau'} \equiv \T^*_{\btau} \T_{\btau'}$. For $g \in \cL$, we can write
\[
\H_{\btau,\btau'} g ( \obx_1 ) = \E_{\obx_2 } [ \H_{\btau,\btau'} ( \obx_1 , \obx_2 ) g ( \obx_2 ) ] 
\]
where
\[
\begin{aligned}
& \H_{\btau,\btau'} ( \obx_1 , \obx_2 ) \\ = & \frac{r_q^2}{d_q} \E_{\obtheta} \Big[ \sigma_{\bd,\btau} ' \lp \lb \< \obtheta^\pq , \obx_1^\pq \> / \sqrt{d_q} \rb_{q \in [Q]} \rp  \sigma_{\bd,\btau'} ' \lp \lb \< \obtheta^\pq , \obx_2^\pq \> / \sqrt{d_q} \rb_{q \in [Q]} \rp \Big] \< \obx^\pq_1 , \obx^\pq_2 \>.
\end{aligned}
\]
We recall the decomposition of $\sigma_{\bd,\btau} '$ in terms of tensor product of Gegenbauer polynomials:
\[
\begin{aligned}
 \sigma_{\bd,\btau} '  ( \ox_1^{(1)} , \ldots , \ox_1^{(Q)} ) = &\sum_{\bk \in \posint^Q} \lambda_{\bk}^{\bd} ( \sigma_{\bd,\btau}') B(\bd , \bk) Q^{\bd}_{\bk} \Big( \sqrt{d_1} \ox_1^{(1)} , \ldots, \sqrt{d_Q} \ox_1^{(Q)} \Big), \\
 \lambda_{\bk}^{\bd} ( \sigma_{\bd,\btau}') = & \E_{\obx} \Big[  \sigma_{\bd,\btau} '  ( \ox_1^{(1)} , \ldots , \ox_1^{(Q)} ) Q^{\bd}_{\bk } \Big( \sqrt{d_1} \ox_1^{(1)} , \ldots, \sqrt{d_Q} \ox_1^{(Q)}  \Big) \Big] .
 \end{aligned}
\]
Following the same computations as in Lemma \ref{lem:formula_NT_Kernel}, we get
\[
\begin{aligned}
 \H_{\btau,\btau'} ( \obx_1 , \obx_2 ) = &  \sum_{\bk \in \posint^Q} A^\pq_{(\btau,\btau'), \bk} Q^{\bd}_{\bk} \lp \lb \< \obx^\pq_1 , \obx^\pq_2 \> \rb_{q \in [Q]} \rp,
\end{aligned}
\]
where 
\begin{equation}\label{eq:def_gamma_UB_NT}
\begin{aligned}
A^\pq_{(\btau , \btau ' ) , \bk} =&  r_q^2 \cdot \Big[ t_{d_q, k_q-1}  \lambda_{\bk_{q-}}^{\bd} ( \sigma_{\bd,\btau}') \lambda_{\bk_{q-} }^{\bd} ( \sigma_{\bd,\btau' }') B(\bd, \bk_{q-}) \\
& \phantom{AAAAAAA} +  s_{d_q, k_q+1}  \lambda_{\bk_{q+} }^{\bd} ( \sigma_{\bd,\btau}')\lambda_{\bk_{q+} }^{\bd} ( \sigma_{\bd,\btau'} ')  B(\bd, \bk_{q+} ) \Big],
\end{aligned}
\end{equation}
with $\bk_{q+} = (k_1 , \ldots , k_q+1 , \ldots , k_Q)$ and $\bk_{q-} = (k_1 , \ldots , k_q-1 , \ldots , k_Q)$, and convention $t_{d_q,-1} = 0$,
\[
s_{d_q, k_q} = \frac{k_q}{2k_q + d_q - 2}, \qquad t_{d_q, k_q} = \frac{k_q + d_q - 2}{2k_q + d_q - 2}.
\]

Recall that for $\bk \in \posint^Q$ and $\bs \in [B(\bd,\bk)]$, $Y^{\bd}_{\bk,\bs} = \bigotimes_{q \in [Q]}  Y^{(d_q)}_{k_q s_q} $ forms an orthogonal basis of $\cL$ and that
\[
\begin{aligned}
\E_{\obx_2} \Big[ Q^{\bd}_{\bk} \lp \lb \< \obx_1 , \obx_2 \> \rb_{q \in [Q]} \rp Y^{\bd}_{\bk,\bs} ( \obx_2) \Big] = & \frac{1}{B(\bd, \bk) } Y^{\bd}_{\bk,\bs} (\obx_1 ) \delta_{\bk,\bs} .
\end{aligned}
\]
We deduce that
\[
\H_{\btau,\btau'} Y^{\bd}_{\bk,\bs} ( \obx_1 ) = \sum_{\bk' \in \posint^Q} A^\pq_{(\btau , \btau ' ) , \bk'}  \E_{\obx_2} \Big[ Q^{\bd}_{\bk'} \lp \lb \< \obx_1 , \obx_2 \> \rb_{q \in [Q]} \rp Y^{\bd}_{\bk,\bs} ( \obx_2) \Big] = \frac{A^\pq_{(\btau , \btau ' ) , \bk}}{B( \bd , \bk)} Y^{\bd}_{\bk,\bs} ( \obx_1 ).
\]
Consider $\{ \T_{\btau} Y^{\bd}_{\bk,\bs} \}_{\bk \in \posint^Q, \bs \in [B(\bd,\bk)]}$. We have:
\[
\begin{aligned}
\< \T_{\btau} Y^{\bd}_{\bk,\bs} , \T_{\btau'} Y^{\bd}_{\bk',\bs'} \>_{L^2 } = & \< Y^{\bd}_{\bk,\bs} , \H_{\btau,\btau'} Y^{\bd}_{\bk' ,\bs '}  \>_{L^2} = \frac{A^\pq_{(\btau , \btau ' ) , \bk}}{B( \bd , \bk)}  \delta_{\bk,\bk'} \delta_{\bs,\bs'}, \\
\K_{\btau,\btau'}\T_{\btau ''}  Y^{\bd}_{\bk,\bs} = & \T_{\btau} \H_{\btau',\btau''}  Y^{\bd}_{\bk,\bs} =   \frac{A^\pq_{(\btau' , \btau '' ) , \bk}}{B( \bd , \bk)} \T_{\btau} Y^{\bd}_{\bk,\bs}.
\end{aligned}
\] 
Hence $\{ \T_{\btau''} Y^{(\bd)}_{\bk,\bs} \}$ forms an orthogonal basis that diagonalizes $\K_{\btau',\btau''}$ (notice that $\T_{\btau} Y^{\bd}_{\bk,\bs}$ is parallel to $\T_{\btau'} Y^{\bd}_{\bk,\bs}$ for any $\btau,\btau' \in \R^Q_{\geq 0}$). Let us consider the subspace $\T_{\btau} (V^{\bd}_{\cQ} )$, the image of $V^{\bd}_{\cQ}$ by the operator $\T_{\btau}$. From Assumptions \ref{ass:activation_lower_upper_NT_aniso}.$(a)$ and \ref{ass:activation_lower_upper_NT_aniso}.$(b)$ and Lemma \ref{lem:convergence_proba_Gegenbauer_coeff}, there exists $\eps_0 \in (0 ,1 ) $ and $d_0$ such that for any $\btau,\btau' \in [1- \eps_0 , 1+ \eps_0]^Q$ and $d \ge d_0$, we have $A^\pq_{(\btau , \btau ' ) , \bk} >0 $ for any $\bk \in \cQ$, and therefore the inverse $\K_{\btau,\btau'}^{-1} \vert_{\T_{\btau} (V^{\bd}_{ \cQ} )}$ (restricted to $\T_{\btau} (V^{\bd}_{ \cQ} )$) is well defined.

\subsubsection{Proof of Theorem \ref{thm:upper_bound_NT_q}}

Let us assume that $\lbrace f_d \rbrace$ is contained in $\bigoplus_{\bk \in \cQ} \bV^\bd_{\bk}$, i.e.\,$\of_{d} = \proj_{\cQ} \of_{d}$. 

Consider 
\[
\hat{f} (\bx ; \bTheta, \ba ) = \sum_{i = 1}^N \< \ba_i , \bx^\pq \> \sigma '  (\<\btheta_i , \bx \> / R).
\]
Define $\balpha_{\btau} ( \obtheta) \equiv \K_{\btau,\btau}^{-1} \T_{\btau} \of_{d} (\obtheta)$ and choose $\ba_i^* = N^{-1} \balpha_{\btau_i} ( \obtheta_i) $, where we denoted $\obtheta_i = ( \obtheta^\pq_i )_{q \in [Q]} $ with $\obtheta^\pq_i = \btheta^\pq_i /\tau^\pq_i \in \S^{d_q-1} (\sqrt{d_q})$ independent of $\btau_i$.

Fix $\eps_0 > 0 $ as prescribed in Lemma \ref{lem:lower_bound_A} and consider the expectation over $\cP_{\eps_0}$ of the $\NT^\pq$ risk (in particular, $\ba^* = (\ba_1^* , \ldots , \ba_N^*) \in \R^{Nd_q}$ are well defined):
\[
\begin{aligned}
\E_{\bTheta_{\eps_0}} [ R_{\NT^\pq}(f_d , \bTheta) ] =& \E_{\bTheta_{\eps_0} } \Big[ \inf_{\ba \in \R^{Nd_q}} \E_{\bx} [(f_d (\bx) - \hat f (\bx; \bTheta, \ba))^2] \Big]\\
\leq & \E_{\bTheta_{\eps_0}} \Big[ \E_{\bx} \Big[(f_d (\bx) - \hat f (\bx; \bTheta, \ba^* (\bTheta) ))^2 \Big] \Big].
\end{aligned}
\]
We can expand the squared loss at $\ba$ as
\begin{equation}
\begin{aligned}
\E_\bx [ (f_d (\bx) - \hat{f} (\bx) )^2 ] = & \| f_d \|^2_{L^2} -2 \sum_{i=1}^N \E_{\bx} [ \<\ba_i , \bx^\pq \> \sigma ' ( \< \btheta_i , \bx \> / R  ) f_d (\bx) ] \\
\label{eq:expansion_squared_loss_NT}
& +  \sum_{i,j=1}^N \E_{\bx} [ \<\ba_i , \bx^\pq \> \<\ba_j , \bx^\pq \> \sigma ' ( \< \btheta_i , \bx \> /R  ) \sigma ' ( \< \btheta_j , \bx \> / R  ) ].
\end{aligned}
\end{equation}
The second term of the expansion \eqref{eq:expansion_squared_loss_NT} around $\ba^*$ verifies
\begin{equation}\label{eq:upper_bound_NT_NU_term1}
\begin{aligned}
& \E_{\bTheta_{\eps_0}} \Big[ \sum_{i=1}^N \E_{\bx} [ \< \ba_i^* , \bx^\pq \> \sigma ' ( \< \btheta_i , \bx \> / R  ) f_d (\bx) ] \Big] \\
= & \E_{\btau_{\eps_0}} \Big[ \E_{\obtheta} \Big[  \balpha_{\btau} ( \overline \btheta)^{\sT}  \E_{\obx} \Big[  \bx^\pq \sigma_{\bd,\btau} ' \lp \lb \< \obtheta^\pq , \obx^\pq \>/ \sqrt{d_q} \rb_{q \in [Q]} \rp  \of_{d} (\obx) \Big] \Big] \Big] \\
= &  \E_{\btau_{\eps_0} } \Big[ \< \K_{\btau,\btau}^{-1} \T_{\btau} \of_{d}  , \T_{\btau} \of_{d} \>_{L^2} \Big]\\
= & \| f_d \|^2_{L^2},
\end{aligned}
\end{equation}
where we used that for each $\btau \in [1-\eps_0,1+\eps_0]^Q$, we have $\T_{\btau}^* \K_{\btau,\btau}^{-1} \T_{\btau} = \id \vert_{V^{\bd}_{\cQ}}$.

Let us consider the third term in the expansion \eqref{eq:expansion_squared_loss_NT} around $\ba^*$: the non diagonal term verifies
\[
\begin{aligned}
& \E_{\bTheta_{\eps_0}} \Big[ \sum_{i\neq j} \E_{\bx} [ \< \ba_i^*, \bx^\pq \>\< \ba_j^* , \bx^\pq \> \sigma ' ( \< \btheta_i , \bx \> / R  ) \sigma ' ( \< \btheta_j , \bx \> / R ) ] \Big] \\
= & ( 1 - N^{-1} ) \E_{\btau_{\eps_0}^1 , \btau_{\eps_0}^2,\obtheta_1 , \obtheta_2} \Big[  \balpha_{\btau^1} ( \obtheta_1)^\sT    E_{\overline \bx} \Big[ \sigma_{\bd,\btau}'   \lp \lb \< \obtheta^\pq_1 , \obx^\pq \> /\sqrt{d_q} \rb_{q \in [Q]} \rp  \\
& \phantom{AAAAAA} \times \sigma_{\bd,\btau'}  ' \lp \lb \< \obtheta^\pq_2 , \obx^\pq \> /\sqrt{d_q} \rb_{q \in [Q]} \rp \bx^\pq (\bx^\pq)^\sT \Big] \balpha_{\btau^2} ( \obtheta_2)  \Big]  \\
= & ( 1 - N^{-1} ) \E_{\btau_{\eps_0}^1 , \btau_{\eps_0}^2,\obtheta_1 , \obtheta_2} \Big[  \K^{-1}_{\btau^1,\btau^1} \T_{\btau^1} \of_{d} ( \obtheta_1 )^\sT  \K_{\btau^1 , \btau^2 } ( \obtheta_1 , \obtheta_2 ) \K^{-1}_{\btau^2,\btau^2} \T_{\btau^2} \of_{d} ( \obtheta_2 ) \Big] \\
= &  ( 1 - N^{-1} )  \E_{\btau_{\eps_0}^1, \btau_{\eps_0}^2 } \Big[ \< \K_{\btau^1,\btau^1}^{-1} \T_{\btau^1} \of_{d}  , \K_{\btau^1,\btau^2} \K^{-1}_{\btau^2,\btau^2} \T_{\btau^2} \of_{d} \>_{L^2} \Big].
\end{aligned}
\]
For $\bk \in \cQ$ and $\bs \in [B(\bd , \bk)]$ and $\btau^1,\btau^2 \in [1-\eps_0,1+\eps_0]^Q$, we have
\[
\begin{aligned}
\T_{\btau^1}^* \K^{-1}_{\btau^1,\btau^1} \K_{\btau^1,\btau^2} \K^{-1}_{\btau^2,\btau^2} \T_{\btau^2}  Y^\bd_{\bk,\bs} =&  \Big(  \T_{\btau^1}^* \K^{-1}_{\btau^1,\btau^1} \T_{\btau^1} \Big) \cdot \Big( \T_{\btau^2}^*  \K^{-1}_{\btau^2,\btau^2} \T_{\btau^2} \Big) \cdot Y^\bd_{\bk,\bs} =  Y^\bd_{\bk,\bs}  . 
\end{aligned}
\]
Hence for any $\btau^1,\btau^2 \in [1-\eps_0,1+\eps_0]^Q$, $\T_{\btau^1}^* \K^{-1}_{\btau^1,\btau^1} \K_{\btau^1,\btau^2} \K^{-1}_{\btau^2,\btau^2} \T_{\btau^2}   = \id
\vert_{V^{\bd}_{\cQ}}$. Hence
\begin{equation}\label{eq:upper_bound_NT_NU_term2}
\E_{\bTheta_{\eps_0}} \Big[ \sum_{i\neq j} \E_{\bx} [ \< \ba_i^*, \bx^\pq \>\< \ba_j^* , \bx^\pq \> \sigma ' ( \< \btheta_i , \bx \> / R ) \sigma ' ( \< \btheta_j , \bx \> /R ) ] \Big]  = (1 - N^{-1} ) \| f_d \|^2_{L^2}.
\end{equation}
The diagonal term verifies
\[
\begin{aligned}
 & \E_{\bTheta_{\eps_0}} \Big[ \sum_{i \in [N]} \E_{\bx} [ \< \ba_i^* ,\bx^\pq \>^2  \sigma ' ( \< \btheta_i , \bx \> / R ) \sigma  '( \< \btheta_j , \bx \> / R  ) ] \Big] \\
 = & N^{-1} \E_{\btau_{\eps_0} , \obtheta} \Big[  \balpha_{\btau} ( \obtheta)^\sT \K_{\btau,\btau} ( \obtheta , \obtheta ) \balpha_{\btau} ( \obtheta)  \Big]   \\
\le & N^{-1} \Big[ \max_{\obtheta, \btau \in [1- \eps_0, 1+\eps_0]^Q}  \| \K_{\btau,\btau} ( \obtheta , \obtheta ) \|_{\op}\Big] \cdot \E_{\btau_{\eps_0}} [ \| \K^{-1}_{\btau,\btau} \T_{\btau} \of_{d} \|^2_{L^2} ].
\end{aligned}
\]
We have, from Lemma \ref{lem:decomposition_U_NT_NU},
\[
\K_{\btau , \btau} ( \obtheta , \obtheta ) = \alpha^\pq \id_{d_q} + \beta^\pq \obtheta^\pq (\obtheta^\pq)^2
\]
where
\[
\begin{aligned}
\begin{bmatrix} 
\alpha^\pq \\
\beta^\pq
\end{bmatrix} = &  [d_q(d_q-1) (\tau^\pq)^4 ]^{-1} \begin{bmatrix} 
d_q (\tau^\pq)^4 & - (\tau^\pq)^2 \\
- (\tau^\pq)^2 & 1  
\end{bmatrix} \\
& \phantom{AAAAAAAAAAAAA}\times \begin{bmatrix} 
\E_{\bx} [ \< \bx^\pq , \bx^\pq \> \sigma'_{\bd,\btau} ( \ox_1^{(1)} , \ldots , \ox_1^{(Q)} )^2 ]  \\
\E_{\bx} [(\ox^\pq_1)^2 \sigma'_{\bd,\btau} ( \ox_1^{(1)} , \ldots , \ox_1^{(Q)} )^2 ]
\end{bmatrix}.
\end{aligned}
\]
Hence from Lemma \ref{lem:coupling_convergence_Gaussian} and for $\eps_0$ small enough, there exists $C>0$ such that for $d$ large enough
\[
 \sup_{\btau \in [1- \eps_0,1+\eps_0]^Q} \| \K_{\btau,\btau} ( \obtheta , \obtheta ) \|_{\op}    \leq C \frac{r_q^2}{d_q} = C d^{\kappa_q}.
\]
Furthermore
\[
\begin{aligned}
\| \K^{-1}_{\btau,\btau} \T_{\btau} \of_{d} \|^2_{L^2} = & \sum_{\bk \in \cQ} \frac{B(\bd , \bk) }{A^\pq_{(\btau,\btau),\bk} } \sum_{\bs \in [B(\bd,\bk)] } \lambda^\bd_{\bk,\bs} ( \of_{d})^2 \\
\le & \lbb \max_{k \in \cQ} \frac{B(\bd,\bk) }{A^\pq_{(\btau,\btau),\bk}  }  \rbb \cdot \| \proj_{\cQ} f_d \|^2_{L^2}.
\end{aligned}
\]
From Lemma \ref{lem:lower_bound_A}, we get
\[
\E_{\btau_{\eps_0}} [ \| \K^{-1}_{\btau,\btau} \T_{\btau} \of_{d} \|^2_{L^2} ] \leq C d^{\gamma - \kappa_q} \cdot \| \proj_{\cQ}  f_d \|^2_{L^2}.
\]
Hence,
\begin{equation}\label{eq:upper_bound_NT_NU_term3}
 \E_{\bTheta_{\eps_0}} \Big[ \sum_{i \in [N]} \E_{\bx} [ \< \ba_i^* ,\bx^\pq \>^2  \sigma ' ( \< \btheta_i , \bx \> / R  ) \sigma  '( \< \btheta_j , \bx \> / R ) ] \Big] \leq C  \frac{d^{\gamma}}{N} \| \proj_{\cQ}  f_d \|^2_{L^2}.
\end{equation}

Combining Eq.~\eqref{eq:upper_bound_NT_NU_term1}, Eq.~\eqref{eq:upper_bound_NT_NU_term2} and Eq.~\eqref{eq:upper_bound_NT_NU_term3}, we get
\[
\begin{aligned}
& \E_{\bTheta_{\eps_0}} \Big[ \E_{\bx} \Big[(f_d (\bx) - \hat f (\bx; \bTheta, \ba^* (\bTheta) ))^2 \Big] \Big] \\ 
=& \| f_d \|_{L^2}^2 - 2 \| f_d \|_{L^2}^2 + (1 - N^{-1})  \| f_d \|_{L^2}^2 +N^{-1} \E_{\btau_{\eps_0} , \obtheta} \Big[  \balpha_{\tau} ( \overline \btheta)^\sT \K_{\tau,\tau} ( \obtheta , \obtheta ) \balpha_{\tau} ( \overline \btheta)  \Big]  \\
 \le &C \frac{d^{\gamma}}{ N} \| \proj_{\cQ}  f_d \|^2_{L^2}.
\end{aligned}
\]
By Markov's inequality, we get for any $\eps >0$ and $d$ large enough,
\[
\begin{aligned}
\P ( R_{\NT^\pq}(f_d , \bTheta) > \eps \cdot \| f_d \|_{L^2}^2 ) \leq & \P ( \lbrace R_{\NT^\pq}(f_d , \bTheta) > \eps \cdot  \| f_d \|_{L^2}^2 \rbrace \cap \cP_{\eps_0} )  + \P (\cP_{\eps_0}^c ) \\ \leq & C' \frac{d^{\gamma}}{ N } + \P (\cP_{\eps_0}^c ).
\end{aligned}
\]
The assumption that $ N = \omega_d ( d^\gamma )$ and Lemma \ref{lem:bound_proba_cPeps} conclude the proof.

\clearpage

\section{Proof of Theorem \ref{thm:NN} in the main text}
%
%\begin{theorem}[Approximation error for NN]\label{thm:NN}
%Assume that $\sigma\in C^{\infty}(\reals)$ satisfies the same assumptions as in Theorem \ref{thm:NT}. Further assume that $\sup_{x \in \R} \vert \sigma''(x)\vert < \infty$. If $d_0^{\ell + \delta}\le N\le d_0^{\ell+1 - \delta}$ for some $\delta > 0$ independent of $N,d$, then the approximation error of NN models (\ref{eqn:NT_model}) is
%\begin{align}\label{eqn:bound_NN_risk}
% R_{\NN, N}(f_*) \le (1 + o_d(1)) \cdot \| \proj_{> \ell + 1} f_* \|_{L^2}^2. 
%\end{align}
%Moreover, the quantity $R_{\NN, N}(f_*)$ is independent of $\kappa \ge 0$. 
%\end{theorem}

%\begin{proof}
\noindent
{\bf Step 1. Show that $R_{\NN, 2 N}(f_*) \le \inf_{\bW \in \R^{N \times d}} R_{\NT, N}(f_*, \bW)$. }

Define the neural tangent model with $N$ neurons by $\hat f_{\NT, N}(\bx; \bs; \bW) = \sum_{i = 1}^N \< \bs_i, \bx\> \sigma'(\< \bw_i, \bx\>)$ and the neural networks with $N$ neurons by $\hat f_{\NN, N}(\bx; \bW, \bb) = \sum_{i = 1}^N b_i \sigma(\< \bw_i, \bx\>)$. For any $\bW \in \R^{N \times d}$, $\bs \in \R^N$, and $\eps > 0$, we define
\[
\begin{aligned}
\hat g_N(\bx; \bW, \bs, \eps) \equiv&~ \eps^{-1}\Big(\hat f_{\NN, N}(\bx; \bW + \eps \bs, \bfone) - \hat f_{\NN, N}(\bx; \bW, \bfone)\Big),\\
\cE(\bx; \bW, \bs, \eps) =&~ \hat g_N(\bx; \bW, \bs, \eps) - \hat f_{\NT, N}(\bx; \bs; \bW). 
\end{aligned}
\]
Then by Taylor expansion, there exists $(\tilde \bw_i)_{i \in [N]}$ such that
\[
\vert \cE(\bx; \bW, \bs, \eps) \vert = \frac{\eps}{2} \Big\vert \sum_{i = 1}^N \< \bs_i, \bx\>^2 \sigma''(\< \tilde \bw_i, \bx\>) \Big\vert. 
\]
By the boundedness assumption of $\sup_{x \in \R} \vert \sigma''(x) \vert$, we have 
\[
\lim_{\eps \to 0+}\| \cE(\cdot; \bW, \bs, \eps) \|_{L^2}^2 = 0, 
\]
and hence 
\[
\lim_{\eps \to 0+}\| f_* - \hat g_N(\cdot; \bW, \bs, \eps) \|_{L^2}^2 = \| f_* - \hat f_{\NT, N}(\cdot ; \bs; \bW) \|_{L^2}^2. 
\]
Note that $\hat g_N$ can be regarded as a function in $\cF_{\NN}^{2N}$ and $\hat f_{\NT, N} \in \cF_{\NN}^{N}(\bW)$, this implies that 
\begin{equation}\label{eqn:NN_NT_risk_relation}
R_{\NN, 2N}(f_*) \le \inf_{\bW \in \R^{N \times d}} R_{\NT, N}(f_*, \bW). 
\end{equation}

\noindent
{\bf Step 2. Give upper bound of $\inf_{\bW \in \R^{N \times d}} R_{\NT, N}(f_*, \bW)$. } We take $\overline \bW = (\bar \bw_i)_{i \le N}$ with $\bar \bw_i = \bU \bar \bv_i$, where $\bar \bv_i \sim \Unif(\S^{d_0 -1}(r^{-1}))$, and denote $\overline \bV = (\bar \bv_i)_{i \le N}$. Then we have 
\[
\cG_{\NT}^{N}(\overline \bV) \equiv \Big\{ f(\bx) = \bar f(\bU^\sT \bx): \bar f(\bz) = \sum_{i = 1}^N \< \bar \bs_i, \bz\> \sigma'(\< \bar \bv_i, \bz\>), \bar \bs_i \in \R^{d_0}, i \le N \Big\} \subseteq \cF_{\NT}^N(\overline \bW). 
\]
It is easy to see that, when $f_*(\bx) = \vphi(\bU^\sT \bx)$, we have 
\[
\inf_{\hat f \in \cG_{\NT}^N(\overline \bV)} \E[(f_*(\bx) - \hat f(\bx) )^2] = \inf_{\hat f \in \cF_{\NT}^N(\overline \bV)} \E[(\vphi(\bz) - \hat f(\bz))^2],
\]
where $\cF_{\NT}^N(\overline \bV)$ is the class of neural tangent model on $\R^{d_0}$
\[
\cF_{\NT}^N(\overline \bV) = \Big\{ \bar f(\bz) = \sum_{i = 1}^N \< \bar \bs_i, \bz\> \sigma'(\< \bar \bv_i, \bz\>): \bar \bs_i \in \R^{d_0}, i \le N  \Big\}. 
\]
Moreover, by Theorem \ref{thm:NT} in the main text, when $d_0^{\ell + \delta} \le N \le d_0^{\ell + 1 - \delta}$ for some $\delta > 0$ independent of $N, d$, we have 
\[
\inf_{\hat f \in \cF_{\NT}^N(\overline \bV)} \E[(\vphi(\bz) - \hat f(\bz) )^2] = (1 + o_{d, \P}(1)) \cdot \| \proj_{> \ell + 1} \vphi \|_{L^2}^2 = (1 + o_{d, \P}(1)) \cdot \| \proj_{> \ell + 1} f_* \|_{L^2}^2. 
\]
As a consequence, we have 
\[
\begin{aligned}
\inf_{\bW \in \R^{N \times d}} R_{\NT, N}(f_*, \bW) \le&~ \inf_{\hat f \in \cF_{\NT}^N(\overline \bW)} \E[(f_*(\bx) - \hat f(\bx) )^2] \le \inf_{\hat f \in \cG_{\NT}^N(\overline \bV)} \E[(f_*(\bx) - \hat f(\bx))^2]\\
=&~ \inf_{\hat f \in \cF_{\NT}^N(\overline \bV)} \E[(\vphi(\bz) - \hat f(\bz))^2] = (1 + o_{d, \P}(1)) \cdot \| \proj_{> \ell + 1} f_* \|_{L^2}^2. 
\end{aligned}
\]
Combining with Eq. (\ref{eqn:NN_NT_risk_relation}) gives that, when $d_0^{\ell + \delta} \le N \le d_0^{\ell + 1 - \delta}$, we have
\[
R_{\NN, N}(f_*) \le (1 + o_{d}(1)) \cdot \| \proj_{> \ell + 1} f_* \|_{L^2}^2. 
\]

\noindent
{\bf Step 3. Show that $R_{\NN, N}(f_*)$ is independent of $\kappa$. }

We let $\tilde r = d^{\tilde \kappa / 2}$ and $\mathring r = d^{\mathring r / 2}$ for some $\tilde \kappa \neq \mathring \kappa$. Suppose we have $\tilde \bx = \bU \tilde \bz_{1} + \bU^\perp \bz_{2}$ and $\mathring \bx = \bU \mathring \bz_{1} + \bU^\perp \bz_{2}$, where $\tilde \bz_{1} \sim \Unif(\S^{d_0 - 1}(\tilde r \sqrt d_0))$, $\mathring \bz_{1} \sim \Unif(\S^{d_0 - 1}(\mathring r \sqrt d_0))$, and $\bz_{2} \sim \Unif(\S^{d - d_0 - 1}(\sqrt{d - d_0}))$. Moreover, we let $\tilde f_*(\tilde \bx) = \vphi(\bU^\sT \tilde \bx / \tilde r)$ and $\mathring f_*(\mathring \bx) = \vphi(\bU^\sT \mathring \bx / \mathring r)$ for some function  $\vphi: \R^{d_0} \to \R$. 

Then, for any $\tilde \bW = (\tilde \bw_i)_{i \le N} \subseteq \R^{d}$ and $\tilde \bb = (\tilde b_i)_{i \le N} \subseteq \R$, there exists $(\tilde \bv_{1, i})_{i \le N} \subseteq \R^{d_0}$ and $(\tilde \bv_{2, i})_{i \le N} \subseteq \R^{d - d_0}$ such that $\tilde \bw_i = \bU \tilde \bv_{1, i} + \bU^\perp \tilde \bv_{2, i}$. We define $\mathring \bv_{1, i} = \tilde r \cdot \tilde \bv_{1, i} / \mathring r$, $\mathring \bw_i = \bU \mathring \bv_{1, i} + \bU^\perp \tilde \bv_{2, i}$, $\mathring \bW = (\mathring \bw_i)_{i \le N}$, and $\mathring \bb = \tilde \bb$. Then we have 
\[
 \E_{\mathring \bx}[(\mathring f_*(\mathring \bx) - f_{\NN, N}(\mathring \bx; \mathring \bW, \mathring \bb))^2] = \E_{\tilde \bx}[(\tilde f_*(\tilde \bx) - f_{\NN, N}(\tilde \bx; \tilde \bW, \tilde \bb))^2].
\]
On the other hand, for any $\mathring \bW = (\mathring \bw_i)_{i \le N} \subseteq \R^{d}$ and $\mathring \bb = (\mathring b_i)_{i \le N} \subseteq \R$, we can find $\tilde \bW = (\tilde \bw_i)_{i \le N} \subseteq \R^{d}$ and $\tilde \bb = (\tilde b_i)_{i \le N} \subseteq \R$ such that the above equation holds. This proves that $R_{\NN, N}(f_*)$ is independent of $\kappa$.

%\end{proof}

\clearpage

\section{Convergence of the Gegenbauer coefficients}

In this section, we prove a string of lemmas that are used to show convergence of the Gegenbauer coefficients. 

\subsection{Technical lemmas}

First recall that for $q \in [Q]$ we denote $\tau^{(q)} \equiv \| \btheta^{(q)} \|_2 / \sqrt{d_q} $ where $\btheta^{(q)}$ are the $d_q$ coordinates of $\btheta \sim \Unif ( \S^{D - 1} ( \sqrt{D}))$ associated to the $q$-th sphere of $\PS^\bd$. We show that $\tau^{(q)}$ is $(1/d_q)$-sub-Gaussian.

\begin{lemma}\label{lem:tail_tau}
There exists constants $c,C>0$ such that for any $\eps > 0$,
\[
\P ( | \tau^{(q)} -1| > \eps ) \leq C \exp ( - c d_q  \eps^2  ).
\]
\end{lemma} 

\begin{proof}[Proof of Lemma \ref{lem:tail_tau}]
Let $\bG \sim \normal ( 0 , \id_{D})$. We consider the random vector $\bU \equiv \bG / \| \bG \|_2 \in \R^D$. We have $\bU \sim \Unif ( \S^{D-1} ( 1))$. We denote $N_{d_q} = G_1^2 + \ldots + G_{d_q}^2$ and $N_{D} = G_1^2 + \ldots + G_{D}^2$. The random variable $\tau^{(q)}$ has the same distribution as
\[
\tau^\pq \equiv \| \btheta^\pq \|_2 / \sqrt{d_q}  {\stackrel{{\rm d}}{=}} \frac{\sqrt{N_{d_q} / d_q}}{\sqrt{N_D / D}}.
\]
Hence,
\begin{equation}\label{eq:decomposition_bound_tau}
\begin{aligned}
\P ( | \tau^\pq -1| > \eps ) =& \P \lp \Bigg\vert \frac{\sqrt{N_{d_q} / d_q}}{\sqrt{N_D / D}} - 1 \Bigg\vert > \eps \rp \\
\leq & \P  \Big( \Big\vert \sqrt{N_{d_q} / d_q} - 1 \Big\vert > \eps /2 \Big) + \P  \Big( \Big\vert \sqrt{N_D /D} - 1 \Big\vert > \eps /(2 +2 \eps) \Big),
\end{aligned}
\end{equation}
where we used the fact that 
\[
|a-1| \leq \frac{\eps}{2} \text{ and } |b-1| \leq \frac{\eps}{2+2\eps} \Rightarrow \Big\vert \frac{a}{b} - 1 \Big\vert \leq \eps.
\]                                                                   
Let us first consider $N_{d_q}$ with $\eps \in (0,2]$. The $G_i^2$ are sub-exponential random variables with
\[
\E \Big[ e^{\lambda(G_i^2 - 1)} \Big] \leq e^{2 \lambda^2}, \qquad \forall | \lambda | < 1/4.
\] 
From standard sub-exponential concentration inequality, we get
\begin{equation}\label{eq:subexponential_concentration}
\P  \Big( \Big\vert N_{d_q} / d_q - 1 \Big\vert > \eps  \Big) \leq 2 \exp \Big(- d_q \eps \min (1,\eps) /8 \Big).
\end{equation}
Hence, for $\eps \in (0,2]$, we have
\[
\P  \Big( \Big\vert \sqrt{N_{d_q} / d_q} - 1 \Big\vert > \eps /2 \Big) \leq  \P  \Big( \Big\vert N_{d_q} / d_q- 1 \Big\vert > \eps /2 \Big) \leq  2 \exp \Big( - d_q \eps^2 /32 \Big),
\]
while for $\eps > 2$, 
\[
\begin{aligned}
\P  \Big( \Big\vert \sqrt{N_{d_q} / d_q} - 1 \Big\vert > \eps /2 \Big) \leq   \P  \Big(  N_{d_q} / d_q  > (\eps/2 + 1 )^2 \Big) \leq & \P  \Big(  N_{d_q} / d_q  - 1 > \eps^2 / 4  \Big)  \\
\leq &  \exp \Big( - d_q \eps^2 /32 \Big).
\end{aligned}
\]
In the case of $N_D$, applying \eqref{eq:subexponential_concentration} with $\eps/(2 + 2 \eps) \leq 1$ shows that
\[
\begin{aligned}
\P  \Big( \Big\vert \sqrt{N_{D} / D} - 1 \Big\vert > \eps/(2 + 2 \eps) \Big) \leq & \P  \Big( \Big\vert N_{D} / D- 1 \Big\vert > \eps/(2 + 2 \eps) \Big) \\
\leq  & 2 \exp \Big( - D \eps^2 /(32 (1+\eps)^2 ) \Big).
\end{aligned}
\]
Combining the above bounds into \eqref{eq:decomposition_bound_tau} yields for $\eps \geq 0$,
\[
\begin{aligned}
\P ( | \tau^\pq -1| > \eps ) \leq & 2 \exp \Big( - d_q \eps^2 /32 \Big) + 2 \exp \Big(  - D \eps^2 /(32 (1+\eps)^2 ) \Big) \\
\leq & 4 \exp \Big(  - \eps^2 \min \Big( d_q , D/(1+\eps)^2  \Big) /32 \Big) .
\end{aligned}
\]
Notice that $| \tau^\pq -1| \leq  \sqrt{D/d_q} -1 $ and we only need to consider $\eps \in [ 0 , \sqrt{D/d_q} -1 ]$. We conclude that for any $\eps \geq 0$, we have
\[
\P ( | \tau^\pq -1| > \eps )  \leq 4 \exp \Big( - d_q \eps^2 /32 \Big).
\]
\end{proof}

We consider an activation function $\sigma : \R \to \R$. Fix $\btheta \in \S^{D-1} (\sqrt{D})$ and recall that $\bx = ( \bx^{(1)} , \ldots \bx^{(Q)} ) \in \PS^\bd_\bkappa$. We recall that $\bx \sim \Unif (\PS^\bd_\bkappa) = \mu_\bd^\bkappa$ while $\obx \sim \Unif (\PS^\bd ) = \mu_\bd$. Therefore, for a given $\obtheta$, $\lb \< \obtheta^\pq , \obx^\pq \> / \sqrt{d_q} \rb_{q \in [Q]} \sim \Tilde \mu_{\bd}^1$ as defined in Eq.~\eqref{eq:def_psd_mu1_d}. Therefore we reformulate $\sigma ( \< \btheta , \cdot \> / R)$ as a function $\sigma_{\bd,\btau}$ from $\ps^\bd$ to $\R$:
\begin{equation}\label{eq:def_sigma_d_tau_appendix}
\begin{aligned}
\sigma ( \< \btheta , \bx \> / R ) = & \sigma \lp \sum_{q \in [Q]}  \tau^\pq \cdot (r_q/R) \cdot  \< \obtheta^\pq , \obx^\pq \> / \sqrt{d_q} \rp \\
\equiv & \sigma_{\bd,\btau} \lp \lb \< \obtheta^\pq , \obx^\pq \> / \sqrt{d_q} \rb_{q \in [Q]} \rp.
\end{aligned}
\end{equation}
We will denote in the rest of this section $\alpha_q = \tau^\pq  r_q / R$ for $q = 1 , \ldots , Q$. Notice in particular that $\alpha_q \propto d^{\eta_q + \kappa_q - \xi} $ where we recall that $\xi = \max_{q \in [Q]} \lb \eta_q + \kappa_q \rb$. Without loss of generality, we will assume that the (unique) maximum is attained on the first sphere, i.e.\,$\xi= \eta_1 + \kappa_1$ and $\xi > \eta_q + \kappa_q$ for $q \ge 2$.

\begin{lemma}\label{lem:coupling_convergence_Gaussian}
Assume $\sigma$ is an activation function with $\sigma (u)^2 \leq c_0 \exp ( c_1 u^2 / 2)$ almost surely, for some constants $c_0 > 1$ and $c_1 < 1$. We consider the function $\sigma_{\bd,\btau} : \ps^\bd \to \R$ associated to $\sigma$, as defined in Eq.~\eqref{eq:def_sigma_d_tau}. 

Then
\begin{enumerate}
\item[$(a)$] $\E_{G \sim \normal(0, 1)}[\sigma(G)^2] < \infty$. 
\item[$(b)$] Let $\bw^\pq$ be unit vectors in $\R^{d_q}$ for $q = 1, \ldots , Q$. There exists $\eps_0 = \eps_0 (c_1)$ and $d_0=d_0(c_1)$ such that, for $\overline \bx = ( \bx^{(1)} , \ldots, \bx^{(Q)} ) \sim \mu_{\bd}^\bkappa$,
\begin{align}
\sup_{d \ge d_0} \, \sup_{\btau \in [1 - \eps_0, 1+ \eps_0]^Q } \E_{\overline \bx} \lsb \sigma_{\bd, \btau} \lp \lb \<\bw^\pq,\obx^\pq\> \rb_{q \in [Q]} \rp^2  \rsb < \infty\, .
\end{align}
\item[$(c)$] Let $\bw^\pq$ be unit vectors in $\R^{d_q}$ for $q = 1, \ldots , Q$. Fix integers $\bk = (k_1 , \ldots , k_Q) \in \posint^Q$. Then for any $\delta > 0$, there exists constants $\eps_0 = \eps_0 (c_1 , \delta)$ and $d_0 = d_0 ( c_1 , \delta)$, and a coupling of $G\sim\normal(0,1)$ and $\overline \bx = ( \bx^{(1)} , \ldots , \bx^{(Q)})  \sim \mu_\bd^\bkappa$ such that for any $d \geq d_0$ and $\btau \in [1- \eps_0 , 1+\eps_0]^Q$
\begin{align}
\begin{aligned}
  \E_{\obx, G} \Big[  \Big( \Big[ \prod_{q \in [Q]} \lp 1 - \< \bw^\pq, \obx^\pq \>^2 /d_q  \rp^{ k_q} \Big]   \sigma_{\bd, \btau} \lp \lb \<\bw^\pq,\obx^\pq\> \rb_{q \in [Q]} \rp - \sigma(G) \Big)^2\Big] <& \delta.
\end{aligned}
\end{align}
\end{enumerate}
\end{lemma}

\begin{proof}[Proof of Lemma \ref{lem:coupling_convergence_Gaussian}]

Part $(a)$ is straightforward. 

For part $(b)$, recall that the probability distribution of $\< \bw^\pq , \obx^\pq \>$ when $\obx^\pq \sim \Unif ( \S^{d_q - 1} ( \sqrt{d_q}) )$ is given by
\begin{align}
\tilde\tau^1_{d_q-1}(\de x) &= C_{d_q}\, \left(1-\frac{x^2}{d_q}\right)^{\frac{d_q-3}{2}}\bfone_{x\in [-\sqrt{d_q},\sqrt{d_q}]}\de x\, ,\label{eq:taud-def}\\
C_{d_q} & = \frac{\Gamma(d_q-1)}{2^{d_q-2}\sqrt{d_q} \,\Gamma((d_q-1)/2)^2}\, . \
\end{align}
A simple calculation shows that $C_{n} \to (2\pi)^{-1/2}$ as $n \to\infty$, and hence $\sup_n C_{n} \le \overline{C}<\infty$. 
Therefore for $\btau \in [1 - \eps , 1+ \eps]^Q$, we have
\[
\begin{aligned}
& \E_{\overline \bx}\lbb \sigma_{\bd, \btau} \lp \lb \<\bw^\pq,\obx^\pq\> \rb_{q \in [Q]} \rp^2 \rbb  \\
 = & \int_{\prod_{q \in [Q]} [-\sqrt{d_q} , \sqrt{d_q} ] } \sigma_{\bd, \btau} \lp \ox_1^{(1)} , \ldots , \ox_1^{(Q)}  \rp^2  \prod_{q \in [Q]} \lp C_{d_q}  \left(1-\frac{(\ox_1^\pq)^2}{d_q}\right)^{\frac{d_q-3}{2}}\, \de \ox_1^\pq \rp  \\
\le &  \overline{C}^Q \int_{\R^Q} c_0 \exp \lp  c_1 \lp \sum_{q \in [Q]} \alpha_q \ox_1^\pq \rp^2 /2 \rp  \prod_{q \in [Q]} \lp \exp \Big( - \frac{d_q - 3}{2d_q}(\ox_1^\pq)^2 \Big)  \,\de \ox_1^\pq \rp \\
= & c_0 \overline{C}^Q \int_{\R^Q}\exp \lp  - \obx_1^\sT \bM \obx_1 /2 \rp \lp \prod_{q \in [Q]} \de \ox_1^\pq \rp 
\end{aligned}
\]
where we denoted $\obx_1 = ( \ox_1^{(1)} , \ldots , \ox_1^{(Q)})$ and $\bM \in \R^{Q \times Q}$ with 
\[
M_{qq} = \frac{d_q - 3}{d_q} - c_1^2 \alpha_q^2  , \qquad M_{qq'} = -c_1 \alpha_q \alpha_{q'}, \qquad \text{for $q \neq q' \in [Q]$.}
\]
Recalling the definition of $\alpha_q = \tau^\pq r_q / R$, with $r_q = d^{(\eta_q + \kappa_q)/2}$ and $R = d^{\xi/2} (1+ o_d (1))$. Hence for any $\eps>0$, uniformly on  $\btau \in [1 - \eps, 1+ \eps]^Q$, we have $\alpha_q \to 0$ for $q \geq 2$ and $\lim\sup_{d\to\infty} |\alpha_1 -1 | \leq \eps$. Hence if we choose $\eps_0 < c_1^{-1} - 1$, there exists $c>0$ such that for $d$ sufficiently large $\bM \succeq c \id_Q$ and for any $\btau \in [1-\eps_0 , 1+ \eps_0]^Q$
\[
 \E_{\overline \bx} \lbb \sigma_{\bd, \btau} \lp \lb \<\bw^\pq,\obx^\pq\> \rb_{q \in [Q]} \rp^2 \rbb  \leq c_0 \overline{C}^Q \int_{\R^Q}\exp \lp  - c \| \obx_1 \|^2_2 /2 \rp \lp \prod_{q \in [Q]} \de \ox_1^\pq \rp < \infty. 
\]

Finally, for part $(c)$, without loss of generality we will take $\bw^\pq = \be^\pq_1$ so that $\< \bw^\pq , \obx^\pq \> = \ox^\pq_1$. From part $(b)$, there exists $\eps >0$ and $d_0$ such that
\[
\begin{aligned}
& \sup_{d \ge d_0} \sup_{\tau \in [1- \eps , 1+\eps]}  \E_{\overline \bx} \E_{\obx, G} \left[  \Big[ \prod_{q \in [Q]} \lp 1 - \< \bw^\pq, \obx^\pq \>^2 /d_q  \rp^{ 2 k_q} \Big]   \sigma_{\bd, \btau} \lp \lb \<\bw^\pq,\obx^\pq\> \rb_{q \in [Q]} \rp^2 \right] \\
\le &  \sup_{d \ge d_0} \sup_{\tau \in [1- \eps , 1+\eps]}  \E_{\overline \bx} \Big[ \sigma_{\bd, \btau} \lp \lb \<\bw^\pq,\obx^\pq\> \rb_{q \in [Q]} \rp^2 \Big] < \infty.
\end{aligned}
\]
Consider $\bG \sim \normal ( 0 , \id_Q )$ and an arbitrary coupling between $\obx$ and $\bG$. For any $M>0$ 
we can choose $\sigma_M$ bounded continuous so that
for any $d$ and $\btau \in [1 - \eps , 1 + \eps]^Q$,
\begin{equation} \label{eq:truncation_sigma_weak_cv}
\begin{aligned}
&\E_{\obx, \bG} \lbb \lp  \prod_{q \in [Q]} \lp1 - (\ox_1^\pq)^2 / d_q\rp^{k_q} \cdot \sigma \lp \sum_{q \in [Q]} \alpha_q \ox_1^\pq \rp   -  \prod_{q \in [Q]} \lp1 - G_q^2 / d_q\rp^{k_q} \cdot \sigma \lp \sum_{q \in [Q]} \alpha_q G_q \rp \rp^2 \rbb \\
\le & \E_{\obx, \bG} \lbb \lp  \prod_{q \in [Q]} \lp1 - (\ox_1^\pq)^2 / d_q\rp^{k_q} \cdot \sigma_M \lp \sum_{q \in [Q]} \alpha_q \ox_1^\pq \rp   -  \prod_{q \in [Q]} \lp1 - G_q^2 / d_q\rp^{k_q} \cdot \sigma_M \lp \sum_{q \in [Q]} \alpha_q G_q \rp \rp^2 \rbb +\frac{1}{M}\, .
\end{aligned}
\end{equation}
It is therefore sufficient to prove the claim for $\sigma_M$. Letting $\bxi_q \sim\normal(0,\id_{d_q-1})$ independently for each $q \in [Q]$ and independent of $\bG$, we construct the coupling via
\begin{align}\label{eq:coupling}
\ox^\pq_1 = \frac{G_q\sqrt{d_q}}{\sqrt{G_q^2+\|\bxi_q\|_2^2}}\, ,\;\;\; \obx_{- 1}^\pq = \frac{\bxi_q\sqrt{d_q}}{\sqrt{G_q^2+\|\bxi_q\|_2^2}}\, ,  \quad q \in [Q],
\end{align}
where we set $\obx^\pq = (\ox^\pq_1,\obx_{- 1}^\pq)$ for each $q \in [Q]$. We thus have $(\ox^\pq_1,\obx_{- 1}^\pq) \to \bG$ almost surely, hence the limit superior of Eq.~\eqref{eq:truncation_sigma_weak_cv} is by weak convergence bounded by $1/M$ for any arbitrary $M$. Furthermore, noticing that $\alpha_q \to 0$ uniformly on $\btau \in [1 - \eps , 1 + \eps]^Q$ for $q \ge 2$, we have by bounded convergence
\begin{equation}\label{eq:bound_coupling_2}
\lim_{d \to \infty} \sup_{\btau \in [1 - \eps,1+\eps]^Q} \E_{\obx, \bG} \lbb \lp  \prod_{q \in [Q]} \lp1 - G_q^2 / d_q\rp^{k_q} \cdot \sigma \lp \sum_{q \in [Q]} \alpha_q G_q \rp   -   \sigma \lp \alpha_1 G_1 \rp \rp^2 \rbb = 0.
\end{equation}
We further have $\lim_{(d,\tau^{(1)} ) \to (\infty , 1)} \alpha_1 = 1$. Hence, by bounded convergence,
\begin{equation}\label{eq:bound_coupling_3}
\lim_{(d,\tau^{(1)}) \to (\infty , 1)} \E_{G_1} \lbb \lp    \sigma \lp \alpha_1 G_1 \rp  - \sigma(G_1)\rp^2 \rbb = 0.
\end{equation}
Combining Eq.~\eqref{eq:truncation_sigma_weak_cv} with the coupling \eqref{eq:coupling} and Eqs~\eqref{eq:bound_coupling_2} and \eqref{eq:bound_coupling_3} yields the result.
\end{proof}

Consider the expansion of $\sigma_{\bd,\btau}$ in terms of tensor product of Gegenbauer polynomials. We have
\[
\sigma (\<\btheta , \bx \>/ R ) =  \sum_{\bk \in \posint^Q} \lambda^{\bd}_{\bk} (\sigma_{\bd,\btau}  ) B(\bd,\bk) Q^{\bd}_{\bk} \lp \lb \< \obtheta^\pq , \obx^\pq \> \rb_{q \in [Q]} \rp, 
\]
where
\[
\lambda^{\bd}_{\bk} (\sigma_{\bd,\btau} ) =  \E_{\obx} \Big[ \sigma_{\bd,\btau} \lp \ox_1^{(1)} , \ldots , \ox^{(Q)}_1 \rp  Q^{\bd}_{\bk} \Big( \sqrt{d_1}\ox_1^{(1)} , \ldots , \sqrt{d_Q} \ox_1^{(Q)} \Big)  \Big].
\]
with the expectation taken over $\obx = (\obx^{(1)} , \ldots , \obx^{(Q)} ) \sim \mu_{\bd} \equiv \Unif ( \PS^\bd )$. We will need the following lemma, which is direct consequence of Rodrigues formula, to get the scaling of the Gegenbauer coefficents of $\sigma_{\bd,\btau}$.

\begin{lemma}\label{lem:formula_gegenbauer_prod_gegenbauer}
Let $\bk = (k_1 , \ldots , k_Q) \in \posint^Q$ and denote $|\bk | = k_1 + \ldots + k_Q$. Assume that the activation function $\sigma$ is $|\bk|$-times weakly differentiable and denote $\sigma^{(|\bk|)}$ its $|\bk|$-weak derivative. Let $\alpha_q = \tau^\pq r_q / R$ for $q = 1 , \ldots , Q$. Then
\begin{equation}\label{eq:formula_gegenbauer_coefficients}
\lambda^{\bd}_{\bk} ( \sigma_{\bd,\btau}  ) = \lp \prod_{q \in [Q]} \alpha^{k_q}_q \rp \cdot   R(\bd,\bk) \cdot \E_{\obx } \left[ \lp \prod_{q \in [Q]} \lp 1 - \frac{(\ox_1^\pq)^2}{d_q} \rp^{k_q} \rp \cdot \sigma^{(|\bk|)} \lp \sum_{q \in [Q]}  \alpha_q \ox_1^\pq \rp  \right],
\end{equation}
where $\obx  \sim \Unif( \PS^\bd ) $ and
\[
R(\bd,\bk) = \prod_{q \in [Q]} \frac{d_q^{k_q/2} \Gamma ((d_q - 1)/2)}{2^{k_q} \Gamma ( k_q + (d_q - 1)/2)} .
\]
Furthermore,
\begin{equation} \label{eq:lim_BR}
\lim_{d \to \infty} B (\bd , \bk ) R ( \bd , \bk )^2  = \frac{1}{\bk !},
\end{equation}
where $\bk ! = k_1 ! \ldots k_Q!$.
\end{lemma}

\begin{proof}[Proof of Lemma \ref{lem:formula_gegenbauer_prod_gegenbauer}]
We have
\begin{equation}\label{eq:double_integration_gegenbauer}
\begin{aligned}
& \lambda^{\bd}_{\bk} (\sigma_{\bd,\btau} )\\
 =  &  \E_{\obx} \Big[ \sigma_{\bd,\btau} \lp \ox_1^{(1)} , \ldots , \ox^{(Q)}_1 \rp  Q^{\bd}_{\bk} \Big( \sqrt{d_1}\ox_1^{(1)} , \ldots , \sqrt{d_Q} \ox_1^{(Q)} \Big)  \Big] \\
=& \E_{\obx^{(1)}, \ldots , \obx^{(Q-1)}} \lbb \E_{\obx^{(Q)}} \lbb  \sigma \lp \sum_{q \in [Q-1]} \alpha_q \ox^\pq_1 + \alpha_Q \ox^{(Q)}_1 \rp Q^{(d_Q)}_{k_Q} ( \sqrt{d_Q} \ox^{(Q)}_1 ) \rbb \prod_{q \in [Q-1]}  Q^{(d_q)}_{k_q} (   \sqrt{d_q} \ox^\pq_1)\rbb,
\end{aligned}
\end{equation}
where we used the definition \eqref{def:prod_gegenbauer} of tensor product of Gegenbauer polynomials.

Consider the integration with respect to $\obx^{(Q)}$. Denote for ease of notations $u = \alpha_1 \ox^{(1)}_1 + \ldots + \alpha_{Q-1} \ox^{(Q-1)}_1$. We use the Rodrigues formula for the Gegenbauer polynomials (see Eq.~\eqref{eq:Rogrigues_formula}):
\begin{equation}\label{eq:integration_z_gegenbauer}
\begin{aligned}
& \E_{\obx^{(Q)} \sim \Unif (\S^{d_Q - 1} ( \sqrt{d_Q} ) )}  \Big[ \sigma \lp u + \alpha_Q \ox^{(Q)}_1 \rp  Q^{(d_Q)}_{k_Q} \Big( \sqrt{d_Q} \ox^{(Q)}_1 \Big) \Big] \\
 =& \frac{\omega_{d_Q-2}}{\omega_{d_Q-1}} \int_{[-1, 1]} \sigma \lp \alpha_Q \sqrt{d_Q} t+ u \rp  Q^{(d_Q)}_{k_Q} ( d_Q t )   (1 - t^2)^{(d_Q-3)/2} \de t \\
=& (-1/2)^{k_Q} \frac{\Gamma((d_Q - 1)/2)}{\Gamma(k_Q + (d_Q- 1)/2)} \cdot \frac{\omega_{d_Q-2}}{\omega_{d_Q-1}} \int_{[-1, 1]}\sigma \lp \alpha_Q \sqrt{d_Q} t+ u\rp  \Big( \frac{\de }{\de t}\Big)^{k_Q} (1 - t^2)^{k_Q + (d_Q-3)/2} \de t \\
=& \alpha_Q^{k_Q} 2^{-k_Q} d_Q^{k_Q/2} \frac{\Gamma((d_Q - 1)/2)}{\Gamma(k_Q + (d_Q - 1)/2)}\cdot \frac{\omega_{d_Q-2}}{\omega_{d_Q-1}} \int_{[-1, 1]} (1 - t^2)^{k_Q} \sigma^{(k_Q)} \lp \alpha_Q \sqrt{d_Q} t+ u \rp   (1 - t^2)^{(d_Q-3)/2} \de t \\
 = & \alpha_Q^{k_Q} \frac{d_Q^{k_Q/2} \Gamma ((d_Q - 1)/2)}{2^{k_Q}\Gamma ( k_Q + (d_Q - 1)/2)} \E_{\obx^{(Q)} \sim \Unif (\S^{d_Q- 1} ( \sqrt{d_Q} ) )} \lbb \lp 1 - (\ox_1^{(Q)})^2/d_Q \rp^{k_Q} \sigma^{(k_Q)} \lp \alpha_Q \ox^{(Q)}_1+ u\rp \rbb.
\end{aligned}
\end{equation}
Iterating Eq.~\eqref{eq:integration_z_gegenbauer} over $q \in [Q]$ and Eq.~\eqref{eq:double_integration_gegenbauer} yield the desired formula \eqref{eq:formula_gegenbauer_coefficients}.

Furthermore, for each $q \in [Q]$,
\[
\begin{aligned}
k_q ! B(d_q , k_q ) = & (2k_q + d_q - 2) \prod_{j=0}^{k_q-2} ( j + d_q -1 ), \\
 \frac{\Gamma ((d_q - 1)/2) }{2^{k_q} \Gamma (k_q+ (d_q - 1)/2)} =&  \prod_{j=0}^{k_q-1} \frac{1}{2j + d_q - 1 }.
\end{aligned}
\]
Combining these two equations yields
\begin{equation}\label{eq:bound_B_Rsquared_vphi}
\begin{aligned}
& k_q ! B(d_q , k_q ) \frac{d_q^{k_q} \Gamma ((d_q - 1)/2)^2 }{2^{2k_q} \Gamma (k_q+ (d_q - 1)/2)^2}  \\
= & \frac{2k_q + d_q - 2 }{2k_q + d_q - 3} \cdot  \lp \prod_{j=0}^{k_q-2} \frac{j + d_q - 1 }{2j + d_q - 1} \rp \cdot  \lp \prod_{j=0}^{k_q-1} \frac{d_q }{2j + d_q - 1} \rp,
\end{aligned}
\end{equation}
which converges to $1$ when $d_q \to \infty$. We deduce that
\[
\lim_{d \to \infty} B (\bd ,\bk) R ( \bd ,\bk )^2  = \frac{1}{\bk !}.
\]
\end{proof}

\subsection{Proof of convergence in probability of the Gegenbauer coefficients}

\begin{lemma}\label{lem:convergence_proba_Gegenbauer_coeff}
Let $\bk = (k_1 , \ldots , k_Q) \in \posint^Q$ and denote $| \bk | = k_1 + \ldots + k_Q$. Assume that the activation function $\sigma$ is $|\bk|$-times weakly differentiable and denote $\sigma^{(|\bk|)}$ its $|\bk|$-weak derivative. Assume furthermore that there exist constants $c_0 > 0$ and $c_1 < 1$ such that $\sigma^{(|\bk|)} (u)^2 \leq c_0 \exp ( c_1 u^2 / 2)$ almost surely. 

Then for any $\delta > 0$, there exists $\eps_0 \in ( 0 ,1)$ and $d_0$ such that for any $d \geq d_0$ and $\btau \in [ 1- \eps_0 , 1 + \eps_0]^Q$,
\[
\Bigg\vert \lp \prod_{q \in [Q] } d^{(\xi - \eta_q - \kappa_q)k_q} \rp B(\bd,\bk) \lambda^{\bd}_{\bk} ( \sigma_{\bd,\btau}  )^2 - \frac{\mu_{|\bk|} ( \sigma)^2 }{\bk !} \Bigg\vert \leq \delta.
\]
\end{lemma}

\begin{proof}[Proof of Lemma \ref{lem:convergence_proba_Gegenbauer_coeff}]
From Lemma \ref{lem:formula_gegenbauer_prod_gegenbauer}, we have
\begin{equation}\label{eq:formula_conv_proba_gegenbauer}
\begin{aligned}
&\lp \prod_{q \in [Q] } d^{(\xi - \eta_q - \kappa_q)k_q} \rp B(\bd,\bk) \lambda^{\bd}_{\bk} ( \sigma_{\bd,\btau}  )^2 \\
= &\lp \prod_{q \in [Q]} \alpha^{2k_q}_q d^{(\xi - \eta_q - \kappa_q)k_q}  \rp \cdot   [ B(\bd , \bk) R(\bd,\bk)^2] \\
& \phantom{AAAAAAAAA} \times \E_{\obx } \left[\prod_{q \in [Q]} \lp 1 - \frac{(\ox_1^\pq)^2}{d_q} \rp^{k_q}  \cdot \sigma^{(|\bk|)} \lp \sum_{q \in [Q]}  \alpha_q \ox_1^\pq \rp  \right]^2.
\end{aligned}
\end{equation}
Recall $\alpha_q = \tau^\pq r_q /R$ with $r_q = d^{(\kappa_q + \eta_q)/2}$ and $R = d^{\xi/2} (1 + o_d (1) )$. Hence, we have
\begin{equation}\label{eq:conv_gegenbauer_bound_1}
\begin{aligned}
\lim_{(d,\btau) \to ( \infty , \bone )}  \prod_{q \in [Q]} \alpha^{2k_q}_q d^{(\xi - \eta_q - \kappa_q)k_q}  & = 1.
\end{aligned}
\end{equation}
Furthermore, from Lemma \ref{lem:formula_gegenbauer_prod_gegenbauer}, we have
\begin{equation}\label{eq:conv_gegenbauer_bound_1_5}
\begin{aligned}
\lim_{d \to \infty} B (\bd , \bk ) R ( \bd , \bk )^2  & = \frac{1}{\bk !}.
\end{aligned}
\end{equation}
We can apply Lemma \ref{lem:coupling_convergence_Gaussian} to the activation function $\sigma^{(|\bk|)}$. In particular part $(c)$ of the lemma implies that there exists $\eps_0 \in (0,1)$ such that for $d$ sufficiently large, we have for any $\btau \in [1- \eps_0 , 1+\eps_0]^Q$,
\begin{equation}\label{eq:conv_gegenbauer_bound_2}
 \Bigg\vert E_{\obx } \lbb  \lp \prod_{q \in [Q]} \lp 1 - \frac{(\ox_1^\pq)^2}{d_q} \rp^{k_q} \rp \cdot \sigma^{(|\bk|)} \lp \sum_{q \in [Q]}  \alpha_q \ox_1^\pq \rp \rbb -  \E_{G} [ \sigma^{(|\bk|)} ( G) ] \Bigg\vert \leq \delta/2.
\end{equation}
From Eq.~\eqref{eq:weak_derivative_hermite_coefficient}, we have $\E_{G} [ \sigma^{(|\bk|)} ( G) ] = \mu_{|\bk|} ( \sigma)$. Combining Eqs.~\eqref{eq:conv_gegenbauer_bound_1} and \eqref{eq:conv_gegenbauer_bound_2} into Eq.~\eqref{eq:formula_conv_proba_gegenbauer} yields the result.
\end{proof}

\begin{lemma}\label{lem:convergence_Gegenbauer_coeff_0_l}
Let $k$ be a non negative integer and denote $\bk = (k ,0, \ldots , 0) \in \posint^Q$, where we recall that without loss of generality we choose $q = 1$ as the unique $\arg \max_{q \in [Q]} \lb \eta_q + \kappa_q \rb$. Assume that the activation function $\sigma$ verifies $\sigma (u)^2 \leq c_0 \exp ( c_1 u^2 / 2)$ almost surely for some constants $c_0 > 0$ and $c_1 <1$. 

Then for any $\delta > 0$, there exists $\eps_0 = \eps_0 (c_1 , \delta)$ and $d_0 = d_0 ( c_1 , \delta)$ such that for any $d \geq d_0$ and $\btau \in [ 1- \eps_0 , 1 + \eps_0]^Q$,
\[
\Bigg\vert  B(d_{1},k) \lambda^{\bd}_{\bk} ( \sigma_{\bd,\btau}  )^2 - \frac{\mu_{k} ( \sigma)^2 }{ k !} \Bigg\vert \leq \delta.
\]
\end{lemma}

\begin{proof}[Proof of Lemma \ref{lem:convergence_Gegenbauer_coeff_0_l}]
Recall the correspondence \eqref{eq:Gegen-to-Hermite} between Gegenbauer and Hermite polynomials. Note for any monomial $m_l (x) = x^k$, we can apply Lemma \ref{lem:coupling_convergence_Gaussian}.$(c)$ to $m_l ( \ox_1^{(q_\xi)} ) \sigma$ and find a coupling such that for any $\eta >0$, there exists $\eps_0 >0$ and
\begin{equation}\label{eq:cv_gegenbauer_monomials}
\lim_{d \to \infty} \sup_{\btau \in [1- \eps_0 , 1+\eps_0]^Q}  \E_{\obx, G} \Big[\Big(m_k ( \ox_1^{(q_\xi)} ) \sigma_{\bd, \btau} ( \ox^{(1)}_1 , \ldots , \ox^{(Q)}_1) - m_k(G) \sigma(G) \Big)^2\Big] \leq \eta.
\end{equation}
We have 
\[
 [B(d_1 , k) k !]^{1/2} \lambda^{\bd}_{\bk} ( \sigma_{\bd,\btau}  ) = \E_{\obx} [ \sigma_{\bd,\btau} ( \ox_1^{(1)} , \ldots, \ox_1^{(Q)} ) Q^{(d_1)}_k ( \sqrt{d_1} \ox^{(1)}_1 )  [B(d_1 , k) k !]^{1/2} ].
\]
Using the asymptotic correspondence between Gegenbauer polynomials and Hermite polynomials \eqref{eq:Gegen-to-Hermite} 
\[
\lim_{d \to \infty} \Coeff\{ Q_k^{(d)}( \sqrt d x) \, B(d, k)^{1/2} \} = \Coeff\left\{ \frac{1}{(k!)^{1/2}}\,\He_k(x) \right\}\, ,
\]
and Eq.~\eqref{eq:cv_gegenbauer_monomials}, we get for any $\delta >0$, there exists $\eps_0 >0$ such that for $d$ sufficiently large, we have for any $\btau \in [1- \eps_0 , 1+\eps_0]^Q$,
\[
\begin{aligned}
\Big\vert E_{\obx} \lbb \sigma_{\bd,\btau} ( \ox^{(1)}_1 , \ldots , \ox^{(Q)}_1)  Q^{(d_1)}_k ( \sqrt{d_1} \ox^{(1)}_1 )  [B(d_1 , k) k !]^{1/2} \rbb - \E_{G} [ \sigma(G) \He_k (G)] \Big\vert \leq \delta ,
\end{aligned}
\]
which concludes the proof.
\end{proof}

\clearpage

\section{Bound on the operator norm of Gegenbauer polynomials}

\begin{proposition}[Bound on the Gram matrix]\label{prop:Delta_bound_aniso}
Let $\bk \in \posint^Q$ and denote $\gamma = \sum_{q \in [Q]} \eta_q k_q$. Let $n \le d^{\gamma}/e^{A_d\sqrt{\log d}}$ for any $A_d\to\infty$. Let $(\obx_i)_{i \in [n]}$ with $\obx_i = ( \lb \obx^\pq_i \rb_{q \in [Q]} ) \sim \Unif ( \PS^\bd )  $ independently, and $Q_{k_q}^{(d_q)}$ be the $k_q$'th Gegenbauer polynomial
with domain $[-d_q, d_q]$. Consider the random matrix $\bW =
(\bW_{ij})_{i, j \in [n]} \in \R^{n \times n}$, with 
\[
\bW_{ij} =
Q^\bd_{\bk} ( \lb \< \obx_i^\pq , \obx_j^\pq \> \rb_{q \in [Q] } ) = \prod_{q\in[Q]} Q^{(d_q)}_{k_q} (  \< \obx^\pq_i , \obx^\pq_j \> ).
\]
Then we have 
\[
\lim_{d, n \to \infty} \E[\| \bW - \id_n \|_{\op}] = 0. 
\]
\end{proposition}

\begin{corollary}[Uniform bound on the Gram matrix]\label{coro:Delta_bound_aniso_unif}
Let $n \le d^{\gamma}/e^{A_d\sqrt{\log d}}$ for some $\gamma > 0$ and any $A_d\to\infty$. Let $(\obx_i)_{i \in [N]}$ with $\obx_i = ( \lb \obx^\pq_i \rb_{q \in [Q]} ) \sim \Unif ( \PS^\bd )  $ independently. Consider for any $\bk \in \posint^Q$, the random matrix $\bW_\bk =
((\bW_{\bk})_{ij})_{i, j \in [n]} \in \R^{n \times n}$ as defined in Proposition \ref{prop:Delta_bound_aniso}. Denote:
\[
\cQ = \blb \bk \in \posint^Q \Big\vert \sum_{q \in [Q]} \eta_q k_q < \gamma \brb.
\]
Then we have 
\[
\sup_{\bk \in \cQ^c} \E[\| \bW_\bk - \id_n \|_{\op}] = o_{d,\P} (1). 
\]
\end{corollary}

\begin{proof}[Proof of Corollary \ref{coro:Delta_bound_aniso_unif}]
For each $q \in [Q]$, we consider $\bDelta^\pq = \bW^\pq_k - \id_n$ where $\bW^\pq_k = ( (\bW^\pq_k)_{ij} )_{i,j \in [n]}$ with 
\[
(\bW^\pq_{k})_{ij} = Q^{(d_q)}_k ( \< \obx_i^\pq , \obx^\pq_j  \> ).
\]
Then, defining $\gamma_q \equiv \gamma/ \eta_q$, we have
\[
\begin{aligned}
&\E\Big[\sup_{k \ge 2 \gamma_q + 3} \| \bW_k^\pq - \id_n \|_{\op}^2 \Big] \le   \E\Big[\sum_{k \ge 2 \gamma_q + 3} \| \bW_k^\pq - \id_n \|_F^2 \Big] \\
 =&  n(n-1) \sum_{k \ge 2 \gamma_q + 3} \E [ Q_k^{(d_q)} ( \< \obx^\pq , \oby^\pq \> )^2] = n (n-1) \sum_{k \ge 2 \gamma_q + 3} B(d_q, k)^{-1}.
\end{aligned}
\]
For $d$ sufficiently large, there exists $C>0$  such that for any $ p \geq m \equiv \lceil 2 \gamma_q+3 \rceil$:
\[
\begin{aligned}
\frac{B(d_q,m)}{B(d_q,p)} = \prod_{k = m}^{p-1} \frac{(2k+d_q-2)}{(2k +d_q)}\cdot \frac{(k+1)}{(k+d_q-2)}  \leq  &  \prod_{k = m}^{p-1} \frac{1}{1 + (d_q-3)/(k+1)} \\
 \leq & \prod_{k = m}^{p-1} e^{- \frac{m+1}{d_q-2 + m}\cdot \frac{d_q - 2 }{k+1}} \leq \frac{C}{p^2}\, . 
\end{aligned}
\]
Hence, there exists constant $C'$, such that for large $d$, we have
\[
\sum_{k \ge 2 \gamma_q+ 3} B(d_q, k)^{-1} \le C' \cdot B(d_q, m)^{-1}. 
\]
Recalling that $B(d_q, m) = \Theta_d(d^{\eta_q m}) = \omega_d(d^{2\gamma})$, and $n = o_d(d^{\gamma})$, we deduce
\begin{equation}\label{eqn:W_after_2l_terms}
\E\Big[\sup_{k \ge 2 \gamma_q + 3} \| \bW_k^\pq - \id_n \|_{\op}^2 \Big] = o_{d}(1).
\end{equation}

Let us now consider $\bDelta = \bW_{\bk} - \id_n$. We will denote $\bDelta^\pq = \bW^\pq_{k_q} - \id_n$. Then it is easy to check (recall the diagonal elements of $\bW^{(d_q)}_{k_q} $ are equal to one) that for any $q \in [Q]$
\[
\begin{aligned}
\bDelta = \Big(\bigodot_{q' \neq q} \bW^\pqp_{k_{q'}} \Big) \odot \bDelta^\pq
\end{aligned}
\]
where $\bA \odot \bB$ denotes the Hadamard product, or entrywise product, $( \bA \odot \bB )_{i,j \in [n]} = ( A_{ij} B_{ij} )_{i,j \in [n]}$. We recall the following inequality on the operator norm of Hadamard product of two matrices, with $\bA$ positive definite:
\[
\| \bA \odot \bB \|_{\op} \leq \Big( \max_{ij} \bA_{ij} \Big) \| \bB \|_{\op}.
\]
Hence, in particular
\[
\| \bDelta \|_{\op} \leq  \Big(\prod_{q' \neq q} \max_{ij} [ (\bW^\pqp_{k_{q'}} )_{ij} ] \Big) \| \bDelta^\pq\|_{\op}
\]
Consider $\cI = [0 , 2\gamma_1 + 3 [ \times \ldots \times [0 , 2\gamma_Q + 3[ \cap \posint^Q$. Then, from Eq.~\eqref{eqn:W_after_2l_terms}, we get directly
\begin{equation}\label{eq:uniform_bound_first}
\sup_{\bk \in \cI^c} \| \bW_{\bk} - \id_n \|_{\op} = o_{d,\P}(1).
\end{equation}
Furthermore, $\cI \cap \cQ$ is finite and from Proposition \ref{prop:Delta_bound_aniso}, we directly get
\begin{equation}\label{eq:uniform_bound_second}
\sup_{\bk \in \cI \cap \cQ } \| \bW_{\bk} - \id_n \|_{\op} = o_{d,\P}(1).
\end{equation}
Combining bounds \eqref{eq:uniform_bound_first} and \eqref{eq:uniform_bound_second} yields the result.

\end{proof}

\subsection{Proof of Proposition \ref{prop:Delta_bound_aniso}}

The proof follows closely the proof of the uniform case presented in \cite{ghorbani2019linearized}. For completeness, we copy here the relevant lemmas.

\noindent
\textbf{Step 1. Bounding operator norm by moments.}

Denote $\bDelta = \bW - \id_n$. We define for each $q \in [Q]$, $\bW^{(d_q)}_{k_q} = (Q_{k_q}^{(d_q)}(\< \obx_i^\pq, \obx_j^\pq \>) )_{ij\in[n]}$ and $\bDelta^\pq =\bW^{(d_q)}_{k_q} - \id_n$. Then it is easy to check (recall the diagonal elements of $\bW^{(d_q)}_{k_q} $ are equal to one)
\[
\begin{aligned}
\bDelta = \bDelta^{(1)} \odot \ldots \odot  \bDelta^{(Q)},
\end{aligned}
\]
where $\bA \odot \bB$ denotes the Hadamard product, or entrywise product, $( \bA \odot \bB )_{i,j \in [n]} = ( A_{ij} B_{ij} )_{i,j \in [n]}$.
For any sequence of integers $p=p(d)$, we have 
\begin{align}
\E[\|\bDelta\|_{\op}]\le \E[\Trace(\bDelta^{2p})^{1/(2p)}]\le \E[\Trace(\bDelta^{2p})]^{1/(2p)}\label{eq:MomentBound}
\end{align}
To prove the proposition, it suffices to show that for any sequence $A_d \to \infty$, we have
\begin{equation}
\lim_{d, n \to \infty, n = O_d(d^{\gamma}  e^{- A_d \sqrt{ \log d} } )} \E[\Trace( \bDelta^{2p} ) ]^{1/(2p)} = 0. 
\end{equation}
In the following, we calculate $\E[\Trace(\bDelta^{2p})]$. 
We have
\[
\begin{aligned}
\E[\Trace(\bDelta^{2p})] = & \sum_{\bi = (i_1, \ldots, i_{2p}) \in [n]^{2p}} \E[\Delta_{i_1 i_2} \Delta_{i_2 i_3} \ldots \Delta_{i_{2p} i_1}]\\
 =&  \sum_{\bi = (i_1, \ldots, i_{2p}) \in [n]^{2p}} \prod_{q \in [Q]} \E[ \Delta^\pq_{i_1 i_2} \Delta^\pq_{i_2 i_3} \ldots \Delta^\pq_{i_{2p} i_1}] ,
\end{aligned}
\]
where we used that $\obx^\pq$ and $\obx^{(q')}$ are independent for $q \neq q'$. 

We will denote for any $\bi = (i_1, \ldots, i_k) \in [n]^k$, define for each $q \in [Q]$
\[
M^\pq_{\bi} = \begin{cases} \E[\Delta^\pq_{i_1 i_2} \cdots \Delta^\pq_{i_k i_1}] & k \ge 2, \\
1&  k =1\, .
\end{cases}
\]
Similarly, we define $M_{\bi}$ associated to $\bDelta$,
\[
M_{\bi} = \prod_{q \in [Q]} M^\pq_{\bi}.
\]

To calculate these quantities, we will apply repeatedly the following identity, which is an immediate consequence of  Eq.~\eqref{eq:ProductGegenbauer}.
For any $i_1, i_2, i_3$ distinct, we have 
\[
\E_{\btheta_{i_2}}[\Delta_{i_1 i_2}^\pq \Delta_{i_2 i_3}^\pq] = \frac{1}{B(d_q,k_q)}\Delta_{i_1 i_3}^\pq. 
\]
Throughout the proof, we will denote by $C,C',C''$ constants that may depend on $k$ but not on $p,d,n$. The value of these constants is allowed to change from line to line.

\noindent
\textbf{Step 2. The induced graph and equivalence of index sequences.}

For any index sequence $\bi = (i_1, i_2, \ldots, i_{2p}) \in [n]^{2p}$, we defined an undirected multigraph $G_\bi = (V_\bi, E_\bi)$ associated to index sequence $\bi$. The vertex set $V_\bi$ is the set of distinct elements in $i_1, \ldots, i_{2p}$. The edge set $E_{\bi}$ is formed as follows: for any $j \in [2 p]$  we add an edge between $i_j$ and $i_{j+1}$ (with convention $2 p + 1 \equiv 1$). Notice that this could be a self-edge, or a repeated edge:  $G_\bi = (V_\bi, E_\bi)$  will be --in general-- a multigraph. We denote $v(\bi) = \vert V_\bi \vert$ to be the number of vertices of $G_\bi$, and $e(\bi) = \vert E_\bi \vert$ to be the number of edges (counting  multiplicities). In particular,  $e(\bi) = k$ for $\bi \in [n]^k$. We define
\[
\cT_\star(p) = \{\bi \in [n]^{2p}: G_\bi \text{ does not have self edge}\}. 
\]

For any two index sequences $\bi_1, \bi_2$, we say they are equivalent $\bi_1 \asymp \bi_2$, if  the two graphs $G_{\bi_1}$ and $G_{\bi_2}$
are isomorphic, i.e. there exists an edge-preserving bijection of their vertices (ignoring vertex labels). We denote the equivalent class of $\bi$ to be 
\[
\cC(\bi) = \{ \bj: \bj \asymp \bi\}.
\]
We define the quotient set $\cQ(p)$ by
\[
\cQ(p) = \{ \cC(\bi): \bi \in [n]^{2p} \}. 
\]

The following Lemma was proved in \cite[Proposition 3]{ghorbani2019linearized}

\begin{lemma}\label{lem:equivalent_class}
The following properties holds for all sufficiently large $n$ and $d$:
\begin{itemize}
\item[$(a)$] For any equivalent index sequences $\bi = (i_1, \ldots, i_{2p}) \asymp \bj = (j_1, \ldots, j_{2p})$, we have $M_{\bi}^\pq = M_{\bj}^\pq$. 
\item[$(b)$] For any index sequence $\bi \in [n]^{2p} \setminus \cT_\star(p)$, we have $M_{\bi} = 0$. 
\item[$(c)$] For any index sequence $\bi \in \cT_\star(p)$, the degree of any vertex in $G_\bi$ must be even. 
\item[$(d)$] The number of equivalent classes $\vert \cQ(p) \vert \le (2p)^{2p}$. 
\item[$(e)$] Recall that $v(\bi) = \vert V_\bi\vert$ denotes the number  of distinct elements in $\bi$. Then, for any 
$\bi\in [n]^{2p}$, the number of elements in the corresponding equivalence class satisfies $\vert \cC(\bi)\vert \le v(\bi)^{v(\bi)} \cdot n^{v(\bi)}\le p^p n^{v(\bi)}$. 
\end{itemize}
\end{lemma}

In view of property $(a)$ in the last lemma, given an equivalence class $\cC=\cC(\bi)$, we will write $M_{\cC} = M_{\bi}$ for the corresponding value.

\noindent
\textbf{Step 3. The skeletonization process.} 

For  multi-graph $G$, we say that one of its vertices is \emph{redundant}, if it has degree 2. 
For any index sequence $\bi \in \cT_\star(p) \subset [n]^{2p}$ (i.e. such that $G_\bi$ does not have self-edges), we denote by $r(\bi) \in \N_+$ to be the redundancy of $\bi$, and by
 $\sk(\bi)$ to be the skeleton of $\bi$, both defined  by the following skeletonization process. Let $\bi_0 = \bi \in [n]^{2p}$. 
For any integer $s\ge 0$, if $G_{\bi_s}$ has no redundant vertices then stop and set $\sk(\bi)= \bi_s$.
Otherwise, select a redundant vertex $\bi_s(\ell)$ arbitrarily (the $\ell$-th element of $\bi_s$). If 
$\bi_s(\ell-1) \neq \bi_s(\ell+1)$, then remove $\bi_s(\ell)$ from the graph (and from the sequence), together with its adjacent edges, and connect $\bi_s(\ell-1)$ and
$\bi_s(\ell+1)$ with an edge, and denote $\bi_{s+1}$ to be the resulting index sequence, i.e., $\bi_{s+1} = (\bi_s(1), \ldots, \bi_s(\ell - 1), \bi_s(\ell + 2), \ldots, \bi_s(\endd))$. If $\bi_s(\ell-1) = \bi_s(\ell+1)$, then remove $\bi_s(\ell)$ from the graph (and from the sequence), together with its adjacent edges, and denote $\bi_{s+1}$ to be the resulting index sequence, i.e., $\bi_{s+1} = (\bi_s(1), \ldots, \bi_s(\ell - 1), \bi_s(\ell + 1), \bi_s(\ell + 2), \ldots, \bi_s(\endd))$. (Here $\ell+1$, and $\ell-1$ have to be interpreted modulo $\vert \bi_s \vert$, the length of $\bi_s$.) The redundancy of  $\bi$, denoted by $r(\bi)$,  is the number of vertices removed during the skeletonization process. 

It is easy to see that the outcome of this process is independent of the order in which we select vertices.

\begin{lemma}\label{lem:skeleton}
For the above skeletonization process, the following properties hold
\begin{itemize}
\item[$(a)$] If $\bi \asymp \bj \in [n]^p$, then $\sk(\bi) \asymp \sk(\bj)$. That is, the skeletons of equivalent index sequences are equivalent. 
\item[$(b)$] For any $\bi = (i_1, \ldots, i_k) \in [n]^k$, and $q\in [Q]$, we have
\begin{align*}
M_{\bi}^\pq = \frac{M_{\sk(\bi)}^\pq }{ B(d_q, k_q)^{r(\bi)}}. 
\end{align*}
\item[$(c)$] For any $\bi \in \cT_\star(p) \subset [n]^{2p}$, its skeleton is either formed by a single element, or an index sequence whose graph has the property that 
every vertex has degree greater or equal to $4$. 
\end{itemize}
\end{lemma}

Given an index sequence $\bi \in \cT_\star(p)  \subset [n]^{2p}$, we say
$\bi$ is of type 1, if $\sk(\bi)$ contains only one index. We say  $\bi$ is of type 2 if $\sk(\bi)$ is not empty (so that by Lemma 
\ref{lem:skeleton}, $G_{\sk(\bi)}$ can only contain vertices with degree greater or equal to $4$). Denote the class of type 1 index sequence (respectively type 2 index sequence) 
by $\cT_1(p)$ (respectively $\cT_2(p)$). 
We also denote by $\tcT_a(p)$, $a\in\{1,2\}$ the set of equivalence
classes of sequences in $\cT_a(p)$. This definition makes sense since the equivalence class of the skeleton of a sequence only depends on the equivalence class of the sequence itself.

\noindent
\textbf{ Step 4. Type 1 index sequences.} 

Recall that $v(\bi)$ is the number of vertices in $G_\bi$, and  $e(\bi)$ is the number of edges in $G_\bi$ (which coincides with the length of $\bi$). 
We consider $\bi \in \cT_1(p)$. Since for $\bi \in \cT_1(p)$, every edge of $G_\bi$ must be at most a double edge. Indeed, if $(u_1,u_2)$ had multiplicity larger than $2$ in $G_{\bi}$,
neither $u_1$ nor $u_2$ could be deleted during the skeletonization process, contradicting the assumption that $\sk(\bi)$ contains a single vertex. 
Therefore, we must have $\min_{\bi \in \cT_1} v(\bi) = p + 1$. According the Lemma \ref{lem:skeleton}.$(b)$, for every $\bi \in \cT_1(p)$, we have 
\[
M_\bi =  \prod_{q \in [Q]} M_{\bi}^\pq = \prod_{q \in [Q]} 1/B(d_q, k_q)^{v(\bi)-1} = \frac{1}{B(\bd, \bk)^{v (\bi ) - 1} }.
\]
Note by Lemma \ref{lem:equivalent_class}.$(e)$, the number of elements in the equivalence class of $\bi$ is $\vert \cC(\bi) \vert \le p^p \cdot n^{v(\bi)}$. Hence we get 
\begin{equation}
\max_{\bi \in \cT_1(p)} \big[\vert \cC(\bi) \vert \vert M_{\bi}\vert \big] \le \sup_{\bi \in \cT_1(p)}  \big[ p^p n^{v(\bi)} / B(\bd, \bk)^{v(\bi) - 1}\big] = p^p n^{p + 1} / B(\bd, \bk)^{p}\, .
\end{equation}
Therefore, denoting $K = \sum_{q \in [Q]} \eta_q k_q$,
\begin{align}
& \sum_{\bi\in\cT_1(p)} M_{\bi} = \sum_{\cC \in \tcT_1(p)} \vert\cC \vert\,  \vert M_{\cC}\vert\\
\le & |\cQ(p)|p^p \frac{n^{p + 1}}{ B(\bd, \bk)^{p}} \le (Cp)^{3p}n^{p+1} d^{-K p}\, . \label{eqn:bound_T1}
\end{align}
where in the last step we used Lemma \ref{lem:equivalent_class} and the fact that for $q \in [Q]$, $B(d_q,k_q)\ge C_0 d_q^{k_q}$ for some $C_0>0$.

\noindent
\textbf{ Step 5. Type 2 index sequences.} 

We have the following simple lemma bounding $M_\bi$, copied from \cite[Proposition 3]{ghorbani2019linearized}. This bound is useful when  $\bi$ is a skeleton.
\begin{lemma}\label{lem:M_estimate}
For any $q \in [Q]$, there exists constants $C$ and $d_0$ depending uniquely on $k_q$ such that, for any $d \ge d_0(k_q)$, and any index sequence $\bi \in [n]^m$ with  
$2\le m\le d_q/(4k_q)$, we have 
\[
\vert M_\bi^\pq \vert \le \left(C m^{k_q} \cdot d_q^{-  k_q }\right)^{m/2}\, .
\]
\end{lemma}

Suppose $\bi \in \cT_2(p)$, and denote $v(\bi)$ to be the number of vertices in $G_\bi$. We have, for a sequence $p = o_d (d) $, and each $q \in [Q]$
\begin{align*}
|M_\bi^\pq | &\stackrel{(1)}{=} \frac{|M_{\sk(\bi)}^\pq|}{B(d_q, k_q)^{r(\bi)}}\\
&\stackrel{(2)}{\le}  \left(\frac{Ce(\sk(\bi))}{d_q}\right)^{k_q \cdot e( \sk(\bi))/2} (C'd_q)^{- r(\bi) k_q} \\
& \stackrel{(3)}{\le}   \left(\frac{Cp}{d_q}\right)^{k_q \cdot e(\sk(\bi))/2} (C'd_q)^{- r(\bi) k_q}   \\
& \stackrel{(4)}{\le}   \left(\frac{Cp}{d_q}\right)^{k_q \cdot v(\sk(\bi))} (C'd_q)^{- r(\bi) k_q}   \\
&\stackrel{(5)}{\le} C^{v(\bi)}p^{k_q \cdot v(\sk(\bi))}  d_q^{- (v(\sk(\bi)) + r(\bi)) \cdot k_q}\\
&\stackrel{(6)}{\le}(Cp)^{k_q\cdot v(\bi)}  d_q^{- v(\bi) k_q}\, .
\end{align*}
Here $(1)$ holds by Lemma \ref{lem:skeleton}.$(b)$; $(2)$ by Lemma \ref{lem:M_estimate}, and the fact that $\sk(\bi) \in [n]^{e(\sk(\bi))}$, together by $B(d_q,k_q)\ge C_0 d_q^{k_q}$;
$(3)$ because $e(\sk(\bi))\le 2p$; $(4)$ by Lemma \ref{lem:skeleton}.$(c)$, implying  that for $\bi \in \cT_2(p)$,  each vertex of $G_{\sk(\bi)}$ has degree greater or equal to $4$, so that $v(\sk(\bi)) \le e(\sk(\bi))/2$ (notice that for $d\ge d_0(k_q)$ we can assume $Cp/d_q<1$). 
Finally, $(5)$ follows since $r(\bi), v(\sk(\bi))\le v(\bi)$, and $(6)$   the definition of $r(\bi)$ implying $r(\bi) = v(\bi) - v(\sk(\bi))$. 

Hence we get 
\[
|M_\bi |  \le \prod_{q \in [Q]} (Cp)^{k_q\cdot v(\bi)}  d_q^{- v(\bi) k_q}
\]

Note by Lemma \ref{lem:equivalent_class}.$(e)$, the number of elements in equivalent class $\vert \cC(\bi) \vert \le p^{v(\bi)} \cdot n^{v(\bi)}$. 
Since $v(\bi)$ depends only on the equivalence class of $\bi$, we will write, with a slight abuse of notation $v(\bi) = v(\cC(\bi))$. 
Notice that the number of equivalence classes with $v(\cC) = v$ is upper bounded by the number multi-graphs with $v$ vertices and  $2p$ edges, which is
at most $v^{4p}$.
Denoting $\alpha = \max_{q \in [Q]} \lb 1 / \eta_q \rb $, we have
\begin{align}
\sum_{\bi\in\cT_2(p)} M_{\bi} &\le\sum_{\cC\in\tcT_2(p)}\vert \cC \vert \vert M_{\cC}\vert \\
&\le \sum_{\cC\in\tcT_2(p)}  (Cp^\alpha)^{(K+1) v(\cC)} \left(\frac{n}{d^K}\right)^{v(\cC)} \\
& \le \sum_{v=2}^{2p} v^{4p} \left(\frac{Cnp^{\alpha (K+1)}}{d^K}\right)^{v}. 
\end{align}
Define $\eps = Cnp^{\alpha (K+1)} /d^K$. We will assume hereafter that $p$ is selected such that
\begin{align}
2p\le -\log\left(\frac{Cnp^{\alpha (K+1)} }{d^K}\right)\,. \label{eq:p-condition}
\end{align}
By calculus and condition (\ref{eq:p-condition}), the function $F(v) = v^{4p}\eps^v$ is maximized over $v\in [2,2p]$ at $v=2$, whence
\begin{align}
\sum_{\bi\in\cT_2(p)} M_{\bi} &\le 2p\, F(2) \le C^p \left(\frac{n}{d^K}\right)^{2}\, .\label{eqn:bound_T2}
\end{align}

\noindent
\textbf{Step 6. Concluding the proof. }

Using  Eqs. (\ref{eqn:bound_T1}) and (\ref{eqn:bound_T2}), we have, for any $p = o_d (d)$ satisfying Eq.~\eqref{eq:p-condition}, we have
\begin{align}
\E[\Trace(\bDelta^{2p})] & = \sum_{\bi = (i_1, \ldots, i_{2p}) \in [N]^{2p}} M_\bi = \sum_{\bi \in \cT_1(p)} M_{\bi}+\sum_{\bi\in\cT_2(p)}  M_{\bi}\\
& \le  (Cp)^{3p}\frac{n^{p+1}}{d^{Kp}} +  C^p \left(\frac{n}{d^K}\right)^{2}\, .
\end{align}
Form Eq.~\eqref{eq:MomentBound}, we obtain
\begin{align}
\E[\|\bDelta\|_{\op}] \le C \left\{ p^{3/2} n^{1/(2p)}\sqrt{\frac{n}{d^{K}}}  + \left(\frac{n}{d^{K}}\right)^{1/p} \right\}. 
\end{align}
Finally setting $n = d^K e^{-2A\sqrt{\log d}}$ and $p = (K/A)\sqrt{\log d}$, this yields
\begin{align}
\E[\|\bDelta\|_{\op}] \le C \left\{ e^{-\frac{A}{4}\sqrt{\log d}} +e^{-2A^2/K}\right\}\,.
\end{align}
Therefore, as long as $A\to \infty$, we have $\E[\|\bDelta\|_{\op}]\to 0$. It is immediate to check that the above choice of $p$
satisfies the required conditions $p = o_d (d)$ and Eq.~\eqref{eq:p-condition} for all $d$ large enough.

\clearpage

\section{Technical lemmas}

We put here one technical lemma that is used in the proof of Theorem \ref{thm:NT_lower_upper_bound_aniso}.(a).

\begin{lemma}\label{lem:bound_inv_block_matrix}
Let $\bD = ( \bD^{qq'} )_{q,q' \in [Q]} \in \R^{DN \times DN}$ be a symmetric $Q$ by $Q$ block matrix with $\bD^{qq'} \in \R^{d_q N\times d_{q'} N}$. Denote $\bB = \bD^{-1}$. Assume that $\bD$ satisfies the following properties:
\begin{enumerate}
\item For any $q \in [Q]$, there exists $c_q,C_q >0$ such that we have with high probability 
\[
 0 < \frac{r_q^2}{d_q}	c_q = d^{\kappa_q} c_q \leq \lambda_{\min} (\bD^{qq}) \leq \lambda_{\max} (\bD^{qq}) \leq \frac{r_q^2}{d_q} C_q = d^{\kappa_q} C_q < \infty,
\]
as $d \to \infty$.
\item For any $q \neq q' \in [Q]$, we have $\sigma_{\max} ( \bD^{qq'} ) = o_{d,\P} (r_q r_{q'} / \sqrt{d_q d_{q'}}) = o_{d,\P}(d^{(\kappa_q + \kappa_{q'} )/2})$. 
\end{enumerate}
Then for any $q\neq q' \in [Q]$, we have 
\begin{equation}\label{eq:bound_inv_block_matrix}
\| \bB^{qq} \|_{\op} = O_{d,\P} \lp \frac{d_q}{ r_q^2} \rp = O_{d,\P} ( d^{-\kappa_q}) , \qquad \| \bB^{qq'} \|_{\op} = o_{d,\P} \lp  \frac{\sqrt{d_q d_{q'}}}{r_q r_{q'}} \rp = o_{d,\P} (d^{-(\kappa_q + \kappa_{q'} )/2} ).
\end{equation}
\end{lemma}

\begin{proof}[Proof of Lemma \ref{lem:bound_inv_block_matrix}]
Let us show the result recursively on the integer $Q$. Note that the case $Q = 1$ is direct.

Consider $\bD = ( \bD^{qq'} )_{q,q' \in [Q] }$. Denote $\Tilde D = D- d_Q$, $\bA = ( \bD^{qq'} )_{q,q' \in [Q-1]} \in \R^{\Tilde D  N \times \Tilde D  N}$ and $\bC = [ (\bD^{1Q})^\sT, \ldots , (\bD^{(Q-1)Q})^\sT ]^\sT \in \R^{d_Q N \times \Tilde D  N}$ such that
\[
\bD = \begin{bmatrix} 
\bA & \bC \\
\bC^\sT & \bD^{QQ}
\end{bmatrix}.
\]
Assume that $\bA^{-1}$ verifies Eq.~\eqref{eq:bound_inv_block_matrix}. Denote
\[
\bB =  \begin{bmatrix} 
\bR & \bT \\
\bT^\sT & \bB^{QQ}
\end{bmatrix}.
\]
From the two by two blockmatrix inversion, we have:
\[
\begin{aligned}
\bB^{QQ} = & ( \bD^{QQ} - \bC^\sT \bA^{-1} \bC )^{-1},\\
\bT = & - \bA^{-1} \bC \bB^{QQ}.
\end{aligned}
\]
We have
\[
\begin{aligned}
\Big\Vert \bC^\sT \bA^{-1} \bC \Big\Vert_{\op} \leq & \sum_{q,q'\in [Q-1]} \Big\Vert (\bD^{qQ})^\sT (\bA^{-1})_{qq'} \bD^{q'Q} \Big\Vert_{\op} \\
 =&    \sum_{q,q'\in [Q-1]} o_{d,\P} \lp \frac{ r_q r_{Q} }{ \sqrt{d_q d_{Q}}} \rp \cdot O_{d,\P} \lp  \frac{\sqrt{d_q d_{q'}}}{r_q r_{q'}} \rp \cdot o_{d,\P} \lp \frac{ r_{q'} r_{Q} }{ \sqrt{d_{q'} d_{Q}}} \rp \\
 = & o_{d, \P} (r_Q^2 /d_Q ),
\end{aligned}
\]
where we used in the second line the properties on $\bD$ and our assumption on $\bA^{-1}$. Hence $  \bD^{QQ} - \bC^\sT \bA^{-1} \bC   \preceq (r_q^2/d_q) ( c_q - o_{d,\P} (1) ) \id$ and $\| \bB^{QQ} \|_{\op} = O_{d,\P} ( d_q / r_q^2)$. 

Furthermore, for $q < Q$,
\[
\begin{aligned}
 \bB^{qQ}  = - \sum_{q' \in [Q-1]} (\bA^{-1})_{qq'} \bC_{q'} \bB^{QQ} 
\end{aligned}
\]
Hence
\[
\begin{aligned}
\Big\Vert \bB^{qQ}  \Big\Vert_{\op} \leq  & \sum_{q' \in [Q-1]} \Big\Vert (\bA^{-1})_{qq'} \bD^{q'Q} \bB^{QQ} \Big\Vert_{\op} \\
= &  \sum_{q,q'\in [Q-1]} O_{d,\P} \lp \frac{\sqrt{d_q d_{q'}}  }{ r_q r_{q'} } \rp \cdot o_{d,\P} \lp  \frac{r_q r_{Q} }{\sqrt{d_{q'} d_{Q}} } \rp \cdot O_{d,\P} \lp \frac{ d_Q }{ r_Q^2 } \rp = o_{d, \P} \lp \frac{\sqrt{d_q d_{Q}}  }{ r_q r_{Q} } \rp,
\end{aligned}
\]
which finishes the proof.
\end{proof}

\subsection{Useful lemmas from \cite{ghorbani2019linearized}}

For completeness, we reproduce in this section lemmas proven in \cite{ghorbani2019linearized}. 

\begin{lemma}\label{lem:non_decreasing_N}
The number $B(d, k)$ of independent degree-$k$ spherical harmonics on $\S^{d-1}$ is non-decreasing in $k$ for any fixed $d \ge 2$. 
\end{lemma}

\begin{lemma}\label{lem:gegenbauer_coefficients}
For any fixed $k$, let $Q_k^{(d)}(x)$ be the $k$-th Gegenbauer polynomial. We expand 
\[
Q_k^{(d)}(x) = \sum_{s = 0}^k p_{k, s}^{(d)} x^s. 
\]
Then we have 
\[
p_{k, s}^{(d)} = O_d(d^{-k/2 - s/2}). 
\]
\end{lemma}

\begin{lemma}\label{lem:random_matrix_bound}
Let $N = o_d(d^{\ell+1})$ for a fixed integer $\ell$. Let $(\bw_i)_{i \in [N]} \sim \Unif(\S^{d-1})$ independently. Then as $d \to \infty$, we have
\[
\max_{i \neq j} | \< \bw_i , \bw_j \> | = O_{d,\P} ((\log d)^{k/2} d^{ - k/2}). 
\]
\end{lemma}

\begin{proposition}[Bound on the Gram matrix]\label{prop:Delta_bound}
Let $N \le d^{k}/e^{A_d\sqrt{\log d}}$ for a fixed integer $k$ and any $A_d\to\infty$. Let $(\btheta_i)_{i \in [N]} \sim \Unif(\S^{d-1}(\sqrt d))$
independently, and $Q_k^{(d)}$ be the $k$'th Gegenbauer polynomial
with domain $[-d, d]$. Consider the random matrix $\bW =
(\bW_{ij})_{i, j \in [N]} \in \R^{N \times N}$, with $\bW_{ij} =
Q_k^{(d)}(\< \btheta_i, \btheta_j\>)$. Then we have 
\[
\lim_{d, N \to \infty} \E[\| \bW - \id_d \|_{\op}] = 0. 
\]
\end{proposition}

\end{document}